\documentclass{article} 
\usepackage{arxiv}
\usepackage[utf8]{inputenc} 
\usepackage[T1]{fontenc}    
\usepackage{hyperref}       
\usepackage{url}            
\usepackage{booktabs}       
\usepackage{amsfonts}       
\usepackage{nicefrac}       
\usepackage{microtype}      
\usepackage{xcolor}         
\usepackage{graphicx}
\usepackage{subcaption}
\usepackage{wrapfig}
\usepackage{enumitem}
\usepackage{amsmath}
\usepackage{amsthm}
\usepackage{algorithm,algorithmic}
\usepackage{comment}

\usepackage{subcaption}
\usepackage{wrapfig}
\usepackage{enumitem}
\usepackage{amssymb}
\usepackage{multirow}
\usepackage{mathtools}
\usepackage{dsfont}
\usepackage{adjustbox}
\usepackage{array}

\usepackage{amsmath,amsfonts,bm}

\def\eqref#1{equation~\ref{#1}}

\def\1{\bm{1}}

\def\vb{{\bm{b}}}

\def\vd{{\bm{d}}}

\def\vq{{\bm{q}}}
\def\vr{{\bm{r}}}

\def\vu{{\bm{u}}}

\def\mC{{\bm{C}}}

\def\mF{{\bm{F}}}
\def\mG{{\bm{G}}}
\def\mH{{\bm{H}}}
\def\mI{{\bm{I}}}

\def\mQ{{\bm{Q}}}

\def\mS{{\bm{S}}}

\def\mV{{\bm{V}}}

\def\mY{{\bm{Y}}}

\DeclareMathAlphabet{\mathsfit}{\encodingdefault}{\sfdefault}{m}{sl}
\SetMathAlphabet{\mathsfit}{bold}{\encodingdefault}{\sfdefault}{bx}{n}

\def\gA{{\mathcal{A}}}
\def\gB{{\mathcal{B}}}

\def\gI{{\mathcal{I}}}

\def\gL{{\mathcal{L}}}

\def\gR{{\mathcal{R}}}
\def\gS{{\mathcal{S}}}

\def\gU{{\mathcal{U}}}
\def\gV{{\mathcal{V}}}
\def\gW{{\mathcal{W}}}
\def\gX{{\mathcal{X}}}
\def\gY{{\mathcal{Y}}}
\def\gZ{{\mathcal{Z}}}

\def\sP{{\mathbb{P}}}
\def\sQ{{\mathbb{Q}}}

\newcommand{\E}{\mathbb{E}}

\newcommand{\R}{\mathbb{R}}

\DeclareMathOperator*{\argmin}{arg\,min}

\usepackage{titletoc}
\usepackage{tikz}
\usetikzlibrary{arrows.meta,positioning}

\allowdisplaybreaks
\newtheorem{theorem}{Theorem}[section]
\newtheorem{lemma}[theorem]{Lemma}
\newtheorem{definition}[theorem]{Definition}
\newtheorem{proposition}[theorem]{Proposition}
\newtheorem{corollary}[theorem]{Corollary}
\newtheorem{remark}[theorem]{Remark}

\newtheorem{assumption}[theorem]{Assumption}

\newtheorem*{algorithm_state*}{Algorithm}

\AtBeginDocument{
  \setlength{\abovedisplayskip}{6pt plus 2pt minus 2pt}
  \setlength{\belowdisplayskip}{6pt plus 2pt minus 2pt}
  \setlength{\abovedisplayshortskip}{3pt plus 1pt minus 1pt}
  \setlength{\belowdisplayshortskip}{3pt plus 1pt minus 1pt}
  \setlength{\jot}{2pt} 
}

\title{Doubly Robust Proxy Causal Learning with Neural Mean Embeddings}

\author{
\begin{tabular}{c}
Bariscan Bozkurt,
Alexandre Galashov,
Dimitri Meunier,
Zikai Shen,\\
Arthur Gretton,
Houssam Zenati
\end{tabular}
\\[0.5em]
University College London
}

\begin{document}
\maketitle

\begin{abstract}

We propose a novel doubly robust neural mean-embedding framework for proxy causal learning (PCL) for continuous and structured treatments. PCL uses proxy variables and bridge functions to identify causal effects under unobserved confounding: either an outcome bridge or a treatment bridge can identify the target, while doubly robust methods combine both. Our framework learns both bridges with data-adaptive neural features and combines them into doubly robust estimators of population, treatment-conditional, and heterogeneous dose-response functions. Two features of our result are particularly notable: the treatment bridge can be learned without density ratios, and the same treatment bridge yields both population and heterogeneous responses, eliminating the need for a separate bridge fit. Our analysis shows that the doubly robust error is controlled by final-stage error and the smaller of the two projected bridge errors, with no separate density-ratio estimation error. Across challenging continuous- and structured-treatment benchmarks, our method outperforms neural single-bridge estimators and fixed-feature doubly robust estimators. Our implementation is avaliable at \url{https://github.com/BariscanBozkurt/DoublyRobust-PCL-Network}.\looseness=-1

\end{abstract}

\vspace{-3mm}
\section{{Introduction and related work}}
\vspace{-2mm}

\begin{wrapfigure}{r}{0.36\textwidth}
\vspace{-1.2\baselineskip}
\centering
\begin{tikzpicture}[
    >=Latex,
    line width=0.95pt,
    shorten >=2pt,
    shorten <=2pt,
    every node/.style={font=\small},
    observed/.style={
        circle,
        draw=black,
        fill=yellow!70!olive!35,
        minimum size=10mm,
        inner sep=0pt
    },
    latent/.style={
        circle,
        draw=black,
        dashed,
        fill=white,
        minimum size=10mm,
        inner sep=0pt
    },
    sedge/.style={->},
    dedge/.style={
        <->,
        dashed,
        line width=1.15pt,
        dash pattern=on 4pt off 2pt
    },
    dedgexu/.style={
        <->,
        dashed,
        line width=1.35pt,
        dash pattern=on 4pt off 2pt
    }
]

\node[latent]   (U) at (0, 2.55) {$U$};
\node[observed] (Z) at (-2.10, 0.55) {$Z$};
\node[observed] (X) at (0, 0.55) {$X$};
\node[observed] (W) at (2.10, 0.55) {$W$};
\node[observed] (A) at (-1.35, -1.35) {$A$};
\node[observed] (Y) at (1.35, -1.35) {$Y$};

\draw[sedge] (U) -- (Z);
\draw[sedge] (U) -- (A);
\draw[sedge] (U) -- (Y);

\draw[sedge] (X) -- (Z);
\draw[sedge] (X) -- (W);
\draw[sedge] (X) -- (A);
\draw[sedge] (X) -- (Y);

\draw[sedge] (A) -- (Y);
\draw[sedge] (W) -- (Y);

\draw[dedgexu] (X) -- (U);
\draw[dedge]   (U) -- (W);
\draw[dedge]   (Z) -- (A);

\end{tikzpicture}
\vspace{-0.6\baselineskip}
\caption{Causal graph for PCL setting.}
\label{fig:pcl_graph_tikz}
\vspace{-1.2\baselineskip}
\end{wrapfigure}

Causal inference from observational data is difficult when treatment assignment
is confounded. Standard adjustment assumes that all relevant confounders are
observed, yielding \(Y^{(a)}\perp A\mid X\) under consistency, no interference,
and positivity. This condition underlies covariate adjustment and propensity-score methods \citep{Rubin1980, rosRub83, Guido_Imbens_Nonparametric}, outcome-regression methods \citep{Hill_BayesianNonparametric}, and representation-learning estimators for treatment effects \citep{pmlr_v48_johansson16, NEURIPS2018_a50abba8}. Yet important determinants of treatment and outcome may be unrecorded or only
indirectly measured. Instrumental-variable methods offer one route beyond observed-confounder adjustment \citep{Reiersl1945ConfluenceAB,Stock_who_invented_IV,Newey_nonparametric_IV}, but
require their own exclusion and relevance conditions that are often restrictive.\looseness=-1

Proxy causal learning (PCL) uses a different source of information: two observed proxy variables \(Z\) and \(W\) that are informative about the unobserved confounding \citep{Kuroki2014Mesurement, Miao2018Identifying, TchetgenTchetgen2024IntroProximal}. Figure~\ref{fig:pcl_graph_tikz} illustrates the standard PCL setting. Here, \(X\) denotes the pre-treatment covariates, while \(Z\) and \(W\) provide complementary treatment-side and outcome-side information about the latent confounder \(U\). The treatment \(A\) and outcome \(Y\) are confounded by observed covariates \(X\) and latent variables \(U\). Conditional independence and completeness conditions permit identification without recovering \(U\) itself. Identification is expressed through bridge functions whose conditional expectations reproduce regression or weighting equations. Learning these bridges replaces latent-confounder inference by the solution of Fredholm integral equations of the first kind. 

Existing PCL estimators mainly follow outcome-bridge or treatment-bridge routes. Outcome-bridge methods learn a function \(h(A,X,W)\) satisfying
\(
    \E\{Y-h(A,X,W)\mid A,X,Z\}=0,
\)
and then recover response curves by averaging or conditioning the learned bridge. 
Linear and semiparametric bridge estimators were developed in proximal causal inference \citep{TchetgenTchetgen2024IntroProximal, semiparametricProximalCausalInference}; sieve and proxy-control estimators were studied by \citet{deaner2023proxycontrolspaneldata}; and RKHS-methods provide flexible nonparametric bridge estimators \citep{Mastouri2021ProximalCL, singh2023kernelmethodsunobservedconfounding}. Neural outcome-bridge estimators replace fixed features with learned representations, improving scalability to nonlinear and high-dimensional settings \citep{xu2021deep, kompa2022deep}. These methods learn the regression-type bridge \(h(A,X,W)\), but leave open the neural estimation of the complementary weighting-type bridge based on \(Z\) {, i.e. treatment bridge functions.}

Treatment bridges give the complementary weighting route and are analogous to inverse-probability weighting \citep{rosRub83,Robins2000marginal}. For a given causal estimand, a treatment bridge \(\varphi(A,X,Z)\) is characterized by a moment equation of the form
\(
    \E\{r_0(A,X,W)-\varphi(A,X,Z)\mid A,X,W\}=0,
\)
where {\(r_0\) denotes the estimand-specific Radon--Nikodym
derivative; e.g., for the population dose-response curve,
\(r_{0,\mathrm{ATE}}:=d(\sP_A\otimes\sP_{X,W})/d\sP_{A,X,W}\).}
Semiparametric \citep{semiparametricProximalCausalInference}, minimax \citep{Kallus2021Causal}, and kernel-based treatment-bridge estimators \citep{bozkurt2025density} have been developed for proxy causal inference, including extensions to continuous treatments and treatment-conditional dose responses \citep{bozkurt2025density}. Recent neural treatment-bridge work moves toward adaptive treatment-side estimation \citep{zhang2025neural}, but is developed for population binary treatment effects rather than continuous or structured treatments. {Restricting treatment to two options cannot capture responses across many modern settings of interest, including in medicine (drug combinations and doses) and advertising (ad designs). Learned neural features let our estimators use these richer treatment descriptions and adapt their representation to the data.}

Doubly robust PCL combines the outcome- and treatment-bridge routes, mirroring classical doubly robust causal inference, where validity can be retained when one of two complementary components is correctly specified \citep{BangRobins2005, Chernozhukov2018DML, Kennedy2024DRReview}. This is especially useful in PCL because the two bridges use the proxies differently, and it is rarely known in advance whether the treatment or outcome bridge will be more reliable. Semiparametric and kernel doubly robust proxy estimators have been developed for population responses \citep{semiparametricProximalCausalInference, wu2024doubly, bozkurt2025densityratiofreedoublyrobust}, but continuous-treatment methods rely on fixed kernel representations. Beyond population responses, \citet{Sverdrup2023ProximalCL} develop a proximal P-learner for binary-treatment CATE contrasts. More generally, applications may require full heterogeneous response surfaces \((a,v)\mapsto f_{\mathrm{CATE}}(a,v)\), describing responses in the subpopulation \(V=v\), or treatment-conditional surfaces \((a,a')\mapsto f_{\mathrm{ATT}}(a,a')\), describing how units observed at \(A=a'\) would respond under intervention \(a\). Existing neural doubly robust PCL methods, however, do not cover these targets for continuous or
structured treatments.

{We address this gap by proposing a neural doubly robust PCL framework that learns both bridge routes for continuous and structured treatments. Its treatment-side component, \textsc{TreatmentNet}, is a two-stage neural mean embedding estimator that learns the projected treatment bridge with data-adaptive features by minimizing a density-ratio-free risk. Combining it with a neural outcome bridge and a final residual regression, we propose \textsc{DR-PCL-Net} for population, heterogeneous, and treatment-conditional response surfaces. For heterogeneous responses, we derive a novel conditional self-normalization identity showing that one population treatment bridge can be reused over all supported \((a,v)\). For treatment-conditional responses, a joint objective learns the full \((a,a')\) treatment-bridge surface in one fit.
{We adapt the proximal closed-form updates of \citet{galashov2025closedformlayeroptimization} to train both bridges. By retaining information from earlier batches, these updates address minibatch instability in the outcome-bridge method of \citet{xu2021deep}.}\looseness=-1
}

In summary, our contributions are fourfold. (i) We introduce \textsc{TreatmentNet} that utilizes density-ratio-free treatment-bridge learning with adaptive neural features and full response functions for continuous or structured-treatments.
{(ii) We  combine  neural treatment and outcome bridges in \textsc{DR-PCL-Net}, giving doubly robust population, heterogeneous, and treatment-conditional responses. A single population bridge serves all supported subgroups; a joint fit covers the treatment-conditional surface.} (iii) We derive nonasymptotic projected error bounds for approximate empirical minimizers that separate approximation, empirical-process, final-stage, and optimization errors. Under bounded complementary errors, the doubly robust remainder is controlled by the smaller projected bridge error, without strong bridge recovery or a separate density-ratio estimation error.
{(iv) In our experiements, we first construct synthetic examples that isolate challenging regimes in each individual bridge route, which are addressed when both are combined in a doubly robust approach.   We then evaluate population, heterogeneous and treatment-conditional responses on challenging benchmarks, including  image-valued treatments. Our method outperforms competing approaches in the majority of scenarios.}\looseness=-1


\vspace{-3mm}
\section{{Problem setup and target causal functions}}
\vspace{-2mm}
\label{sec:ProblemSetup}
 We observe i.i.d. copies $O = (Y, A, X, Z, W) \sim \sP_O$, where $A \in \gA$ is a possibly continuous or structured treatment, {$Y \in \gY\subseteq\R^{d_Y}$}, $X \in \gX$ contains measured covariates, $Z \in \gZ$ is a treatment-inducing proxy, and $W \in \gW$ is an outcome-inducing proxy. The latent confounder is denoted by $U \in \gU$. For heterogeneous effects write $X = (S, V)$, where $V$ indexes the subgroup conditioned on and $S$ contains the remaining covariates. {We use the notation \(\sP\) for probability laws; lowercase \(p\) denotes a density or mass function when it exists.} We study the following causal response functions\looseness=-1
\begin{align*}
    f_{\mathrm{ATE}}(a) \coloneqq \E \{Y^{(a)}\}, \quad f_{\mathrm{CATE}}(a, v) \coloneqq \E \{Y^{(a)} \mid V = v\}, \quad f_{\mathrm{ATT}}(a, a') \coloneqq \E \{Y^{(a)} \mid A = a'\}
\end{align*}
These are, respectively, the population, heterogeneous, and treatment-conditional dose-response functions, with domains \(\mathcal A\), \(\mathcal A\times\mathcal V\), and \(\mathcal A^2\). The subscripts follow standard discrete-treatment terminology: ATE, CATE, and ATT refer to average, conditional average, and average-on-the-treated treatment effects. Here, they index the corresponding response functions rather than discrete-treatment contrasts. Assume {consistency, \(Y=Y^{(A)}\), and latent exchangeability, \(Y^{(a)}\perp A\mid(U,X)\)}. For the population, heterogeneous, and treatment-conditional targets, respectively, impose
\(
\sP_{X,U}\ll\sP_{X,U\mid A=a},\enspace
\sP_{S,U\mid V=v}\ll\sP_{S,U\mid A=a,V=v},\enspace
\sP_{X,U\mid A=a'}\ll\sP_{X,U\mid A=a}.
\)
Under these conditions, the targets admit the latent-variable representations: (i) \(f_{\mathrm{ATE}}(a)=\E[\E[Y\mid A=a,X,U]]\), (ii) \(f_{\mathrm{CATE}}(a,v)=\E[\E[Y\mid A=a,S,V=v,U]\mid V=v]\), and (iii) \(f_{\mathrm{ATT}}(a,a')=\E[\E[Y\mid A=a,X,U]\mid A=a']\). These representations show that the causal functions would be identified by adjustment if the latent confounder \(U\) were observed. Proxy causal learning instead identifies the same functions without observing or estimating \(U\), by using the proxies \(Z\) and \(W\) to learn bridge functions.\looseness=-1

Intuitively, \(Z\) is informative about the latent confounding through treatment assignment, while \(W\) is informative about the same latent confounding through the outcome. We formalize these proxy roles through the following conditional independence assumptions, and then impose completeness to ensure that the proxies carry enough information for bridge identification. 

\begin{assumption}[Proxy conditional independencies]
The structural causal model is compatible with
Figure~\ref{fig:pcl_graph_tikz} and obeys
\(
Y\perp Z\mid(U,A,X),\enspace W\perp(A,Z)\mid(U,X).
\)
\label{assum:CSM_and_ProxyIndependence}
\label{assum:ProxyConditionalIndependenceAssumptions}
\end{assumption}

Assumption~\ref{assum:ProxyConditionalIndependenceAssumptions} formalizes the roles of $Z$ and $W$ as treatment and outcome proxies, respectively. In addition, identification requires that these proxies be sufficiently informative about the latent confounder $U$ conditional on $(A,X)$ \citep{Miao2018Identifying, TchetgenTchetgen2024IntroProximal, semiparametricProximalCausalInference}. This is captured by the following completeness assumption. \looseness=-1

\begin{assumption}[Completeness]
\label{assum:Completeness}
For $\sP_{A,X}$-almost all $(a,x) \in \gA \times \gX$, and for every square-integrable function $\ell \in L^2(\sP_{U \mid A=a, X=x})$: (i) $\E[\ell(U) \mid A=a, X=x, Z] = 0 \enspace \sP_{Z \mid A=a, X=x}\text{-a.e.}  \Longrightarrow\ \ell(U) = 0 \enspace \sP_{U \mid A=a, X=x}\text{-a.e.},$ (ii) $\E[\ell(U) \mid A=a, X=x, W] = 0 \enspace \sP_{W \mid A=a, X=x}\text{-a.e.}  \Longrightarrow\ \ell(U) = 0 \enspace \sP_{U \mid A=a, X=x}\text{-a.e.}.$
\end{assumption}

Assumption~\ref{assum:Completeness} is the standard sufficient condition used in PCL and nonparametric instrumental-variable regression
\citep{Miao2018Identifying,Newey_nonparametric_IV,
semiparametricProximalCausalInference,TchetgenTchetgen2024IntroProximal}. For fixed \(a, x\), it requires that no nonzero square-integrable function of \(U\) have zero conditional mean given the relevant proxy. In finite discrete models it reduces to a full-rank condition, while sufficiently rich exponential-family conditional laws provide familiar continuous sufficient conditions. Appendix~\ref{app:identification_assumptions} reviews these examples from \citep{Newey_nonparametric_IV, TchetgenTchetgen2024IntroProximal} and separately discusses the additional range conditions required for bridge existence.
\looseness=-1

\vspace{-3mm}
\section{Identification of causal functions}
\vspace{-2mm}
\label{sec:Identification_Main}

This section first reviews the established outcome-bridge and treatment-bridge representations, then develops our general-treatment extensions. We derive doubly robust identification of the full heterogeneous response surface \(f_{\mathrm{CATE}}(a,v)\) and the treatment-conditional response surface \(f_{\mathrm{ATT}}(a,a')\) for continuous or structured treatments. We also give an alternative CATE identification result: after conditional normalization, any valid population treatment bridge induces a valid CATE bridge. Our proposed method uses this identity to construct a density-ratio-free heterogeneous-response estimator that reuses the population bridge rather than fitting a separate target-specific bridge.\looseness=-1


\noindent\textbf{Outcome bridge identification.} The standard route in PCL uses outcome bridge function \citep{Miao2018Identifying, Mastouri2021ProximalCL, xu2021deep, kompa2022deep}: it seeks a function \(h(A,X,W)\) such that \(\E[h(A,X,W)\mid A,X,Z]=\E[Y\mid A,X,Z]\). Once such a bridge exists, response functions are obtained by averaging or conditioning \(h\) over the relevant observed distribution.\looseness=-1

\begin{theorem}[Identification via the outcome bridge]
\label{theorem:OutcomeBridgeIdentificationAllCausalFunctions}
{Suppose Assumption~\ref{assum:ProxyConditionalIndependenceAssumptions} and Assumption~\ref{assum:Completeness}~(i) hold.} Furthermore, suppose that for $\sP_{A,X}$-almost all $(a,x) \in \gA \times \gX$, there exists an outcome bridge function $h_0(a,x,\cdot) \in L^2\!\left(\sP_{W \mid A=a, X=x}\right)$
such that
\begin{equation}
    \E[Y \mid A=a, X=x, Z]
    =
    \E[h_0(a,x,W) \mid A=a, X=x, Z],
    \qquad
    \sP_{Z \mid A=a, X=x}\text{-a.e.}
    \label{eq:OutcomeBridgeEquation}
\end{equation}
 Then, the causal functions are identified by (i) $f_{\mathrm{ATE}}(a) = \E[h_0(a,X,W)]$, (ii) $f_{\mathrm{CATE}}(a,v) = \E[h_0(a,X,W) \mid V=v]$, (iii) $f_{\mathrm{ATT}}(a,a') = \E[h_0(a,X,W) \mid A=a']$.
\end{theorem}

The population dose-response formula is the standard outcome-bridge result of \citet{Miao2018Identifying}; the conditional formulas follow by conditioning the same representation \citep[Theorem~1]{singh2023kernelmethodsunobservedconfounding}.


\noindent\textbf{Treatment bridge identification. }We next turn to treatment bridge identification, which complements the outcome-bridge view and is closer in spirit to inverse propensity weighting \citep{rosRub83, Robins2000marginal}. This perspective has been developed in semiparametric, minimax, and kernel forms, including recent extensions to continuous treatments and density-ratio-free formulations \citep{semiparametricProximalCausalInference, Kallus2021Causal, wu2024doubly, bozkurt2025density, bozkurt2025densityratiofreedoublyrobust}. It is also the identification route that will motivate our treatment-side estimator.

Define the target specific Radon--Nikodym derivatives for population, heterogeneous, and treatment-conditional response curves: (i) \(r_{0,\mathrm{ATE}}
\coloneqq {d(\sP_A\otimes\sP_{X,W})}/{d\sP_{A,X,W}},\) (ii) \(r_{0,\mathrm{CATE}}(\cdot,v,\cdot,\cdot)
\coloneqq {d(\sP_{A\mid V=v}\otimes\sP_{S,W\mid V=v})}
        /{d\sP_{A,S,W\mid V=v}},\) (iii)  \(r_{0,\mathrm{ATT}}(a,a',\cdot,\cdot)
\coloneqq {d\sP_{X,W\mid A=a'}}/{d\sP_{X,W\mid A=a}}.\)
When common dominating measures exist, these are, respectively,
\(p_A(a)/p_{A\mid X,W}(a\mid x,w)\),
\(p_{A\mid V}(a\mid v)/p_{A\mid S,V,W}(a\mid s,v,w)\), and
\(p_{X,W\mid A}(x,w\mid a')/p_{X,W\mid A}(x,w\mid a)\), on their denominator supports.\looseness=-1

\begin{theorem}[Identification via treatment bridge]
\label{theorem:TreatmentBridgeIdentificationAllCausalFunctions}
Suppose Assumptions~\ref{assum:ProxyConditionalIndependenceAssumptions} and \ref{assum:Completeness}~(ii) hold. Assume that for $\sP_{A,X}$-almost all $(a,x) \in \gA \times \gX$, there exist treatment bridge functions $\varphi_0^{\mathrm{ATE}}(a,x,\cdot) \in L^2\!\left(\sP_{Z \mid A=a, X=x}\right)$ and, for any $a' \in \gA$, $\varphi_0^{\mathrm{ATT}}(a,a',x,\cdot) \in L^2\!\left(\sP_{Z \mid A=a, X=x}\right)$. In addition, assume that for $\sP_{A,S,V}$-almost all $(a,s,v) \in \gA \times \gS \times \gV$, there exists $\varphi_0^{\mathrm{CATE}}(a,v,s,\cdot) \in L^2\!\left(\sP_{Z \mid A=a, S=s, V=v}\right)$. Suppose these bridge functions satisfy, almost everywhere with respect to the corresponding conditional law of $W$,
\begin{align}
    \E[ {\varphi_0^{\mathrm{ATE}}}(a,X,Z) \mid A=a, X, W]
    &=r_{0,\mathrm{ATE}}(a,X,W),
    \label{eq:TreatmentBridgeEq_ATE} 
    \\
    \E[ {\varphi_0^{\mathrm{CATE}}}(a,v,S,Z) \mid A=a, S, V=v, W]
    &=r_{0,\mathrm{CATE}}(a,v,S,W),
    \label{eq:TreatmentBridgeEq_CATE} 
    \\
    \E[{\varphi_0^{\mathrm{ATT}}}(a,a',X,Z) \mid A=a, X, W]
    &=r_{0,\mathrm{ATT}}(a,a',X,W).
    \label{eq:TreatmentBridgeEq_ATT}
\end{align}
Then, the causal functions are identified by (i) $f_{\mathrm{ATE}}(a) = \E[Y  {\varphi_0^{\mathrm{ATE}}}(a,X,Z) \mid A=a]$, (ii) $f_{\mathrm{CATE}}(a,v) = \E[Y  {\varphi_0^{\mathrm{CATE}}}(a,v,S,Z) \mid A=a, V=v]$, (iii) $f_{\mathrm{ATT}}(a,a') = \E[Y {\varphi_0^{\mathrm{ATT}}}(a,a',X,Z) \mid A=a].$
\end{theorem}

The ATE and ATT statements follow from the treatment-bridge identification results of \citet{bozkurt2025density}. We prove the new CATE extension in Appendix~\ref{Appendix:Appendix_CATEIdentification_withTreatmentBridge}.

\noindent\textbf{Doubly robust identification. }Combining outcome and treatment bridges yields doubly robust formulas that remain valid when either bridge is correctly specified and underlie semiparametric and kernel DR proxy methods \citep{semiparametricProximalCausalInference, wu2024doubly, bozkurt2025densityratiofreedoublyrobust}. The population dose-response formula recovers existing results; { the heterogeneous and treatment-conditional formulas are new extensions developed here.}\looseness=-1

The doubly robust identification theorem below extends the bridge structure that appears in the influence functions derived in Appendix~\ref{sec:Appendix_EIF_Discrete} for discrete treatments. That appendix treats \(A\) as discrete, and \(V\) as discrete for the heterogeneous target. Here, we state the corresponding bridge identities directly as identification formulas for general treatments and covariates, including continuous \(A\) and \(V\).\looseness=-1

\begin{theorem}[Doubly robust identification]
\label{theorem:DoublyRobustIdentifications}
Suppose the causal and target-specific latent overlap conditions stated in Section~\ref{sec:ProblemSetup}, together with Assumption~\ref{assum:ProxyConditionalIndependenceAssumptions}, hold. Let the candidate bridges \(h\), \(\varphi^{\mathrm{ATE}}\),
\(\varphi^{\mathrm{CATE}}\), and \(\varphi^{\mathrm{ATT}}\) be measurable and square-integrable, and define
\begin{align}
\Psi_{\mathrm{ATE}}(a;h,\varphi^{\mathrm{ATE}})
&=\E[h(a,X,W)]
 +\E[\varphi^{\mathrm{ATE}}(a,X,Z)\{Y-h(a,X,W)\}\mid A=a],
\label{eq:DoublyRobustATE_identification}\\
\Psi_{\mathrm{CATE}}(a,v;h,\varphi^{\mathrm{CATE}})
&=\E[h(a,X,W)\mid V=v]\nonumber\\
&\quad+\E[\varphi^{\mathrm{CATE}}(a,v,S,Z)\{Y-h(a,X,W)\}\mid A=a,V=v],
\label{eq:DoublyRobustCATE_identification}\\
\Psi_{\mathrm{ATT}}(a,a';h,\varphi^{\mathrm{ATT}})
&=\E[h(a,X,W)\mid A=a']\nonumber\\
&\quad+\E[\varphi^{\mathrm{ATT}}(a,a',X,Z)\{Y-h(a,X,W)\}\mid A=a].
\label{eq:DoublyRobustATT_identification}
\end{align}
For each target, the corresponding
functional equals the causal response if either (i) \(h\) satisfies
Eq.~\ref{eq:OutcomeBridgeEquation} and Assumption~\ref{assum:Completeness}~(i) holds, or (ii) its treatment-bridge candidate
satisfies the corresponding equation in
Eqs.~\ref{eq:TreatmentBridgeEq_ATE}--\ref{eq:TreatmentBridgeEq_ATT} and Assumption~\ref{assum:Completeness}~(ii)
holds. 
\end{theorem}

The ATE representation in Eq.~\ref{eq:DoublyRobustATE_identification} is the covariate-adjusted analogue of the doubly robust dose-response formula in \citet{bozkurt2025densityratiofreedoublyrobust}. Appendix~\ref{sec:Appendix_DoublyRobustIdentification} proves our novel heterogeneous
and treatment-conditional extensions.

\paragraph{{One population treatment bridge for heterogeneous responses.}}
We show that conditional normalization of any valid population treatment bridge produces a valid CATE bridge. This gives a new identifying formula and, subsequently, a density-ratio-free estimator that reuses a single population bridge over all supported \((a,v)\). For any valid ATE bridge, define \(D_0(a,v):=\E\{\varphi_0^{\mathrm{ATE}}(a,X,Z)\mid A=a,V=v\}
={d(\sP_A\otimes\sP_V)}/{d\sP_{A,V}}(a,v),\)
where the second equality follows from the conditional-expectation property of
Radon--Nikodym derivatives.\looseness=-1
\begin{proposition}[ATE bridge induces CATE bridge]
\label{prop:AlternativeCATEidentification}
At points where \(0<D_0(a,v)<\infty\), the Radon--Nikodym chain rule gives
\(
r_{0,\mathrm{ATE}}(a,(s,v),w)
=D_0(a,v)r_{0,\mathrm{CATE}}(a,v,s,w)
\enspace\text{a.e.}
\)
Hence
\(\varphi_0^{\mathrm{CATE}}(a,v,s,z)
:=\varphi_0^{\mathrm{ATE}}(a,(s,v),z)/D_0(a,v)\) solves
Eq.~\ref{eq:TreatmentBridgeEq_CATE}, and
\[
f_{\mathrm{CATE}}(a,v)
=\frac{\E\{Y\varphi_0^{\mathrm{ATE}}(a,X,Z)\mid A=a,V=v\}}
       {D_0(a,v)}.
\]
For an arbitrary ATE-scale candidate \(\varphi(a,X,Z)\), let
\(D_\varphi(a,v)=\E\{\varphi(a,X,Z)\mid A=a,V=v\}\). Whenever all terms are
finite and \(D_\varphi(a,v)\ne0\), define
\begin{align}
\Psi_{\mathrm{CATE}}^{\mathrm{SN}}(a,v;h,\varphi)
&=\E\{h(a,X,W)\mid V=v\}+\frac{\E[\varphi(a,X,Z)\{Y-h(a,X,W)\}\mid A=a,V=v]}
{D_\varphi(a,v)}.
\nonumber
\end{align}
This functional identifies \(f_{\mathrm{CATE}}(a,v)\) if either \(h\) is a
valid outcome bridge or \(\varphi\) is a valid ATE treatment bridge.
\end{proposition}
{Thus, heterogeneous responses reuse one population bridge and retain double robustness, without fitting another treatment bridge or sampling treatments conditional on \(V\) (Appendix~\ref{Appendix:HeterogeneousTreatmentAlgorithm}).}

\vspace{-3mm}
\section{{Neural mean embedding approach for dose-response curve}}
\vspace{-2mm}
\label{sec:NeuralMeanEmbeddingSectionDoseResponse}

We now construct our novel {\textsc{TreatmentNet}} and the doubly robust estimator {\textsc{DR-PCL-Net}} in combination with {\textsc{OutcomeNet}} for the population dose-response.

\noindent\textbf{Notation and sample splitting.} Let \(b\in\{h,\varphi\}\) index the outcome- and treatment-bridge routes, respectively. We use uppercase letters for random variables and lowercase letters for their realizations. Given \(\mathcal D=\{o_i\}_{i=1}^n\) with \(o_i = (a_i, x_i, w_i, z_i, y_i)\), partition the observation indices into three disjoint base folds:  \(\gI_s^{(b)}\) denotes the base fold assigned to Stage~\(s\) of route \(b\). Stage~1 uses
\(\gI_1^{(b)}\) to estimate the conditional mean embedding of the relevant proxy features. Stage~2 uses \(\gI_2^{(b)}\) to learn the bridge after substituting the Stage~1 estimate. Stage~3 uses the held-out observations to estimate the causal response, using \(\gI_3^{(h)}\) for {\textsc{OutcomeNet}}, \(\gI_3^{(\varphi)}\) for {\textsc{TreatmentNet}}, and \({\gI_3^{(\mathrm{DR})}}\) for {\textsc{DR-PCL-Net}}.

\noindent\textbf{{\textsc{TreatmentNet}} and the ratio-free risk.}
Define the conditional operator \((T_\varphi\varphi)(a,x,w)
\coloneqq \E\{\varphi(a,x,Z)\mid A=a,X=x,W=w\}\) and denote \(g_\varphi \coloneqq T_\varphi\varphi\).
For \({Q_\varphi}:=(A,X,W)\sim
{\sP_{Q_\varphi}}\) and independent copies
\({Q_{\varphi,j}}=(A_j,X_j,W_j)\),
\({j\in\{1,2\}}\), set
\begin{equation}
\gR_{\mathrm{ATE}}^{\mathrm{RF}}(u)
:=\E\{u(A_1,X_1,W_1)^2\}
-2\E\{u(A_1,X_2,W_2)\}.
\label{eq:DensRatioFreeRisk}
\end{equation}
The linear term integrates under \(\sP_A\otimes\sP_{X,W}\), rather than the
observed joint law.

\begin{lemma}[Ratio-free risk characterization]
\label{lemma:densRatioFreeRiskCharacterization}
If \(\sP_A\otimes\sP_{X,W}\ll\sP_{A,X,W}\), with derivative
\(r_{0,\mathrm{ATE}}\in L^2({\sP_{Q_\varphi}})\), then
\(
\gR_{\mathrm{ATE}}^{\mathrm{RF}}(u)
-\gR_{\mathrm{ATE}}^{\mathrm{RF}}(r_{0,\mathrm{ATE}})
=\|u-r_{0,\mathrm{ATE}}\|_{L^2({\sP_{Q_\varphi}})}^2 .
\)
\end{lemma}
Under treatment-bridge existence, the lemma shows that minimizing
\(\gR_{\mathrm{ATE}}^{\mathrm{RF}}(T_\varphi\varphi)\) targets
\(
T_\varphi\varphi=r_{0,\mathrm{ATE}}
\)
without estimating or evaluating the Radon--Nikodym derivative. However, evaluating this loss requires ${T_\varphi\varphi}$ which is not readily accessible. For a generic nonlinear bridge, evaluating the projected loss would itself require solving an inner conditional-expectation problem that changes with the bridge parameters. A DeepIV-style alternative would model \(\sP_{Z\mid A,X,W}\) and approximate this projection by Monte Carlo integration during training \citep{Hartford2016CounterfactualPW}. Our construction follows a neural mean-embedding approach \citep{xu_neuralBackdoor2023a, xu2021deep}. In particular, to evaluate the ratio-free criterion at \(u=T_\varphi\varphi\), we parameterize the treatment bridge by a linear head over separable learnable features,
\(
\varphi(a,x,z)
=\bm{\varphi}^\top\!\left[
\phi_{AX,2}^{(\varphi)}(a,x)\otimes
\phi_{Z,2}^{(\varphi)}(z)\right].
\)
For this class of treatment bridge functions, define the conditional feature mean
\(
\mu_Z^{(\varphi)}(a,x,w)
:=\E\!\left\{\phi_{Z,2}^{(\varphi)}(Z)
\mid A=a,X=x,W=w\right\}.
\) Linearity of conditional expectation then gives
\begin{equation}
(T_\varphi\varphi)(a,x,w)
=\bm{\varphi}^\top\!\left[
\phi_{AX,2}^{(\varphi)}(a,x)
\otimes\mu_Z^{(\varphi)}(a,x,w)\right].
\label{eq:projected_treatment_bridgeT}
\end{equation}
Stage~1 estimates this vector-valued conditional mean through
\(
\widehat\mu_Z^{(\varphi)}(a,x,w)
:=
\{\widehat \mV^{(\varphi)}\}^{\top}
\phi_{AXW,1}^{(\varphi)}
(a,x,w;\widehat\theta_1^{(\varphi)}),
\)
where \((\widehat\theta_1^{(\varphi)},\widehat \mV^{(\varphi)})\) are obtained by regressing \(\phi_{Z,2}^{(\varphi)}(Z)\) on \(\phi_{AXW,1}^{(\varphi)}(A,X,W)\). For fixed neural features, this single Stage~1 regression evaluates the projection for every linear bridge head: Stage~2 simply replaces \(\mu_Z^{(\varphi)}\) by \(\widehat\mu_Z^{(\varphi)}\) in Equation~\ref{eq:projected_treatment_bridgeT}. It therefore avoids full conditional-law modeling and repeated proxy sampling. The same factorization applies to \textsc{OutcomeNet}, with \(W\mid A,X,Z\) \citep{xu2021learning}.\looseness=-1

\noindent\textbf{Proximal bridge objectives.}
At iteration \(t\), let \(\gB_{s,t}^{(\varphi)}\subset\gI_s^{(\varphi)}\) be the Stage-\(s\) mini-batch and \({m_{s,t}^{(\varphi)}}:=|\gB_{s,t}^{(\varphi)}|\). The first-stage feature-regression objective is
\begin{align}
\widehat{\gL}_{\varphi,1,t}^{\mathrm{prox}}(\theta_1^{(\varphi)}, \mV^{(\varphi)})
&=\frac{1}{ {m_{1,t}^{(\varphi)}}}\sum_{i\in\gB_{1,t}^{(\varphi)}}
\left\|
\phi_{Z,2}^{(\varphi)}(z_i)
-\{\mV^{(\varphi)}\}^{\top}
\phi_{AXW,1}^{(\varphi)}(a_i,x_i,w_i)
\right\|_2^2 +\lambda_1^{(\varphi)}
\|\mV^{(\varphi)}-\widehat{\mV}^{(\varphi)}_t\|_F^2,
\nonumber
\end{align}
Here, $\theta_1^{(\varphi)}$ denotes the parameters of the first-stage neural network \(\phi_{AXW,1}^{(\varphi)}(a_i,x_i,w_i)\). Let \(\hat\mu_{Z,t}^{(\varphi)}\) denote the fitted conditional feature mean.
For \(i,j\in\gB_{2,t}^{(\varphi)}\), define
\(
\vu_{ij,t}:=\phi_{AX,2}^{(\varphi)}(a_i,x_j)
\otimes\hat\mu_{Z,t}^{(\varphi)}(a_i,x_j,w_j).
\)
The second-stage proximal training objective is
\begin{align}
\widehat{\gL}_{\varphi,2,t}^{\mathrm{RF,U}}(\theta_2^{(\varphi)}, \bm{\varphi})
&=\frac{1}{{m_{2,t}^{(\varphi)}}}\sum_{i\in\gB_{2,t}^{(\varphi)}}
(\bm{\varphi}^\top\vu_{ii,t})^2
-\frac{2}{{m_{2,t}^{(\varphi)}}
\bigl({m_{2,t}^{(\varphi)}}-1\bigr)}
\sum_{\substack{i,j\in\gB_{2,t}^{(\varphi)}\\i\ne j}}
\bm{\varphi}^\top\vu_{ij,t}
+\lambda_2^{(\varphi)}
\|\bm{\varphi}-\widehat{\bm{\varphi}}_t\|_2^2,\nonumber
\end{align}
where \(\theta_2^{(\varphi)}\) collects the parameters of the second stage networks \(\phi_{AX,2}^{(\varphi)}\) and \(\phi_{Z,2}^{(\varphi)}\). Minimizing this objective yields the treatment bridge estimate $\hat \varphi$.

\noindent\textbf{\textsc{OutcomeNet}.}
\textsc{OutcomeNet} follows the two-stage neural conditional-embedding construction of
\citet{xu2021deep,xu2021learning}. Stage 1 estimates
\(\mu_W^{(h)}(a,x,z)=\E\{\phi_{W,2}^{(h)}(W)\mid A=a,X=x,Z=z\}\). Stage 2
regresses the outcome \(Y\) on
\(\phi_{AX,2}^{(h)}(A,X)\otimes\hat\mu_W^{(h)}(A,X,Z)\), using \(h(a,x,w)=\{\mH^{(h)}\}^{\top}[
\phi_{AX,2}^{(h)}(a,x)\otimes\phi_{W,2}^{(h)}(w)]\).
{We employ a new training procedure,  centering the penalties on both linear layers at their previous values, following \citet{galashov2025closedformlayeroptimization}.} The full objectives and closed-form updates are given in Appendix~\ref{sec:DFPCL_Xu_Review}. On the
evaluation fold, \(\hat f_{\mathrm{ATE}}^{(h)}(a)
=\frac{1}{|\gI_3^{(h)}|}
\sum_{i\in\gI_3^{(h)}}\hat h(a,x_i,w_i).\)

\noindent\textbf{Training schedule.}
For each paired \(\gB_{1,t}^{(b)}\) and \(\gB_{2,t}^{(b)}\),
\(b\in\{h,\varphi\}\), we alternate \(T_{b,1}\) first-stage updates and
\(T_{b,2}\) second-stage updates. Stage 1 holds the proxy-side target feature
map fixed, updates the Stage-1 input featurizer by gradient-based steps, and
then updates \(\mV^{(b)}\) by its history-centered closed-form solve. To account for the dependence of the Stage-1 optimum on the proxy feature representation, Stage 2 differentiates through an auxiliary closed-form Stage-1 solve; the linear bridge heads are then updated analytically \citet{galashov2025closedformlayeroptimization, xu2021deep,xu2021learning}. Appendix~\ref{sec:NeuralMeanEmbeddingSectionDoseResponse_Appendix} gives the complete objectives and algorithm.

\noindent\textbf{Final response regressions.}
On held-out folds \(\gI_3^{(\varphi)}\) and \(\gI_3^{(\mathrm{DR})}\), define the \textsc{TreatmentNet} and
\textsc{DR-PCL-Net} pseudo-outcomes
\(
\gamma_i^{(\varphi)}
:=
y_i\widehat\varphi(a_i,x_i,z_i),
\enspace
\gamma_i^{(\mathrm{DR})}
:=
\widehat\varphi(a_i,x_i,z_i)
\{y_i-\widehat h(a_i,x_i,w_i)\}.
\)
Regressing these pseudo-outcomes on \(a_i\) gives functions
\begin{align}
\widehat k_\varphi(a)
\approx
\E\!\left[
Y\widehat\varphi(a,X,Z)\mid A=a
\right],
\enspace
\widehat k_{\mathrm{DR}}(a)
\approx
\E\!\left[
\widehat\varphi(a,X,Z)
\{Y-\widehat h(a,X,W)\}
\mid A=a
\right].
    \label{eq:DR_ThirdStageMain}
\end{align}
The resulting estimators for \textsc{TreatmentNet} and \textsc{DR-PCL-Net} are given by
\(
\widehat f_{\mathrm{ATE}}^{(\varphi)}(a)
=
\widehat k_\varphi(a),
\enspace
\widehat f_{\mathrm{ATE}}^{\mathrm{DR}}(a)
=
\widehat f_{\mathrm{ATE}}^{(h)}(a)
+
\widehat k_{\mathrm{DR}}(a),
\) respectively,
where the \textsc{OutcomeNet} average in the second expression can be evaluated on the same held-out fold. These are the sample analogues of the treatment-bridge and doubly robust identification formulas in Section~\ref{sec:Identification_Main}.

{
\noindent\textbf{Extensions beyond population response.}
For heterogeneous responses, Proposition~\ref{prop:AlternativeCATEidentification} leads to a self-normalized estimator: a joint regression on \((A,V)\) learns the bridge-weighted numerator and bridge normalizer, whose stabilized ratio estimates \(f_{\mathrm{CATE}}(a,v)\); see Appendix~\ref{Appendix:HeterogeneousTreatmentAlgorithm}. The doubly robust version replaces the numerator by the corresponding bridge-weighted residual correction as outlined in Appendix~\ref{Appendix:HeterogeneousDRAlgorithm}.
{For treatment-conditional responses, our pairwise objective learns \(\varphi^{\mathrm{ATT}}(a,a',x,z)\) jointly over both treatment arguments. A single final regression then estimates \(f_{\mathrm{ATT}}(a,a')\), avoiding a separate fit for each reference level \(a'\) (Appendices~G.2--G.3).}
}

\vspace{-3mm}
\section{{Projected-error analysis}}
\vspace{-2mm}
\label{sec:consistency_main}

{

Appendix~\ref{app:RademacherBounds} derives nonasymptotic bounds for the sample-split estimators under stagewise approximate empirical minimization \citep[Section 3.2]{vanderVaart1996weak}. The \textsc{DR-PCL-Net} error decomposes into final averaging and regression errors and a doubly robust remainder. Under bounded complementary bridge errors, this remainder is controlled by the smaller of the outcome- and treatment-side projected bridge errors. Consequently, once the final-stage errors vanish, projected convergence of either bridge is sufficient for consistency. Here we summarize the resulting population dose-response rate consequences; the corresponding CATE and ATT results are given in Appendices~\ref{app:self_normalized_cate_consistency} and~\ref{app:att_treatment_bridge_consistency}--\ref{app:dr_att_consistency}. The bridge error bounds control projected conditional-moment errors \citep{bennett2025inference} rather than strong \(L^2\)-recovery of either solution to the underlying ill-posed inverse problems, and no separate density-ratio estimation error appears.

}

Let
\(
Q_h:=(A,X,Z), Q_\varphi:=(A,X,W),\enspace \text{and} \enspace
m_0(q_h):=\E[Y\mid Q_h=q_h]\in\mathbb R^{d_Y}.
\)
For vector-valued \(u\) and scalar \(v\), respectively, define the conditional expectation operators \((T_hu)(a,x,z)
\coloneqq \E[u(a,x,W)\mid A=a,X=x,Z=z],\) and \((T_\varphi v)(a,x,w) \coloneqq \E[v(a,x,Z)\mid A=a,X=x,W=w].\) We measure the bridge learning through projected conditional-moment residuals: \(\mathcal R_h^{\mathrm{proj}}
\coloneqq 
\|T_h\hat h-m_0\|_{L^2(\sP_{Q_h};\mathbb R^{d_Y})}^2, \) and \(\mathcal R_\varphi^{\mathrm{proj}}
\coloneqq
\|T_\varphi\hat\varphi-r_{0,\mathrm{ATE}}\|_{L^2(\sP_{Q_\varphi})}^2.\) When valid bridges exist, these equal \(\|T_h(\widehat h-h_0)\|_2^2\) and \(\|T_\varphi(\widehat\varphi-\varphi_0^{\mathrm{ATE}})\|_2^2\),
respectively. At empirical-risk stage \(j\), the optimization gap \(\xi_j\ge0\) is  defined by
\(
\widehat R_j(\widehat\vartheta_j)
\le
\inf_{\vartheta\in\mathcal V_j}\widehat R_j(\vartheta)+\xi_j.
\)

For \(b\in\{h,\varphi\}\), let \(\nu_b=(n_b,m_b)\) denote the two bridge-training sample sizes. Suppressing confidence arguments and fixed bounded multipliers given in the Appendix~\ref{app:RademacherBounds}, define
\[
\rho_{h,\nu_h}
\coloneqq \kappa_{h,1,\nu_h}+\kappa_{h,2,\nu_h}
+\Delta_{h,1,\nu_h}+\Delta_{h,2,\nu_h},\enspace \rho_{\varphi,\nu_\varphi}
\coloneqq \kappa_{\varphi,1,\nu_\varphi}+\kappa_{\varphi,2,\nu_\varphi}
+\Delta_{\varphi,1,\nu_\varphi}
+\Delta_{\varphi,2,\nu_\varphi}^{\mathrm{RF},U}.
\]
Here the \(\kappa\)-terms are approximation errors and the
\(\Delta\)-terms are uniform empirical-process deviations for the
corresponding loss classes; see also \citet{bartlett2005} and \citet{foster2023orthogonal}. In particular,
\(\rho_{\varphi,\nu_\varphi}\) contains no separate density-ratio
estimation error. Let \(t_h,t_\varphi\), and \({t_{\mathrm{DR}}}\) be the held-out sample sizes for the {\textsc{OutcomeNet}} average, {\textsc{TreatmentNet}} response regression, and {\textsc{DR-PCL-Net}} final stage, respectively.

Suppressing confidence arguments and fixed bounded multipliers, also define
\(
\rho_{\varphi,3,t_\varphi}
\coloneqq
\kappa_{\varphi,3,t_\varphi}
+\Delta_{\varphi,3,t_\varphi},
\enspace
\rho_{\mathrm{DR},3,{t_{\mathrm{DR}}}}
\coloneqq
\kappa_{\mathrm{DR},{t_{\mathrm{DR}}}}
+\Delta_{\mathrm{DR},{t_{\mathrm{DR}}}}.
\)
These are the final-stage regression errors for {\textsc{TreatmentNet}} and the
{\textsc{DR-PCL-Net}} residual correction in Eq.~\ref{eq:DR_ThirdStageMain}. Furthermore, writing \(\xi_{h,1,\nu_h},\xi_{h,2,\nu_h}\) and \(\xi_{\varphi,1,\nu_\varphi},\xi_{\varphi,2,\nu_\varphi}\) for the bridge-stage optimization gaps and \(\xi_{\varphi,3,t_\varphi},\xi_{\mathrm{DR},3,{t_{\mathrm{DR}}}}\) for those of the final regressions, define, again up to fixed bounded constants, \(\tilde\rho_{h,\nu_h}
\coloneqq
\rho_{h,\nu_h}+\xi_{h,1,\nu_h}+\xi_{h,2,\nu_h},
\enspace
\tilde\rho_{\varphi,\nu_\varphi}
\coloneqq
\rho_{\varphi,\nu_\varphi}+\xi_{\varphi,1,\nu_\varphi}+\xi_{\varphi,2,\nu_\varphi}, \enspace \tilde\rho_{\varphi,3,t_\varphi}
\coloneqq
\rho_{\varphi,3,t_\varphi}+\xi_{\varphi,3,t_\varphi},
\enspace
\tilde\rho_{\mathrm{DR},3,{t_{\mathrm{DR}}}}
\coloneqq
\rho_{\mathrm{DR},3,{t_{\mathrm{DR}}}}+\xi_{\mathrm{DR},3,{t_{\mathrm{DR}}}}.\) Corollaries~\ref{cor:outcomenet_projected_consistency} and~\ref{cor:treatmentnet_projected_consistency_plugin} give, for the two bridge routes,
\(
\mathcal R_b^{\mathrm{proj}}
=
O_p\!\left(\tilde\rho_{b,\nu_b}\right),
\enspace b\in\{h,\varphi\}.
\)

\begin{theorem}[Population dose-response error bounds]
\label{thm:main_dose_response_consistency}
Suppose the conditions of
Corollaries~\ref{cor:outcomenet_projected_consistency} and
\ref{cor:treatmentnet_projected_consistency_plugin}
and of
Theorems~\ref{thm:outcome_dose_response_consistency},
\ref{thm:outcome_empirical_averaging_error},
\ref{thm:treatment_dose_response_rademacher}, and
\ref{thm:dr_ate_rademacher}
hold. For the \textsc{DR-PCL-Net} bound, also suppose
Assumption~\ref{assum:bounded_complementary_bridge_errors}
holds. Then the corresponding squared-error rates are:
\begin{list}{\textup{(\roman{enumi})}}{
    \usecounter{enumi}
    \setlength{\leftmargin}{1.8em}
    \setlength{\labelwidth}{1.3em}
    \setlength{\labelsep}{0.5em}
    \setlength{\itemsep}{0.25em}
    \setlength{\parsep}{0pt}
    \setlength{\topsep}{0.25em}}
\item {\textsc{OutcomeNet}: }
\(\displaystyle
\|\hat f_{\mathrm{ATE}}^{(h)}-f_{\mathrm{ATE}}
\|_{L^2(\sP_A;\mathbb R^{d_Y})}^2
=O_p\!\left( t_h^{-1} + {\tilde\rho_{h,\nu_h}}\right).
\)

\item {\textsc{TreatmentNet}: }
\(\displaystyle
\|\hat f_{\mathrm{ATE}}^{(\varphi)}-f_{\mathrm{ATE}}
\|_{L^2(\sP_A;\mathbb R^{d_Y})}^2
=O_p\!\left( {\tilde\rho_{\varphi,3,t_\varphi}} +
{\tilde\rho_{\varphi,\nu_\varphi}}
\right).
\)

\item {\textsc{DR-PCL-Net}: }
\(\displaystyle
\begin{aligned}[t]
\|\hat f_{\mathrm{ATE}}^{\mathrm{DR}}-f_{\mathrm{ATE}}
\|_{L^2(\sP_A;\mathbb R^{d_Y})}^2
&=O_p\!\left(
{t_{\mathrm{DR}}^{-1}}
+{\tilde\rho_{\mathrm{DR},3,{t_{\mathrm{DR}}}}}\right.\left.
+\min\{{\tilde\rho_{h,\nu_h}},{\tilde\rho_{\varphi,\nu_\varphi}}\}
\right).
\end{aligned}
\)
\end{list}
\end{theorem}

{Theorem~\ref{thm:main_dose_response_consistency} allows consistent causal-response estimation without \(L^2\) convergence of either bridge.} Under Assumption~\ref{assum:bounded_complementary_bridge_errors}, consistency follows if
\(
t_{\mathrm{DR}}\to\infty,
\enspace
\widetilde\rho_{\mathrm{DR},3,t_{\mathrm{DR}}}=o_p(1),
\enspace
\min\!\left\{
\widetilde\rho_{h,\nu_h},
\widetilde\rho_{\varphi,\nu_\varphi}
\right\}
=o_p(1).
\)
Importantly, our result does not require controlling strong $L^2$-norm convergence of either $\widehat{h}$ to $h_0$ or $\widehat{\varphi}$ to ${\varphi_0^{\mathrm{ATE}}}$. Since the bridge functions $h_0$ and ${\varphi_0^{\mathrm{ATE}}}$ are solutions to ill-posed inverse problems, controlling $L^2$-norms would require restrictive assumptions on measure of ill-posedness \citep{blundell2007semi, chen2018optimal}. {The projected bridge-rate inputs are provided by Corollaries~\ref{cor:outcomenet_projected_consistency} and~\ref{cor:treatmentnet_projected_consistency_plugin}, while Theorem~\ref{thm:dr_ate_rademacher} propagates them through the doubly robust remainder and the final-stage errors.}



\vspace{-3mm}
\section{Numerical experiments}
\vspace{-2mm}
\label{sec:NumericalExperiments}

We evaluate the proposed neural doubly robust proxy estimators against neural baselines {\textsc{TreatmentNet}},
{{\textsc{OutcomeNet}} \citep[][with our updated optimization]{xu2021deep}}, NMMR \citep{kompa2022deep}, and fixed-kernel feature baselines DRKPV \citep{bozkurt2025densityratiofreedoublyrobust}, PKDR \citep{wu2024doubly}, and KPV \citep{Mastouri2021ProximalCL, singh2023kernelmethodsunobservedconfounding}. Performance is measured by causal MSE between estimated and ground-truth response curves over the evaluation grid. The results use ten data-generating seeds and three split/initialization repetitions per seed; repetitions are averaged within each seed, and standard errors are computed across the ten seed-level means. Baselines are included only when compatible with the benchmark. Implementation details and neural architectures used in the experiments are given in Appendix~\ref{sec:Appendix_NumericalExperimentsSupp}.

\noindent\textbf{{Bridge-route stress test.}} We use two analytic DGPs with closed-form bridges to isolate route-specific failure modes; Appendix~\ref{app:BridgeRoutePreference} gives the complete DGPs and bridge derivations. In Scenario~A, increasing \(\kappa_2\) strengthens treatment selection through the outcome-relevant latent factor \(U_2\). The exact outcome bridge remains invariant, whereas worsening overlap increases the second moment of the treatment bridge. In Scenario~B, decreasing \(\lambda_{W_2}\) weakens the outcome-proxy measurement of \(U_2\) and amplifies the corresponding components of the outcome bridge. The treatment bridge remains invariant because treatment depends only on \(U_1\). Oracle and learned errors exhibit the corresponding reversal, while \textsc{DR-PCL-Net} remains below the stressed route. 
\begin{table}[H]
\centering
\caption{Bridge-route stress tests at $n=4000$. The oracle diagnostics replace the corresponding learned bridge by its closed-form value while retaining held-out final-stage for estimation. \looseness=-1}
\label{tab:main_bridge_route}
\scriptsize
\setlength{\tabcolsep}{3.2pt}
\renewcommand{\arraystretch}{0.94}
\begin{adjustbox}{max width=\textwidth}
\begin{tabular}{@{}llccc cc@{}}
\toprule
& & \multicolumn{3}{c}{Learned estimators}
& \multicolumn{2}{c}{Oracle diagnostics} \\
\cmidrule(lr){3-5}\cmidrule(l){6-7}
Scenario & Stress level & {\textsc{OutcomeNet}} & {\textsc{TreatmentNet}} & {\textsc{DR-PCL-Net}}
& Outcome bridge & Treatment bridge \\
\midrule
\multirow{4}{*}{A}
& ${\kappa_2}=0.0$ & $3.799\pm0.517$ & $4.775\pm0.414$ & $\mathbf{2.403\pm0.319}$ & $\underline{2.339\pm0.351}$ & $2.551\pm0.275$ \\
& ${\kappa_2}=0.2$ & $4.034\pm0.610$ & $4.998\pm0.846$ & $\mathbf{2.714\pm0.311}$ & $\underline{2.339\pm0.351}$ & $2.850\pm0.429$ \\
& ${\kappa_2}=0.4$ & $4.877\pm0.702$ & $7.974\pm1.230$ & $\mathbf{4.175\pm0.546}$ & $\underline{2.339\pm0.351}$ & $4.375\pm0.557$ \\
& ${\kappa_2}=0.6$ & $9.689\pm1.077$ & $27.388\pm3.307$ & $\mathbf{8.907\pm1.274}$ & $\underline{2.339\pm0.351}$ & $11.323\pm1.166$ \\
\midrule
\multirow{5}{*}{B}
& ${\lambda_{W_2}}=0.9$ & $1.259\pm0.074$ & $1.591\pm0.157$ & $\mathbf{0.762\pm0.058}$ & $\underline{0.677\pm0.125}$ & $1.130\pm0.100$ \\
& ${\lambda_{W_2}}=0.7$ & $1.979\pm0.157$ & $1.592\pm0.186$ & $\mathbf{1.494\pm0.134}$ & $\underline{0.902\pm0.124}$ & $1.130\pm0.100$ \\
& ${\lambda_{W_2}}=0.6$ & $2.605\pm0.253$ & $\mathbf{1.944\pm0.216}$ & $2.408\pm0.199$ & $1.236\pm0.163$ & $\underline{1.130\pm0.100}$ \\
& ${\lambda_{W_2}}=0.5$ & $4.759\pm0.793$ & $\mathbf{2.400\pm0.255}$ & $3.205\pm0.288$ & $2.305\pm0.316$ & $\underline{1.130\pm0.100}$ \\
& ${\lambda_{W_2}}=0.4$ & $7.604\pm1.081$ & $\mathbf{2.936\pm0.359}$ & $5.057\pm0.421$ & $3.483\pm0.636$ & $\underline{1.130\pm0.100}$ \\
\bottomrule
\end{tabular}
\end{adjustbox}
\end{table}
Table~\ref{tab:main_bridge_route} demonstrates the results for both scenarios: the treatment route deteriorates while the oracle outcome bridge remains invariant in Scenario~A, whereas the outcome route deteriorates while the oracle treatment bridge remains invariant in Scenario~B. \textsc{DR-PCL-Net} stays below the deliberately stressed route throughout both sweeps, illustrating protection when the more stable bridge is unknown a priori rather than oracle route selection.

We next evaluate the performance across population, structured-treatment, heterogeneous, and treatment-conditional benchmarks.

\noindent\textbf{Continuous treatment population response benchmarks.} For population dose-response estimation, we use the Legalized Abortion and Crime benchmark \citep{LegalizedAbortionData,Mastouri2021ProximalCL},
the calibrated semi-synthetic Deaner IM/IR grade-retention designs
\citep{Fruehwirth_GradeRetentions,deaner2023proxycontrolspaneldata}, and the
nonlinear Demand design of \citet{xu2021deep}. 

\noindent\textbf{Structured-treatment population response benchmark.} To assess scalability beyond low-dimensional treatments, we adopt the structured-treatment dSprite design \citep{dsprites17, xu2021deep}, in which the treatment is a flattened \(64\times64\) image.

\noindent\textbf{Heterogeneous and treatment-conditional benchmarks.} We also evaluate the continuous treatment-conditional surface \(f_{\mathrm{ATT}}(a,a')=\E\{Y^{(a)}\mid A=a'\}\) and the heterogeneous curve \(v\mapsto f_{\mathrm{CATE}}(1,v)\) in designs adapted from \citet{wu2024doubly} and \citet{Abrevaya02102015}, respectively.

\begin{table}[H]
\centering

\caption{Population dose-response MSE (lower is better) for Legalized Abortion \& Crime, Deaner, and demand design datasets. \emph{Div.}: no screened
configuration converged.}
\label{tab:main_causal_mse}
\scriptsize
\setlength{\tabcolsep}{2.2pt}
\renewcommand{\arraystretch}{0.94}
\begin{adjustbox}{max width=\textwidth}
\begin{tabular}{@{}lccccc|cc@{}}
\toprule
Benchmark (sample size; scale)
& PKDR & DRKPV & NMMR-U & NMMR-V & \textsc{OutcomeNet}
& \textsc{TreatmentNet} & \textsc{DR-PCL-Net} \\
\midrule
Abortion \& Crime (1.5k; $10^3\!\times$)
& $20.335\pm0.904$ & $19.646\pm2.837$ & \emph{Div.}
& $19.774\pm3.601$ & $9.366\pm1.246$ & $51.455\pm4.126$
& $\mathbf{8.630\pm0.615}$ \\
Deaner--IM (3k; $10^3\!\times$)
& $2.756\pm0.366$ & $2.073\pm0.203$ & $1.405\pm0.089$
& $0.958\pm0.132$ & $0.857\pm0.045$ & $5.729\pm0.125$
& $\mathbf{0.805\pm0.050}$ \\
Deaner--IR (3k; $10^3\!\times$)
& $3.947\pm0.638$ & $1.028\pm0.064$ & $1.635\pm0.021$
& $1.809\pm0.149$ & $0.790\pm0.132$ & $9.101\pm0.139$
& $\mathbf{0.640\pm0.127}$ \\
Demand (2k)
& $107.86\pm23.96$ & $89.18\pm3.74$ & $25.58\pm0.77$
& $\mathbf{22.39\pm1.61}$ & $32.95\pm3.08$ & $199.67\pm15.77$
& $26.74\pm2.66$ \\
Demand (5k)
& $129.54\pm17.23$ & $85.60\pm2.65$ & $22.30\pm1.34$
& $33.99\pm3.72$ & $23.62\pm2.09$ & $174.11\pm16.50$
& $\mathbf{18.98\pm1.95}$ \\
Demand (10k)
& $172.52\pm17.28$ & $82.53\pm1.21$ & $19.12\pm1.81$
& $53.18\pm3.77$ & $22.42\pm1.81$ & $153.84\pm5.57$
& $\mathbf{17.37\pm1.07}$ \\
Demand (15k)
& $134.48\pm14.36$ & $83.82\pm2.26$ & $\mathbf{14.55\pm1.16}$
& $42.70\pm2.37$ & $21.15\pm1.93$ & $144.97\pm4.24$
& $16.38\pm1.37$ \\
Demand (20k)
& $208.81\pm21.97$ & $80.30\pm0.90$ & $16.39\pm2.50$
& $48.64\pm3.46$ & $20.18\pm1.61$ & $145.76\pm4.40$
& $\mathbf{15.77\pm1.14}$ \\
\bottomrule
\end{tabular}
\end{adjustbox}
\end{table}

\begin{figure}[t]
\centering
\includegraphics[width=.327\textwidth]{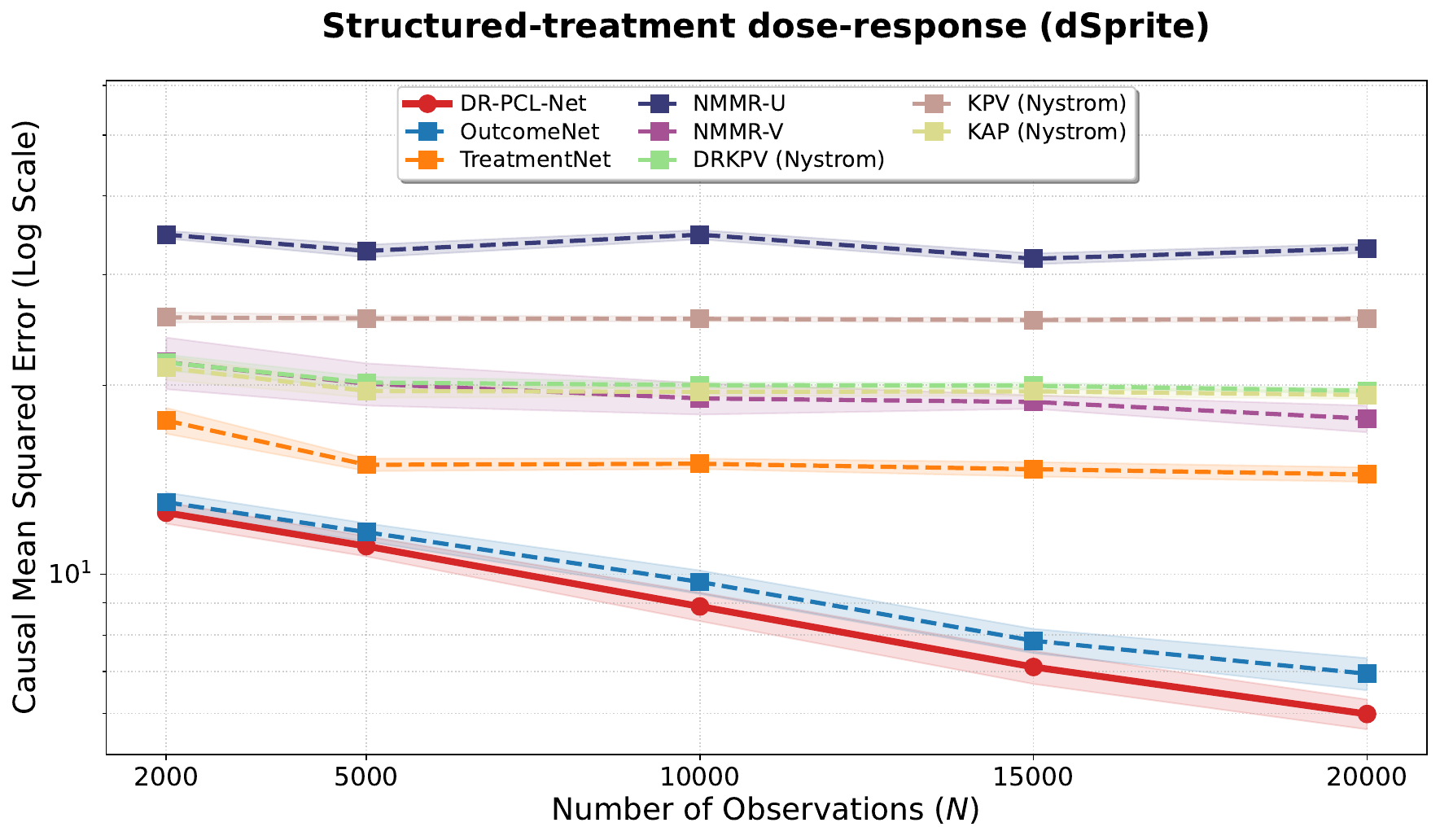}\hfill
\includegraphics[width=.327\textwidth]{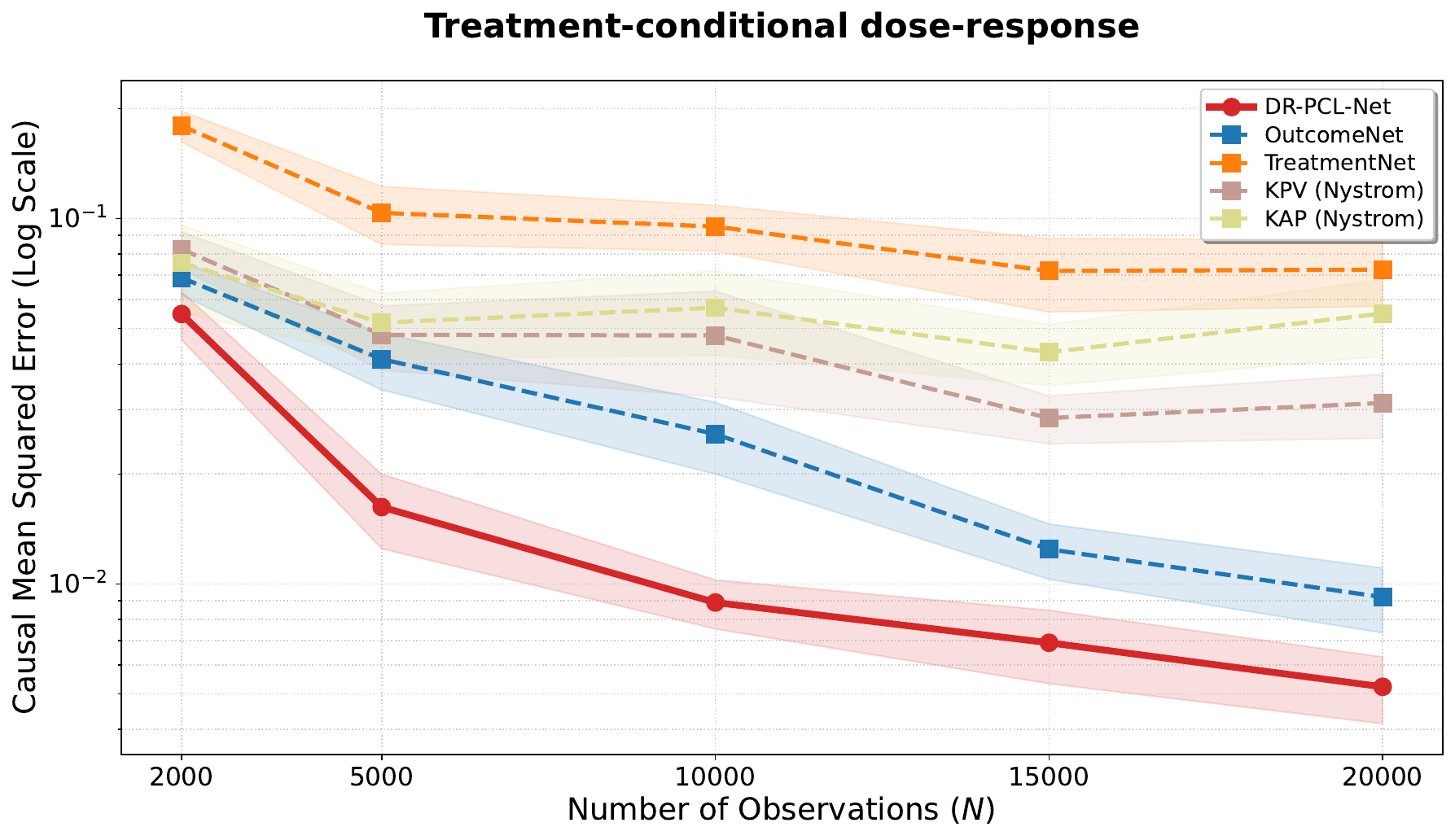}\hfill
\includegraphics[width=.327\textwidth]{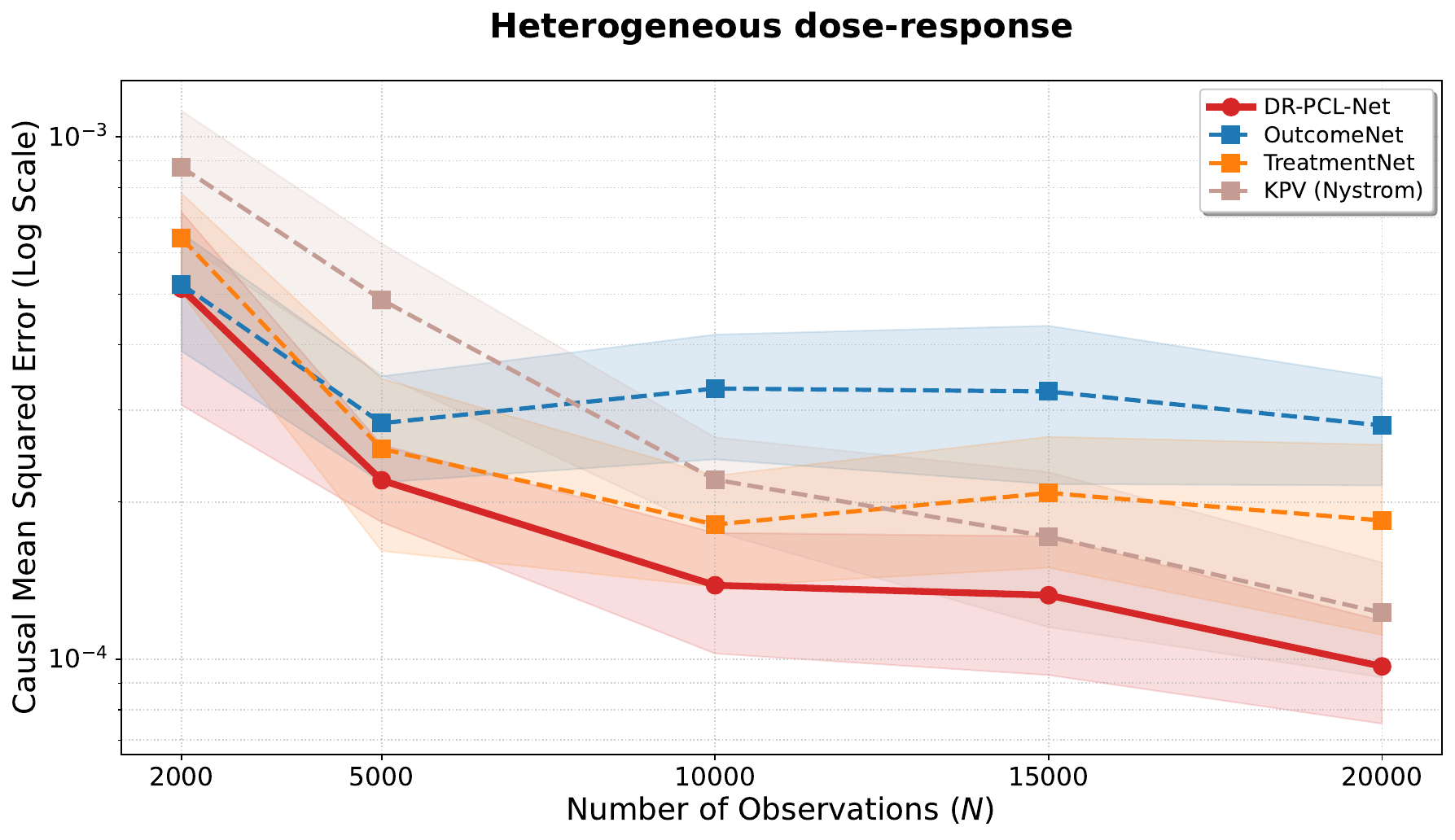}
\caption{Causal-MSE size sweeps for dSprite (left), treatment-conditional dose-response (center), and heterogeneous dose-response (right). Each point averages repetitions within a data-generating seed; shaded regions are 95\% confidence intervals across the seed-level means. }
\label{fig:structured_att_cate_results}
\end{figure}
\noindent\textbf{Overall results.}
\textsc{DR-PCL-Net} attains the lowest reported mean MSE in six of the eight settings in Table~\ref{tab:main_causal_mse}, including Abortion \& Crime, both Deaner benchmarks, and three of the five Demand sample sizes; it remains competitive in the other two. Figure~\ref{fig:structured_att_cate_results} shows the same overall pattern on the structured-treatment, treatment-conditional, and heterogeneous benchmarks, where \textsc{DR-PCL-Net} is consistently among the strongest methods across sample sizes. 

\vspace{-3mm}
\section{Conclusion}
\vspace{-2mm}
\label{sec:Conclusion}
We introduced a density-ratio-free neural treatment bridge estimator and combined it with a neural outcome bridge to estimate population, heterogeneous, and treatment-conditional response curves under unobserved confounding. The resulting doubly robust estimator requires no auxiliary density-ratio estimate and is suitable for continuous and structured treatment settings. We also established projected-norm error bounds showing that the doubly robust error is controlled by the final averaging/regression errors and the smaller of the outcome- and treatment-side projected bridge errors. Empirically, \textsc{DR-PCL-NET} achieves the best mean performance in most reported settings.

\noindent\textbf{Limitations and future work.} {The statistical analyzes approximate empirical minimizers, but does not establish convergence guarantees for the implemented nonconvex stochastic training procedure. Deriving such optimizer-specific guarantees, potentially using techniques from \citet{chen2025towards}, is left for future work.}

\subsection*{AI use statement}
AI-assisted tools were used to improve the language and organization of the manuscript, assist with code editing, and help clarify and cross-check mathematical concepts and derivations. All resulting text, code, and mathematical arguments were independently reviewed and verified by the authors, who remain fully responsible for the content of the manuscript and the accompanying implementation.






\bibliography{bibliography}

\newpage
\appendix

\startcontents[appendix]

\section*{{Appendix contents}}
\printcontents[appendix]{}{1}{\setcounter{tocdepth}{2}}
\bigskip

\newpage

\section{Interpretation and scope of the completeness assumptions}
\label{app:identification_assumptions}

The conditions in Section~\ref{sec:Identification_Main} are standard
sufficient conditions for proximal identification
results
\citep{Miao2018Identifying,semiparametricProximalCausalInference,TchetgenTchetgen2024IntroProximal}. Closely related injectivity conditions also underlie nonparametric instrumental-variable regression \citep{Newey_NIV,Newey_nonparametric_IV}. PCL also does not remove causal assumptions: it retains well-defined interventions, no interference, consistency, latent exchangeability, and overlap.  Relative to adjustment on observed covariates, it replaces exchangeability given \(X\) with exchangeability given \((U,X)\), together with the proxy conditional independence conditions \(Y\perp Z\mid(U,A,X)\) and \(W\perp(A,Z)\mid(U,X)\). These restrictions are most credible when temporal ordering, design information, or scientific knowledge supports \(Z\) as a treatment-inducing proxy (negative-control exposure) and \(W\) as an outcome-inducing proxy (negative-control outcome). Because \(U\) is unobserved, neither the proxy restrictions nor latent exchangeability and overlap are generally verifiable from the observed law alone. Below, we give the discussion on the completeness assumptions.

\paragraph{Completeness is an information condition.} For fixed \((a,x)\) and \(Q\in\{Z,W\}\), define

\[
T_{Q,a,x}\ell(q)
    :=\E\{\ell(U)\mid Q=q,A=a,X=x\}.
\]

Assumption~\ref{assum:Completeness} states that
\(\ker(T_{Q,a,x})=\{0\}\).  Hence, no nonzero square-integrable
direction of \(U\) is invisible to the relevant proxy after conditioning on
\((A,X)\). The two clauses of Assumption~\ref{assum:Completeness} have route-specific roles: \(Z\)-completeness transfers the outcome-bridge equation to the latent law, whereas \(W\)-completeness does so for the treatment bridge.

In the discrete case, it can be understood as a rank-condition. Within an \((a,x)\)-stratum, \(U\) has \(K\) states and
\(Q\) has \(L\) states with positive probabilities.  Then \(Q\)-completeness
is equivalent to \citet{Miao2018Identifying, semiparametricProximalCausalInference}

\[
\operatorname{rank}\!\left[
  \sP(Q=q_j\mid U=u_k,A=a,X=x)
\right]_{j\leq L,\,k\leq K}=K.
\]

For continuous proxies, a standard sufficient condition is that, within
each \((a,x)\)-stratum, the conditional law of \(U\) given \(Q=q\) belongs to a sufficiently rich exponential family as described in \citet{Newey_nonparametric_IV}[Theorem 2.2]. Specifically, suppose that, relative to a common dominating measure,
\[
p_{U\mid Q,A,X}(u\mid q,a,x)
=
b_{a,x}(u)
\exp\!\left\{
\eta_{a,x}(q)^\top \tau_{a,x}(u)
-
\Psi_{a,x}\!\left(\eta_{a,x}(q)\right)
\right\},
\]
where \(b_{a,x}(u)>0\), the sufficient statistic \(\tau_{a,x}\) is one-to-one in \(u\), and the support of the natural parameter \(\eta_{a,x}(Q)\), conditional on \(A=a,X=x\), contains a nonempty open set. Under the usual finite-exponential-moment conditions, these assumptions imply \(Q\)-completeness \citep{Newey_nonparametric_IV}. Intuitively, \(\tau_{a,x}\) must distinguish different latent states, while variation in \(Q\) must move the conditional law of \(U\) through sufficiently many directions to reveal each of them. For example, it holds if \(U\mid Z=z,A=a,X=x\sim\mathcal N(\alpha_{a,x}+\Gamma_{a,x}z,\Sigma_{a,x})\), \(\Sigma_{a,x}\) is positive definite, \(\Gamma_{a,x}\) has row rank \(d_U\), and the support of \(Z\) contains an open set; see \citet{Newey_nonparametric_IV}[Theorem 2.3].

\paragraph{Sufficient conditions for bridge existence.}  We next give a Picard-type sufficient condition for the existence of the treatment bridge. A parallel result for outcome-bridge existence is given by \citet[Appendix~B.1]{xu2021deep}. The proxy restrictions and \(Z\)-completeness ensure that the target derivative belongs to the closure of the range of the treatment-bridge operator; the Picard condition below ensures that it belongs to the range itself.

For fixed \(a, x\), recall that the conditional expectation operators are defined as 
\[
\begin{aligned}
T_h &:L^2(\sP_{W\mid A=a,X=x})
       \longrightarrow L^2(\sP_{Z\mid A=a,X=x}), \quad T_h b = \E[b(W) \mid A = a, X = x, Z = \cdot]\\
T_\varphi &:L^2(\sP_{Z\mid A=a,X=x})
       \longrightarrow L^2(\sP_{W\mid A=a,X=x}), \quad T_\varphi g = \E[g(Z) \mid A = a, X = x, W = \cdot].
\end{aligned}
\]

Iterated expectation gives
\[
\langle T_h b,c\rangle_{L^2(\sP_{Z\mid a,x})}
=\E\{b(W)c(Z)\mid A=a,X=x\}
=\langle b,T_\varphi c\rangle_{L^2(\sP_{W\mid a,x})},
\]
Hence \(T_\varphi=T_h^*\), and therefore \(T_h=T_\varphi^*\).

Solving the outcome-bridge equation requires its target regression to belong to \(\operatorname{ran}(T_h)\). Likewise, solving the treatment-bridge equation requires the relevant Radon--Nikodym derivative to belong to \(\operatorname{ran}(T_\varphi)\).

For the population and treatment-conditional targets, define the latent Radon--Nikodym derivatives
\[
\begin{aligned}
r^U_{0,\mathrm{ATE}}(a,x,u)
&:=\frac{d\sP_{X,U}}{d\sP_{X,U\mid A=a}}(x,u),\\
r^U_{0,\mathrm{ATT}}(a,a',x,u)
&:=\frac{d\sP_{X,U\mid A=a'}}
        {d\sP_{X,U\mid A=a}}(x,u).
\end{aligned}
\]
For \(\tau\in\{\mathrm{ATE},\mathrm{ATT}\}\), let \(r^U_\tau\) denote the corresponding latent derivative and let \(r_\tau\) denote, respectively, \(r_{0,\mathrm{ATE}}(a,x,\cdot)\) or \(r_{0,\mathrm{ATT}}(a,a',x,\cdot)\). 

\begin{proposition}[Sufficient conditions for treatment-bridge existence]
\label{prop:treatment_bridge_existence}
Suppose Assumption~\ref{assum:ProxyConditionalIndependenceAssumptions} and
Assumption~\ref{assum:Completeness}~(i) hold. Assume that the numerator law in
the definition of \(r^U_\tau\) is absolutely continuous with respect to its
denominator and \(r^U_\tau\in L^2(\sP_{U\mid A=a,X=x})\). Suppose
\(T_\varphi\), equivalently \(T_h=T_\varphi^*\), is compact, and let
\(\{(\mu_j,e_j,f_j)\}_{j\geq1}\) be a singular system for
\(T_\varphi\), where \(\mu_j>0\) and
\[
T_\varphi e_j=\mu_j f_j,
\qquad
T_hf_j=\mu_j e_j.
\]
If
\begin{equation}
\sum_{j\geq1}\frac{
  |\langle r_\tau,f_j\rangle_{L^2(\sP_{W\mid A=a,X=x})}|^2}
 {\mu_j^2}<\infty,
\label{eq:treatment_bridge_picard_condition}
\end{equation}
then there exists
\(\varphi_0^{\tau}\in
L^2(\sP_{Z\mid A=a,X=x})\) satisfying
\[
T_\varphi\varphi_0^{\tau}=r_\tau,
\]
and a solution is given by
\begin{align}
\varphi_0^{\tau}
=\sum_{j\geq1}
  \frac{\langle r_\tau,f_j\rangle_{L^2(\sP_{W\mid A=a,X=x})}}
       {\mu_j}e_j.    
\label{eq:TreatmentPicardSolution}
\end{align}

\end{proposition}

\begin{proof}
Because the source and target laws extend through the same conditional law of
\(W\) given \((U,X)\), the projection property of Radon--Nikodym derivatives
and the proxy restrictions give
\[
r_\tau(W)
=\E\{r^U_\tau(U)\mid A=a,X=x,W\}
\quad \sP_{W\mid A=a,X=x}\text{-a.e.}
\]
Thus \(r_\tau\in L^2(\sP_{W\mid A=a,X=x})\) by conditional Jensen. Let
\(b\in\ker(T_h)\) and set
\(\ell_b(U):=\E\{b(W)\mid A=a,X=x,U\}\). Conditional Jensen gives
\(\ell_b\in L^2(\sP_{U\mid A=a,X=x})\). Moreover,
\[
0=(T_hb)(Z)
=\E\{\ell_b(U)\mid A=a,X=x,Z\},
\]
where the second equality follows from
\(W\perp Z\mid(U,A,X)\). Assumption~\ref{assum:Completeness}~(i) therefore
implies \(\ell_b(U)=0\) almost surely. Consequently,
\[
\begin{aligned}
\langle r_\tau,b\rangle_{L^2(\sP_{W\mid A=a,X=x})}
&=\E\{r^U_\tau(U)b(W)\mid A=a,X=x\}\\
&=\E\{r^U_\tau(U)\ell_b(U)\mid A=a,X=x\}=0.
\end{aligned}
\]
Hence \(r_\tau\in\ker(T_h)^\perp=\ker(T_\varphi^*)^\perp\).
Equation~\ref{eq:treatment_bridge_picard_condition} and Picard's theorem now
imply \(r_\tau\in\operatorname{ran}(T_\varphi)\) and give the solution in Eq.~\ref{eq:TreatmentPicardSolution} by \citep{linearIntegralEquationskress2013}[Theorem 15.18].
\end{proof}

Any jointly measurable and suitably square-integrable ATE solution gives a CATE bridge through Proposition~\ref{prop:AlternativeCATEidentification} whenever \(0<D_0(a,v)<\infty\).For ATT, suppose that, for almost every \((a,a',x)\), a solution of the corresponding conditional bridge equation can be chosen so that the resulting map \((a,a',x,z)\mapsto\varphi_0^{\mathrm{ATT}}(a,a',x,z)\) is jointly measurable and square-integrable. Then it solves the joint surface equation
\(
T_{\mathrm{ATT}}\varphi_0^{\mathrm{ATT}}
=
r_{0,\mathrm{ATT}}.
\)

\section{Treatment bridge identification of the heterogeneous dose-response }
\label{Appendix:Appendix_CATEIdentification_withTreatmentBridge}

In this appendix, we prove the treatment-bridge identification results for the heterogeneous dose-response \(f_{\mathrm{CATE}}(a,v)\) stated in {Theorem~\ref{theorem:TreatmentBridgeIdentificationAllCausalFunctions}} and Proposition~\ref{prop:AlternativeCATEidentification}. The corresponding ATE and ATT statements follow from \citet{bozkurt2025density}. To match the notation of the main text, we write \(X=(S,V)\) and {express changes of measure through probability laws and Radon--Nikodym derivatives; density or mass-function quotients are used only as dominated-case interpretations}.

\begin{proof}[Proof of the treatment bridge identification for CATE statement in {Theorem~\ref{theorem:TreatmentBridgeIdentificationAllCausalFunctions}}]
Fix \(a \in \gA\) and \(v \in \gV\), and write \(\varphi_0 = \varphi_0^{\mathrm{CATE}}\) for brevity. By assumption, \(\varphi_0\) satisfies
\begin{equation}
\E[\varphi_0(a,v,S,Z) \mid A=a, S, V=v, W]
= {r_{0,\mathrm{CATE}}(a,v,S,W)}
\qquad
\sP_{W \mid A=a, S, V=v}\text{-a.e.}
\label{eq:app_cate_bridge_eq}
\end{equation}
Define the latent derivative
\[
r^U_{0,\mathrm{CATE}}(a,v,s,u)
:=
\frac{d\sP_{S,U\mid V=v}}
     {d\sP_{S,U\mid A=a,V=v}}(s,u),
\qquad
r^U_{0,\mathrm{CATE}}(a,v,\cdot,\cdot)
\in L^2(\sP_{S,U\mid A=a,V=v}).
\]
Using \(W \perp A \mid U,X\) from Assumption~\ref{assum:CSM_and_ProxyIndependence}, with \(X=(S,V)\), and the projection property of Radon-Nikodym derivatives give
\[
\E\{r^U_{0,\mathrm{CATE}}(a,v,S,U)\mid A=a,S,V=v,W\}
=r_{0,\mathrm{CATE}}(a,v,S,W).
\]

Comparing this identity with Eq.~\ref{eq:app_cate_bridge_eq}, followed by
Assumption~\ref{assum:Completeness}~(ii), yields
\begin{align}
\E\{\varphi_0^{\mathrm{CATE}}(a,v,S,Z)
      \mid A=a,S,V=v,U\}
=r^U_{0,\mathrm{CATE}}(a,v,S,U).
\label{eq:app_cate_bridge_identity}
\end{align}

Now define
\(
m_a(S,U,V) \coloneqq \E[Y \mid A=a,S,V,U].
\)
Then
\begin{align*}
&\E[Y \varphi_0(a,v,S,Z) \mid A=a,V=v] =
\E\!\left[
\E[Y \varphi_0(a,v,S,Z) \mid A=a,S,V=v,U]
\;\middle|\;
A=a,V=v
\right] \\
&\quad=
\E\!\left[
\E[Y \mid A=a,S,V=v,U]\,
\E[\varphi_0(a,v,S,Z) \mid A=a,S,V=v,U]
\;\middle|\;
A=a,V=v
\right] \\
&\quad=
\E\!\left[
m_a(S,U,v)\,
r^U_{0,\mathrm{CATE}}(a,v,S,U)
\;\middle|\;
A=a,V=v
\right] = \E[m_a(S,U,v) \mid V=v],
\end{align*}
where the second equality uses \(Y \perp Z \mid U,A,X\) from Assumption~\ref{assum:CSM_and_ProxyIndependence}, with \(X=(S,V)\), the third equality uses Equation \ref{eq:app_cate_bridge_identity}, and the last equality is the projection property of Radon-Nikodym derivative.

Finally, by consistency and latent exchangeability conditional on \((X,U)=(S,V,U)\),
\begin{align*}
\E[m_a(S,U,v) \mid V=v]
&=\E[Y^{(a)} \mid V=v]
=f_{\mathrm{CATE}}(a,v)\\
&=\E[Y \varphi_0^{\mathrm{CATE}}(a,v,S,Z) \mid A=a,V=v].
\end{align*}
\end{proof}

We next prove our alternative identification of $\varphi_0^{\mathrm{CATE}}$ using the treatment dose-response bridge function.
\begin{proof}[Proof of Proposition~\ref{prop:AlternativeCATEidentification}]

Let \(X=(S,V)\). By the ATE bridge equation and iterated expectations,
\begin{align*}
D_0(a,v)
&=\E\{\varphi_0^{\mathrm{ATE}}(a,X,Z)\mid A=a,V=v\}\\
&=\E\{r_{0,\mathrm{ATE}}(a,X,W)\mid A=a,V=v\}
=\frac{d(\sP_A\otimes\sP_V)}{d\sP_{A,V}}(a,v).
\end{align*}
The last equality is the conditional-expectation property of Radon--Nikodym
derivatives. The chain rule therefore yields
\[
r_{0,\mathrm{ATE}}(a,(s,v),w)
=D_0(a,v)r_{0,\mathrm{CATE}}(a,v,s,w)
\quad\text{a.e.}
\]
If \(0<D_0(a,v)<\infty\), taking conditional expectations shows that \(\varphi_0^{\mathrm{ATE}}(a,(s,v),z)/D_0(a,v)\) solves Eq.~\ref{eq:TreatmentBridgeEq_CATE}. The treatment-bridge identification of CATE then gives the ratio formula in Proposition~\ref{prop:AlternativeCATEidentification}.

It remains to verify self-normalized double robustness. If \(h=h_0\), the conditional outcome-bridge moment makes the residual numerator in the expression of \(\Psi_{\mathrm{CATE}}^{\mathrm{SN}}(a,v;h,\varphi)\) equal to zero. If instead \(\varphi=\varphi_0^{\mathrm{ATE}}\), division by \(D_0(a,v)\) produces the valid CATE bridge above, so the native CATE doubly robust identity applies. Thus the self-normalized functional equals \(f_{\mathrm{CATE}}(a,v)\) under
either route.
\end{proof}

\section{Doubly robust identification: dose, heterogeneous, and conditional response curves}
\label{sec:Appendix_DoublyRobustIdentification}

In this section, we prove the doubly robust identification results in {Theorem~\ref{theorem:DoublyRobustIdentifications}}. We show that each doubly robust formula remains valid if either the outcome bridge or the corresponding treatment bridge is correctly specified.

\begin{proof}
We prove the result separately for the ATE, CATE, and ATT. Throughout, we use the notation \(X=(S,V)\), and we write \(h_0(a,X,W)\) for the outcome bridge.

The ATE formula is the covariate-adjusted version of \citet{bozkurt2025densityratiofreedoublyrobust}[Theorem 2.7]. So, below, we prove CATE and ATT extensions.

\paragraph{CATE.}
Consider
\[
\Psi_{\mathrm{CATE}}(a,v;h,\varphi)
=
\E[\varphi(a,v,S,Z)\{Y-h(a,X,W)\}\mid A=a,V=v]
+\E[h(a,X,W)\mid V=v].
\]

If \(h=h_0\), then
\begin{align*}
\Psi_{\mathrm{CATE}}(a,v;h_0,\varphi)
&=
\E\!\left[
\varphi(a,v,S,Z)\,
\E[Y-h_0(a,X,W)\mid A=a,S,V=v,Z]
\middle| A=a,V=v
\right] \\
&\qquad + \E[h_0(a,X,W)\mid V=v] \\
&=
\E[h_0(a,X,W)\mid V=v]
=
f_{\mathrm{CATE}}(a,v),
\end{align*}
where the second equality again follows from Equation \ref{eq:OutcomeBridgeEquation}, since \(X=(S,V)\), and the last equality follows from the outcome bridge identification of CATE.

If \(\varphi=\varphi_0^{\mathrm{CATE}}\), then
\begin{align*}
&\E[\varphi_0^{\mathrm{CATE}}(a,v,S,Z)h(a,X,W)\mid A=a,V=v] \\
&\qquad=
\E\!\left[
\E[\varphi_0^{\mathrm{CATE}}(a,v,S,Z)\mid A=a,S,V=v,W]\,
h(a,X,W)
\middle| A=a,V=v
\right] \\
&\qquad=
\E\!\left[
r_{0,\mathrm{CATE}}(a,v,S,W)\,h(a,X,W)
\middle| A=a,V=v
\right] 
=
\E[h(a,X,W)\mid V=v].
\end{align*}
Hence,
\begin{align*}
\Psi_{\mathrm{CATE}}(a,v;h,\varphi_0^{\mathrm{CATE}})
&=
\E[Y\varphi_0^{\mathrm{CATE}}(a,v,S,Z)\mid A=a,V=v]
\\&\quad-\E[\varphi_0^{\mathrm{CATE}}(a,v,S,Z)h(a,X,W)\mid A=a,V=v]\\
&\quad+\E[h(a,X,W)\mid V=v] \\
&=
\E[Y\varphi_0^{\mathrm{CATE}}(a,v,S,Z)\mid A=a,V=v]
=
f_{\mathrm{CATE}}(a,v),
\end{align*}
where the last equality follows from the treatment bridge identification of CATE.

\paragraph{ATT.}
Consider
\[
\Psi_{\mathrm{ATT}}(a,a';h,\varphi)
=
\E[\varphi(a,a',X,Z)\{Y-h(a,X,W)\}\mid A=a]
+\E[h(a,X,W)\mid A=a'].
\]

If \(h=h_0\), then
\begin{align*}
\Psi_{\mathrm{ATT}}(a,a';h_0,\varphi)
&=
\E\!\left[
\varphi(a,a',X,Z)\,
\E[Y-h_0(a,X,W)\mid A=a,X,Z]
\middle| A=a
\right] \\
&\qquad + \E[h_0(a,X,W)\mid A=a'] \\
&=
\E[h_0(a,X,W)\mid A=a']
=
f_{\mathrm{ATT}}(a,a'),
\end{align*}
by Equation \ref{eq:OutcomeBridgeEquation} and outcome bridge identification of ATT.

If \(\varphi=\varphi_0^{\mathrm{ATT}}\), then
\begin{align*}
&\E[\varphi_0^{\mathrm{ATT}}(a,a',X,Z)h(a,X,W)\mid A=a] \\
&\qquad=
\E\!\left[
\E[\varphi_0^{\mathrm{ATT}}(a,a',X,Z)\mid A=a,X,W]\,
h(a,X,W)
\middle| A=a
\right] \\
&\qquad=
\E\!\left[
r_{0,\mathrm{ATT}}(a,a',X,W)\,h(a,X,W)
\middle| A=a
\right]  =
\E[h(a,X,W)\mid A=a'].
\end{align*}
Therefore,
\begin{align*}
\Psi_{\mathrm{ATT}}(a,a';h,\varphi_0^{\mathrm{ATT}})
&=
\E[Y\varphi_0^{\mathrm{ATT}}(a,a',X,Z)\mid A=a]
\\&-\E[\varphi_0^{\mathrm{ATT}}(a,a',X,Z)h(a,X,W)\mid A=a]
+\E[h(a,X,W)\mid A=a'] \\
&=
\E[Y\varphi_0^{\mathrm{ATT}}(a,a',X,Z)\mid A=a]
=
f_{\mathrm{ATT}}(a,a'),
\end{align*}
where the last equality follows from treatment bridge identification of ATT.

This proves the double robustness of all three identification formulas.
\end{proof}

\section{Influence-function representations for discrete targets}
\label{sec:Appendix_EIF_Discrete}

In this section, we derive and formally establish the Efficient Influence Functions (EIFs) for the bridge-based target causal functions. We focus specifically on the discrete treatment setting, where $A$ takes values in a finite set $\mathcal{A}$, and $Y$ is scalar. The restriction to discrete \(A\) is important: pointwise dose-response functionals are pathwise differentiable in the usual semiparametric sense for atoms \(A=a\), whereas continuous-treatment pointwise effects require a different local or smoothed efficiency theory \citep{kennedy2017, Colangelo2020, wu2024doubly, zenati2025doubledebiased}. For the CATE result below we also take \(V\) to be discrete, so that \(\mathbb P(V=v)>0\).

Throughout this section, \(O=(Y,A,X,Z,W)\), and for CATE we write \(X=(S,V)\) and use \(h_0(a,S,V,W)\) as shorthand for \(h_0(a,X,W)\). Define
\[
\pi_a:=\mathbb P(A=a),\qquad
\pi_v:=\mathbb P(V=v),\qquad
\pi_{a\mid v}:=\mathbb P(A=a\mid V=v).
\]

For positive-mass treatment and conditioning cells, the Radon--Nikodym
derivatives defined in the main text can be written as
\begin{align*}
r_{0,\mathrm{ATE}}(a,\cdot,\cdot)
&=\frac{d\sP_{X,W}}{d\sP_{X,W\mid A=a}},\\
r_{0,\mathrm{CATE}}(a,v,\cdot,\cdot)
&=\frac{d\sP_{S,W\mid V=v}}{d\sP_{S,W\mid A=a,V=v}},\\
r_{0,\mathrm{ATT}}(a,a',\cdot,\cdot)
&=\frac{d\sP_{X,W\mid A=a'}}{d\sP_{X,W\mid A=a}}.
\end{align*}
All expectations are taken under the true observed law unless otherwise stated.

\paragraph{Semiparametric regularity.}
In addition to the bridge existence, uniqueness, square-integrability, and
positivity conditions stated below, we use the same tangent-space regularity
condition as in \citep{semiparametricProximalCausalInference} and
\citep{bozkurt2025densityratiofreedoublyrobust}. Namely, the relevant
conditional expectation operators are assumed to have dense range; the stronger
surjectivity condition on the operators and their adjoints used by
\citep{semiparametricProximalCausalInference} is sufficient. For the ATE and ATT formulas,
one may take
\[
Tg:=\E[g(W,A,X)\mid Z,A,X],
\qquad
T^\ast r:=\E[r(Z,A,X)\mid W,A,X],
\]
with \(T:L^2(\sP_{W,A,X})\to L^2(\sP_{Z,A,X})\). For the CATE formula, with
\(X=(S,V)\), take
\[
T_{\mathrm C}g:=\E[g(W,A,S,V)\mid Z,A,S,V],
\qquad
T_{\mathrm C}^\ast r:=\E[r(Z,A,S,V)\mid W,A,S,V].
\]
This condition is used only to conclude that the influence functions derived
below are canonical gradients, hence efficient.

\begin{theorem}[Efficient influence functions for discrete targets]
\label{theorem:EIFs_formulation}
Let Assumptions~\ref{assum:ProxyConditionalIndependenceAssumptions} and~\ref{assum:Completeness} hold. Suppose that there exists a unique outcome bridge \(h_0\) satisfying Eq.~\ref{eq:OutcomeBridgeEquation}, and unique treatment bridges \({\varphi_0^{\mathrm{ATE}}},{\varphi_0^{\mathrm{CATE}}}, {\varphi_0^{\mathrm{ATT}}}\) satisfying the treatment bridge equations \ref{eq:TreatmentBridgeEq_ATE}. {Assume, for the targets under consideration, that \(\sP_{X,W}\ll\sP_{X,W\mid A=a}\), \(\sP_{S,W\mid V=v}\ll\sP_{S,W\mid A=a,V=v}\), and \(\sP_{X,W\mid A=a'}\ll\sP_{X,W\mid A=a}\), respectively.}
In particular,
\(\pi_a>0\), \(\pi_{a'}>0\), \(\pi_v>0\), and \(\pi_{a\mid v}>0\) whenever the
corresponding target is considered. Assume also the semiparametric regularity
condition above.

Then \(f_{\mathrm{ATE}}(a)\), \(f_{\mathrm{ATT}}(a,a')\), and, for discrete \(V\), \(f_{\mathrm{CATE}}(a,v)\) are pathwise differentiable. Their EIFs are given by

\medskip
\noindent\textbf{(i) ATE.} For any \(a\in\mathcal A\),
\begin{align}
\psi^{\mathrm{EIF}}_{\mathrm{ATE}}(O;a)
&=
\frac{\mathds 1\{A=a\}}{\pi_a}\,
{\varphi_0^{\mathrm{ATE}}}(a,X,Z)
\bigl\{Y-h_0(A,X,W)\bigr\}
+h_0(a,X,W)
-f_{\mathrm{ATE}}(a).
\label{eq:Our_EIF_ATE}
\end{align}

\medskip
\noindent\textbf{(ii) ATT / conditional counterfactual mean.} For any
\(a,a'\in\mathcal A\),
\begin{align}
\psi^{\mathrm{EIF}}_{\mathrm{ATT}}(O;a,a')
&=
\frac{\mathds 1\{A=a\}}{\pi_a}\,
{\varphi_0^{\mathrm{ATT}}}(a,a',X,Z)
\bigl\{Y-h_0(A,X,W)\bigr\} \\
&+
\frac{\mathds 1\{A=a'\}}{\pi_{a'}}
\bigl\{h_0(a,X,W)-f_{\mathrm{ATT}}(a,a')\bigr\}.
\label{eq:Our_EIF_ATT}
\end{align}

\medskip
\noindent\textbf{(iii) CATE, with discrete \(V\).} For any \(a\in\mathcal A\)
and any \(v\) with \(\pi_v>0\),
\begin{align}
\psi^{\mathrm{EIF}}_{\mathrm{CATE}}(O;a,v)
&=
\frac{\mathds 1\{V=v\}}{\pi_v}
\Bigg[
\frac{\mathds 1\{A=a\}}{\pi_{a\mid v}}\,
{\varphi_0^{\mathrm{CATE}}}(a,v,S,Z)
\bigl\{Y-h_0(A,S,V,W)\bigr\}
\nonumber\\
&\hspace{4.5cm}
+h_0(a,S,v,W)
-f_{\mathrm{CATE}}(a,v)
\Bigg].
\label{eq:Our_EIF_CATE}
\end{align}

\end{theorem}

\begin{proof}
The ATE formula is the covariate-adjusted version of
\citet[Theorem~B.4]{bozkurt2025densityratiofreedoublyrobust}. It is also
obtained from the EIF of \citet[Theorem~3.1]{semiparametricProximalCausalInference}
after replacing the classical treatment bridge \(q_0\), which satisfies {\(\E[q_0(a,X,Z)\mid A=a,X,W] =\pi_a^{-1}r_{0,\mathrm{ATE}}(a,X,W)\)}, by the normalized bridge \({\varphi_0^{\mathrm{ATE}}}(a,X,Z)=\pi_a q_0(a,X,Z)\). We give the derivations for ATT and CATE because the conditioning events change the derivative and the normalization.

Let {\(\{\sP_t:t\in(-\delta,\delta)\}\)} be a regular parametric submodel through the true law {\(\sP_0=\sP\)}, and let
\[
\dot\ell(O):=\left.\frac{\partial}{\partial t}\log p_t(O)\right|_{t=0}
\]
be its score. Let \(h_t\) denote a differentiable path of outcome bridges along
the submodel, and write
\(\dot h:=\left.\partial h_t/\partial t\right|_{t=0}\).

\paragraph{ATT.}
Fix \(a,a'\in\mathcal A\) and write
\[
\psi:=f_{\mathrm{ATT}}(a,a')
=\E[h_0(a,X,W)\mid A=a'].
\]
Define
\[
N_t:=\E_t[\mathds 1\{A=a'\}h_t(a,X,W)],
\qquad
D_t:= {\sP_t(A=a')},
\]
so that \(\psi_t=N_t/D_t\). Differentiating the ratio at \(t=0\) gives
\begin{align}
\left.\frac{\partial \psi_t}{\partial t}\right|_{t=0}
&=
\frac{1}{\pi_{a'}}
\left[
\left.\frac{\partial N_t}{\partial t}\right|_{t=0}
-\psi
\left.\frac{\partial D_t}{\partial t}\right|_{t=0}
\right].
\label{eq:att_ratio_derivative}
\end{align}
The denominator derivative is
\[
\left.\frac{\partial D_t}{\partial t}\right|_{t=0}
=
\E[\mathds 1\{A=a'\}\dot\ell(O)].
\]
For the numerator,
\begin{align}
\left.\frac{\partial N_t}{\partial t}\right|_{t=0}
&=
\E[\mathds 1\{A=a'\}h_0(a,X,W)\dot\ell(O)]
+
\E[\mathds 1\{A=a'\}\dot h(a,X,W)].
\label{eq:att_num_derivative}
\end{align}

It remains to express the second term in Eq.~\ref{eq:att_num_derivative} as an
inner product with the score. The bridge equation along the submodel is
\[
\E_t[Y-h_t(A,X,W)\mid A,X,Z]=0.
\]
Differentiating this equation and using the conditional mean-zero bridge
residual gives, on the event \(\{A=a\}\),
\begin{align}
\E[\dot h(a,X,W)\mid A=a,X,Z]
=
\E[\{Y-h_0(a,X,W)\}\dot\ell(O)\mid A=a,X,Z].
\label{eq:att_linearized_bridge}
\end{align}
The ATT treatment bridge satisfies
\[
{
\E[ {\varphi_0^{\mathrm{ATT}}}(a,a',X,Z)\mid A=a,X,W]
=r_{0,\mathrm{ATT}}(a,a',X,W).}
\]
Therefore, for any {measurable \(g(X,W)\) for which the following expectations are finite},
\begin{align}
\E[\mathds 1\{A=a\}{\varphi_0^{\mathrm{ATT}}}(a,a',X,Z)g(X,W)]
&=
\pi_a
\E[{\varphi_0^{\mathrm{ATT}}}(a,a',X,Z)g(X,W)\mid A=a]
\nonumber\\
&{=
\pi_a\E[r_{0,\mathrm{ATT}}(a,a',X,W)g(X,W)\mid A=a]}
\nonumber\\
&=
\pi_a \E[g(X,W)\mid A=a'].
\label{eq:att_reweighting_identity}
\end{align}
Taking \(g=\dot h(a,\cdot,\cdot)\) in Eq.~\ref{eq:att_reweighting_identity} yields
\[
\E[\mathds 1\{A=a'\}\dot h(a,X,W)]
=
\frac{\pi_{a'}}{\pi_a}
\E[\mathds 1\{A=a\}{\varphi_0^{\mathrm{ATT}}}(a,a',X,Z)\dot h(a,X,W)].
\]
Combining this identity with Eq.~\ref{eq:att_linearized_bridge} gives
\[
\frac{1}{\pi_{a'}}
\E[\mathds 1\{A=a'\}\dot h(a,X,W)]
=
\E\!\left[
\frac{\mathds 1\{A=a\}}{\pi_a} {\varphi_0^{\mathrm{ATT}}}(a,a',X,Z)
\{Y-h_0(a,X,W)\}\dot\ell(O)
\right].
\]
Substituting this display and Eq.~\ref{eq:att_num_derivative} into
Eq.~\ref{eq:att_ratio_derivative} gives
\[
\left.\frac{\partial \psi_t}{\partial t}\right|_{t=0}
=
\E[\psi^{\mathrm{EIF}}_{\mathrm{ATT}}(O;a,a')\dot\ell(O)].
\]
The candidate in Eq.~\ref{eq:Our_EIF_ATT} is mean zero because the first term has
conditional mean zero given \((A=a,X,Z)\), and the second term is centered
conditional on \(A=a'\).

It remains to check efficiency. Let \(\mathcal F=(Z,A,X)\) and
\(\varepsilon:=Y-h_0(A,X,W)\). Under the bridge-restricted semiparametric model,
the tangent space can be written as \(\Lambda=\Lambda_1+\Lambda_2\), where
\[
\Lambda_1
=
\{s_1(\mathcal F)\in L^2(\sP):\E[s_1]=0\},
\]
and
\[
\Lambda_2
=
\left\{
s_2\in L^2(\sP):
\E[s_2\mid \mathcal F]=0,\ 
\E[\varepsilon s_2\mid\mathcal F]\in\overline{\mathcal R(T)}
\right\}.
\]
Set \(H_{\mathrm{ATT}}(X,W):=h_0(a,X,W)-\psi\), and decompose
\(G_{\mathrm{ATT}}:=\psi^{\mathrm{EIF}}_{\mathrm{ATT}}\) as
\[
G_{\mathrm{ATT},1}:=\E[G_{\mathrm{ATT}}\mid\mathcal F]
=
\frac{\mathds 1\{A=a'\}}{\pi_{a'}}
\E[H_{\mathrm{ATT}}(X,W)\mid\mathcal F],
\qquad
G_{\mathrm{ATT},2}:=G_{\mathrm{ATT}}-G_{\mathrm{ATT},1}.
\]
Then \(G_{\mathrm{ATT},1}\in\Lambda_1\). Also
\(\E[G_{\mathrm{ATT},2}\mid\mathcal F]=0\). Moreover
\(\E[\varepsilon G_{\mathrm{ATT},2}\mid\mathcal F]\) is a square-integrable
\(\mathcal F\)-measurable function; by the dense-range condition for \(T\), it
belongs to \(\overline{\mathcal R(T)}\). Hence
\(G_{\mathrm{ATT},2}\in\Lambda_2\), so
\(G_{\mathrm{ATT}}\in\Lambda\). Since it represents the pathwise derivative and
belongs to the tangent space, it is the canonical gradient.

\paragraph{CATE.}
Fix \(a\in\mathcal A\) and \(v\) with \(\pi_v>0\). Write
\[
\psi:=f_{\mathrm{CATE}}(a,v)
=
\E[h_0(a,S,v,W)\mid V=v].
\]
Define
\[
N_t:=\E_t[\mathds 1\{V=v\}h_t(a,S,v,W)],
\qquad
D_t:=  {\sP_t(V=v)},
\]
so that \(\psi_t=N_t/D_t\). The quotient rule gives
\begin{align}
\left.\frac{\partial \psi_t}{\partial t}\right|_{t=0}
&=
\frac{1}{\pi_v}
\left[
\left.\frac{\partial N_t}{\partial t}\right|_{t=0}
-\psi
\left.\frac{\partial D_t}{\partial t}\right|_{t=0}
\right],
\label{eq:cate_ratio_derivative}
\end{align}
with
\[
\left.\frac{\partial D_t}{\partial t}\right|_{t=0}
=
\E[\mathds 1\{V=v\}\dot\ell(O)]
\]
and
\begin{align}
\left.\frac{\partial N_t}{\partial t}\right|_{t=0}
&=
\E[\mathds 1\{V=v\}h_0(a,S,v,W)\dot\ell(O)]
+
\E[\mathds 1\{V=v\}\dot h(a,S,v,W)].
\label{eq:cate_num_derivative}
\end{align}

The bridge equation along the submodel is
\[
\E_t[Y-h_t(A,S,V,W)\mid A,S,V,Z]=0.
\]
Differentiating at \(t=0\) gives
\begin{align}
\E[\dot h(A,S,V,W)\mid A,S,V,Z]
=
\E[\{Y-h_0(A,S,V,W)\}\dot\ell(O)\mid A,S,V,Z].
\label{eq:cate_linearized_bridge}
\end{align}
The CATE treatment bridge satisfies
\[
{
\E[{\varphi_0^{\mathrm{CATE}}}(a,v,S,Z)\mid A=a,S,V=v,W]
=r_{0,\mathrm{CATE}}(a,v,S,W).}
\]
Consequently, for any {measurable \(g(S,W)\) for which the following expectations are finite},
\begin{align}
&\E[\mathds 1\{V=v\}\mathds 1\{A=a\} {\varphi_0^{\mathrm{CATE}}}(a,v,S,Z)g(S,W)]
\nonumber\\
&\qquad
{=
\pi_v\pi_{a\mid v}
\E[r_{0,\mathrm{CATE}}(a,v,S,W)g(S,W)\mid A=a,V=v]}
\nonumber\\
&\qquad
{=
\pi_v\pi_{a\mid v}\E[g(S,W)\mid V=v]}
\nonumber\\
&\qquad
=
\pi_{a\mid v}\E[\mathds 1\{V=v\}g(S,W)].
\label{eq:cate_reweighting_identity}
\end{align}
Taking \(g=\dot h(a,\cdot,v,\cdot)\) in Eq.~\ref{eq:cate_reweighting_identity}
gives
\[
\E[\mathds 1\{V=v\}\dot h(a,S,v,W)]
=
\frac{1}{\pi_{a\mid v}}
\E[\mathds 1\{V=v\}\mathds 1\{A=a\}
{\varphi_0^{\mathrm{CATE}}}(a,v,S,Z)\dot h(a,S,v,W)].
\]
Combining this equation with Eq.~\ref{eq:cate_linearized_bridge} yields
\[
\frac{1}{\pi_v}
\E[\mathds 1\{V=v\}\dot h(a,S,v,W)]
=
\E\!\left[
\frac{\mathds 1\{V=v\}}{\pi_v}
\frac{\mathds 1\{A=a\}}{\pi_{a\mid v}}
{\varphi_0^{\mathrm{CATE}}}(a,v,S,Z)
\{Y-h_0(A,S,V,W)\}\dot\ell(O)
\right].
\]
Substituting this identity and Eq.~\ref{eq:cate_num_derivative} into
Eq.~\ref{eq:cate_ratio_derivative} gives
\[
\left.\frac{\partial \psi_t}{\partial t}\right|_{t=0}
=
\E[\psi^{\mathrm{EIF}}_{\mathrm{CATE}}(O;a,v)\dot\ell(O)].
\]
Thus Eq.~\ref{eq:Our_EIF_CATE} is an influence function. It is mean zero by the
outcome bridge equation and by the centering of
\(h_0(a,S,v,W)-f_{\mathrm{CATE}}(a,v)\) conditional on \(V=v\).

For efficiency, let \(\mathcal F_{\mathrm C}:=(Z,A,S,V)\) and
\(\varepsilon:=Y-h_0(A,S,V,W)\). The tangent space is
\(\Lambda_{\mathrm C}=\Lambda_{\mathrm C,1}+\Lambda_{\mathrm C,2}\), where
\[
\Lambda_{\mathrm C,1}
=
\{s_1(\mathcal F_{\mathrm C})\in L^2(\sP):\E[s_1]=0\},
\]
and
\[
\Lambda_{\mathrm C,2}
=
\left\{
s_2\in L^2(\sP):
\E[s_2\mid \mathcal F_{\mathrm C}]=0,\ 
\E[\varepsilon s_2\mid\mathcal F_{\mathrm C}]
\in\overline{\mathcal R(T_{\mathrm C})}
\right\}.
\]
Set \(H_{\mathrm C}(S,W):=h_0(a,S,v,W)-\psi\), and write
\(G_{\mathrm C}:=\psi^{\mathrm{EIF}}_{\mathrm{CATE}}(O;a,v)\). Then
\[
G_{\mathrm C,1}:=\E[G_{\mathrm C}\mid\mathcal F_{\mathrm C}]
=
\frac{\mathds 1\{V=v\}}{\pi_v}
\E[H_{\mathrm C}(S,W)\mid\mathcal F_{\mathrm C}],
\qquad
G_{\mathrm C,2}:=G_{\mathrm C}-G_{\mathrm C,1}.
\]
The first component belongs to \(\Lambda_{\mathrm C,1}\). The second satisfies
\(\E[G_{\mathrm C,2}\mid\mathcal F_{\mathrm C}]=0\), and
\(\E[\varepsilon G_{\mathrm C,2}\mid\mathcal F_{\mathrm C}]\) is a
square-integrable \(\mathcal F_{\mathrm C}\)-measurable function. By the
dense-range condition for \(T_{\mathrm C}\), this conditional moment belongs to
\(\overline{\mathcal R(T_{\mathrm C})}\). Hence
\(G_{\mathrm C,2}\in\Lambda_{\mathrm C,2}\), so
\(G_{\mathrm C}\in\Lambda_{\mathrm C}\). Since it represents the pathwise
derivative and lies in the tangent space, it is the canonical gradient. This
proves Eq.~\ref{eq:Our_EIF_CATE} and completes the proof.
\end{proof}

\begin{remark}[Relation to classical proximal ATE identification]
\label{rem:ate_eif_relation_to_cui}
In the discrete-treatment case, our ATE representation is a normalized version of
\citet{semiparametricProximalCausalInference}. For a fixed treatment level
\(a\), their treatment bridge may be written as
\[
\E[q^{\mathrm{ATE}}_0(a,X,Z)\mid A=a,X,W]
= {\frac{1}{\pi_a}r_{0,\mathrm{ATE}}(a,X,W).}
\]
Their corresponding doubly robust representation of the counterfactual mean is
\begin{align}
 f_{\mathrm{ATE}}^{{\mathrm{DR}}}(a;h_0,q^{\mathrm{ATE}}_0)
 &=
 \E\!\big[\mathds 1\{A=a\}q^{\mathrm{ATE}}_0(a,X,Z)
 \{Y-h_0(A,X,W)\}\big]
 +\E[h_0(a,X,W)].
\label{eq:Cui_ATE_DR_identification}
\end{align}
Equivalently, its influence-function moment is
\[
\psi^{\mathrm{EIF}}_{q^{\mathrm{ATE}}_0}(O;a)
=
\mathds 1\{A=a\}q^{\mathrm{ATE}}_0(a,X,Z)
\{Y-h_0(A,X,W)\}
+h_0(a,X,W)-f_{\mathrm{ATE}}(a).
\]
Our bridge in Eq.~\ref{eq:TreatmentBridgeEq_ATE} satisfies
\[
\E[{\varphi_0^{\mathrm{ATE}}}(a,X,Z)\mid A=a,X,W]
= {r_{0,\mathrm{ATE}}(a,X,W).}
\]
Thus, when the bridge solution is unique,
\({\varphi_0^{\mathrm{ATE}}}(a,X,Z)
= {\pi_a}
q^{\mathrm{ATE}}_0(a,X,Z)\). Substituting this into
Eq.~\ref{eq:Cui_ATE_DR_identification} gives
\[
\E\!\left[ {\varphi_0^{\mathrm{ATE}}}(a,X,Z)
\{Y-h_0(a,X,W)\}\mid A=a
\right]
+\E[h_0(a,X,W)],
\]
which is exactly Eq.~\ref{eq:DoublyRobustATE_identification}. Therefore the normalized bridge does not change the discrete EIF; it rewrites the same moment as a conditional expectation given \(A=a\). This form is the one used by our neural mean-embedding estimator and avoids an explicit sparse indicator function when the target is indexed by a continuous or structured treatment value.

\end{remark}

\begin{remark}[Relation to the classical proximal ATT EIF]
\label{rem:att_eif_relation_to_cui}
The ATT formula above is not obtained by multiplying the whole ATE EIF by
\(\mathds 1\{A=a'\}/\pi_{a'}\). The bridge-residual term remains on the event
\(\{A=a\}\), while the centering term is evaluated on the event \(\{A=a'\}\).

Consider the binary case \(a=0\), \(a'=1\), and define
\(\mu_0:=f_{\mathrm{ATT}}(0,1)=\E[Y^{(0)}\mid A=1]\). Cui et al.'s ATT bridge
is restricted to the control arm and satisfies
\begin{align*}
\E[h^{\mathrm{ATT}}_0(X,W)\mid A=0,X,Z]
&=
\E[Y\mid A=0,X,Z],\\
\E[q^{\mathrm{ATT}}_0(X,Z)\mid A=0,X,W]
&= {\frac{\pi_1}{\pi_0}
r_{0,\mathrm{ATT}}(0,1,X,W).}
\end{align*}
Their proof of the ATT EIF first derives the influence function for
\(\mu_0\), namely
\begin{align}
\psi^{\mathrm{EIF}}_{\mu_0}(O)
=
\frac{\mathds 1\{A=0\}}{\pi_1}
q^{\mathrm{ATT}}_0(X,Z)\{Y-h^{\mathrm{ATT}}_0(X,W)\}
+
\frac{\mathds 1\{A=1\}}{\pi_1}
\{h^{\mathrm{ATT}}_0(X,W)-\mu_0\}.
\label{eq:Cui_ATT_EIF}
\end{align}
Equivalently,
\begin{align}
\mu_0
&=
\E\!\left[
\frac{\mathds 1\{A=0\}}{\pi_1}
q^{\mathrm{ATT}}_0(X,Z)\{Y-h^{\mathrm{ATT}}_0(X,W)\}
+
\frac{\mathds 1\{A=1\}}{\pi_1}h^{\mathrm{ATT}}_0(X,W)
\right]
\nonumber\\
&=
\E\!\left[
\frac{\mathds 1\{A=0\}}{\pi_1}
q^{\mathrm{ATT}}_0(X,Z)\{Y-h^{\mathrm{ATT}}_0(X,W)\}
\right]
+
\E[h^{\mathrm{ATT}}_0(X,W)\mid A=1].
\label{eq:Cui_ATT_Identification}
\end{align}
Our normalized ATT bridge satisfies

\[
{
\E[{\varphi_0^{\mathrm{ATT}}}(0,1,X,Z)\mid A=0,X,W]
=r_{0,\mathrm{ATT}}(0,1,X,W).}
\]
Thus, under uniqueness,
\({\varphi_0^{\mathrm{ATT}}}(0,1,X,Z)=(\pi_0/\pi_1)q^{\mathrm{ATT}}_0(X,Z)\), and
Eq.\ref{eq:Our_EIF_ATT} reduces to Eq.~\ref{eq:Cui_ATT_EIF} with
\(h^{\mathrm{ATT}}_0(X,W)=h_0(0,X,W)\).

The usual binary ATT contrast is
\[
\tau_{\mathrm{ATT}}
:=\E[Y^{(1)}-Y^{(0)}\mid A=1]
=
\E[Y\mid A=1]-f_{\mathrm{ATT}}(0,1).
\]
Its EIF is therefore
\[
\frac{\mathds 1\{A=1\}}{\pi_1}\{Y-\E[Y\mid A=1]\}
-
\psi^{\mathrm{EIF}}_{\mathrm{ATT}}(O;0,1).
\]
Hence \(\psi^{\mathrm{EIF}}_{\mathrm{ATT}}(O;a,a')\) in
Theorem~\ref{theorem:EIFs_formulation} is the EIF for the conditional
counterfactual mean \(f_{\mathrm{ATT}}(a,a')\); the EIF for an ATT contrast is
obtained by adding the observed-arm component and taking the difference.
\end{remark}

\begin{remark}[Continuous and structured treatments]
\label{rem:continuous_structured_treatments}
The EIFs in Theorem~\ref{theorem:EIFs_formulation} are discrete-treatment EIFs. They justify the corresponding discrete doubly robust moment equations. The identification formulas in Theorem~\ref{theorem:DoublyRobustIdentifications} are still stated for continuous or structured \(A\), provided the bridge equations are interpreted through regular conditional laws and the required positivity conditions hold. However, the pointwise semiparametric efficiency theory for continuous treatments is not obtained by {formally replacing the discrete atom weight
\(\mathds 1\{A=a\}/\pi_a\) by
\(r_{0,\mathrm{ATE}}(a,X,W)\)}. Existing continuous-treatment DR theory typically works with localized or smoothed target functionals \citep{kennedy2017, Colangelo2020, wu2024doubly, zenati2025doubledebiased}. Our contribution in the continuous and high-dimensional setting is therefore an identification and estimation construction based on normalized treatment bridges and neural mean embeddings, not a closed-form pointwise continuous-treatment EIF.

\end{remark}

\begin{remark}[Heterogeneous dose-response identification]
\label{rem:heterogeneous_dr_identification}
Equation~\ref{eq:DoublyRobustATE_identification} extends the doubly robust proximal identification formula from population dose-response curves to the heterogeneous response function \(f_{\mathrm{CATE}}(a,v)=\E[Y^{(a)}\mid V=v]\). This is distinct from the binary ATT case: the conditioning event is an observed context variable \(V=v\), not the
realized treatment arm.

\end{remark}

\section{Neural mean embedding-based proxy causal learning for dose-response}
\label{sec:NeuralMeanEmbeddingSectionDoseResponse_Appendix}

We estimate the dose-response curve using an outcome bridge, a treatment bridge, or their doubly robust combination. Both bridge estimators have the analogous architecture: a first-stage neural conditional mean embedding, a second-stage linear bridge head built on learned features, and a final causal function evaluation stage. This allows the representation itself to adapt to the data and makes the resulting estimators suitable for stochastic optimization in high-dimensional and structured settings in contrast to fixed RKHS features in \citet{bozkurt2025densityratiofreedoublyrobust}.

Our construction consists of three components:
\begin{enumerate}
    \item \textbf{Outcome bridge network.} We begin with the Deep Feature Proxy Causal Learning (DFPCL) architecture of \citet{xu2021deep}, which provides a neural parameterization of the outcome bridge estimator. We refine this component by incorporating proximal closed-form updates for the final linear layer \citep{galashov2025closedformlayeroptimization}. 

    \item \textbf{Treatment bridge network.} We next introduce a neural estimator of the treatment bridge function in Section~\ref{Sec:NeuralTreatmentApproach_Appendix}. This construction follows similar principles to the outcome-side network: adaptive feature learning, proximal closed-form updates for the last linear layer.

    \item \textbf{Neural doubly robust unification.} Finally, in Section~\ref{appendix:DoublyRobustAlgorithmDoseResponse}, we combine the outcome- and treatment-bridge components into a fully neural doubly robust estimator of the dose-response curve. This final stage is designed to leverage the complementary strengths of both bridge functions.
\end{enumerate}

\subsection{{Dose-response curve estimation: outcome bridge approach}}
\label{sec:DFPCL_Xu_Review}

We first describe the outcome-bridge component, which builds on the DFPCL architecture of \citet{xu2021deep}. We refer to this component as {\textsc{OutcomeNet}}. To distinguish it from the treatment-side construction introduced later in Section~\ref{Sec:NeuralTreatmentApproach_Appendix}, we attach the superscript \((h)\) to all quantities associated with the outcome bridge.

The outcome bridge satisfies
\begin{align*}
    \E[Y \mid A = a, X= x, Z = z] = \E[h(a, x, W) \mid A = a, X = x, Z = z].
\end{align*}

Let \(Q_h:=(A,X,Z)\) and define
\[
\begin{aligned}
(T_hu)(a,x,z)
&:=\E\{u(a,x,W)\mid A=a,X=x,Z=z\},\\
m_0(q_h)
&:=\E(Y\mid Q_h=q_h)\in\mathbb R^{d_Y}.
\end{aligned}
\]
Consider the population least-squares risk
\begin{equation}
\gR_{h,2}(u)
:=\E\!\left[\left\|Y-(T_hu)(Q_h)\right\|_2^2\right].
\label{eq:outcome_population_risk}
\end{equation}
If \(h_0\) is an outcome bridge, then \(T_hh_0=m_0\), and the conditional
least-squares decomposition gives
\[
\gR_{h,2}(u)-\gR_{h,2}(h_0)
=
\|T_hu-m_0\|_{L^2(\sP_{Q_h};\mathbb R^{d_Y})}^2
=
\|T_h(u-h_0)\|_{L^2(\sP_{Q_h};\mathbb R^{d_Y})}^2.
\]
Thus, minimizing \(\gR_{h,2}\) targets the outcome-bridge equation in projected
norm. Its empirical analogue cannot be formed directly from pointwise
observations without estimating the unknown conditional mean
\((T_hu)(a_i,x_i,z_i)\). Replacing it by the
pointwise value \(u(a_i,x_i,w_i)\) would instead target an ordinary regression
of \(Y\) on \((A,X,W)\), not the bridge equation.

Therefore, we parametrize the bridge and the relevant conditional proxy features as follows. Let \(Y\in\mathbb R^{d_Y}\), so that \(h(a,x,w)\in\mathbb R^{d_Y}\). If \(d_{AX,2}^{(h)}\) and
\(d_{W,2}^{(h)}\) are the output dimensions of
\(\phi_{AX,2}^{(h)}\) and \(\phi_{W,2}^{(h)}\), respectively, define
\(d_{h,2}:=d_{AX,2}^{(h)}d_{W,2}^{(h)}\) and take
\(\mH^{(h)}\in\mathbb R^{d_{h,2}\times d_Y}\). Then
\begin{equation}
h(a,x,w)
=
\left\{\mH^{(h)}\right\}^{\top}
\left(
\phi_{AX,2}^{(h)}(a,x) \otimes
\phi_{W,2}^{(h)}(w)
\right).
\label{Eq:DFPCL_SecondStageModel}
\end{equation}

\begin{align}
\mu_W^{(h)}(a, x, z) &= \E\!\left[\phi_{W,2}^{(h)}(W)\mid A=a,X=x,Z=z\right]\nonumber
\\
\hat \mu_W^{(h)}(a, x, z; \mV^{(h)})
&=
\left\{\mV^{(h)}\right\}^\top \phi_{AXZ,1}^{(h)}(a,x,z).
\label{Eq:DFPCL_FirstStageModel}
\end{align}

For the parameterization in Eq.~\ref{Eq:DFPCL_SecondStageModel}, linearity
of conditional expectation gives
\[
(T_hh)(a,x,z)
=
\left\{\mH^{(h)}\right\}^{\top}
\left\{
\phi_{AX,2}^{(h)}(a,x)
\otimes
\mu_W^{(h)}(a,x,z)
\right\}.
\]
Stage~1 therefore estimates only the conditional feature mean
\(\mu_W^{(h)}\) through the neural mean embedding in
Eq.~\ref{Eq:DFPCL_FirstStageModel}. Stage~2 substitutes this estimate into
the sample analogue of Eq.~\ref{eq:outcome_population_risk}. The embedding is
an estimation device for the projected bridge, rather than an additional
identification assumption.

More specifically:
\begin{itemize}
    \item \(\phi_{AXZ,1}^{(h)}:\gA\times\gX\times\gZ\to\mathbb R^{d_{h,1}}\) denotes the first-stage feature extractor, parameterized by \(\theta^{(h)}_1\).
    
    \item \(\phi_{AX,2}^{(h)}\) and \(\phi_{W,2}^{(h)}\) denote the second-stage feature extractors for \((A, X)\), and \(W\), jointly parameterized by \(\theta^{(h)}_2\).
    
    \item \(\mV^{(h)}\) is the first-stage linear operator mapping the learned features of \((A,X,Z)\) to predicted proxy features in the representation space of \(W\).
    
    \item {\(\mH^{(h)}\in\mathbb R^{d_{h,2}\times d_Y}\) is the second-stage linear head which maps the tensorized representation of \((A,X,W)\) to \(h(a,x,w)\in\mathbb R^{d_Y}\).}
\end{itemize}

This parameterization separates the two roles played by the neural model. The first stage learns how the information contained in \((A,X,Z)\) predicts the proxy representation of \(W\). The second stage then uses this learned proxy representation to fit the outcome bridge. In this way, the architecture preserves the logic of proximal bridge estimation while allowing the representation spaces themselves to be adapted during training.

The first stage learns the neural conditional mean embedding \(\hat \mu_W^{(h)}(a, x, z)\) by minimizing the regularized least-squares objective
\begin{align}
    \hat{\gL}_{h,1}(\theta_{1}^{(h)}, \mV^{(h)})
    &=
    \frac{1}{|\gI_1^{(h)}|}
    \sum_{i \in \gI_1^{(h)}}
    \left\|
    \phi_{W,2}^{(h)}({w}_i)
    -
    {\left\{\mV^{(h)}\right\}^\top
    \phi_{AXZ,1}^{(h)}({a}_i, {x}_i, {z}_i)}
    \right\|_2^2
    +
    \lambda^{(h)}_{1}\|\mV^{(h)}\|_F^2,
    \label{Eq:DFPCL_FirstStageLoss}
\end{align}
where \(\lambda^{(h)}_{1}\) is the regularization parameter for the first-stage linear layer and \(\|\cdot\|_F\) denotes the Frobenius norm. Importantly, although Equation \ref{Eq:DFPCL_FirstStageLoss} depends on the second-stage proxy featurizer \(\phi_{W,2}^{(h)}\), the parameters \(\theta_{W,2}^{(h)}\) are not updated in this stage. Instead, \(\phi_{W,2}^{(h)}(W)\) is treated as the target representation to be predicted from \((A,X,Z)\).

In the second stage, the learned first-stage operator \(\hat{\mV}^{(h)}\) is used to construct predicted proxy features on \(\mathcal{I}_2^{(h)}\). The outcome bridge is then estimated by minimizing

\begin{align}
    \hat{\gL}_{h,2}(\theta_{2}^{(h)}, \mH^{(h)})
    &=
    \frac{1}{|\mathcal{I}_2^{(h)}|}
    \sum_{i \in \mathcal{I}_2^{(h)}}
    \left\|
    {y}_i
    -
    \left\{\mH^{(h)}\right\}^{\top}\vq_{i}^{(h)}
    \right\|_2^2
    +
    \lambda^{(h)}_{2}\|\mH^{(h)}\|_F^2,
    \label{Eq:DFPCL_SecondStageLoss}
\end{align}
where \(\theta_{2}^{(h)}=\{\theta_{A,2}^{(h)},\theta_{X,2}^{(h)},\theta_{W,2}^{(h)}\}\) collects the parameters of the second-stage featurizers, and \(\vq_{i}^{(h)}\in\mathbb R^{d_{h,2}}\) is defined for \(i \in \gI_2^{(h)}\) by
\[\vq_{i}^{(h)} =
    \phi_{AX,2}^{(h)}({a}_i,{x}_i)
    \otimes
    \hat{\mV}^{(h)}(\theta_{W,2}^{(h)})^\top
    \phi_{AXZ,1}^{(h)}({a}_i,{x}_i,{z}_i).
\]

A key difficulty is that Equation \ref{Eq:DFPCL_SecondStageLoss} depends on \(\theta_{W,2}^{(h)}\) only implicitly, through the first-stage optimum \(\hat{\mV}^{(h)}(\theta_{W,2}^{(h)})\). Indeed, changing \(\theta_{W,2}^{(h)}\) changes the target representation in the first stage, and therefore also changes the first-stage optimum itself. Following \citet{xu2021deep, xu2021learning}, this dependence is handled approximately by differentiating through the closed-form solution for \(\hat{\mV}^{(h)}(\theta_{W,2}^{(h)})\), while treating the first-stage feature representation \(\phi_{AXZ,1}^{(h)}\) as fixed. This yields a tractable approximation to the full bi-level gradient.

The resulting learning procedure alternates between gradient-based updates of the neural feature extractors and analytical updates of the final linear layers.

\paragraph{{Neural parameter update (stage 1).}}
The first-stage feature extractor is updated by a gradient step on Equation \ref{Eq:DFPCL_FirstStageLoss} with learning rate \(\eta_1^{(h)}\):
\(
\theta_{1}^{(h)}
\leftarrow
\theta_{1}^{(h)}
-
\eta_1^{(h)}
\frac{\partial \hat{\gL}_{h,1}}{\partial \theta_{1}^{(h)}}.
\)

\paragraph{Closed-form update of the first-stage linear layer.} For index set $\gI$, define
\begin{align*}
    \mF_h(\gI) = \left[\phi_{AXZ,1}^{(h)}(a_i, x_i, z_i) \right]_{i \in \gI} \in \mathbb{R}^{d_{h,1} \times n_h}, \quad \mS_h(\gI) = \left[\phi_{W, 2}^{(h)} (w_i)\right]_{i \in \gI} \in \mathbb{R}^{d_{W,2}^{(h)} \times n_h}.
\end{align*}

Here \(d_{W,2}^{(h)}\) is the output dimension of \(\phi_{W,2}^{(h)}\), and \(d_{h,1}\) is the output dimension of \(\phi_{AXZ,1}^{(h)}\).

Fixing the feature extractors, the first-stage operator \(\mV^{(h)}\) is updated analytically as a minimizer of Eq.~\ref{Eq:DFPCL_FirstStageLoss} by
\begin{align}
    \hat{\mV}^{(h)}(\theta_{W,2}^{(h)})
    =
    \left(
    \mF_h(\gI_1^{(h)})\mF_h(\gI_1^{(h)})^\top
    +
    |\gI_1^{(h)}|\lambda^{(h)}_{1}\mI
    \right)^{-1}
    \mF_h(\gI_1^{(h)}) \mS_h(\gI_1^{(h)})^\top,
    \label{eq:OutcomeBridgeFirstLayerClosedFormSolution}
\end{align}

\paragraph{{Neural parameter update (stage 2).}}
The second-stage featurizers are updated by a gradient step on Equation \ref{Eq:DFPCL_SecondStageLoss} with learning rate \(\eta_2^{(h)}\):
\(
\theta_{2}^{(h)}
\leftarrow
\theta_{2}^{(h)}
-
\eta_2^{(h)}
\frac{\partial \hat{\gL}_{h,2}}{\partial \theta_{2}^{(h)}}.
\)
For \(\theta_{W,2}^{(h)}\), the gradient is propagated through the closed-form expression in Equation \ref{eq:OutcomeBridgeFirstLayerClosedFormSolution}.

\paragraph{Closed-form update of the second-stage linear layer.}

For \(\gI_2^{(h)}\), define
\begin{align*}
    \vq_i^{(h)}
    =
    \phi_{AX,2}^{(h)}({a}_i,{x}_i)
\otimes
\hat{\mV}^{(h)\top}
\phi_{AXZ,1}^{(h)}({a}_i,{x}_i,{z}_i)
\end{align*}

Writing
\[
\mQ^{(h)}(\gI_2^{(h)})
:=\left[\vq_i^{(h)}\right]_{i\in\gI_2^{(h)}}
\in\mathbb R^{d_{h,2}\times|\gI_2^{(h)}|},
\qquad
\mY(\gI_2^{(h)})
:=\left[y_i\right]_{i\in\gI_2^{(h)}}
\in\mathbb R^{d_Y\times|\gI_2^{(h)}|},
\]
the matrix head is updated analytically as the minimizer of
Equation~\ref{Eq:DFPCL_SecondStageLoss}:
\begin{align}
    \hat{\mH}^{(h)}
    =
    \left(
    \mQ^{(h)}(\gI_2^{(h)}) \mQ^{(h)}(\gI_2^{(h)})^\top
    +
    |\gI_2^{(h)}| \lambda^{(h)}_{2}\mI
    \right)^{-1}
    \mQ^{(h)}(\gI_2^{(h)}) \mY(\gI_2^{(h)})^\top,
    \label{eq:OutcomeBridgeSecondLayerClosedFormSolution}
\end{align}

Following \citet{xu2021deep, xu2021learning}, it is helpful to perform several first-stage updates before each second-stage update. In practice, we carry out \(T_{h,1}\) first-stage updates for every \(T_{h,2}\) second-stage updates, typically with \(T_{h,1}\ge T_{h,2}\), so that the neural conditional mean embedding remains sufficiently accurate throughout training.

\paragraph{Proximal closed-form updates.} The analytical updates in Equations \ref{eq:OutcomeBridgeFirstLayerClosedFormSolution} and \ref{eq:OutcomeBridgeSecondLayerClosedFormSolution} are exact minimizers of the corresponding quadratic objectives on a fixed dataset. In stochastic mini-batch training, however, directly replacing the last-layer parameters by these batchwise optima can lead to unstable trajectories, since each batch induces a different local optimum. Following \citet{galashov2025closedformlayeroptimization}, we therefore replace the zero-centered ridge penalties by proximal penalties centered at the previous iterates. This keeps the batchwise closed-form updates close to the current parameter state and substantially stabilizes training.

In the present two-stage setting, there is one additional subtlety. The stage-2 loss depends on the first-stage optimum through the predicted proxy features. Consequently, in addition to the persistent first-stage iterate \(\hat{\mV}^{(h)}_t\), we introduce an auxiliary first-stage solution that is recomputed \emph{on the fly} on the current second-stage batch. This auxiliary operator is not stored as a separate model parameter. Its only role is to evaluate the second-stage loss with the current proxy representation \(\phi_{W,2}^{(h)}\). To define it, note that \(\gI^{(h)}_2\) must also retain the proxy observations \( w_i\), even though they do not appear explicitly in the quadratic bridge loss.

Let \(\mathcal B^{(h)}_{1, t} \subset \gI^{(h)}_1\) and \(\mathcal B^{(h)}_{2, t} \subset \gI^{(h)}_2\) denote mini-batches sampled at iteration \(t\). We define the first-stage proximal batch loss
\begin{align}
\hat{\gL}_{h,1, t}^{\mathrm{prox}}\!\left(\theta_1^{(h)},\mV^{(h)};\theta_2^{(h)},\mathcal B_{1, t}^{(h)}\right)
&=
\frac{1}{|\mathcal B_{1, t}^{(h)}|}
\sum_{i\in \mathcal B_{1, t}^{(h)}}
\left\|
\phi_{W,2}^{(h)}( w_i)
-
\left(\mV^{(h)}\right)^\top
\phi_{AXZ,1}^{(h)}( a_i, x_i, z_i)
\right\|_2^2 \nonumber\\
&+
\lambda_1^{(h)}
\left\|
\mV^{(h)}-\hat{\mV}^{(h)}_t
\right\|_F^2,
\label{Eq:DFPCL_FirstStageProximalLoss}
\end{align}

and, suppressing fixed neural parameters from its arguments, write the second-stage proximal loss as
\begin{align}
\hat{\gL}_{h,2,t}^{\mathrm{prox}}\!\left(\theta_2^{(h)}, \mH^{(h)}\right)
&=
\frac{1}{|\mathcal B_{2, t}^{(h)}|}
\sum_{i\in \mathcal B_{2, t}^{(h)}}
\left\|
 y_i
-
\left\{\mH^{(h)}\right\}^{\top}
\big(
\phi_{AX,2}^{(h)}( a_i, x_i)
\otimes \hat \mu_W^{(h)}(a_i, x_i, z_i; \check \mV^{(h)}_t)
\big)
\right\|_2^2
\nonumber\\
&\qquad
+
\lambda_2^{(h)}
\left\|
\mH^{(h)}-\hat{\mH}^{(h)}_t
\right\|_F^2,
\label{Eq:DFPCL_SecondStageProximalLoss}
\end{align}

where \(\lambda_1^{(h)}\) and \(\lambda_2^{(h)}\) are proximal regularization parameters, \(\hat{\mV}^{(h)}_t\) and {\(\hat{\mH}^{(h)}_t\)} are the previous iterates, and \(\check{\mV}^{(h)}_t\) denotes the auxiliary on-the-fly first-stage solution computed from the current second-stage batch.

More precisely, we define the persistent first-stage update by
\begin{align}
\hat{\mV}^{(h)}_{t+1}\!\left(\theta_1^{(h)};\theta_2^{(h)},\mathcal B_{1, t}^{(h)}\right)
=
\argmin_{\mV^{(h)}}
\hat{\gL}_{h,1, t}^{\mathrm{prox}}\!\left(\theta_1^{(h)},\mV^{(h)};\theta_2^{(h)},\mathcal B_{1, t}^{(h)}\right),
\label{eq:minimizer_first_stage_batch}
\end{align}
and the auxiliary on-the-fly first-stage operator by
\begin{align}
\check{\mV}^{(h)}_{t}\!\left(\theta_1^{(h)};\theta_2^{(h)},\mathcal B_{2, t}^{(h)}\right)
=
\argmin_{\mV^{(h)}}
\hat{\gL}_{h,1, t}^{\mathrm{prox}}\!\left(\theta_1^{(h)},\mV^{(h)};\theta_2^{(h)},\mathcal B_{2, t}^{(h)}\right).
\label{eq:minimizer_first_stage_batch_aux}
\end{align}
The first quantity is the actual stage-1 parameter update, while the second is an auxiliary batchwise solve used only inside the stage-2 objective. 

Given \(\check{\mV}^{(h)}_{t}\), the second-stage proximal update is

\begin{align}
\hat{\mH}^{(h)}_{t+1}\!\left(\theta_1^{(h)};\theta_2^{(h)},\mathcal B_2^{(h)}\right)
=
\argmin_{\mH^{(h)}}
\hat{\gL}_{h,2,t}^{\mathrm{prox}}\!\left(\mH^{(h)}\right).
\label{eq:minimizer_second_stage_batch}
\end{align}

Expanding these minimizers yields the proximal closed-form updates. For any batch \(\mathcal B\), we have
\[
\mF_h(\gB)
=
\begin{bmatrix}
\phi_{AXZ,1}^{(h)}(a_i,x_i,z_i)
\end{bmatrix}_{i\in\mathcal B},
\qquad
\mS_h(\gB)
=
\begin{bmatrix}
\phi_{W,2}^{(h)}(w_i)
\end{bmatrix}_{i\in\mathcal B},
\]
where the tuples \((a_i,x_i,z_i,w_i)\) are read from the corresponding batch. Then
\begin{align}
\hat{\mV}^{(h)}_{t+1}\!\left(\theta_1^{(h)};\theta_2^{(h)},\mathcal B_{1, t}^{(h)}\right)
&=
\left(
\mF_h(\gB_{1,t}^{(h)})
\mF_h(\gB_{1,t}^{(h)})^\top
+
|\mathcal B_{1,t}^{(h)}|\lambda_1^{(h)} \mI
\right)^{-1}\cdot\nonumber\\&\quad\left(
\mF_h(\gB_{1,t}^{(h)})
\mS_h(\gB_{1, t}^{(h)})^\top
+
|\mathcal B_{1, t}^{(h)}|\lambda_1^{(h)}\hat{\mV}^{(h)}_t
\right),
\label{eq:OutcomeBridgeFirstLayerProximalClosedFormSolution}
\\
\check{\mV}^{(h)}_{t}\!\left(\theta_1^{(h)};\theta_2^{(h)},\mathcal B_2^{(h)}\right)
&=
\left(
\mF_h(\gB_{2,t}^{(h)})
\mF_h(\gB_{2,t}^{(h)})^\top
+
|\mathcal B_{2,t}^{(h)}|\lambda_1^{(h)} \mI
\right)^{-1}\cdot\nonumber\\&\quad\left(
\mF_h(\gB_{2,t}^{(h)})
\mS_h(\gB_{2, t}^{(h)})^\top
+
|\mathcal B_{2, t}^{(h)}|\lambda_1^{(h)}\hat{\mV}^{(h)}_t
\right),
\label{eq:OutcomeBridgeFirstLayerProximalClosedFormSolutionAux}
\end{align}

The corresponding second-stage features on
\(\mathcal B_{2,t}^{(h)}\) are
\[
\vq_i^{(h)}
:=
\phi_{AX,2}^{(h)}({a}_i,{x}_i)
\otimes
\left\{\check{\mV}^{(h)}_t\right\}^{\top}
\phi_{AXZ,1}^{(h)}({a}_i,{x}_i,{z}_i),
\qquad i\in\mathcal B_{2,t}^{(h)}.
\]
We collect them columnwise as
\[
\mQ^{(h)}(\mathcal B_{2,t}^{(h)})
:=
\begin{bmatrix}\vq_i^{(h)}\end{bmatrix}_{i\in\mathcal B_{2,t}^{(h)}}
\in\mathbb R^{d_{h,2}\times|\mathcal B_{2,t}^{(h)}|},
\qquad
\mY(\mathcal B_{2,t}^{(h)})
:=
\begin{bmatrix}y_i\end{bmatrix}_{i\in\mathcal B_{2,t}^{(h)}}
\in\mathbb R^{d_Y\times|\mathcal B_{2,t}^{(h)}|}.
\]
Then, the second-stage
proximal update is
\begin{equation}
\begin{aligned}
\hat{\mH}^{(h)}_{t+1}
&=
\left(
\mQ^{(h)}(\mathcal B_{2, t}^{(h)})
\mQ^{(h)}(\mathcal B_{2, t}^{(h)})^\top
+
|\mathcal B_{2,t}^{(h)}|\lambda_2^{(h)}\mI
\right)^{-1}
\\
&\quad{}\times\left(
\mQ^{(h)}(\mathcal B_{2, t}^{(h)})
\mY(\mathcal B_{2, t}^{(h)})^\top
+
|\mathcal B_{2, t}^{(h)}|\lambda_2^{(h)}\hat{\mH}^{(h)}_t
\right).
\end{aligned}
\label{eq:OutcomeBridgeSecondLayerProximalClosedFormSolution}
\end{equation}

Equivalently, the batch-size factors may be absorbed into \(\lambda_1^{(h)}\) and \(\lambda_2^{(h)}\); we write them explicitly here only to match the averaged losses in Equations \ref{Eq:DFPCL_FirstStageProximalLoss} and \ref{Eq:DFPCL_SecondStageProximalLoss}.

Finally, note that the stage-2 loss depends on \(\theta_{W,2}^{(h)}\) through the auxiliary operator \(\check{\mV}^{(h)}_t\). In practice, gradients with respect to \(\theta_{W,2}^{(h)}\) are therefore propagated through the closed-form mapping in Equation \ref{eq:OutcomeBridgeFirstLayerProximalClosedFormSolutionAux}, while the first-stage representation \(\phi_{AXZ,1}^{(h)}\) is treated as fixed, exactly as in the DFPCL update \citep{xu2021deep}.

\paragraph{Dose-response estimation.} The third stage of {\textsc{OutcomeNet}} is a sample average of the outcome bridge function. In particular, on an independent evaluation fold $\gI_3^{(h)}$, {\textsc{OutcomeNet}} estimates the dose-response curve by
\begin{align*}
    \hat f_{\mathrm{ATE}}^{(h)}(a) = \frac{1}{|\gI_3^{(h)}|} \sum_{i \in \gI_3^{(h)}} \hat h (a, x_i, w_i)
\end{align*}
For \(\phi_{AX,2}^{(h)}(\cdot,\cdot) =\phi_{A,2}^{(h)}(\cdot)\otimes\phi_{X,2}^{(h)}(\cdot)\), define
\begin{align*}
\widehat{\mu}_{XW}^{(h)}
&:=
\frac{1}{|\gI_3^{(h)}|}
\sum_{i\in\gI_3^{(h)}}
\phi_{X,2}^{(h)}(x_i)\otimes\phi_{W,2}^{(h)}(w_i).
\end{align*}
The estimator can then be written equivalently as \(\hat f_{\mathrm{ATE}}^{(h)}(a)
=
\left\{\hat{\mH}^{(h)}\right\}^{\top}
\left[
\phi_{A,2}^{(h)}(a)\otimes\widehat{\mu}_{XW}^{(h)}
\right].\)

This empirical average is valid for continuous and discrete treatment because the learned bridge can be evaluated at any inquiry treatment level $a$; it is not a regression on the observed treatment.

The complete OutcomeNet training procedure is summarized in Algorithm \ref{algo:DeepOutcomeBridgeFunction}.
\begin{algorithm}[H]
{\footnotesize
\textbf{Input:} Observations \(\mathcal D=\{o_i\}_{i=1}^n\), with
\(o_i=(a_i,x_i,w_i,z_i,y_i)\), and disjoint folds
\(\gI_1^{(h)},\gI_2^{(h)},\gI_3^{(h)}\).\\
\textbf{Parameters:} Neural parameters
\(\theta_1^{(h)},\theta_2^{(h)}\); linear layers
\(\mV^{(h)},\mH^{(h)}\).\\
\textbf{Hyperparameters:} Learning rates
\(\eta_1^{(h)},\eta_2^{(h)}\), regularization coefficients
\(\lambda_1^{(h)},\lambda_2^{(h)}\), and update counts
\(T_{h,1},T_{h,2}\).\looseness=-1\\
\textbf{Output:} Fitted outcome bridge \(\hat h\) and population
dose-response estimate \(\hat f_{\mathrm{ATE}}^{(h)}\).
\begin{algorithmic}[1]
\REPEAT
\STATE Sample mini-batches
\(\gB_1^{(h)}\subset\gI_1^{(h)}\) and
\(\gB_2^{(h)}\subset\gI_2^{(h)}\).
\FOR{\(s=1,\ldots,T_{h,1}\)}
\STATE Update \(\theta_1^{(h)}\) by one gradient step on
Equation~\ref{Eq:DFPCL_FirstStageProximalLoss}, holding the proxy-target
featurizer fixed.
\STATE Update the persistent operator \(\hat\mV^{(h)}\) by
Equation~\ref{eq:OutcomeBridgeFirstLayerProximalClosedFormSolution} on
\(\gB_1^{(h)}\).
\ENDFOR
\FOR{\(s=1,\ldots,T_{h,2}\)}
\STATE Compute \(\check\mV^{(h)}\) on \(\gB_2^{(h)}\) using
Equation~\ref{eq:OutcomeBridgeFirstLayerProximalClosedFormSolutionAux}.
\STATE Update \(\theta_2^{(h)}\) by one gradient step on
Equation~\ref{Eq:DFPCL_SecondStageProximalLoss}; differentiate through
\(\check\mV^{(h)}\) for the proxy-side features.
\STATE Update \(\hat\mH^{(h)}\) by
Equation~\ref{eq:OutcomeBridgeSecondLayerProximalClosedFormSolution}.
\ENDFOR
\UNTIL{convergence}
\STATE Set
\(\hat h(a,x,w)=\{\hat\mH^{(h)}\}^{\top}
[\phi_{AX,2}^{(h)}(a,x)\otimes\phi_{W,2}^{(h)}(w)]\).
\STATE Set
\(\hat f_{\mathrm{ATE}}^{(h)}(a)
=|\gI_3^{(h)}|^{-1}\sum_{i\in\gI_3^{(h)}}\hat h(a,x_i,w_i)\).
\STATE \textbf{Return} \(\hat h\) and \(\hat f_{\mathrm{ATE}}^{(h)}\).
\end{algorithmic}
}
\caption{Population dose-response estimation with \textsc{OutcomeNet}.}
\label{algo:DeepOutcomeBridgeFunction}
\end{algorithm}

\subsection{{Dose-response curve estimation: treatment bridge approach}}
\label{Sec:NeuralTreatmentApproach_Appendix}
We now introduce the treatment bridge-based component of our neural mean embedding framework, which we refer to as {\textsc{TreatmentNet}}. The target of this component is the treatment bridge $\varphi_0^{\mathrm{ATE}}(a, x, z)$, which satisfies the conditional moment relation
\begin{align*}
    g_{\varphi, 0}(a, x, w) &\coloneqq \E[\varphi_0^{\mathrm{ATE}}(a, X, Z) \mid A = a, X, W] \\&= {r_{0,\mathrm{ATE}}}(a, X, W), \quad {r_{0,\mathrm{ATE}}}(a, x, w) = \frac{d \left(\sP_A \otimes \sP_{X, W}\right)}{d \sP_{A,X,W}}.
\end{align*}
In contrast to the {\textsc{OutcomeNet}} part of this method, we do not have access to the target of the treatment bridge method, density ratios, and the mainstream approach is to plug-in estimated density ratios into the objective function to learn the treatment bridge function. Recent work \citep{bozkurt2025density, bozkurt2025densityratiofreedoublyrobust} provides a {density-ratio-free} objective to learn the treatment bridge function, and we follow this as characterized in Lemma~\ref{lemma:densRatioFreeRiskCharacterization} which is proved below.

\begin{proof}[Proof of Lemma~\ref{lemma:densRatioFreeRiskCharacterization}]
The ratio-free risk is defined in Eq.~\ref{eq:DensRatioFreeRisk} as 
\[
\gR_{\mathrm{ATE}}^{\mathrm{RF}}( {u}) \coloneqq \E\{ {u}(A_1, X_1, W_1)^2\} - 2 \E\{ {u}(A_1, X_2, W_2)\}.
\]
Let $\sP_Q = \sP_{A, X, W}$ and $\sP_\otimes = \sP_A \otimes \sP_{X, W}$. Then, by change of measure
\begin{align*}
    \gR_{\mathrm{ATE}}^{\mathrm{RF}}( {u}) &= \int  {u}^2 d \sP_Q - 2 \int  {u} d\sP_\otimes = \int \left(  {u}^2 - 2  {u} r_{0, \mathrm{ATE}}\right) d \sP_Q\\
    &= \| {u} - r_{0, \mathrm{ATE}}\|^2_{L^2(\sP_Q)} - \|r_{0, \mathrm{ATE}}\|^2_{L^2(\sP_Q)}
\end{align*}
Substituting $ {u} = r_{0, \mathrm{ATE}}$ and subtracting proves the claim.
\end{proof}

The treatment bridge and the first-stage neural conditional mean embedding are parameterized as
\begin{align}
\varphi(a,x,z)
&=
\bm{\varphi}^\top
\left(
\phi_{AX,2}^{(\varphi)}(a,x)\otimes \phi_{Z,2}^{(\varphi)}(z)
\right),\nonumber
\\
\mu_{Z}^{(\varphi)}(a, x, w) & = \E[\phi_{Z,2}^{(\varphi)}(Z)\mid A=a,X=x,W=w] \nonumber
\\
\hat \mu_{Z}^{(\varphi)}(a, x, w; \mV^{(\varphi)})
&=
\left(\mV^{(\varphi)}\right)^\top \phi_{AXW,1}^{(\varphi)}(a,x,w).\nonumber
\end{align}
Here:
\begin{itemize}
    \item \(\phi_{AXW,1}^{(\varphi)}:\gA\times\gX\times\gW\to\mathbb R^{d_{\varphi,1}}\) is the first-stage feature extractor, parameterized by \(\theta_1^{(\varphi)}\).
    \item \(\phi_{AX,2}^{(\varphi)}\) and \(\phi_{Z,2}^{(\varphi)}\) are the second-stage feature extractors, jointly parameterized by \(\theta_2^{(\varphi)}\).
    \item \(\mV^{(\varphi)}\) is the first-stage linear operator mapping features of \((A,X,W)\) to predicted proxy features.
    \item \(\bm{\varphi}\) is the second-stage linear head mapping the tensorized representation of \((A,X,Z)\) to the bridge value \(\varphi(a,x,z)\).\\
    \item Consequently, the fitted projected bridge function is \[
    \hat g_{\varphi}(a, x, w) = \bm{ \varphi}^\top \left[\phi_{AX,2}^{(\varphi)}(a,x)\otimes\hat \mu_{Z}^{(\varphi)}(a, x, w; \mV^{(\varphi)})\right]
    \]
\end{itemize}

As on the outcome side, {\textsc{TreatmentNet}} is trained by a bi-level multi-stage procedure. For fixed current second-stage proxy features, define the first stage sample risk on first stage batch ${\gB_{1, t}}$:
\begin{align}
\hat \gR_{\varphi, 1}(\theta_1^{(\varphi)}, \mV^{(\varphi)}; \theta_2^{\varphi}, {\gB_{1, t}}) &=
\frac{1}{|\mathcal B_{1, t}^{(\varphi)}|}
\sum_{i\in \mathcal B_{1, t}^{(\varphi)}}
\left\|
\phi_{Z,2}^{(\varphi)}( z_i)
-
\left(\mV^{(\varphi)}\right)^\top
\phi_{AXW,1}^{(\varphi)}( a_i, x_i, w_i)
\right\|_2^2\\
\hat \gL_{\varphi, 1, t}(\theta_1^{(\varphi)}, \mV^{(\varphi)}; \theta_2^{\varphi}, {\gB_1}) &= \hat \gR_{\varphi, 1}(\theta_1^{(\varphi)}, \mV^{(\varphi)}; \theta_2^{\varphi}, {\gB_1}) + \lambda_1^{(\varphi)}
\left\|
\mV^{(\varphi)}-\hat{\mV}^{(\varphi)}_t
\right\|_F^2. \label{Eq:Alternative_DFPCL_FirstStageLoss}
\end{align}

the targets \(\phi_{Z,2}^{(\varphi)}(Z)\) depend on the second-stage proxy featurizer, but \(\theta_{Z,2}^{(\varphi)}\) is not updated in this stage; the feature map acts only as the target representation.

On batch $\gB$, define 
\begin{align*}
    F_\varphi(\gB) = \left[\phi_{AXW, 1}^{(\varphi)}(a_i, x_i, w_i)\right]_{i \in \gB}, \quad S_\varphi(\gB) = \begin{bmatrix}
\phi_{Z,2}^{(\varphi)}(z_i)
\end{bmatrix}_{i\in\mathcal B}.
\end{align*}

Then, the proximal first-stage minimizer is
\begin{align}
\hat{\mV}^{(\varphi)}_{t+1}\!\left(\theta_1^{(\varphi)};\theta_2^{(\varphi)},\mathcal B_{1, t}^{(\varphi)}\right)
&=
\left(
F_\varphi(\gB_{1, t}^{(\varphi)})
F_\varphi(\gB_{1, t}^{(\varphi)})^\top
+
{|\mathcal B_{1,t}^{(\varphi)}|}\lambda_1^{(\varphi)}\mI
\right)^{-1}\cdot\nonumber\\&\quad
\left(
F_\varphi(\gB_{1, t}^{(\varphi)})
S_\varphi(\gB_{1, t}^{(\varphi)})^\top
+
{|\mathcal B_{1,t}^{(\varphi)}|}\lambda_1^{(\varphi)}\hat{\mV}^{(\varphi)}_t
\right).
\label{eq:TreatmentBridgeFirstLayerProximalClosedFormSolution}
\end{align}

The second stage is trained on a mini-batch \(\mathcal B_{2, t}^{(\varphi)} \subset \gI_{2}^{(\varphi)}\). As on the outcome side, the second-stage prediction depends on the current first-stage optimum. We therefore distinguish between:
\begin{itemize}
    \item the \emph{persistent} first-stage operator \(\hat{\mV}_t^{(\varphi)}\), which is updated on batches from \(\gI_1^{(\varphi)}\), and
    \item an \emph{auxiliary on-the-fly} operator \(\check{\mV}_t^{(\varphi)}\), recomputed on the current second-stage batch \(\mathcal B_{2, t}^{(\varphi)}\).
\end{itemize}
The auxiliary operator is not stored as a separate model parameter. Its only role is to evaluate the second-stage loss using the current proxy feature representation \(\phi_{Z,2}^{(\varphi)}\), exactly as in the outcome-bridge formulation.

Formally, \(\check{\mV}_t^{(\varphi)}\) is defined as the minimizer of the first-stage proximal objective evaluated on \(\mathcal B_2^{(\varphi)}\):
\begin{align}
\check{\mV}^{(\varphi)}_{t}\!\left(\theta_1^{(\varphi)};\theta_2^{(\varphi)},\mathcal B_2^{(\varphi)}\right) 
&=
\left(
F_\varphi(\gB_{2, t}^{(\varphi)})
F_\varphi(\gB_{2, t}^{(\varphi)})^\top
+
|\mathcal B_{2,t}^{(\varphi)}|\lambda_1^{(\varphi)}\mI
\right)^{-1}\cdot\nonumber\\&\quad
\left(
F_\varphi(\gB_{2, t}^{(\varphi)})
S_\varphi(\gB_{2, t}^{(\varphi)})^\top
+
|\mathcal B_{2, t}^{(\varphi)}|\lambda_1^{(\varphi)}\hat{\mV}^{(\varphi)}_t
\right).
\label{eq:TreatmentBridgeFirstLayerProximalClosedFormSolutionAux}
\end{align}

Let $\gB_{2, t}^{(\varphi)}$ be the second-stage batch sampled at iteration $t$. For every ordered pair $(i, j)$ on $\gB_{2, t}^{(\varphi)}$, define
\begin{align}
    \vu_{ij, t}^{(\varphi)} = \phi_{AX, 2}^{(\varphi)}(a_i, x_j) \otimes \left\{\check \mV_t^{(\varphi)}\right\}^\top \phi_{AXW, 1}^{(\varphi)}(a_i, x_j, w_j).
\end{align}

Then, the plug-in unbiased ratio-free loss at iteration $t$ is 
\begin{align}
\hat{\gL}_{\varphi, 2, t}^{\mathrm{RF, U}}(\varphi) &= \hat \gR_{\varphi, 2, t}^{\mathrm{RF, U}}(\varphi) + \lambda_{2}^{(\varphi)} \|\bm{\varphi} - \hat {\bm{\varphi}}_t\|_2^2\label{eq:TreatmentBridgeRFSecondStageLoss}\\
    \hat \gR_{\varphi, 2, t}^{\mathrm{RF, U}}(\varphi) &= \frac{1}{|\gB_{2, t}^{(\varphi)}|} \sum_{i \in \gB_{2, t}^{(\varphi)}} \left(\bm{\varphi}^\top \vu_{ii, t}^{(\varphi)}\right)^2 - \frac{2}{|\gB_{2, t}^{(\varphi)}| (|\gB_{2, t}^{(\varphi)}| - 1)}\sum_{\substack{i, j \in \gB_{2, t}^{(\varphi)}\\ i \ne j}} \bm{\varphi}^\top \vu_{ij, t}^{(\varphi)}.
    \label{eq:TreatmentBridgeRFSecondStageRisk}
\end{align}

A difficulty, analogous to the one encountered in OutcomeNet, is that Equation \ref{eq:TreatmentBridgeRFSecondStageLoss} depends on \(\theta_{Z,2}^{(\varphi)}\) only implicitly, through the first-stage optimum \(\check{\mV}_t^{(\varphi)}\). Following \citet{xu2021deep, xu2021learning}, we handle this dependence by differentiating through the closed-form mapping in Equation \ref{eq:TreatmentBridgeFirstLayerProximalClosedFormSolutionAux}, while treating the first-stage feature matrix \(\Phi_{AXW,1}^{(\varphi)}\) as fixed when computing gradients with respect to \(\theta_{Z,2}^{(\varphi)}\).

Conditional on the separately fitted first-stage representations, the off-diagonal term is the complete U-statistic for the product-law expectation. For the head update, define
\begin{align*}
    \mG_t^{(\varphi)} = \frac{1}{|\gB_{2, t}^{(\varphi)}|} \sum_{i \in \gB_{2, t}^{(\varphi)}} \vu_{ii, t}^{(\varphi)} {\vu_{ii, t}^{(\varphi)}}^\top, \quad \vb_{U, t}^{(\varphi)} = \frac{1}{|\gB_{2, t}^{(\varphi)}| (|\gB_{2, t}^{(\varphi)}| - 1)}\sum_{\substack{i, j \in \gB_{2, t}^{(\varphi)}\\ i \ne j}} \vu_{ij, t}^{(\varphi)}.
\end{align*}
Then, the proximal second-stage minimizer is 
\begin{align}
    \hat{\bm{\varphi}}_{t + 1} = \left(\mG_{t}^{(\varphi)} + \lambda_2^{(\varphi)} \mI\right)^{-1}\left(\vb_{U, t}^{(\varphi)} + \lambda_2^{(\varphi)} {\hat{\bm{\varphi}}_t}\right)
    \label{eq:TreatmentBridgeSecondLayerProximalClosedFormSolution}
\end{align}

The featurizer $\phi_{AX, 2}^{(\varphi)}$ is updated directly through Eq.~\ref{eq:TreatmentBridgeRFSecondStageLoss}.

\paragraph{Dose-response curve estimation with the treatment bridge.} The preceding two-stage procedure yields an estimator $\hat{\bm{\varphi}}$ of the treatment bridge satisfying the identifying relation in Eq.\ref{eq:TreatmentBridgeEq_ATE}. To recover the final dose-response curve, we must still estimate the conditional mean $\E[Y \varphi_{0}^{\mathrm{ATE}}(a, X, Z) \mid A = a]$. We distinguish two cases of continuous treatment and discrete treatment domains.

\begin{itemize}
    \item \textit{Continuous treatment domain}:

Unlike the outcome-bridge formulation, the target is itself a regression
function of \(a\). On the independent fold \(\gI_3^{(\varphi)}\), define
\[
\mathcal D_3^{(\varphi)}
:=
\{(a_i,y_i^{(\varphi)})\}_{i\in\gI_3^{(\varphi)}},
\qquad
y_i^{(\varphi)}
:=y_i\hat\varphi(a_i,x_i,z_i)\in\mathbb R^{d_Y}.
\]
We fit the network \(k^{\varphi}(\cdot;\theta_3^{\varphi})\) by minimizing
the empirical mean-squared error
\begin{equation}
\widehat{\gL}_{\varphi,3}(\theta_3^{\varphi})
:=
\frac{1}{|\gI_3^{(\varphi)}|}
\sum_{i\in\gI_3^{(\varphi)}}
\left\|
y_i^{(\varphi)}-k^{\varphi}(a_i;\theta_3^{\varphi})
\right\|_2^2.
\label{eq:TreatmentBridgeThirdStageLoss}
\end{equation}
Let \(\hat\theta_3^{\varphi}\in
\argmin_{\theta_3^{\varphi}}\widehat{\gL}_{\varphi,3}(\theta_3^{\varphi})\).
Since the population minimizer of squared error is the conditional mean,
\[
k^{\varphi}(a;\hat\theta_3^{\varphi})
\approx
\E\{Y\hat\varphi(a,X,Z)\mid A=a\},
\qquad
\hat f_{\mathrm{ATE}}^{(\varphi)}(a)
:=k^{\varphi}(a;\hat\theta_3^{\varphi}).
\]

\item \textit{Discrete treatment domain}:

When \(\gA\) is finite, the MSE regression reduces to a within-treatment
average. Writing
\(t_{\varphi,a}:=\sum_{i\in\gI_3^{(\varphi)}}\mathbf 1\{a_i=a\}\),
\[
\hat f_{\mathrm{ATE}}^{(\varphi)}(a)
=
\frac{1}{t_{\varphi,a}}
\sum_{\substack{i\in\gI_3^{(\varphi)}\\a_i=a}}
y_i\hat\varphi(a,x_i,z_i),
\qquad t_{\varphi,a}>0.
\]
\end{itemize}

The complete training procedure for the treatment bridge function is summarized in Algorithm~\ref{algo:DeepTreatmentBridgeFunction}. 

\begin{algorithm}[H]
{\footnotesize
\textbf{Input:} Observations \(\mathcal D=\{o_i\}_{i=1}^n\) and disjoint
folds \(\gI_1^{(\varphi)},\gI_2^{(\varphi)},\gI_3^{(\varphi)}\).\\
\textbf{Parameters:} Neural parameters
\(\theta_1^{(\varphi)},\theta_2^{(\varphi)},\theta_3^{\varphi}\);
linear layers \(\mV^{(\varphi)},\bm\varphi\).\\
\textbf{Hyperparameters:} Learning rates
\(\eta_1^{(\varphi)},\eta_2^{(\varphi)}\), regularization coefficients
\(\lambda_1^{(\varphi)},\lambda_2^{(\varphi)}\), and update counts
\(T_{\varphi,1},T_{\varphi,2}\).\\
\textbf{Output:} Fitted treatment bridge \(\hat\varphi^{\mathrm{ATE}}\) and population dose-response estimate \(\hat f_{\mathrm{ATE}}^{(\varphi)}\).\looseness=-1
\begin{algorithmic}[1]
\REPEAT
\STATE Sample mini-batches
\(\gB_1^{(\varphi)}\subset\gI_1^{(\varphi)}\) and
\(\gB_2^{(\varphi)}\subset\gI_2^{(\varphi)}\).
\FOR{\(s=1,\ldots,T_{\varphi,1}\)}
\STATE Update \(\theta_1^{(\varphi)}\) by one gradient step on
Equation~\ref{Eq:Alternative_DFPCL_FirstStageLoss}, holding the proxy-target
featurizer fixed.
\STATE Update the persistent operator \(\hat\mV^{(\varphi)}\) by
Equation~\ref{eq:TreatmentBridgeFirstLayerProximalClosedFormSolution}.
\ENDFOR
\FOR{\(s=1,\ldots,T_{\varphi,2}\)}
\STATE Compute \(\check\mV^{(\varphi)}\) on
\(\gB_2^{(\varphi)}\) using
Equation~\ref{eq:TreatmentBridgeFirstLayerProximalClosedFormSolutionAux}.
\STATE Update \(\theta_2^{(\varphi)}\) by one gradient step on the fixed
ratio-free U-objective in
Equation~\ref{eq:TreatmentBridgeRFSecondStageLoss}, differentiating through
\(\check\mV^{(\varphi)}\) for the proxy-side features.
\STATE Recompute \(\check\mV^{(\varphi)}\), the projected features,
\(\mG^{(\varphi)}\), and \(\vb_U^{(\varphi)}\) after the feature update.
\STATE Update \(\hat{\bm\varphi}\) by the proximal closed-form solve in
Equation~\ref{eq:TreatmentBridgeSecondLayerProximalClosedFormSolution}.
\ENDFOR
\UNTIL{convergence}
\STATE Set
\(\hat\varphi^{\mathrm{ATE}}(a,x,z)
=\hat{\bm\varphi}^{\top}
[\phi_{AX,2}^{(\varphi)}(a,x)\otimes\phi_{Z,2}^{(\varphi)}(z)]\).
\IF{\(A\) is continuous}
\STATE For \(i\in\gI_3^{(\varphi)}\), set
\(y_i^{(\varphi)}=y_i\hat\varphi^{\mathrm{ATE}}(a_i,x_i,z_i)\), and fit
\(k^{\varphi}\) by the MSE in
Equation~\ref{eq:TreatmentBridgeThirdStageLoss}.
\STATE Set
\(\hat f_{\mathrm{ATE}}^{(\varphi)}(a)
=k^{\varphi}(a;\hat\theta_3^{\varphi})\).
\ELSE
\STATE Set \(\hat f_{\mathrm{ATE}}^{(\varphi)}(a)\) to the
within-\(a\) average of
\(y_i\hat\varphi^{\mathrm{ATE}}(a,x_i,z_i)\) over
\(\gI_3^{(\varphi)}\).
\ENDIF
\STATE \textbf{Return}
\(\hat\varphi^{\mathrm{ATE}}\) and
\(\hat f_{\mathrm{ATE}}^{(\varphi)}\).
\end{algorithmic}
}
\caption{Population dose-response estimation with density-ratio-free
\textsc{TreatmentNet}. No density ratio is estimated or supplied as a
training target.}
\label{algo:DeepTreatmentBridgeFunction}
\end{algorithm}

\subsection{Dose-response curve estimation: doubly robust approach}
\label{appendix:DoublyRobustAlgorithmDoseResponse}

We now combine the learned outcome and treatment bridges into a doubly robust estimator of the dose-response curve. Recall the identifying representation
\[
\Psi_{\mathrm{ATE}}(a;h_0,{\varphi_0^{\mathrm{ATE}}})
=
\E[{\varphi_0^{\mathrm{ATE}}}(a,X,Z)\{Y-h_0(a,X,W)\}\mid A=a]
+
\E[h_0(a,X,W)].
\]

Given estimators \(\hat h\) from {\textsc{OutcomeNet}} and \(\hat\varphi\) from {\textsc{TreatmentNet}}, the only remaining task is to estimate the conditional expectation terms appearing in this formula. Let \(\gI_3^{(\mathrm{DR})}\) be an evaluation fold independent of the bridge fits.

\begin{itemize}
    \item \textit{Continuous treatment domain:}

Define the residual pseudo-outcome data
\[
\begin{aligned}
\mathcal D_3^{(\mathrm{DR})}
&:=
\{(a_i,y_i^{(\mathrm{DR})})\}_{i\in\gI_3^{(\mathrm{DR})}},
\\
y_i^{(\mathrm{DR})}
&:=
\hat\varphi(a_i,x_i,z_i)
\{y_i-\hat h(a_i,x_i,w_i)\}\in\mathbb R^{d_Y}.
\end{aligned}
\]
We fit \(k^{\mathrm{DR}}(\cdot;\theta_3^{\mathrm{DR}})\) by minimizing
the empirical mean-squared error
\begin{equation}
\widehat{\gL}_{\mathrm{DR},3}(\theta_3^{\mathrm{DR}})
:=
\frac{1}{|\gI_3^{(\mathrm{DR})}|}
\sum_{i\in\gI_3^{(\mathrm{DR})}}
\left\|
y_i^{(\mathrm{DR})}
-k^{\mathrm{DR}}(a_i;\theta_3^{\mathrm{DR}})
\right\|_2^2.
\label{eq:DRThirdStageLoss}
\end{equation}
Let \(\hat\theta_3^{\mathrm{DR}}\in
\argmin_{\theta_3^{\mathrm{DR}}}
\widehat{\gL}_{\mathrm{DR},3}(\theta_3^{\mathrm{DR}})\). The population
minimizer is
\[
\E[\hat\varphi(a,X,Z)\{Y-\hat h(a,X,W)\}\mid A=a].
\]
Combining this fitted correction with the outcome-bridge average, computed on
the same evaluation fold, gives
\begin{align}
\hat f_{\mathrm{ATE}}^{\mathrm{DR}}(a)
&:=
\frac{1}{|\gI_3^{(\mathrm{DR})}|}
\sum_{i\in\gI_3^{(\mathrm{DR})}}\hat h(a,x_i,w_i)
+k^{\mathrm{DR}}(a;\hat\theta_3^{\mathrm{DR}}).
\label{eq:DREstimator}
\end{align}

\item  \textit{Discrete treatment domain}:

When \(A\) is finite, the correction regression reduces to a saturated
within-treatment average. Let
\(t_{\mathrm{DR},a}:=
\sum_{i\in\gI_3^{(\mathrm{DR})}}\mathbf 1\{a_i=a\}\). Then, for
\(t_{\mathrm{DR},a}>0\),
\[
\begin{aligned}
\hat f_{\mathrm{ATE}}^{\mathrm{DR}}(a)
&=
\frac{1}{|\gI_3^{(\mathrm{DR})}|}
\sum_{i\in\gI_3^{(\mathrm{DR})}}\hat h(a,x_i,w_i)
\\&\quad+
\frac{1}{t_{\mathrm{DR},a}}
\sum_{\substack{i\in\gI_3^{(\mathrm{DR})}\\a_i=a}}
\hat\varphi(a,x_i,z_i)\{y_i-\hat h(a,x_i,w_i)\}.
\end{aligned}
\]

\end{itemize}

\begin{algorithm}[H]
{\footnotesize
\textbf{Input:} Observations \(\mathcal D=\{o_i\}_{i=1}^n\), bridge-fitting
folds, and an independent evaluation fold \(\gI_3^{(\mathrm{DR})}\).\\
\textbf{Subroutines:} The bridge-fitting stages of
Algorithms~\ref{algo:DeepOutcomeBridgeFunction}
and~\ref{algo:DeepTreatmentBridgeFunction}.\\
\textbf{Output:} Doubly robust population dose-response estimate
\(\hat f_{\mathrm{ATE}}^{\mathrm{DR}}\).
\begin{algorithmic}[1]
\STATE Fit \(\hat h\) and \(\hat\varphi^{\mathrm{ATE}}\) on the bridge-fitting folds with Algorithms~\ref{algo:DeepOutcomeBridgeFunction}
and~\ref{algo:DeepTreatmentBridgeFunction}, respectively.
\STATE For each \(i\in\gI_3^{(\mathrm{DR})}\), form
\[
y_i^{(\mathrm{DR})}
=\hat\varphi^{\mathrm{ATE}}(a_i,x_i,z_i)
\{y_i-\hat h(a_i,x_i,w_i)\}.
\]
\IF{\(A\) is continuous}
\STATE Fit \(k^{\mathrm{DR}}(\cdot;\theta_3^{\mathrm{DR}})\) by the MSE in
Equation~\ref{eq:DRThirdStageLoss}.
\STATE Set
\[
\hat f_{\mathrm{ATE}}^{\mathrm{DR}}(a)
=
\frac{1}{|\gI_3^{(\mathrm{DR})}|}
\sum_{i\in\gI_3^{(\mathrm{DR})}}\hat h(a,x_i,w_i)
+k^{\mathrm{DR}}(a;\hat\theta_3^{\mathrm{DR}}).
\]
\ELSE
\STATE Replace the correction regression by the within-\(a\) average of
\(y_i^{(\mathrm{DR})}\), and add the outcome-bridge average over
\(\gI_3^{(\mathrm{DR})}\).
\ENDIF
\STATE \textbf{Return} \(\hat f_{\mathrm{ATE}}^{\mathrm{DR}}\).
\end{algorithmic}
}
\caption{Population dose-response estimation with \textsc{DR-PCL-Net}.}
\label{algo:DeepDoublyRobustPCL}
\end{algorithm}

\section{Neural mean embedding-based proxy causal learning for heterogeneous dose-response}
\label{sec:NeuralMeanEmbedding_HeterogeneousDoseResponse}

In this section, we derive the doubly robust algorithm for heterogeneous dose-response curve. In particular, similar to dose-response counterpart, we first derive the outcome bridge-based algorithm. Then, we derive the analogous treatment bridge algorithm, and then combine both into doubly robust algorithm.

\subsection{Heterogeneous dose-response estimation: outcome bridge method}
\label{Appendix:HeterogeneousOutcomeAlgorithm}

Write \(X=(S,V)\), and use
\(h(a,v,s,w):=h(a,(s,v),w)\) when displaying the effect modifier
separately. No CATE-specific outcome bridge is required: we learn
\(\hat h(a,X,W)\) exactly as in
Appendix~\ref{sec:DFPCL_Xu_Review}, with \(V\) included in \(X\).
Outcome-bridge identification gives
\begin{equation}
f_{\mathrm{CATE}}(a,v)
=\E[h_0(a,v,S,W)\mid V=v].
\label{eq:HTE_OutcomeBridge_Expansion}
\end{equation}

For this target, we use the factorized second-stage architecture
\begin{equation}
\begin{aligned}
\hat h(a,v,s,w)
&=
\left\{\hat{\mH}^{(h)}\right\}^{\top}
\left[
\phi_{A,2}^{(h)}(a)\otimes
\phi_{V,2}^{(h)}(v)\otimes
c^{(h)}(s,w)
\right],\\
c^{(h)}(s,w)
&:=
\phi_{S,2}^{(h)}(s)\otimes\phi_{W,2}^{(h)}(w),
\end{aligned}
\label{eq:HeterogeneousOutcomeBridgeFactorization}
\end{equation}
where
\(\hat{\mH}^{(h)}\in\mathbb R^{d_{h,2}\times d_Y}\) is the fitted
linear head; it is a column vector when \(d_Y=1\). Conditional on the
data used to fit \(\hat h\), let
\[
\mu_{SW\mid V}^{(h)}(v)
:=\E[c^{(h)}(S,W)\mid V=v].
\]
Linearity of the fitted head gives
\[
\E[\hat h(a,v,S,W)\mid V=v]
=
\left\{\hat{\mH}^{(h)}\right\}^{\top}
\left[
\phi_{A,2}^{(h)}(a)\otimes
\phi_{V,2}^{(h)}(v)\otimes
\mu_{SW\mid V}^{(h)}(v)
\right].
\]

Let \(\gI_3^{(h)}\) be an independent third-stage split of size
\(n_{h,3}\). We estimate the single conditional mean embedding
\(\mu_{SW\mid V}^{(h)}\) by regressing the joint feature responses
\(c_i^{(h)}:=c^{(h)}(s_i,w_i)\) on \(v_i\). With a regression network
\(m^{(h)}(\cdot;\theta_3^{(h)})\), let
\[
\hat\theta_3^{(h)}
\in\argmin_{\theta}
\frac{1}{n_{h,3}}\sum_{i\in\gI_3^{(h)}}
\left\|
c_i^{(h)}-m^{(h)}(v_i;\theta)
\right\|_2^2
\]
and write
\(\hat\mu_{SW\mid V}^{(h)}(v)
:=m^{(h)}(v;\hat\theta_3^{(h)})\).
The resulting estimator is
\begin{equation}
\hat f_{\mathrm{CATE}}^{(h)}(a,v)
:=
\left\{\hat{\mH}^{(h)}\right\}^{\top}
\left[
\phi_{A,2}^{(h)}(a)\otimes
\phi_{V,2}^{(h)}(v)\otimes
\hat\mu_{SW\mid V}^{(h)}(v)
\right].
\label{eq:HeterogeneousOutcomeBridgeFinalEstimator}
\end{equation}
Because the learned embedding depends only on \(v\), one third-stage fit
can be reused for every intervention query \(a\); no retraining over
treatment levels is required. At the oracle bridge and population
embedding, this expression equals \(f_{\mathrm{CATE}}(a,v)\).

\begin{algorithm}[H]
{\footnotesize
\textbf{Input:} Observations
\(\mathcal D=\{(a_i,v_i,s_i,w_i,z_i,y_i)\}_{i=1}^n\) and disjoint folds
\(\gI_1^{(h)},\gI_2^{(h)},\gI_3^{(h)}\).\\
\textbf{Subroutine:} The bridge-fitting stages of
Algorithm~\ref{algo:DeepOutcomeBridgeFunction}.\\
\textbf{Output:} Fitted outcome bridge \(\hat h\) and heterogeneous
response estimate \(\hat f_{\mathrm{CATE}}^{(h)}\).
\begin{algorithmic}[1]
\STATE Set \(X=(S,V)\) and fit \(\hat h\) with the factorized
representation in
Equation~\ref{eq:HeterogeneousOutcomeBridgeFactorization}.
\STATE For each \(i\in\gI_3^{(h)}\), form
\[
c_i^{(h)}
=\phi_{S,2}^{(h)}(s_i)\otimes\phi_{W,2}^{(h)}(w_i).
\]
\STATE Fit \(m^{(h)}(\cdot;\theta_3^{(h)})\) by minimizing
\[
\frac{1}{|\gI_3^{(h)}|}
\sum_{i\in\gI_3^{(h)}}
\|c_i^{(h)}-m^{(h)}(v_i;\theta_3^{(h)})\|_2^2,
\]
and set
\(\hat\mu_{SW\mid V}^{(h)}(v)=m^{(h)}(v;\hat\theta_3^{(h)})\).
\STATE Construct \(\hat f_{\mathrm{CATE}}^{(h)}(a,v)\) using
Equation~\ref{eq:HeterogeneousOutcomeBridgeFinalEstimator}.
\STATE \textbf{Return} \(\hat h\) and
\(\hat f_{\mathrm{CATE}}^{(h)}\).
\end{algorithmic}
}
\caption{Heterogeneous dose-response estimation with
\textsc{OutcomeNet}.}
\label{algo:HeterogeneousDeepOutcomeBridgeFunction}
\end{algorithm}

\subsection{Heterogeneous dose-response estimation: treatment bridge method}
\label{Appendix:HeterogeneousTreatmentAlgorithm}

We now extend {\textsc{TreatmentNet}} to the heterogeneous dose-response function. As in the outcome-bridge subsection, we write \(X=(S,V)\), where \(V\) denotes the effect modifiers of interest and \(S\) collects the remaining observed covariates.

For a raw CATE bridge candidate \(\varphi(a,v,s,z)\), define
\begin{align*}
g_{\varphi}^{\mathrm{CATE}}(a,v,s,w)
&:=(T_{\varphi}^{\mathrm{CATE}}\varphi)(a,v,s,w)\\
&:=\E[\varphi(a,v,s,Z)\mid A=a,V=v,S=s,W=w].
\end{align*}
The native bridge satisfies
\[
T_{\varphi}^{\mathrm{CATE}}\varphi_0^{\mathrm{CATE}}
= {r_{0,\mathrm{CATE}}},
\qquad
f_{\mathrm{CATE}}(a,v)
=\E[Y\varphi_0^{\mathrm{CATE}}(a,v,S,Z)\mid A=a,V=v].
\]

To characterize a density-ratio-free objective, let
\(\sP_{\mathrm{CATE}}:=\sP_{A,V,S,W}\) and define the conditionally
decoupled law
\begin{equation}
\sQ_{\mathrm{CATE}}(da,dv,ds,dw)
:=
\sP_V(dv)\,
\sP_{A\mid V}(da\mid v)\,
\sP_{S,W\mid V}(ds,dw\mid v).
\label{eq:CATENativeProductLaw}
\end{equation}
Assume \(\sQ_{\mathrm{CATE}}\ll\sP_{\mathrm{CATE}}\), and set
\[
{r_{0,\mathrm{CATE}}}
:=\frac{d\sQ_{\mathrm{CATE}}}{d\sP_{\mathrm{CATE}}}
\in L^2(\sP_{\mathrm{CATE}}).
\]

{Under a common dominating measure, this derivative admits
the version}
\[
{r_{0,\mathrm{CATE}}}(a,v,s,w)
=
\frac{p_{A\mid V}(a\mid v)}
{p_{A\mid V,S,W}(a\mid v,s,w)}
{\quad\text{on the denominator support}.}
\]
For \( {u}\in L^2(\sP_{\mathrm{CATE}})\), define
\begin{equation}
\gR_{\mathrm{CATE}}^{\mathrm{RF}}( {u})
:=
\E_{\sP_{\mathrm{CATE}}}[ {u}(A,V,S,W)^2]
-2\E_{\sQ_{\mathrm{CATE}}}[ {u}(A,V,S,W)].
\label{eq:CATENativeRatioFreeRisk}
\end{equation}
The linear expectation has the conditional-product representation
\[
\E_{\sQ_{\mathrm{CATE}}}[{u}]
=
\E_V\!\left[
\E_{A^\dagger\mid V}\!\left\{
\E_{S,W\mid V}\!\left[{u}(A^\dagger,V,S,W)\right]
\right\}
\right],
\]
where \(A^\dagger\perp(S,W)\mid V\). Since
\(\E_{\sQ_{\mathrm{CATE}}}[{u}]
=\E_{\sP_{\mathrm{CATE}}}[{r_{0,\mathrm{CATE}}} {u}]\), completing the
square gives
\begin{equation}
\gR_{\mathrm{CATE}}^{\mathrm{RF}}( {u})
-\gR_{\mathrm{CATE}}^{\mathrm{RF}}( {r_{0,\mathrm{CATE}}})
=
\| {u}- {r_{0,\mathrm{CATE}}}\|_{L^2(\sP_{\mathrm{CATE}})}^2.
\label{eq:CATENativeRatioFreeExcessRisk}
\end{equation}

Thus, a native density-ratio-free CATE risk exists, but its linear term
requires draws from the conditional product
\(\sP_{A\mid V}\otimes\sP_{S,W\mid V}\) at the same value of \(V\).
A global permutation of \(A\) instead samples from \(\sP_A\) and is
generally biased for Equation~\ref{eq:CATENativeRatioFreeRisk}. When \(V\)
is discrete, within-\(V\) pairing is possible in sufficiently populated
strata. For continuous or high-dimensional \(V\), exact matching is
unavailable and a conditional sampler, smoothing method, or additional
conditional regression is required. We avoid this extra nuisance by
reusing the density-ratio-free ATE bridge through
Proposition~\ref{prop:AlternativeCATEidentification}.

Recall that CATE is identified with treatment bridge function via
\begin{align*}
    f_{\mathrm{CATE}}(a, v) = \frac{\E\{Y \varphi_0^{\mathrm{ATE}}(a, v, S, Z) \mid A = a, V = v\}}{\E\{\varphi_0^{\mathrm{ATE}}(a, v, S, Z) \mid A = a, V = v\}}.
\end{align*}
 
The denominator is the oracle normalizer:
\[
{
D_0(a,v)
:=\E[\varphi_0^{\mathrm{ATE}}(a,v,S,Z)\mid A=a,V=v]
=\frac{d(\sP_A\otimes\sP_V)}{d\sP_{A,V}}(a,v).}
\]
{The last equality follows from the conditional-expectation property of Radon--Nikodym derivatives; see Proposition~\ref{prop:AlternativeCATEidentification}. Under a common dominating measure, it reduces to \(p_A(a)/p_{A\mid V}(a\mid v)\) on the denominator support.}
We learn one
\(\hat\varphi^{\mathrm{ATE}}\) using all \(X=(S,V)\), as described in
Appendix~\ref{Sec:NeuralTreatmentApproach_Appendix}; no native CATE bridge
is fitted.

On an independent third-stage split \(\gI_3^{(\varphi)}\), define
\[
\hat\varphi_k^{\mathrm{ATE}}
:=\hat\varphi^{\mathrm{ATE}}(a_k,v_k,s_k,z_k),
\qquad
\gamma_{N,k}^{(\varphi)}
:=y_k\hat\varphi_k^{\mathrm{ATE}},
\qquad
\gamma_{D,k}^{(\varphi)}
:=\hat\varphi_k^{\mathrm{ATE}}.
\]
We fit a shared network with vector-valued numerator head
\(N^{(\varphi)}\) and scalar denominator head \(D^{(\varphi)}\) by
minimizing the joint empirical mean-squared error
\[
\frac{1}{|\gI_3^{(\varphi)}|}
\sum_{k\in\gI_3^{(\varphi)}}
\left[
\|\gamma_{N,k}^{(\varphi)}-N^{(\varphi)}(a_k,v_k)\|_2^2
+\{\gamma_{D,k}^{(\varphi)}-D^{(\varphi)}(a_k,v_k)\}^2
\right].
\]
Denote the fitted heads by \(\hat N^{(\varphi)}\) and
\(\hat D^{(\varphi)}\). Their population MSE targets are
\begin{align*}
N_{\hat\varphi}(a,v)
&:=\E[Y\hat\varphi^{\mathrm{ATE}}(a,v,S,Z)\mid A=a,V=v],\\
D_{\hat\varphi}(a,v)
&:=\E[\hat\varphi^{\mathrm{ATE}}(a,v,S,Z)\mid A=a,V=v].
\end{align*}
For a small \(\epsilon_D>0\), the stabilized \textsc{TreatmentNet} estimator is
\begin{equation}
\hat f_{\mathrm{CATE}}^{(\varphi)}(a,v)
:=
\frac{\hat N^{(\varphi)}(a,v)}
{\max\{\hat D^{(\varphi)}(a,v),\epsilon_D\}}.
\label{eq:CATESelfNormalizedTreatmentEstimator}
\end{equation}

\begin{algorithm}[H]
{\footnotesize
\textbf{Input:} Observations
\(\mathcal D=\{(a_i,v_i,s_i,w_i,z_i,y_i)\}_{i=1}^n\), disjoint folds
\(\gI_1^{(\varphi)},\gI_2^{(\varphi)},\gI_3^{(\varphi)}\), and
\(\epsilon_D>0\).\\
\textbf{Subroutine:} The bridge-fitting stages of
Algorithm~\ref{algo:DeepTreatmentBridgeFunction}.\\
\textbf{Output:} Self-normalized heterogeneous response estimate
\(\hat f_{\mathrm{CATE}}^{(\varphi)}\).
\begin{algorithmic}[1]
\STATE Set \(X=(S,V)\) and fit the single ATE-scale bridge
\(\hat\varphi^{\mathrm{ATE}}\); do not fit a native CATE bridge.
\STATE For each \(i\in\gI_3^{(\varphi)}\), form
\[
\hat\varphi_i^{\mathrm{ATE}}
=\hat\varphi^{\mathrm{ATE}}(a_i,v_i,s_i,z_i),\qquad
\gamma_{N,i}^{(\varphi)}
=y_i\hat\varphi_i^{\mathrm{ATE}},\qquad
\gamma_{D,i}^{(\varphi)}
=\hat\varphi_i^{\mathrm{ATE}}.
\]
\STATE Jointly fit a shared network with vector head
\(N^{(\varphi)}\) and scalar head \(D^{(\varphi)}\) by minimizing
\[
\frac{1}{|\gI_3^{(\varphi)}|}
\sum_{i\in\gI_3^{(\varphi)}}
\left[
\|\gamma_{N,i}^{(\varphi)}
-N^{(\varphi)}(a_i,v_i)\|_2^2
+\{\gamma_{D,i}^{(\varphi)}
-D^{(\varphi)}(a_i,v_i)\}^2
\right].
\]
\STATE Set
\[
\hat f_{\mathrm{CATE}}^{(\varphi)}(a,v)
=
\frac{\hat N^{(\varphi)}(a,v)}
{\max\{\hat D^{(\varphi)}(a,v),\epsilon_D\}}
\]
as in Equation~\ref{eq:CATESelfNormalizedTreatmentEstimator}.
\STATE \textbf{Return} \(\hat f_{\mathrm{CATE}}^{(\varphi)}\).
\end{algorithmic}
}
\caption{Self-normalized heterogeneous dose-response estimation with
\textsc{TreatmentNet}, reusing the ATE-scale treatment bridge.}
\label{algo:HeterogeneousDeepTreatmentBridgeFunction}
\end{algorithm}

\subsection{Heterogeneous dose-response estimation: doubly robust method}
\label{Appendix:HeterogeneousDRAlgorithm}

Let \(h\) be an outcome-bridge candidate and \(\varphi\) an ATE-scale
treatment-bridge candidate. For finite expectations and
\(D_\varphi(a,v)\ne0\), define
\begin{align*}
M_h(a,v)
&:=\E[h(a,v,S,W)\mid V=v],\\
K_{h,\varphi}(a,v)
&:=\E[\varphi(a,v,S,Z)\{Y-h(a,v,S,W)\}\mid A=a,V=v],\\
D_\varphi(a,v)
&:=\E[\varphi(a,v,S,Z)\mid A=a,V=v].
\end{align*}
The self-normalized functional is
\begin{equation}
\Psi_{\mathrm{CATE}}^{\mathrm{SN}}(a,v;h,\varphi)
:=M_h(a,v)+\frac{K_{h,\varphi}(a,v)}{D_\varphi(a,v)}.
\label{eq:CATESelfNormalizedDRFunctional}
\end{equation}
It equals \(f_{\mathrm{CATE}}(a,v)\) if either \(h=h_0\), or
\(\varphi=\varphi_0^{\mathrm{ATE}}\).

For estimation, reuse \(\hat f_{\mathrm{CATE}}^{(h)}\) from
Appendix~\ref{Appendix:HeterogeneousOutcomeAlgorithm}. On a third-stage
fold \(\gI_3^{(\mathrm{DR})}\), held out from bridge fitting, set
\begin{align*}
\hat\varphi_i^{\mathrm{ATE}}
&:=\hat\varphi^{\mathrm{ATE}}(a_i,v_i,s_i,z_i),
&\hat h_i&:=\hat h(a_i,v_i,s_i,w_i),\\
\gamma_{K,i}^{\mathrm{DR}}
&:=\hat\varphi_i^{\mathrm{ATE}}(y_i-\hat h_i),
&\gamma_{D,i}^{\mathrm{DR}}&:=\hat\varphi_i^{\mathrm{ATE}}.
\end{align*}
As in Appendix~\ref{Appendix:HeterogeneousTreatmentAlgorithm}, jointly
regress these paired labels on \((a_i,v_i)\) by the same joint MSE, with
\(\gamma_{K,i}^{\mathrm{DR}}\) replacing \(\gamma_{N,i}^{(\varphi)}\), obtaining the
vector-valued correction \(\hat K^{\mathrm{DR}}\) and scalar normalizer
\(\hat D^{\mathrm{DR}}\). Thus, only the numerator label differs from
the \textsc{TreatmentNet} third stage. The estimator is
\begin{equation}
\hat f_{\mathrm{CATE}}^{\mathrm{DR}}(a,v)
:=
\hat f_{\mathrm{CATE}}^{(h)}(a,v)
+
\frac{\hat K^{\mathrm{DR}}(a,v)}
{\max\{\hat D^{\mathrm{DR}}(a,v),\epsilon_D\}},
\label{eq:CATESelfNormalizedDREstimator}
\end{equation}
with the same denominator stabilization as in
Equation~\ref{eq:CATESelfNormalizedTreatmentEstimator}.

\begin{algorithm}[H]
{\footnotesize
\textbf{Input:} Observations
\(\mathcal D=\{(a_i,v_i,s_i,w_i,z_i,y_i)\}_{i=1}^n\), bridge-fitting
folds, an independent fold \(\gI_3^{(\mathrm{DR})}\), and
\(\epsilon_D>0\).\\
\textbf{Subroutines:}
Algorithms~\ref{algo:HeterogeneousDeepOutcomeBridgeFunction}
and~\ref{algo:DeepTreatmentBridgeFunction}.\\
\textbf{Output:} Doubly robust heterogeneous response estimate
\(\hat f_{\mathrm{CATE}}^{\mathrm{DR}}\).
\begin{algorithmic}[1]
\STATE Fit \(\hat h\) and
\(\hat f_{\mathrm{CATE}}^{(h)}\) with
Algorithm~\ref{algo:HeterogeneousDeepOutcomeBridgeFunction}.
\STATE With \(X=(S,V)\), fit the ATE-scale bridge
\(\hat\varphi^{\mathrm{ATE}}\) using the bridge stages of
Algorithm~\ref{algo:DeepTreatmentBridgeFunction}.
\STATE For each \(i\in\gI_3^{(\mathrm{DR})}\), set
\[
\hat h_i=\hat h(a_i,v_i,s_i,w_i),\qquad
\hat\varphi_i^{\mathrm{ATE}}
=\hat\varphi^{\mathrm{ATE}}(a_i,v_i,s_i,z_i),
\]
\[
\gamma_{K,i}^{\mathrm{DR}}
=\hat\varphi_i^{\mathrm{ATE}}(y_i-\hat h_i),\qquad
\gamma_{D,i}^{\mathrm{DR}}
=\hat\varphi_i^{\mathrm{ATE}}.
\]
\STATE Jointly regress
\((\gamma_{K,i}^{\mathrm{DR}},\gamma_{D,i}^{\mathrm{DR}})\) on
\((a_i,v_i)\) by squared Euclidean and scalar MSE, obtaining
\(\hat K^{\mathrm{DR}}\) and \(\hat D^{\mathrm{DR}}\).
\STATE Set
\[
\hat f_{\mathrm{CATE}}^{\mathrm{DR}}(a,v)
=\hat f_{\mathrm{CATE}}^{(h)}(a,v)
+\frac{\hat K^{\mathrm{DR}}(a,v)}
{\max\{\hat D^{\mathrm{DR}}(a,v),\epsilon_D\}}
\]
as in Equation~\ref{eq:CATESelfNormalizedDREstimator}.
\STATE \textbf{Return} \(\hat f_{\mathrm{CATE}}^{\mathrm{DR}}\).
\end{algorithmic}
}
\caption{Heterogeneous dose-response estimation with
\textsc{DR-PCL-Net}.}
\label{algo:DeepHeterogeneousDoublyRobustPCL}
\end{algorithm}

\section{Neural mean embedding-based proxy causal learning for treatment-conditional dose response}

We next derive our proposed neural mean embedding algorithms for treatment-conditional dose response estimation. We first present the outcome bridge route. Then, we derive the treatment bridge route and combine it with the outcome bridge method to derive the doubly robust estimator.

\subsection{{Conditional dose-response estimation: outcome bridge method}}
\label{Appendix:ConditionalOutcomeAlgorithm}

For
\(f_{\mathrm{ATT}}(a,a')=\E\{Y^{(a)}\mid A=a'\}\), \(a\) is the
intervention level and \(a'\) indexes the factual reference population.
Outcome-bridge identification gives
\[
f_{\mathrm{ATT}}(a,a')
=\E[h_0(a,X,W)\mid A=a'].
\]
Thus, no ATT-specific outcome bridge is required: we reuse the \textsc{OutcomeNet}
estimate \(\hat h\) from Appendix~\ref{sec:DFPCL_Xu_Review}. For the
factorized second-stage representation, write
\[
\hat h(a,x,w)
=\left\{\hat{\mH}^{(h)}\right\}^{\top}\!\left[
\phi_{A,2}^{(h)}(a)\otimes c_{XW}^{(h)}(x,w)
\right],
\qquad
c_{XW}^{(h)}(x,w)
:=\phi_{X,2}^{(h)}(x)\otimes\phi_{W,2}^{(h)}(w).
\]
Conditional on the fitted bridge, define
\(\mu_{XW\mid A}^{(h)}(a')
:=\E[c_{XW}^{(h)}(X,W)\mid A=a']\).

On an independent third-stage split \(\gI_3^{(h)}\), estimate this
conditional mean by the empirical MSE regression
\[
\hat\mu_{XW\mid A}^{(h)}
\in\argmin_{m}
\frac{1}{|\gI_3^{(h)}|}
\sum_{i\in\gI_3^{(h)}}
\|c_{XW}^{(h)}(x_i,w_i)-m(a_i)\|_2^2.
\]
The resulting estimator is
\begin{equation}
\hat f_{\mathrm{ATT}}^{(h)}(a,a')
:=
\left\{\hat{\mH}^{(h)}\right\}^{\top}\!\left[
\phi_{A,2}^{(h)}(a)\otimes
\hat\mu_{XW\mid A}^{(h)}(a')
\right].
\label{eq:ConditionalOutcomeBridgeFinalEstimator}
\end{equation}
The same third-stage fit yields the entire \((a,a')\)-surface; no retraining
over either treatment argument is required. For finite discrete \(A\), the
MSE regression reduces to the within-level average of
\(c_{XW}^{(h)}(X,W)\); for continuous \(A\), \(m\) is a neural network.
At the oracle bridge and population conditional mean embedding,
Equation~\ref{eq:ConditionalOutcomeBridgeFinalEstimator} equals
\(f_{\mathrm{ATT}}(a,a')\).

\begin{algorithm}[H]
{\footnotesize
\textbf{Input:} Observations \(\mathcal D=\{o_i\}_{i=1}^n\) and disjoint
folds \(\gI_1^{(h)},\gI_2^{(h)},\gI_3^{(h)}\).\\
\textbf{Subroutine:} The bridge-fitting stages of
Algorithm~\ref{algo:DeepOutcomeBridgeFunction}.\\
\textbf{Output:} Treatment-conditional response surface
\(\hat f_{\mathrm{ATT}}^{(h)}\).
\begin{algorithmic}[1]
\STATE Fit the population outcome bridge \(\hat h\); no ATT-specific
outcome bridge is required.
\STATE For each \(i\in\gI_3^{(h)}\), form
\[
c_{XW,i}^{(h)}
=\phi_{X,2}^{(h)}(x_i)\otimes\phi_{W,2}^{(h)}(w_i).
\]
\STATE Fit \(m^{(h)}(\cdot;\theta_3^{(h)})\) by minimizing
\[
\frac{1}{|\gI_3^{(h)}|}
\sum_{i\in\gI_3^{(h)}}
\|c_{XW,i}^{(h)}-m^{(h)}(a_i;\theta_3^{(h)})\|_2^2,
\]
and set
\(\hat\mu_{XW\mid A}^{(h)}(a')
=m^{(h)}(a';\hat\theta_3^{(h)})\).
\STATE Construct \(\hat f_{\mathrm{ATT}}^{(h)}(a,a')\) using
Equation~\ref{eq:ConditionalOutcomeBridgeFinalEstimator}.
\STATE \textbf{Return} the full surface
\(\hat f_{\mathrm{ATT}}^{(h)}\).
\end{algorithmic}
}
\caption{Treatment-conditional dose-response estimation with
\textsc{OutcomeNet}.}
\label{algo:ConditionalDeepOutcomeBridgeFunction}
\end{algorithm}

\subsection{{Conditional dose-response estimation: treatment bridge method}}
\label{sec:NeuralMeanEmbedding_ConditionalDoseResponse}

For the conditional dose-response surface
\(f_{\mathrm{ATT}}(a,a')=\E\{Y^{(a)}\mid A=a'\}\), the first argument
\(a\) is the intervention level, whereas \(a'\) indexes
the factual reference population. Define the probability laws
\begin{align}
\sP_{\mathrm{ATT}}(da,da',dx,dw)
&:=\sP_A(da)\sP_A(da')\sP_{X,W\mid A=a}(dx,dw),\nonumber\\
\sQ_{\mathrm{ATT}}(da,da',dx,dw)
&:=\sP_A(da)\sP_A(da')\sP_{X,W\mid A=a'}(dx,dw).
\label{eq:ATTProductLaws}
\end{align}
Assume \(\sQ_{\mathrm{ATT}}\ll\sP_{\mathrm{ATT}}\) on the treatment pairs of
interest, and let
\[
r_{0,\mathrm{ATT}}
:=\frac{d\sQ_{\mathrm{ATT}}}{d\sP_{\mathrm{ATT}}}
\in L^2(\sP_{\mathrm{ATT}}).
\]
For a raw bridge candidate
\(\varphi:\gA\times\gA\times\gX\times\gZ\to\mathbb R\), define its
conditional projection
\[
(T_{\mathrm{ATT}}\varphi)(a,a',x,w)
=:g_{\varphi}^{\mathrm{ATT}}(a,a',x,w)
:=\E[\varphi(a,a',x,Z)\mid A=a,X=x,W=w].
\]
A valid ATT bridge satisfies
\(T_{\mathrm{ATT}}\varphi_0^{\mathrm{ATT}}=r_{0,\mathrm{ATT}}\).

Let \((A_1,X_1,W_1)\) and \((A_2,X_2,W_2)\) be independent observational
units. For \({u}\in L^2(\sP_{\mathrm{ATT}})\), consider
\begin{equation}
\gR_{\mathrm{ATT}}^{\mathrm{RF}}({u})
:=\E\!\left[{u}(A_1,A_2,X_1,W_1)^2
-2{u}(A_1,A_2,X_2,W_2)\right].
\label{eq:ATTRatioFreeRisk}
\end{equation}
The same change-of-measure argument as in
Lemma~\ref{lemma:densRatioFreeRiskCharacterization} yields
\begin{equation}
\gR_{\mathrm{ATT}}^{\mathrm{RF}}({u})
-\gR_{\mathrm{ATT}}^{\mathrm{RF}}(r_{0,\mathrm{ATT}})
=\|{u}-r_{0,\mathrm{ATT}}\|_{L^2(\sP_{\mathrm{ATT}})}^2.
\label{eq:ATTRatioFreeExcessRisk}
\end{equation}
Thus, the reference level can be treated as an input: one bridge estimates
the entire \((a,a')\)-surface without retraining at each \(a'\).

The first-stage conditional mean embedding is unchanged from
Appendix~\ref{Sec:NeuralTreatmentApproach_Appendix}. We use the compact
parameterization
\begin{align*}
\mu_Z^{(\varphi)}(a,x,w)
&:=\E[\phi_{Z,2}^{(\varphi)}(Z)\mid A=a,X=x,W=w],\\
\hat\mu_Z^{(\varphi)}(a,x,w)
&=\bigl(\hat{\mV}^{(\varphi)}\bigr)^\top
\phi_{AXW,1}^{(\varphi)}(a,x,w),\\
\varphi(a,a',x,z)
&=\bm\varphi^\top\!\left[
\phi_{AA'X,2}^{(\varphi)}(a,a',x)
\otimes\phi_{Z,2}^{(\varphi)}(z)\right],\\
\hat g_{\varphi}^{\mathrm{ATT}}(a,a',x,w)
&=\bm\varphi^\top\!\left[
\phi_{AA'X,2}^{(\varphi)}(a,a',x)
\otimes\hat\mu_Z^{(\varphi)}(a,x,w)\right].
\end{align*}
Here \(\phi_{AA'X,2}^{(\varphi)}\) is one joint featurizer of
\((a,a',x)\). For a second-stage batch of size \(M\), its complete
off-diagonal objective is
\begin{equation}
\widehat{\gR}_{\mathrm{ATT},2}^{\mathrm{RF,U}}(\varphi)
=\frac{1}{M(M-1)}\sum_{i\ne j}
\left\{
\hat g_{\varphi}^{\mathrm{ATT}}(a_i,a_j,x_i,w_i)^2
-2\hat g_{\varphi}^{\mathrm{ATT}}(a_i,a_j,x_j,w_j)
\right\}.
\label{eq:ATTEmpiricalRatioFreeRisk}
\end{equation}
We apply the same alternating feature updates and proximal closed-form head
updates as for \textsc{TreatmentNet}. Crucially, the transported term evaluates the
embedding at \(\hat\mu_Z^{(\varphi)}(a_i,x_j,w_j)\): \(i\) supplies the
source treatment, while \(j\) supplies the intact anchor block
\((a_j,x_j,w_j)\). Replacing \(a_i\) by \(a_j\) changes the transported law.

Finally, freeze the bridge and use an independent third-stage split of size
\(t_{\mathrm{ATT}}\). For each ordered pair \(i\ne j\), form
\[
\gamma_{ij}^{\mathrm{ATT}}
:=y_i\hat\varphi(a_i,a_j,x_i,z_i).
\]
Fit a two-input regression \(k^{\varphi}(\cdot,\cdot;\theta_3^{\varphi})\)
by minimizing the pairwise empirical MSE
\[
\frac{1}{t_{\mathrm{ATT}}(t_{\mathrm{ATT}}-1)}\sum_{i\ne j}
\left\|\gamma_{ij}^{\mathrm{ATT}}
-k^{\varphi}(a_i,a_j;\theta_3^{\varphi})\right\|_2^2,
\qquad
\hat f_{\mathrm{ATT}}^{(\varphi)}(a,a')
:=k^{\varphi}(a,a';\hat\theta_3^{\varphi}).
\]
Indeed, at the oracle bridge,
\[
\E[Y_1\varphi_0^{\mathrm{ATT}}(A_1,A_2,X_1,Z_1)
\mid A_1=a,A_2=a']
=f_{\mathrm{ATT}}(a,a').
\]

\begin{algorithm}[H]
{\footnotesize
\textbf{Input:} Observations \(\mathcal D=\{o_i\}_{i=1}^n\), disjoint
bridge-fitting folds
\(\gI_1^{(\varphi)},\gI_2^{(\varphi)}\), and an independent final-stage
fold \(\gI_3^{(\varphi)}\).\\
\textbf{Parameters:} Neural parameters
\(\theta_1^{(\varphi)},\theta_2^{(\varphi)},\theta_3^{\varphi}\), first-stage
operator \(\mV^{(\varphi)}\), bridge head \(\bm\varphi\), and final network
\(k^{\varphi}\).\\
\textbf{Hyperparameters:} Learning rates
\(\eta_1^{(\varphi)},\eta_2^{(\varphi)}\), proximal coefficients
\(\lambda_1^{(\varphi)},\lambda_2^{(\varphi)}\), and update counts
\(T_{\varphi,1},T_{\varphi,2}\).\\
\textbf{Output:} Fitted joint-anchor bridge
\(\hat\varphi^{\mathrm{ATT}}\) and treatment-conditional response surface
\(\hat f_{\mathrm{ATT}}^{(\varphi)}\).
\begin{algorithmic}[1]
\REPEAT
\STATE Sample mini-batches
\(\gB_1^{(\varphi)}\subset\gI_1^{(\varphi)}\) and
\(\gB_2^{(\varphi)}\subset\gI_2^{(\varphi)}\).
\FOR{\(s=1,\ldots,T_{\varphi,1}\)}
\STATE Update \(\theta_1^{(\varphi)}\) by one gradient step on
Equation~\ref{Eq:Alternative_DFPCL_FirstStageLoss}, holding the proxy-target
featurizer fixed.
\STATE Update the persistent operator \(\hat\mV^{(\varphi)}\) by
Equation~\ref{eq:TreatmentBridgeFirstLayerProximalClosedFormSolution}.
\ENDFOR
\FOR{\(s=1,\ldots,T_{\varphi,2}\)}
\STATE Compute \(\check\mV^{(\varphi)}\) on
\(\gB_2^{(\varphi)}\) using
Equation~\ref{eq:TreatmentBridgeFirstLayerProximalClosedFormSolutionAux}.
\STATE For each ordered \(i\ne j\) in \(\gB_2^{(\varphi)}\), form the
matched and transported projected features
\[
\begin{aligned}
q_{ij}^{\mathrm{mat}}
&=\phi_{AA'X,2}^{(\varphi)}(a_i,a_j,x_i)
\otimes\{\check\mV^{(\varphi)}\}^{\top}
\phi_{AXW,1}^{(\varphi)}(a_i,x_i,w_i),\\
q_{ij}^{\mathrm{tr}}
&=\phi_{AA'X,2}^{(\varphi)}(a_i,a_j,x_j)
\otimes\{\check\mV^{(\varphi)}\}^{\top}
\phi_{AXW,1}^{(\varphi)}(a_i,x_j,w_j).
\end{aligned}
\]
\STATE Update \(\theta_2^{(\varphi)}\) by one gradient step on the fixed
ratio-free U-objective in
Equation~\ref{eq:ATTEmpiricalRatioFreeRisk}, differentiating through
\(\check\mV^{(\varphi)}\) for the proxy-side features.
\STATE Recompute the projected features and update \(\hat{\bm\varphi}\) by
the exact proximal closed-form minimizer of
\[
\widehat{\gR}_{\mathrm{ATT},2}^{\mathrm{RF,U}}(\bm\varphi)
+\lambda_2^{(\varphi)}
\|\bm\varphi-\hat{\bm\varphi}_{\mathrm{old}}\|_2^2,
\]
where \(\hat{\bm\varphi}_{\mathrm{old}}\) is the head before this update.
\ENDFOR
\UNTIL{convergence}
\STATE Set
\[
\hat\varphi^{\mathrm{ATT}}(a,a',x,z)
=\hat{\bm\varphi}^{\top}
[\phi_{AA'X,2}^{(\varphi)}(a,a',x)
\otimes\phi_{Z,2}^{(\varphi)}(z)].
\]
\STATE For all ordered \(i\ne j\) in \(\gI_3^{(\varphi)}\), form
\[
\gamma_{ij}^{\mathrm{ATT}}
=y_i\hat\varphi^{\mathrm{ATT}}(a_i,a_j,x_i,z_i).
\]
\STATE Fit \(k^{\varphi}(a_i,a_j;\theta_3^{\varphi})\) by the ordered-pair
MSE displayed above and set
\(\hat f_{\mathrm{ATT}}^{(\varphi)}(a,a')
=k^{\varphi}(a,a';\hat\theta_3^{\varphi})\).
\STATE \textbf{Return} \(\hat\varphi^{\mathrm{ATT}}\) and the full surface
\(\hat f_{\mathrm{ATT}}^{(\varphi)}\).
\end{algorithmic}
}
\caption{Treatment-conditional dose-response estimation with
density-ratio-free \textsc{TreatmentNet}.}
\label{algo:ConditionalDeepTreatmentBridgeFunction}
\end{algorithm}

\subsection{{Conditional dose-response estimation: doubly robust method}}

Recall the doubly robust functional
\[
\Psi_{\mathrm{ATT}}(a,a';h,\varphi)
=
\E[\varphi(a,a',X,Z)\{Y-h(a,X,W)\}\mid A=a]
+\E[h(a,X,W)\mid A=a'],
\]
which equals \(f_{\mathrm{ATT}}(a,a')\) if either \(h=h_0\) or
\(\varphi=\varphi_0^{\mathrm{ATT}}\), by
Theorem~\ref{theorem:DoublyRobustIdentifications}. Conditional on the fitted
bridges, let \(O_1\) and \(O_2\) be independent final-stage observations and
define
\[
\Gamma_{12}^{\mathrm{ATT},\mathrm{DR}}
:=
\hat\varphi^{\mathrm{ATT}}(A_1,A_2,X_1,Z_1)
\{Y_1-\hat h(A_1,X_1,W_1)\}
+\hat h(A_1,X_2,W_2).
\]
Independence gives
\[
\E[\Gamma_{12}^{\mathrm{ATT},\mathrm{DR}}\mid A_1=a,A_2=a']
=\Psi_{\mathrm{ATT}}(a,a';\hat h,\hat\varphi^{\mathrm{ATT}}).
\]
We therefore regress this ordered-pair pseudo-outcome on \((A_1,A_2)\), as
summarized in Algorithm~\ref{algo:ConditionalDeepDoublyRobustPCL}; the exact
remainder is analyzed in Appendix~\ref{app:dr_att_consistency}.
\begin{algorithm}[H]
{\footnotesize
\textbf{Input:} Observations
\(\mathcal D=\{(a_i,x_i,w_i,z_i,y_i)\}_{i=1}^n\), bridge-fitting folds,
and an independent final-stage fold \(\gI_3^{(\mathrm{DR})}\).\\
\textbf{Subroutines:} Algorithms~\ref{algo:DeepOutcomeBridgeFunction}
and~\ref{algo:ConditionalDeepTreatmentBridgeFunction}.\\
\textbf{Output:} Doubly robust treatment-conditional response surface
\(\hat f_{\mathrm{ATT}}^{\mathrm{DR}}\).
\begin{algorithmic}[1]
\STATE Fit the population outcome bridge \(\hat h\) and the joint-anchor
treatment bridge \(\hat\varphi^{\mathrm{ATT}}\).
\STATE For every ordered pair \(i\ne j\) in
\(\gI_3^{(\mathrm{DR})}\), form
\[
\gamma_{ij}^{\mathrm{ATT},\mathrm{DR}}
=
\hat\varphi^{\mathrm{ATT}}(a_i,a_j,x_i,z_i)
\{y_i-\hat h(a_i,x_i,w_i)\}
+\hat h(a_i,x_j,w_j).
\]
\STATE Fit a two-input network
\(k^{\mathrm{DR}}(\cdot,\cdot;\theta_3^{\mathrm{DR}})\) by minimizing
\[
\frac{1}{t_{\mathrm{DR}}(t_{\mathrm{DR}}-1)}
\sum_{i\ne j}
\|\gamma_{ij}^{\mathrm{ATT},\mathrm{DR}}
-k^{\mathrm{DR}}(a_i,a_j;\theta_3^{\mathrm{DR}})\|_2^2,
\qquad t_{\mathrm{DR}}:=|\gI_3^{(\mathrm{DR})}|.
\]
\STATE Set
\(\hat f_{\mathrm{ATT}}^{\mathrm{DR}}(a,a')
=k^{\mathrm{DR}}(a,a';\hat\theta_3^{\mathrm{DR}})\).
\STATE \textbf{Return} the full surface
\(\hat f_{\mathrm{ATT}}^{\mathrm{DR}}\).
\end{algorithmic}
}
\caption{Treatment-conditional dose-response estimation with
\textsc{DR-PCL-Net}.}
\label{algo:ConditionalDeepDoublyRobustPCL}
\end{algorithm}

\section{{Rademacher-complexity-based error bounds for the proposed estimators}}
\label{app:RademacherBounds}
{We derive Rademacher-complexity-based error bounds for the projected bridges and the induced causal functions.} We assume independent splits of the dataset for each regression stage. The analysis allows stage-specific approximate empirical minimization. 
\begin{assumption}
    \label{assum:suboptimal_error}
    For every empirical-risk stage \(j\) considered below, let \(\widehat{\gR}_j\) denote its normalized empirical criterion and \(\mathcal V_j\) its candidate class. We assume that the fitted estimator \(\hat\vartheta_j\) is measurable and satisfies
\[
\widehat{\gR}_j(\hat\vartheta_j)
\le
\inf_{\vartheta\in\mathcal V_j}\widehat{\gR}_j(\vartheta)
+\xi_j,
\qquad \xi_j\ge 0,
\]
where the possibly data-dependent gap \(\xi_j\) is measured on the same normalized criterion scale as \(\widehat{\gR}_j\).
\end{assumption}
 Exact empirical minimization is the special case \(\xi_j=0\). Note that this assumption is standard in the analysis of empirical risk minimization estimators \citep[see section 3.2]{vanderVaart1996weak}. For the profiled bridge first stages, the corresponding inequality is required uniformly over the candidate proxy-feature map. Because the projected errors below are squared excess risks, the gaps enter their bounds additively. We do not derive algorithm-specific rates for the implemented nonconvex optimizer; the guarantees apply to any fitted estimators satisfying the displayed empirical suboptimality conditions.

\subsection{Background on Rademacher complexity analysis}
\label{Appendix:rademacherBackground}
{The resulting error bounds depend on the complexity of the function classes used at each regression stage.} Therefore, we first provide a review of Rademacher complexity analysis here.

Let \(\mathcal O\) be a
measurable space, let \(\mathcal F\) be a class of measurable functions
\(f:\mathcal O\to[0,B_{\mathcal F}]\), and let
\(\bm O_N=(O_1,\ldots,O_N)\). Write
\[
\sP f:=\E\{f(O)\},
\qquad
\sP_N f:=\frac{1}{N}\sum_{i=1}^{N}f(O_i).
\]

\begin{definition}[Empirical Rademacher complexity]
\label{def:empiricalRademacher}
The empirical Rademacher complexity of \(\mathcal F\) is
\[
\widehat{\mathfrak R}_{\bm O_N}(\mathcal F)
:=
\E_{\sigma}\!\left[
\sup_{f\in\mathcal F}
\frac{1}{N}\sum_{i=1}^{N}\sigma_i f(O_i)
\,\middle|\,\bm O_N
\right],
\]
where \(\sigma_1,\ldots,\sigma_N\) are independent Rademacher variables.
\end{definition}

\begin{proposition}[Uniform Rademacher deviation; \citet{Mohri_Foundations_of_ML}]
\label{prop:uniformRademacherDeviation}
If \(O_1,\ldots,O_N\overset{\mathrm{iid}}{\sim}\sP\), then, for every
\(\delta\in(0,1)\), each of the following statements holds with probability
at least \(1-\delta\), uniformly over \(f\in\mathcal F\):
\[
\sP f
\le
\sP_N f
+
2\widehat{\mathfrak R}_{\bm O_N}(\mathcal F)
+
3B_{\mathcal F}
\sqrt{\frac{\log(2/\delta)}{2N}}.
\]
\[
\sP_N f
\le
\sP f
+
2\widehat{\mathfrak R}_{\bm O_N}(\mathcal F)
+
3B_{\mathcal F}
\sqrt{\frac{\log(2/\delta)}{2N}}.
\]
\end{proposition}

We also need the corresponding bound for an order-two U-process. Let
\(O_1,\ldots,O_m\overset{\mathrm{iid}}{\sim}\sP\), write
\(\bm O_m=(O_1,\ldots,O_m)\), let \(O,O'\) be independent draws from \(\sP\),
and let \(\mathcal L\) be a class of measurable kernels
\(\ell:\mathcal O^2\to\mathbb R\). Define
\[
\sP^{\otimes 2}\ell:=\E\{\ell(O,O')\},
\qquad
\mathbb U_m\ell
:=
\frac{1}{m(m-1)}\sum_{i\ne j}\ell(O_i,O_j).
\]
The kernels need not be symmetric. Indeed, with
\(\ell^{\mathrm{sym}}(o,o'):=\{\ell(o,o')+\ell(o',o)\}/2\),
\[
\mathbb U_m\ell
=
\binom{m}{2}^{-1}\sum_{i<j}\ell^{\mathrm{sym}}(O_i,O_j),
\qquad
\sP^{\otimes 2}\ell^{\mathrm{sym}}=\sP^{\otimes 2}\ell.
\]

Set \(k_m:=\lfloor m/2\rfloor\). Let \(\Pi\) be a uniformly random
permutation of \(\{1,\ldots,m\}\), independent of the sample, and let
\(\sigma_1,\ldots,\sigma_{k_m}\) be independent Rademacher variables.

\begin{definition}[Pairwise empirical Rademacher complexity]
\label{def:pairwiseRademacher}
The random-pairing complexity of \(\mathcal L\) is
\[
\widehat{\mathfrak R}^{(2),U}_{\bm O_m}(\mathcal L)
:=
\E_{\Pi,\sigma}\!\left[
\sup_{\ell\in\mathcal L}
\frac{1}{k_m}\sum_{r=1}^{k_m}
\sigma_r\ell\!\left(O_{\Pi(2r-1)},O_{\Pi(2r)}\right)
\,\middle|\,\bm O_m
\right].
\]
\end{definition}

The random-pairing reduction is classical; see \citet[Appendix~A, Lemma~A.1]{clemencon2008ranking}. 

\begin{proposition}[Uniform U-process Rademacher deviation]
\label{prop:uniformUProcessDeviation}
Suppose \(m\ge 2\), the relevant suprema are measurable, and
\[
M_{\mathcal L}
:=
\sup_{\ell\in\mathcal L}\sup_{o,o'\in\mathcal O}
|\ell(o,o')|
<\infty.
\]
For every \(\delta\in(0,1)\), each of the following statements holds with
probability at least \(1-\delta\), uniformly over \(\ell\in\mathcal L\):
\[
\sP^{\otimes 2}\ell
\le
\mathbb U_m\ell
+2\widehat{\mathfrak R}^{(2),U}_{\bm O_m}(\mathcal L)
+3M_{\mathcal L}\sqrt{\frac{8\log(2/\delta)}{m}},
\]
\[
\mathbb U_m\ell
\le
\sP^{\otimes 2}\ell
+2\widehat{\mathfrak R}^{(2),U}_{\bm O_m}(\mathcal L)
+3M_{\mathcal L}\sqrt{\frac{8\log(2/\delta)}{m}}.
\]
\end{proposition}

\begin{proof}
For a permutation \(\Pi\), define the disjoint-pair average
\[
V_{m,\Pi}\ell
:=
\frac{1}{k_m}\sum_{r=1}^{k_m}
\ell\!\left(O_{\Pi(2r-1)},O_{\Pi(2r)}\right).
\]
Each ordered pair has the same probability with each paired position.
Hence
\(
\E_{\Pi}[V_{m,\Pi}\ell\mid\bm O_m]=\mathbb U_m\ell
\).
For
\[
Z_+
:=
\sup_{\ell\in\mathcal L}
\{\mathbb U_m\ell-\sP^{\otimes 2}\ell\},
\qquad
Z_-
:=
\sup_{\ell\in\mathcal L}
\{\sP^{\otimes 2}\ell-\mathbb U_m\ell\},
\]
Jensen's inequality and ordinary symmetrization of the \(k_m\) independent
pairs give, for either sign,
\[
\E Z_\pm
\le
2\E\!\left[
\widehat{\mathfrak R}^{(2),U}_{\bm O_m}(\mathcal L)
\right].
\]
For \(Z_-\), apply the same symmetrization to \(-\mathcal L\); symmetry of
the Rademacher signs gives
\(
\widehat{\mathfrak R}^{(2),U}_{\bm O_m}(-\mathcal L)
=
\widehat{\mathfrak R}^{(2),U}_{\bm O_m}(\mathcal L)
\).
Replacing one observation changes \(Z_+\) or \(Z_-\) by at most
\(4M_{\mathcal L}/m\). It changes the pairwise empirical complexity by at
most the same amount: a fixed matching changes by at most
\(2M_{\mathcal L}/k_m\), and the replaced observation is included with
probability \(2k_m/m\). Therefore, by McDiarmid's inequality, for a fixed
choice of sign, with probability at least \(1-\delta\),
\[
Z_\pm
\le
\E Z_\pm+t_\delta,
\qquad
\E\!\left[
\widehat{\mathfrak R}^{(2),U}_{\bm O_m}(\mathcal L)
\right]
\le
\widehat{\mathfrak R}^{(2),U}_{\bm O_m}(\mathcal L)+t_\delta,
\]
simultaneously, where
\(
t_\delta=M_{\mathcal L}\sqrt{8\log(2/\delta)/m}
\).
Thus \(Z_\pm\le
2\widehat{\mathfrak R}^{(2),U}_{\bm O_m}(\mathcal L)+3t_\delta\).
The two choices of sign give the two stated inequalities separately.
\end{proof}

Although \(\mathbb U_m\ell\) averages
\(m(m-1)\) dependent terms, its effective sample size in this first-order
bound is \(k_m=\lfloor m/2\rfloor\), which yields the usual
\(m^{-1/2}\) scale.

\subsection{Outcome bridge consistency for dose-response estimation} \label{app:outcome_bridge_consistency}

We establish projected consistency of {\textsc{OutcomeNet}} for
\(Y\in\mathbb R^{d_Y}\). Let
\[
Q_h:=(A,X,Z) \subset \mathcal Q_h,
\qquad
m_0(q_h):=\E[Y\mid Q_h=q_h]\in\mathbb R^{d_Y}.
\]
For \(u:\mathcal A\times\mathcal X\times\mathcal W\to\mathbb R^{d_Y}\),
define
\[
(T_hu)(a,x,z)
:=
\E[u(a,x,W)\mid A=a,X=x,Z=z].
\]
Its projected seminorm is
\[
\|u\|_{T_h}
:=
\|T_hu\|_{L^2(\sP_{Q_h};\mathbb R^{d_Y})}.
\]

Let the independent stage samples be
\(
D_1^{(h)}=\{(\bar q_{h,i},\bar w_i)\}_{i \in \gI_1^{(h)}},
\enspace
D_2^{(h)}=\{(\tilde y_i,\tilde q_{h,i})\}_{i \in \gI_2^{(h)}},
\) with $|\gI_1^{(h)}| = n_h$ and $|\gI_2^{(h)}| = m_h$. The first sample is i.i.d. from \(\sP_{Q_h,W}\); the second is i.i.d. from
\(\sP_{Y,Q_h}\).

We consider a learned-feature bridge class. Let $\nu_h = (n_h, m_h)$ and define 
\[
\Phi_{\nu_h}
\subset\{\phi:\mathcal W\to\mathbb R^{d_{\nu_h}}\},
\qquad
{\mathcal M_{\nu_h}}
\subset\{{\eta}:\mathcal Q_h\to\mathbb R^{d_{\nu_h}}\},
\]
for the class of learned outcome-proxy feature maps and the first-stage conditional mean regressors, respectively. We represent the multi-output head by a matrix
\[
\Theta_{\nu_h}
\subset
\{\theta:\mathcal A\times\mathcal X
\to\mathbb R^{d_Y\times d_{\nu_h}}\},
\]
and the corresponding bridge class is
\(
\mathcal H_{\nu_h}
:=
\left\{
h_{\theta,\phi}(a,x,w)=\theta(a,x)\phi(w)
:
\theta\in\Theta_{\nu_h},\ \phi\in\Phi_{\nu_h}
\right\}.
\)
The usual tensor-product architecture is subsumed in this notation: its final later is absorbed into $\theta(a, x)$.

For each \(\phi\in\Phi_{\nu_h}\), define the population conditional feature mean
\(
\mu_{\phi}^{(h)}(q_h)
:=
\E[\phi(W)\mid Q_h=q_h].
\)
Then, for \(q_h=(a,x,z)\),
\(
(T_hh_{\theta,\phi})(q_h)
=
\langle \theta(a,x), \mu_{\phi}^{(h)}(q_h) \rangle.
\)

The population and empirical first-stage risks are
\begin{align*}
\gR_{h,1}(\phi,{\eta})
&:=
\E\|\phi(W)-{\eta}(Q_h)\|_2^2,\\
\widehat{\gR}_{h,1}(\phi,{\eta})
&:=
\frac1{n_h}\sum_{i=1}^{n_h}
\|\phi(\bar w_i)-{\eta}(\bar q_{h,i})\|_2^2.
\end{align*}

For every \(\phi\), let \({\widehat\mu_{\phi}^{(h)}}\in {\mathcal M_{\nu_h}}\) be a measurable
approximate empirical minimizer, that is,
\[
\widehat{\gR}_{h,1}(\phi, {\widehat\mu_{\phi}^{(h)}})
\le
\inf_{{\eta}\in {\mathcal M_{\nu_h}}}
\widehat{\gR}_{h,1}(\phi, {\eta})
+\xi_{h,1,\nu_h},
\qquad \xi_{h,1,\nu_h}\ge 0,
\]
uniformly over \(\phi\in\Phi_{\nu_h}\); the exact-minimizer case is \(\xi_{h,1,\nu_h}=0\).

Define the population and plug-in
projected predictions
\begin{align*}
m_{\theta,\phi}^{(h)}(q_h)
&:=
\langle \theta(a,x), \mu_{\phi}^{(h)}(q_h) \rangle,\\
\widehat m_{\theta,\phi}^{(h)}(q_h)
&:=
\langle \theta(a,x), {\widehat\mu_{\phi}^{(h)}}(q_h) \rangle.
\end{align*}

The profiled second-stage risk is
\[
\widehat{\gR}_{h,2}(\theta,\phi)
:=
\frac1{m_h}\sum_{i \in \gI_2^{h}}
\left\|
 y_i-\widehat m_{\theta,\phi}^{(h)}( q_{h,i})
\right\|_2^2.
\]

Let \((\hat\theta^{(h)},\hat\phi^{(h)})\in\Theta_{\nu_h}\times\Phi_{\nu_h}\) be a measurable
approximate empirical minimizer of the second-stage risk,
\[
\widehat{\gR}_{h,2}(\hat\theta^{(h)},\hat\phi^{(h)})
\le
\inf_{\theta\in\Theta_{\nu_h},\,\phi\in\Phi_{\nu_h}}
\widehat{\gR}_{h,2}(\theta,\phi)
+\xi_{h,2,\nu_h},
\qquad \xi_{h,2,\nu_h}\ge 0,
\qquad
\hat h:=h_{\hat\theta^{(h)},\hat\phi^{(h)}}.
\]

\begin{assumption}[Outcome-side boundedness]
\label{ass:outcome_boundedness_rademacher}
There are finite constants \(B_Y,B_\phi,{B_\mu},B_\theta\), uniform in
\(\nu_h\), such that
\[
\|Y\|_2\le B_Y\quad\text{a.s.},
\qquad
\sup_{\phi\in\Phi_{\nu_h}}\sup_{w\in\mathcal W}
\|\phi(w)\|_2\le B_\phi,
\qquad
\sup_{ {\eta}\in {\mathcal M_{\nu_h}}}\sup_{q_h\in\mathcal Q_h}
\| {\eta}(q_h)\|_2\le {B_\mu},
\]
and
\[
\sup_{\theta\in\Theta_{\nu_h}}
\sup_{(a,x)\in\mathcal A\times\mathcal X}
\|\theta(a,x)\|_{\mathrm{op}}
\le B_\theta.
\]
Set
\[
M_{h,1}:=(B_\phi+ {B_\mu})^2,
\qquad
M_{h,2}:=(B_Y+B_\theta {B_\mu})^2.
\]
\end{assumption}

Define the approximation errors
\begin{align*}
\kappa_{h,1,\nu_h}
&:=
\sup_{\phi\in\Phi_{\nu_h}}
\inf_{{\eta}\in {\mathcal M_{\nu_h}}}
\| {\eta}-\mu_{\phi}^{(h)}\|_{L^2(\sP_{Q_h};\mathbb R^{d_{\nu_h}})}^2,\\
 {\kappa_{h,2,\nu_h}}
&:=
\inf_{\theta\in\Theta_{\nu_h},\,\phi\in\Phi_{\nu_h}}
\|m_0-m_{\theta,\phi}^{(h)}\|_{L^2(\sP_{Q_h};\mathbb R^{d_Y})}^2.
\end{align*}

The two scalar loss classes are
\begin{align*}
\mathcal L_{h,1,\nu_h}
&:=
\left\{
(q_h,w)\mapsto\|\phi(w)- {\eta}(q_h)\|_2^2
:
\phi\in\Phi_{\nu_h},\  {\eta}\in {\mathcal M_{\nu_h}}
\right\},\\
\mathcal L_{h,2,\nu_h}
&:=
\left\{
(y,q_h)\mapsto\|y- \langle \theta(a,x), {\eta}(q_h) \rangle\|_2^2
:
\theta\in\Theta_{\nu_h},\ {\eta}\in {\mathcal M_{\nu_h}}
\right\}.
\end{align*}
Their envelopes are \(M_{h,1}\) and \(M_{h,2}\), respectively.

To analyze the first- and second-stage errors of outcome bridge network, we will apply Proposition~\ref{prop:uniformRademacherDeviation}. To this end, define 

\begin{align*}
\Delta_{h,1,\nu_h}(\delta)
&:=
2\widehat{\mathfrak R}_{D_1^{(h)}}(\mathcal L_{h,1,\nu_h})
+3M_{h,1}\sqrt{\frac{\log(2/\delta)}{2n_h}},\\
\Delta_{h,2,\nu_h}(\delta)
&:=
2\widehat{\mathfrak R}_{D_2^{(h)}}(\mathcal L_{h,2,\nu_h})
+3M_{h,2}\sqrt{\frac{\log(2/\delta)}{2m_h}}.
\end{align*}

\begin{lemma}[Uniform {\textsc{OutcomeNet}} first-stage error]
\label{lem:outcome_stage1_rademacher}
Let
\[
\epsilon_{h,1,\nu_h}
:=
\sup_{\phi\in\Phi_{\nu_h}}
\| {\widehat\mu_{\phi}^{(h)}}-\mu_{\phi}^{(h)}\|_{L^2(\sP_{Q_h};\mathbb R^{d_{\nu_h}})}^2.
\]
With probability at least \(1-2\delta\),
\[
\epsilon_{h,1,\nu_h}
\le
\kappa_{h,1,\nu_h}
+2\Delta_{h,1,\nu_h}(\delta)
{{}+\xi_{h,1,\nu_h}}.
\]
\end{lemma}

\begin{proof}
For fixed \(\phi\), the vector-valued projection identity gives
\[
\gR_{h,1}(\phi, {\eta})
-\gR_{h,1}(\phi,\mu_{\phi}^{(h)})
=
\| {\eta}-\mu_{\phi}^{(h)}\|_{L^2(\sP_{Q_h};\mathbb R^{d_{\nu_h}})}^2.
\]
Indeed, 
\begin{align*}
    \gR_{h, 1}(\phi,  {\eta}) & = \E \|\phi(W) -  {\mu_\phi^{(h)}}(Q) +  {\mu_\phi^{(h)}}(Q) -  {\eta}(Q)\|_2^2\\
    &= \gR_{h, 1}(\phi,  {\eta}) + \| {\eta} -  {\mu_\phi^{(h)}}\|^2_{L^2(P_Q; \R^{d_{\nu_h}})} + 2 \E \langle \phi(W) -  {\mu_\phi^{(h)}}(Q),  {\mu_\phi^{(h)}}(Q) -  {\eta}(Q)\rangle
\end{align*}
The cross term vanishes after conditioning on \(Q_h\). Define
\[
\Delta_{1,h, \nu_h}
:=
\sup_{\phi\in\Phi_{\nu_h},\, {\eta}\in {\mathcal M_{\nu_h}}}
\left|
\gR_{h,1}(\phi, {\eta})-\widehat \gR_{h,1}(\phi, {\eta})
\right|.
\]

By approximate empirical risk minimization, \(\widehat{\gR}_{h, 1}(\phi, {\widehat\mu_\phi^{(h)}}) \le \widehat{\gR}_{h, 1}(\phi, {\eta})+\xi_{h,1,\nu_h}, \quad \forall {\eta} \in {\mathcal M_{\nu_h}}.\)

Therefore, for every \({\eta}\in {\mathcal M_{\nu_h}}\),
\[
\begin{aligned}
\gR_{h,1}(\phi, {\widehat\mu_\phi^{(h)}})-\gR_{h,1}(\phi, {\mu_\phi^{(h)}})
&\le
\widehat \gR_{h,1}(\phi, {\widehat\mu_\phi^{(h)}})-\gR_{h,1}(\phi, {\mu_\phi^{(h)}})
+\Delta_{1,\nu_h} \\
&\le
\widehat \gR_{h,1}(\phi, {\eta})- \gR_{h,1}(\phi, {\mu_\phi^{(h)}})
+\Delta_{1,\nu_h}
+\xi_{h,1,\nu_h} \\
&\le
\gR_{h,1}(\phi, {\eta})- \gR_{h,1}(\phi, {\mu_\phi^{(h)}})
+2\Delta_{1,\nu_h}
+\xi_{h,1,\nu_h}.
\end{aligned}
\]
Using the projection identity,
\(
\| {\widehat\mu_\phi^{(h)}}- {\mu_\phi^{(h)}}\|_{L^2(\sP_Q; \R^{d_{\nu_h}})}^2
\le
\| {\eta}- {\mu_\phi^{(h)}}\|_{L^2(\sP_Q; \R^{d_{\nu_h}})}^2
+
2\Delta_{1,\nu_h}
+\xi_{h,1,\nu_h}.
\)
Taking the infimum over \( {\eta}\in {\mathcal M_{\nu_h}}\) and then the supremum over
\(\phi\in\Phi_{\nu_h}\) yields
\[
\epsilon_{1, h, \nu_h}
\le
\kappa_{1,h, \nu_h}
+
2\Delta_{1, h, \nu_h}
+\xi_{h,1,\nu_h}.
\]
The gap enters once, additively, because the optimization inequality is applied a single
time in the comparison above and the criterion scale is that of the squared projected error.

Finally, applying Proposition~\ref{prop:uniformRademacherDeviation} to the bounded loss
class \(\mathcal L_{1,h, \nu_h}\) gives
\[
\Delta_{1,h, \nu_h}
\le
\Delta_{1,h, \nu_h}(\delta)
\]
with probability at least \(1- 2\delta\).

\end{proof}

\begin{theorem}[Projected {\textsc{OutcomeNet}} error bound]
\label{thm:outcome_projected_rademacher}
Under Assumption~\ref{ass:outcome_boundedness_rademacher}, let both stages use {measurable approximate empirical minimizers with gaps \(\xi_{h,1,\nu_h},\xi_{h,2,\nu_h}\ge 0\)}. For every \(\delta\in(0,1/4)\), with
probability at least \(1-4\delta\),
\begin{align*}
\|T_h\hat h-m_0\|_{L^2(\sP_{Q_h};\mathbb R^{d_Y})}^2
\le{}&
4 {\kappa_{h,2,\nu_h}}
+6B_\theta^2\kappa_{h,1,\nu_h} +12B_\theta^2\Delta_{h,1,\nu_h}(\delta)
+4\Delta_{h,2,\nu_h}(\delta)\\
&{{}+6B_\theta^2\xi_{h,1,\nu_h}
+2\xi_{h,2,\nu_h}}.
\end{align*}
{Setting \(\xi_{h,1,\nu_h}=\xi_{h,2,\nu_h}=0\) recovers the exact-minimization bound.}
If a square-integrable outcome bridge \(h_0\) satisfies \(T_hh_0=m_0\), the
same bound controls
\[
\|\hat h-h_0\|_{T_h}^2
=
\|T_h\hat h-m_0\|_{L^2(\sP_{Q_h};\mathbb R^{d_Y})}^2.
\]
\end{theorem}

\begin{proof}
Conditional on \(D_1^{(h)}\), every \({\widehat\mu_{\phi}^{(h)}}\) is fixed and
belongs to \({\mathcal M_{\nu_h}}\). Hence the conditional plug-in loss class
is a subclass of \(\mathcal L_{h,2,\nu_h}\). Define
\[
\gR_{h,2}^{\mathrm{plug}}(\theta,\phi)
:=
\E\|Y-\widehat m_{\theta,\phi}^{(h)}(Q_h)\|_2^2.
\]
This expectation is over an independent second-stage draw \((Y,Q_h)\),
conditional on \(D_1^{(h)}\).
Since \(m_0(Q_h)=\E[Y\mid Q_h]\),
\[
\gR_{h,2}^{\mathrm{plug}}(\theta,\phi)
=
\E\|Y-m_0(Q_h)\|_2^2
+\|\widehat m_{\theta,\phi}^{(h)}-m_0\|_{L^2(\sP_{Q_h};\mathbb R^{d_Y})}^2.
\]

Let
\[
\begin{aligned}
\Delta_{2,h, \nu_h}
:={}&
\sup_{\theta\in\Theta_{\nu_h},\, {\eta}\in {\mathcal M_{\nu_h}}}
\Bigg|
\mathbb E\left[
\{Y-\langle\theta(A,X), {\eta}(Q_h)\rangle\}^2
\right]
\\[-0.2em]
&\quad-
\frac1{m_h}
\sum_{i \in \gI_2^{(h)}}
\{y_i-\langle\theta(a_i,x_i),{\eta}(q_{h,i})\rangle\}^2
\Bigg|.
\end{aligned}
\]
Because \( {\widehat\mu_\phi^{(h)}}\in {\mathcal M_{\nu_h}}\), this uniform deviation controls the
plug-in class:
\[
\sup_{\theta,\phi}
\left|
\gR_{h,2}^{\mathrm{plug}}(\theta,\phi)
-
\widehat \gR_{h,2}^{\mathrm{plug}}(\theta,\phi)
\right|
\le
\Delta_{2,\nu_h}.
\]

By approximate empirical risk minimization,
\(
\widehat \gR_{h,2}^{\mathrm{plug}}(\hat\theta,\hat\phi)
\le
\widehat \gR_{h,2}^{\mathrm{plug}}(\theta,\phi)
+\xi_{h,2,\nu_h}
\enspace
\forall(\theta,\phi)\in\Theta_{\nu_h}\times\Phi_{\nu_h}.
\)
Since the plug-in excess risk equals the squared projected error, adding and subtracting the
uniform deviation twice and applying this inequality once gives
\begin{align}
\|\widehat m_{\hat\theta^{(h)},\hat\phi^{(h)}}^{(h)}-m_0\|_{L^2}^2
\le
\inf_{\theta,\phi}
\|\widehat m_{\theta,\phi}^{(h)}-m_0\|_{L^2}^2
+2\Delta_{h,2,\nu_h}
+\xi_{h,2,\nu_h}
\label{eq:outcome_stage2_plugin_erm}
\end{align}

The operator-norm bound gives, uniformly in \((\theta,\phi)\),
\begin{align*}
\|\widehat m_{\theta,\phi}^{(h)}-m_{\theta,\phi}^{(h)}\|_{L^2}^2
&=
\E\left\|
\theta(A,X)\{ {\widehat\mu_{\phi}^{(h)}}(Q_h)-\mu_{\phi}^{(h)}(Q_h)\}
\right\|_2^2 \le
B_\theta^2\epsilon_{h,1,\nu_h}.
\end{align*}
Thus
\[
\inf_{\theta,\phi}
\|\widehat m_{\theta,\phi}^{(h)}-m_0\|_{L^2}^2
\le
2 {\kappa_{h,2,\nu_h}}
+2B_\theta^2\epsilon_{h,1,\nu_h}.
\]

Combining this inequality with the preceding ERM bound gives
\[
\|\widehat m_{\hat\theta^{(h)},\hat\phi^{(h)}}^{(h)}-m_0\|_{L^2}^2
\le
2 {\kappa_{h,2,\nu_h}}
+2B_\theta^2\epsilon_{h,1,\nu_h}
+2\Delta_{h,2,\nu_h}
+\xi_{h,2,\nu_h}.
\]

Finally,
\(T_h\hat h=m_{\hat\theta^{(h)},\hat\phi^{(h)}}^{(h)}\). Hence, using
\(\|u+v\|^2\le 2\|u\|^2+2\|v\|^2\), which doubles every term of the previous display and
contributes a further \(2B_\theta^2\epsilon_{h,1,\nu_h}\),
\begin{align*}
\|T_h\hat h-m_0\|_{L^2}^2
&\le
2\|\widehat m_{\hat\theta^{(h)},\hat\phi^{(h)}}^{(h)}-m_0\|_{L^2}^2
+2B_\theta^2\epsilon_{h,1,\nu_h}\\
&\le
4 {\kappa_{h,2,\nu_h}}
+6B_\theta^2\epsilon_{h,1,\nu_h}
+4\Delta_{h,2,\nu_h}
+2\xi_{h,2,\nu_h}.
\end{align*}
By Lemma~\ref{lem:outcome_stage1_rademacher}, with probability at least $1 - 2 \delta$
\[
\epsilon_{h,1,\nu_h} \le \kappa_{1, h, \nu_h} + 2 \Delta_{1, h, \nu_h}(\delta)
+\xi_{h,1,\nu_h},
\]
and substituting this into the coefficient \(6B_\theta^2\epsilon_{h,1,\nu_h}\) produces the
term \(6B_\theta^2\xi_{h,1,\nu_h}\).

By Proposition~\ref{prop:uniformRademacherDeviation} to the bounded loss
class, with probability $1 - 2\delta$, \(\mathcal L_{2,h, \nu_h}\) gives \(\Delta_{2, h, \nu_h} \le \Delta_{2, h, \nu_h}(\delta).\) Hence, the union bond gives the stated result with probability at least $1 - 4\delta$.

Finally, if $h_0$ is a valid outcome bridge, then $T_h h_0 = m_0$. Therefore,
\[
\|\hat h - h_0\|_{T_h}^2 = \|T_h (\hat h - h_0)\|^2_{L^2(\sP_{Q_h}; \R^{d_{\nu_h}})} = \|T_h \hat h - m_0\|^2_{L^2(\sP_{Q_h}; \R^{d_{\nu_h}})}
\]
\end{proof}

\begin{corollary}[Projected \textsc{OutcomeNet} consistency]
\label{cor:outcomenet_projected_consistency}
For \(\delta\in(0,1/4)\), define
\[
\rho_{h,\nu_h}(\delta)
:=
4\kappa_{h,2,\nu_h}
+6B_\theta^2\kappa_{h,1,\nu_h}
+12B_\theta^2\Delta_{h,1,\nu_h}(\delta)
+4\Delta_{h,2,\nu_h}(\delta).
\]

Define the optimization-aware projected rate
\[
\tilde\rho_{h,\nu_h}(\delta)
:=
\rho_{h,\nu_h}(\delta)
+6B_\theta^2\xi_{h,1,\nu_h}
+2\xi_{h,2,\nu_h}.
\]
Then, with probability at least \(1-4\delta\),
\[
\|T_h\hat h-m_0\|_{L^2(\sP_{Q_h};\mathbb R^{d_Y})}^2
\le \tilde\rho_{h,\nu_h}(\delta).
\]
Let \(\delta_{h,\nu_h}\downarrow0\) satisfy
\(\log(1/\delta_{h,\nu_h})/\min\{n_h,m_h\}\to0\), and abbreviate
\(\rho_{h,\nu_h}:=\rho_{h,\nu_h}(\delta_{h,\nu_h})\) and
\(\tilde\rho_{h,\nu_h}:=\tilde\rho_{h,\nu_h}(\delta_{h,\nu_h})\).
If
\[
\kappa_{h,1,\nu_h}\to0,\qquad
\kappa_{h,2,\nu_h}\to0,\qquad
\widehat{\mathfrak R}_{D_j^{(h)}}(\mathcal L_{h,j,\nu_h})
\xrightarrow{p}0,\quad j\in\{1,2\},
\]
\[
\xi_{h,1,\nu_h}\xrightarrow{p}0,
\qquad
\xi_{h,2,\nu_h}\xrightarrow{p}0,
\]
as \(n_h,m_h\to\infty\), then
\(\tilde\rho_{h,\nu_h}\xrightarrow{p}0\) and
\[
\|T_h\hat h-m_0\|_{L^2(\sP_{Q_h};\mathbb R^{d_Y})}^2
=O_p(\tilde\rho_{h,\nu_h}),
\qquad
\|T_h\hat h-m_0\|_{L^2(\sP_{Q_h};\mathbb R^{d_Y})}^2
\xrightarrow{p}0.
\]

If \(T_hh_0=m_0\), then \(\|\hat h-h_0\|_{T_h}^2\to0\) in probability.
\end{corollary}

\begin{remark}[Projected rather than strong bridge recovery]
The theorem controls the conditional-moment error \(T_h\hat h-m_0\). It does
not imply
\[
\|\hat h-h_0\|_{L^2(\sP_{A,X,W};\mathbb R^{d_Y})}\to0.
\]
Indeed, since $h_0$ is the solution to an ill-posed inverse problem, obtaining strong norm convergence for $\hat{h}$ requires controlling sieve measure of ill-posedness, which generally diverges as sieve dimension grows \citep{blundell2007semi, chen2011rate, chen2018optimal, meunier2024nonparametric,shen2025nonparametricinstrumentalvariableregression, ICLR2025_ead8e195}.  
In the severely ill-posed case, strong-norm convergence can be arbitrarily slow. The result below avoids this difficulty for the outcome-bridge plug-in estimator: using the dual treatment bridge, the dose-response functional is well-posed with respect to the weak residual \(T_h\hat h-m_0\), so projected convergence is sufficient for consistency \citep{deaner2023proxycontrolspaneldata}. This weak-norm control is then used later in the doubly robust analysis.
\end{remark}

\begin{remark}[Outcome dimension]
For fixed \(d_Y\), no coordinatewise union bound is needed. If \(d_Y\) grows,
the dependence of \(B_Y\), \(B_\theta\), and
\(\widehat{\mathfrak R}(\mathcal L_{h,2,\nu_h})\) on \(d_Y\) must be tracked.
\end{remark}

\subsection{From projected outcome-bridge consistency to dose-response consistency}

We now transfer the projected conditional-moment rate for \(\hat h\) to the
dose-response curve. This is the ATE analogue of the well-posedness
inequality in \citet[Theorem~4.1(i)]{deaner2023proxycontrolspaneldata}.
Corollary~\ref{cor:outcomenet_projected_consistency} gives
\[
\|T_h\hat h-m_0\|_{L^2(\sP_{Q_h};\mathbb R^{d_Y})}^2
=
O_p({\tilde\rho_{h,\nu_h}}).
\]
The result below shows that the dose-response functional is continuous in
this projected seminorm; strong recovery of \(h_0\) is not required.

For a candidate bridge \(h\), define the population plug-in dose-response curve
\[
\bar f_h(a)
:=
\E[h(a,X,W)]\in\mathbb R^{d_Y}.
\]

Our goal is to control 
\(
\|\bar{f}_{\hat h} - f_{\mathrm{ATE}}\|^2_{L^2(\sP_A)}
\)
Assume the treatment-bridge identification route in {Theorem~\ref{theorem:TreatmentBridgeIdentificationAllCausalFunctions}}. In particular, suppose
an ATE treatment bridge \( {\varphi_0^{\mathrm{ATE}}}\) exists and satisfies
\[
\E[ {\varphi_0^{\mathrm{ATE}}}(a,X,Z)\mid A=a,X,W]
= {r_{0,\mathrm{ATE}}}(a,X,W).
\]

This equality holds for \(\sP_A\)-almost every \(a\) and
\(\sP_{X,W\mid A=a}\)-almost every \((x,w)\).
For \(\sP_A\)-almost every \(a\), set
\[
c_{\varphi,\mathrm{ATE}}(a)
:=
\left(
\E[\{ {\varphi_0^{\mathrm{ATE}}}(a,X,Z)\}^2\mid A=a]
\right)^{1/2},
\qquad
C_{\varphi,\mathrm{ATE}}
:=
\operatorname*{ess\,sup}_{a\sim\sP_A}
c_{\varphi,\mathrm{ATE}}(a),
\]
and assume \(C_{\varphi,\mathrm{ATE}}<\infty\).

\begin{lemma}[Projected residual controls dose-response error]
\label{lem:outcome_projected_to_dose_response}
For any \(h\in\mathcal H_{\nu_h}\) and \(\sP_A\)-almost every \(a\),
\[
\left\|
f_{\mathrm{ATE}}(a)-\bar f_h(a)
\right\|_2
\le
c_{\varphi,\mathrm{ATE}}(a)
\left\|
\E[Y-h(a,X,W)\mid A=a,X,Z]
\right\|_{L^2(\sP_{X,Z\mid A=a};\mathbb R^{d_Y})}.
\]
\end{lemma}

\begin{proof}
By treatment-bridge identification,
\[
f_{\mathrm{ATE}}(a)
=
\mathbb E[
Y {\varphi_0^{\mathrm{ATE}}}(a,X,Z)
\mid A=a
].
\]

By iterated expectations and the treatment-bridge equation,
\[
\begin{aligned}
\mathbb E[
h(a,X,W) {\varphi_0^{\mathrm{ATE}}}(a,X,Z)
\mid A=a
]
&=
\E[h(a,X,W) {r_{0,\mathrm{ATE}}}(a,X,W)\mid A=a]\\
&=\E[h(a,X,W)]
=\bar f_h(a).
\end{aligned}
\]
The second equality is a change of measure: disintegration of
\( {r_{0,\mathrm{ATE}}}\) gives
\( {r_{0,\mathrm{ATE}}}(a,\cdot)=
\frac{d\sP_{X,W}}{d\sP_{X,W\mid A=a}}\) for \(\sP_A\)-almost every \(a\).
Therefore,
\[
\begin{aligned}
f_{\mathrm{ATE}}(a)-\bar f_h(a)
&=
\E[
 {\varphi_0^{\mathrm{ATE}}}(a,X,Z)
\{Y-h(a,X,W)\}
\mid A=a
].
\end{aligned}
\]
Conditioning on \((X,Z)\),
\[
\begin{aligned}
f_{\mathrm{ATE}}(a)-\bar f_h(a)
&=
\E\!\left[
 {\varphi_0^{\mathrm{ATE}}}(a,X,Z)
\E[
Y-h(a,X,W)
\mid A=a,X,Z]
\mid A=a
\right].
\end{aligned}
\]
Then Cauchy--Schwarz inequality yileds
\[
\left\|
f_{\mathrm{ATE}}(a)-\bar f_h(a)
\right\|_2
\le
c_{\varphi,\mathrm{ATE}}(a)
\left\|
\E[Y-h(a,X,W)\mid A=a,X,Z]
\right\|_{L^2(\sP_{X,Z\mid A=a};\mathbb R^{d_Y})}.
\]
\end{proof}

\begin{theorem}[Outcome-bridge dose-response consistency]
\label{thm:outcome_dose_response_consistency}
Suppose
\begin{align}
\label{eq:projected_outcome_bridge}
\|T_h\hat h-m_0\|_{L^2(\sP_{Q_h};\mathbb R^{d_Y})}^2
=
O_p( {\tilde\rho_{h,\nu_h}}),
\end{align}
and suppose the treatment-bridge conditions above hold. Then
\[
\|\bar f_{\hat h}-f_{\mathrm{ATE}}\|_{L^2(\sP_A;\mathbb R^{d_Y})}^2
\le
C_{\varphi,\mathrm{ATE}}^2
\|T_h\hat h-m_0\|_{L^2(\sP_{Q_h};\mathbb R^{d_Y})}^2.
\]
Consequently,
\[
\|\bar f_{\hat h}-f_{\mathrm{ATE}}\|_{L^2(\sP_A;\mathbb R^{d_Y})}^2
=O_p( {\tilde\rho_{h,\nu_h}}).
\]
In particular, if \( {\tilde\rho_{h,\nu_h}}\to0\), then
\(
\bar f_{\hat h}\to f_{\mathrm{ATE}}
\quad\text{in }L^2(\sP_A;\mathbb R^{d_Y})
\)
in probability.
\end{theorem}

\begin{proof}

For \(h=\hat h\), the previous lemma gives, for every \(a\),
\[
\left\|
f_{\mathrm{ATE}}(a)-\bar f_{\hat h}(a)
\right\|_2^2
\le
C_\varphi(a)^2
\left\|
\mathbb E[Y-\hat h(a,X,W)\mid A=a,X,Z]
\right\|_{L^2(\sP_{X,Z\mid A=a})}^2.
\]
Integrating over \(a\sim \sP_A\),
\[
\begin{aligned}
\|\bar f_{\hat h}-&f_{\mathrm{ATE}}\|_{L^2(\sP_A)}^2
=
\int
\left\|
\bar f_{\hat h}(a)-f_{\mathrm{ATE}}(a)
\right\|^2_2
\,d\sP_A(a)
\\
&\le
\int
C_\varphi(a)^2
\left[
\int
\left\|
\mathbb E[
Y-\hat h(a,X,W)
\mid A=a,X=x,Z=z]
\right\|^2_2
d \sP_{X,Z\mid A=a}(x,z)
\right]
d\sP_A(a)
\\
&\le
C_\varphi^2
\int
\int
\left\|
\mathbb E[
Y-\hat h(a,X,W)
\mid A=a,X=x,Z=z]
\right\|^2_2
d\sP_{X,Z\mid A=a}(x,z)
d\sP_A(a)
\\
&=
C_\varphi^2
\int
\left\|
\mathbb E[
Y-\hat h(A,X,W)
\mid A,X,Z]
\right\|_2^2
dP_{A,X,Z}(a,x,z).
\end{aligned}
\]
Since
\[
\E[Y-\hat h(A,X,W)\mid Q_h]
=m_0(Q_h)-T_h\hat h(Q_h),
\]
we obtain
\[
\|\bar f_{\hat h}-f_{\mathrm{ATE}}\|_{L^2(\sP_A;\mathbb R^{d_Y})}^2
\le
C_{\varphi,\mathrm{ATE}}^2
\|T_h\hat h-m_0\|_{L^2(\sP_{Q_h};\mathbb R^{d_Y})}^2.
\]
Equation~\ref{eq:projected_outcome_bridge} completes the proof.

\end{proof}

\begin{theorem}[Empirical outcome-bridge averaging error]
\label{thm:outcome_empirical_averaging_error}
Let \(\hat h\) be trained on a sample independent of an evaluation sample
\(\mathcal D_3^{(h)}=\{(x_i,w_i)\}_{i \in \gI_3^{(h)}}\), where
\((x_i,w_i)\stackrel{i.i.d.}{\sim}\sP_{X,W}\) and $|\gI_3^{(h)}| = t_h$. Define

\[
\mu_{\hat h}(a):=\bar f_{\hat h}(a)
=\int\hat h(a,x,w)\,d\sP_{X,W}(x,w),
\qquad
\hat\mu_h(a):=\frac{1}{t_h}\sum_{i \in \gI_3^{(h)}}
\hat h(a,x_i,w_i).
\]
Under Assumption~\ref{ass:outcome_boundedness_rademacher},
\[
\sup_{a,x,w}\|\hat h(a,x,w)\|_2
\le B_{h,\infty}:=B_\theta B_\phi
\quad\text{almost surely}.
\]
Then
\[
\mathcal E_{\mu,h}
:=
\|\hat\mu_h-\mu_{\hat h}\|_{L^2(\sP_A;\mathbb R^{d_Y})}^2
=
O_p(t_h^{-1}).
\]
Consequently, if
\(\|\mu_{\hat h}-f_{\mathrm{ATE}}\|_{L^2(\sP_A;\mathbb R^{d_Y})}^2
=O_p( {\tilde\rho_{h,\nu_h}})\), then
\[
\|\hat\mu_h-f_{\mathrm{ATE}}\|_{L^2(\sP_A;\mathbb R^{d_Y})}^2
=
O_p( {\tilde\rho_{h,\nu_h}}+t_h^{-1}).
\]
\end{theorem}

\begin{proof}

Condition on the trained bridge \(\hat h\). For every fixed \(a\), the empirical average
\(\hat\mu_h(a)\) is unbiased for \(\mu_{\hat h}(a)\), because the evaluation
sample is independent of the training sample:
\[
\E[\hat\mu_h(a)\mid \hat h]
=
\int\hat h(a,x,w)\,d\sP_{X,W}(x,w)
=
\mu_{\hat h}(a).
\]
Set \(\xi_i(a):=\hat h(a,x_i,w_i)-\mu_{\hat h}(a)\). Conditional on
\(\hat h\), the vectors \(\xi_i(a)\) are independent and centered. Hence the
cross inner products vanish and
\[
\begin{aligned}
\E\!\left[
\|\hat\mu_h(a)-\mu_{\hat h}(a)\|_2^2
\mid \hat h
\right]
&=\frac{1}{t_h^2}\sum_{i\in\gI_3^{(h)}}
\E[\|\xi_i(a)\|_2^2\mid\hat h] \le
\frac{B_{h,\infty}^2}{t_h}.
\end{aligned}
\]
Here we used
\(
\E[\|\xi_i(a)\|_2^2\mid\hat h]
=\E[\|\hat h(a,X,W)\|_2^2\mid\hat h]-\|\mu_{\hat h}(a)\|_2^2
\le B_{h,\infty}^2
\).
Integrating over \(a\sim \sP_A\) gives
\[
\E\!\left[
\mathcal E_{\mu,h}
\mid \hat h
\right]
=
\int
\E\!\left[
\|\hat\mu_h(a)-\mu_{\hat h}(a)\|_2^2
\mid \hat h
\right]
d\sP_A(a)
\le
\frac{B_{h,\infty}^2}{t_h}.
\]
Therefore, by Markov's inequality, for every \(\delta>0\),
\[
\sP\!\left(
\mathcal E_{\mu,h}
>
\frac{B_{h,\infty}^2}{\delta t_h}
\;\middle|\;
\hat h
\right)
\le
\delta,
\]
and hence \(\mathcal E_{\mu,h}=O_p(t_h^{-1})\).

Finally, by writing
\(
\hat\mu_h-f_{\mathrm{ATE}}
=
(\hat\mu_h-\mu_{\hat h})
+
(\mu_{\hat h}-f_{\mathrm{ATE}}),
\)
we obtain
\[
\|\hat\mu_h-f_{\mathrm{ATE}}\|_{L^2(\sP_A;\mathbb R^{d_Y})}^2
\le
2\|\hat\mu_h-\mu_{\hat h}\|_{L^2(\sP_A;\mathbb R^{d_Y})}^2
+
2\|\mu_{\hat h}-f_{\mathrm{ATE}}\|_{L^2(\sP_A;\mathbb R^{d_Y})}^2.
\]
Combining the empirical averaging bound with
\(\|\mu_{\hat h}-f_{\mathrm{ATE}}\|_{L^2(\sP_A;\mathbb R^{d_Y})}^2
=O_p( {\tilde\rho_{h,\nu_h}})\)
gives the stated rate.

\end{proof}

\subsection{Treatment bridge consistency for dose-response estimation}
\label{app:treatment_bridge_consistency_rademacher}

We establish projected consistency of {\textsc{TreatmentNet}}. The argument is the counterpart of the corresponding analysis of the {\textsc{OutcomeNet}} analysis in Appendix~\ref{app:outcome_bridge_consistency},
but its second stage is an order-two U-statistic.

Let
\(
Q_\varphi:=(A,X,W) \subset \mathcal Q_\varphi.
\) For a measurable
\(u:\mathcal A\times\mathcal X\times\mathcal Z\to\mathbb R\), define
\[
(T_\varphi u)(a,x,w)
:=
\E[u(a,x,Z)\mid A=a,X=x,W=w],
\qquad
\|u\|_{T_\varphi}
:=
\|T_\varphi u\|_{L^2(\sP_{Q_\varphi})}.
\]

Let the independent first- and second-stage samples be
\(
D_1^{(\varphi)}
=
\{( q_{\varphi,i}, z_i)\}_{i\in \gI_1^{(\varphi)}},
\enspace
D_2^{(\varphi)}
=
\{ q_{\varphi,i}\}_{i \in \gI_2^{(\varphi)}},
\)
with $|\gI_1^{(\varphi)}| = n_\varphi$ and $|\gI_2^{(\varphi)}| = m_\varphi$. The first sample is i.i.d. from \(\sP_{Q_\varphi,Z}\), the second is
i.i.d. from \(\sP_{Q_\varphi}\). We consider the learned-feature bridge class. Let $\nu_\varphi = (n_\varphi, m_\varphi)$  and define

\[
\Phi_{\nu_\varphi}
\subset
\{\phi:\mathcal Z\to\mathbb R^{d_{\nu_\varphi}}\},
\qquad
{\mathcal M_{\nu_\varphi}}
\subset
\{{\eta}:\mathcal Q_\varphi\to\mathbb R^{d_{\nu_\varphi}}\},
\]
for the class of learned treatment-proxy feature maps and the first-stage conditional mean regressors, respectively. We define the class of second-stage heads by
\[
\Theta_{\nu_\varphi}
\subset
\{\theta:\mathcal A\times\mathcal X
\to\mathbb R^{d_{\nu_\varphi}}\}.
\]

The bridge class is
\[
\mathcal H_{\nu_\varphi}
:=
\left\{
\varphi_{\theta,\phi}(a,x,z)
=
\langle\theta(a,x),\phi(z)\rangle
:
\theta\in\Theta_{\nu_\varphi},\
\phi\in\Phi_{\nu_\varphi}
\right\}.
\]
This compact class includes the tensor-product {\textsc{TreatmentNet}}: its final linear layer and the second-stage \((A,X)\)-features are absorbed into \(\theta(a,x)\).

For each \(\phi\in\Phi_{\nu_\varphi}\), define the conditional feature mean
\[
\mu_\phi^{(\varphi)}(q_\varphi)
:=
\E[\phi(Z)\mid Q_\varphi=q_\varphi].
\]
Then, for \(q_\varphi=(a,x,w)\),
\[
(T_\varphi\varphi_{\theta,\phi})(q_\varphi)
=
\langle\theta(a,x),\mu_\phi^{(\varphi)}(q_\varphi)\rangle.
\]

For each candidate feature map $\phi \in \Phi_{\nu_\varphi}$, the population and empirical first-stage risks are
\begin{align*}
\gR_{\varphi,1}(\phi, {\eta})
&:=
\E\|\phi(Z)- {\eta}(Q_\varphi)\|_2^2,\\
\widehat{\gR}_{\varphi,1}(\phi, {\eta})
&:=
\frac1{n_\varphi}\sum_{i \in \gI_1^{(\varphi)}}
\|\phi( z_i)- {\eta}( q_{\varphi,i})\|_2^2.
\end{align*}

For every \(\phi\), let \({\widehat\mu_\phi^{(\varphi)}}\in {\mathcal M_{\nu_\varphi}}\) be a measurable
approximate empirical minimizer, that is,
\[
\widehat{\gR}_{\varphi,1}(\phi, {\widehat\mu_\phi^{(\varphi)}})
\le
\inf_{ {\eta}\in {\mathcal M_{\nu_\varphi}}}
\widehat{\gR}_{\varphi,1}(\phi, {\eta})
+\xi_{\varphi,1,\nu_\varphi},
\qquad \xi_{\varphi,1,\nu_\varphi}\ge 0,
\]
uniformly over \(\phi\in\Phi_{\nu_\varphi}\); the exact-minimizer case is
\(\xi_{\varphi,1,\nu_\varphi}=0\).

Define the population and plug-in
projected bridges
\begin{align*}
 {g_{\theta,\phi}^{(\varphi)}}(q_\varphi)
&:=
\langle\theta(a,x),\mu_\phi^{(\varphi)}(q_\varphi)\rangle
=
(T_\varphi\varphi_{\theta,\phi})(q_\varphi),\\
 {\widehat g_{\theta,\phi}^{(\varphi)}}(q_\varphi)
&:=
\langle\theta(a,x), {\widehat\mu_\phi^{(\varphi)}}(q_\varphi)\rangle.
\end{align*}

For \(q=(a,x,w)\) and \(q'=(a',x',w')\), define
\[
q\diamond q':=(a,x',w').
\]
For independent copies \(Q_{\varphi,1}\) and \(Q_{\varphi,2}\) of
\(Q_\varphi\), define, for \( {u}\in L^2(\sP_{Q_\varphi})\),
\[
\gR_{\varphi,2}^{\mathrm{RF}}( {u})
:=
\E\{ {u}(Q_{\varphi,1})^2\}
-
2\E\{ {u}(Q_{\varphi,1}\diamond Q_{\varphi,2})\}.
\]

By Lemma~\ref{lemma:densRatioFreeRiskCharacterization},
\begin{equation}
\gR_{\varphi,2}^{\mathrm{RF}}( {u})
-
\gR_{\varphi,2}^{\mathrm{RF}}(r_0)
=
\| {u}-r_0\|_{L^2(\sP_{Q_\varphi})}^2.
\label{eq:treatment_rf_excess_identity}
\end{equation}

Conditional on \(D_1^{(\varphi)}\), define the plug-in kernel
\begin{equation}
\hat\ell_{\theta,\phi}^{\mathrm{RF}}(q,q')
:=
\{\hat g_{\theta,\phi}^{(\varphi)}(q)\}^2
-
2\hat g_{\theta,\phi}^{(\varphi)}(q\diamond q').
\label{eq:treatmentSecondStageRiskKernel}
\end{equation}

The empirical second-stage objective is the complete ordered U-statistic
\begin{align*}
\widehat{\gR}_{\varphi,2}^{\mathrm{RF},U}(\theta,\phi)
&:=
\frac{1}{m_\varphi(m_\varphi-1)}
\sum_{i\ne j}
\hat\ell_{\theta,\phi}^{\mathrm{RF}}
(\tilde q_{\varphi,i},\tilde q_{\varphi,j})\\
&=
\frac1{m_\varphi}\sum_{i=1}^{m_\varphi}
\{\hat g_{\theta,\phi}^{(\varphi)}
(\tilde q_{\varphi,i})\}^2 
-
\frac{2}{m_\varphi(m_\varphi-1)}
\sum_{i\ne j}
\hat g_{\theta,\phi}^{(\varphi)}
(\tilde q_{\varphi,i}\diamond\tilde q_{\varphi,j}).
\end{align*}

Let \((\hat\theta^{(\varphi)},\hat\phi^{(\varphi)})\in\Theta_{\nu_\varphi}\times\Phi_{\nu_\varphi}\)
be a measurable approximate empirical minimizer of the product-law objective,
\[
\widehat{\gR}_{\varphi,2}^{\mathrm{RF},U}(\hat\theta^{(\varphi)},\hat\phi^{(\varphi)})
\le
\inf_{\theta\in\Theta_{\nu_\varphi},\,
\phi\in\Phi_{\nu_\varphi}}
\widehat{\gR}_{\varphi,2}^{\mathrm{RF},U}(\theta,\phi)
+\xi_{\varphi,2,\nu_\varphi},
\qquad \xi_{\varphi,2,\nu_\varphi}\ge 0,
\qquad
\hat\varphi
:=
\varphi_{\hat\theta^{(\varphi)},\hat\phi^{(\varphi)}}.
\]

\begin{assumption}[Treatment-side boundedness]
\label{ass:treatment_boundedness_rademacher}
There are finite constants \(B_\phi, {B_\mu},B_\theta\), uniform in
\(\nu_\varphi\), such that
\[
\sup_{\phi\in\Phi_{\nu_\varphi}}\sup_{z\in\mathcal Z}
\|\phi(z)\|_2\le B_\phi,
\qquad
\sup_{\eta\in {\mathcal M_{\nu_\varphi}}}
\sup_{q_\varphi\in\mathcal Q_\varphi}
\|\eta(q_\varphi)\|_2\le  {B_\mu},
\]
and
\[
\sup_{\theta\in\Theta_{\nu_\varphi}}
\sup_{(a,x)\in\mathcal A\times\mathcal X}
\|\theta(a,x)\|_2\le B_\theta.
\]
Here \( {B_\mu}\) bounds the first-stage class \( {\mathcal M_{\nu_\varphi}}\);
the corresponding plug-in projected bridges are bounded by
\(B_\theta  {B_\mu}\). Set
\[
M_{\varphi,1}:=(B_\phi+ {B_\mu})^2,
\qquad
M_{\varphi,2}^{\mathrm{RF}}
:=
B_\theta^2 {B_\mu}^2+2B_\theta {B_\mu}.
\]
\end{assumption}

Define the approximation errors
\begin{align*}
\kappa_{\varphi,1,\nu_\varphi}
&:=
\sup_{\phi\in\Phi_{\nu_\varphi}}
\inf_{\eta\in {\mathcal M_{\nu_\varphi}}}
\|\eta-\mu_\phi^{(\varphi)}
\|_{L^2(\sP_{Q_\varphi};\mathbb R^{d_{\nu_\varphi}})}^2,\\
 {\kappa_{\varphi,2,\nu_\varphi}}
&:=
\inf_{\theta\in\Theta_{\nu_\varphi},\,
\phi\in\Phi_{\nu_\varphi}}
\|g_{\theta,\phi}^{(\varphi)}-r_0
\|_{L^2(\sP_{Q_\varphi})}^2.
\end{align*}

The scalar loss classes are
\begin{align*}
\mathcal L_{\varphi,1,\nu_\varphi}
&:=
\left\{
(q_\varphi,z)\mapsto
\|\phi(z)-\eta(q_\varphi)\|_2^2
:
\phi\in\Phi_{\nu_\varphi},\
\eta\in {\mathcal M_{\nu_\varphi}}
\right\},\\
\mathcal L_{\varphi,2,\nu_\varphi}^{\mathrm{RF}}
&:=
\left\{
(q,q')\mapsto
\langle\theta(a,x),\eta(q)\rangle^2
-
2\langle\theta(a,x'),\eta(q\diamond q')\rangle
:
\theta\in\Theta_{\nu_\varphi},\
\eta\in {\mathcal M_{\nu_\varphi}}
\right\}.
\end{align*}
Their envelopes are \(M_{\varphi,1}\) and
\(M_{\varphi,2}^{\mathrm{RF}}\), respectively. 

Define
\begin{align*}
\Delta_{\varphi,1,\nu_\varphi}(\delta)
&:=
2\widehat{\mathfrak R}_{D_1^{(\varphi)}}
(\mathcal L_{\varphi,1,\nu_\varphi})
+
3M_{\varphi,1}
\sqrt{\frac{\log(2/\delta)}{2n_\varphi}},\\
\Delta_{\varphi,2,\nu_\varphi}^{\mathrm{RF},U}(\delta)
&:=
2\widehat{\mathfrak R}^{(2),U}_{D_2^{(\varphi)}}
(\mathcal L_{\varphi,2,\nu_\varphi}^{\mathrm{RF}})
+
3M_{\varphi,2}^{\mathrm{RF}}
\sqrt{\frac{8\log(2/\delta)}{m_\varphi}}.
\end{align*}
The first term uses Proposition~\ref{prop:uniformRademacherDeviation}. The
second uses the U-process bound in
Proposition~\ref{prop:uniformUProcessDeviation}.

\begin{lemma}[Uniform {\textsc{TreatmentNet}} first-stage error]
\label{lem:treatment_stage1_rademacher}
Let
\[
\epsilon_{\varphi,1,\nu_\varphi}
:=
\sup_{\phi\in\Phi_{\nu_\varphi}}
\|\hat\mu_\phi^{(\varphi)}-\mu_\phi^{(\varphi)}
\|_{L^2(\sP_{Q_\varphi};\mathbb R^{d_{\nu_\varphi}})}^2.
\]
With probability at least \(1-2\delta\),
\[
\epsilon_{\varphi,1,\nu_\varphi}
\le
\kappa_{\varphi,1,\nu_\varphi}
+
2\Delta_{\varphi,1,\nu_\varphi}(\delta)
{{}+\xi_{\varphi,1,\nu_\varphi}}.
\]
\end{lemma}

\begin{proof}
For fixed \(\phi\), the vector-valued projection identity gives
\[
\gR_{\varphi,1}(\phi,\eta)
-
\gR_{\varphi,1}(\phi,\mu_\phi^{(\varphi)})
=
\|\eta-\mu_\phi^{(\varphi)}
\|_{L^2(\sP_{Q_\varphi};\mathbb R^{d_{\nu_\varphi}})}^2.
\]
By Proposition~\ref{prop:uniformRademacherDeviation}, each of
\[
\gR_{\varphi,1}(\phi,\eta)
\le
\widehat{\gR}_{\varphi,1}(\phi,\eta)
+
\Delta_{\varphi,1,\nu_\varphi}(\delta)
\]
and
\[
\widehat{\gR}_{\varphi,1}(\phi,\eta)
\le
\gR_{\varphi,1}(\phi,\eta)
+
\Delta_{\varphi,1,\nu_\varphi}(\delta)
\]
holds uniformly with failure probability at most \(\delta\). On their
intersection, {approximate} empirical optimality yields
\[
\gR_{\varphi,1}(\phi,\hat\mu_\phi^{(\varphi)})
\le
\inf_{\eta\in {\mathcal M_{\nu_\varphi}}}
\gR_{\varphi,1}(\phi,\eta)
+
2\Delta_{\varphi,1,\nu_\varphi}(\delta)
{{}+\xi_{\varphi,1,\nu_\varphi}}.
\]
Subtract the projection risk, take the infimum over \(\eta\), and then the
supremum over \(\phi\).
\end{proof}

Conditional on \(D_1^{(\varphi)}\), the definitions of the population and
plug-in projected bridges give, for every \((\theta,\phi)\),
\[
\hat g_{\theta,\phi}^{(\varphi)}(q_\varphi)
-g_{\theta,\phi}^{(\varphi)}(q_\varphi)
=
\left\langle
\theta(a,x),
\hat\mu_\phi^{(\varphi)}(q_\varphi)
-\mu_\phi^{(\varphi)}(q_\varphi)
\right\rangle.
\]
Therefore, Cauchy--Schwarz, Assumption~\ref{ass:treatment_boundedness_rademacher},
and the definition of \(\epsilon_{\varphi,1,\nu_\varphi}\) imply
\begin{equation}
\begin{aligned}
&\sup_{\substack{\theta\in\Theta_{\nu_\varphi}\\
\phi\in\Phi_{\nu_\varphi}}}
\|\hat g_{\theta,\phi}^{(\varphi)}
-g_{\theta,\phi}^{(\varphi)}\|_{L^2(\sP_{Q_\varphi})}^2\\
&\qquad\le
B_\theta^2
\sup_{\phi\in\Phi_{\nu_\varphi}}
\|\hat\mu_\phi^{(\varphi)}-\mu_\phi^{(\varphi)}
\|_{L^2(\sP_{Q_\varphi};\mathbb R^{d_{\nu_\varphi}})}^2
=B_\theta^2\epsilon_{\varphi,1,\nu_\varphi}.
\end{aligned}
\label{eq:treatment_plugin_projection_error}
\end{equation}
This plug-in stability bound transfers the first-stage conditional-mean error
to the projected treatment bridge. It therefore connects the empirical
second-stage risk based on \(\hat g_{\theta,\phi}^{(\varphi)}\) to the
population projection
\(g_{\theta,\phi}^{(\varphi)}=T_\varphi\varphi_{\theta,\phi}\).

\begin{theorem}[Projected {\textsc{TreatmentNet}} error bound]
\label{thm:treatment_projected_plugin_rademacher}
Under the preceding setup and
Assumption~\ref{ass:treatment_boundedness_rademacher}, let both stages use {measurable approximate empirical minimizers with gaps
\(\xi_{\varphi,1,\nu_\varphi},\xi_{\varphi,2,\nu_\varphi}\ge 0\)}. For every
\(\delta\in(0,1/4)\), with probability at least \(1-4\delta\),
\begin{align*}
\|T_\varphi\hat\varphi-r_0
\|_{L^2(\sP_{Q_\varphi})}^2
\le{}&
4 {\kappa_{\varphi,2,\nu_\varphi}}
+
6B_\theta^2\kappa_{\varphi,1,\nu_\varphi} +
12B_\theta^2\Delta_{\varphi,1,\nu_\varphi}(\delta)
+
4\Delta_{\varphi,2,\nu_\varphi}^{\mathrm{RF},U}(\delta)\\
& {{}+6B_\theta^2\xi_{\varphi,1,\nu_\varphi}
+2\xi_{\varphi,2,\nu_\varphi}}.
\end{align*}
 {Setting both gaps to zero recovers the exact-minimization bound.}
If a square-integrable ATE treatment bridge \(\varphi_0^{\mathrm{ATE}}\)
satisfies \(T_\varphi\varphi_0^{\mathrm{ATE}}=r_0\), the same bound controls
\[
\|\hat\varphi-\varphi_0^{\mathrm{ATE}}\|_{T_\varphi}^2
=
\|T_\varphi\hat\varphi-r_0\|_{L^2(\sP_{Q_\varphi})}^2.
\]
\end{theorem}

\begin{proof}
Conditional on \(D_1^{(\varphi)}\), every
\(\hat\mu_\phi^{(\varphi)}\) is fixed and belongs to
\( {\mathcal M_{\nu_\varphi}}\). Hence the plug-in kernel class is a subclass of
\(\mathcal L_{\varphi,2,\nu_\varphi}^{\mathrm{RF}}\).
Moreover, for every \((\theta,\phi)\),
\[
\sP_{Q_\varphi}^{\otimes2}
\hat\ell_{\theta,\phi}^{\mathrm{RF}}
=
\gR_{\varphi,2}^{\mathrm{RF}}
(\hat g_{\theta,\phi}^{(\varphi)}),
\qquad
\mathbb U_{m_\varphi}\hat\ell_{\theta,\phi}^{\mathrm{RF}}
=
\widehat{\gR}_{\varphi,2}^{\mathrm{RF},U}(\theta,\phi).
\]
Proposition~\ref{prop:uniformUProcessDeviation} gives each of the two
uniform inequalities
\begin{align*}
\gR_{\varphi,2}^{\mathrm{RF}}
(\hat g_{\theta,\phi}^{(\varphi)})
&\le
\widehat{\gR}_{\varphi,2}^{\mathrm{RF},U}(\theta,\phi)
+
\Delta_{\varphi,2,\nu_\varphi}^{\mathrm{RF},U}(\delta),\\
\widehat{\gR}_{\varphi,2}^{\mathrm{RF},U}(\theta,\phi)
&\le
\gR_{\varphi,2}^{\mathrm{RF}}
(\hat g_{\theta,\phi}^{(\varphi)})
+
\Delta_{\varphi,2,\nu_\varphi}^{\mathrm{RF},U}(\delta)
\end{align*}
with conditional failure probability at most \(\delta\) per direction.
Indeed, on their intersection, for every
\((\theta,\phi)\in\Theta_{\nu_\varphi}\times\Phi_{\nu_\varphi}\),
\begin{align*}
\gR_{\varphi,2}^{\mathrm{RF}}
(\hat g_{\hat\theta^{(\varphi)},\hat\phi^{(\varphi)}}^{(\varphi)})
&\le
\widehat{\gR}_{\varphi,2}^{\mathrm{RF},U}
(\hat\theta^{(\varphi)},\hat\phi^{(\varphi)})
+\Delta_{\varphi,2,\nu_\varphi}^{\mathrm{RF},U}(\delta)\\
&\le
\widehat{\gR}_{\varphi,2}^{\mathrm{RF},U}(\theta,\phi)
+\Delta_{\varphi,2,\nu_\varphi}^{\mathrm{RF},U}(\delta)
{{}+\xi_{\varphi,2,\nu_\varphi}}\\
&\le
\gR_{\varphi,2}^{\mathrm{RF}}
(\hat g_{\theta,\phi}^{(\varphi)})
+2\Delta_{\varphi,2,\nu_\varphi}^{\mathrm{RF},U}(\delta)
{{}+\xi_{\varphi,2,\nu_\varphi}}.
\end{align*}
{The middle inequality follows from approximate empirical optimality, which inserts \(\xi_{\varphi,2,\nu_\varphi}\).} Subtracting \(\gR_{\varphi,2}^{\mathrm{RF}}(r_0)\), applying Equation~\ref{eq:treatment_rf_excess_identity} to both sides, and taking the infimum over \((\theta,\phi)\) give
\begin{align*}
\|\hat g_{\hat\theta^{(\varphi)},\hat\phi^{(\varphi)}}^{(\varphi)}
-r_0\|_{L^2}^2
\le
\inf_{\theta,\phi}
\|\hat g_{\theta,\phi}^{(\varphi)}-r_0\|_{L^2}^2
+
2\Delta_{\varphi,2,\nu_\varphi}^{\mathrm{RF},U}(\delta)
{{}+\xi_{\varphi,2,\nu_\varphi}}.
\end{align*}
Here and below, the displayed \(L^2\) norm is over
\(\sP_{Q_\varphi}\). By
Equation~\ref{eq:treatment_plugin_projection_error},
\begin{align*}
\inf_{\theta,\phi}
\|\hat g_{\theta,\phi}^{(\varphi)}-r_0\|_{L^2}^2
&\le
2\inf_{\theta,\phi}
\|g_{\theta,\phi}^{(\varphi)}-r_0\|_{L^2}^2
+
2B_\theta^2\epsilon_{\varphi,1,\nu_\varphi}\\
&=
2 {\kappa_{\varphi,2,\nu_\varphi}}
+
2B_\theta^2\epsilon_{\varphi,1,\nu_\varphi}.
\end{align*}
Therefore,
\begin{align*}
\|\hat g_{\hat\theta^{(\varphi)},\hat\phi^{(\varphi)}}^{(\varphi)}
-r_0\|_{L^2}^2
\le
2 {\kappa_{\varphi,2,\nu_\varphi}}
+
2B_\theta^2\epsilon_{\varphi,1,\nu_\varphi}
+
2\Delta_{\varphi,2,\nu_\varphi}^{\mathrm{RF},U}(\delta)
{{}+\xi_{\varphi,2,\nu_\varphi}}.
\end{align*}

Finally,
\(
T_\varphi\hat\varphi
=
g_{\hat\theta^{(\varphi)},\hat\phi^{(\varphi)}}^{(\varphi)}
\).
A second application of
Equation~\ref{eq:treatment_plugin_projection_error} gives
\begin{align*}
\|T_\varphi\hat\varphi-r_0\|_{L^2}^2
&\le
2\|
\hat g_{\hat\theta^{(\varphi)},\hat\phi^{(\varphi)}}^{(\varphi)}
-r_0
\|_{L^2}^2
+
2B_\theta^2\epsilon_{\varphi,1,\nu_\varphi}\\
&\le
4\kappa_{\varphi,2,\nu_\varphi}
+
6B_\theta^2\epsilon_{\varphi,1,\nu_\varphi}
+
4\Delta_{\varphi,2,\nu_\varphi}^{\mathrm{RF},U}(\delta)
{{}+2\xi_{\varphi,2,\nu_\varphi}}.
\end{align*}
{The factor \(2\) on \(\xi_{\varphi,2,\nu_\varphi}\) comes from the
\(\|u+v\|^2\le2\|u\|^2+2\|v\|^2\) step above.} Apply Lemma~\ref{lem:treatment_stage1_rademacher} {, whose bound contributes
\(6B_\theta^2\xi_{\varphi,1,\nu_\varphi}\) through the coefficient
\(6B_\theta^2\epsilon_{\varphi,1,\nu_\varphi}\)}. Its two first-stage events
and the two conditional U-process events fail with total probability at most
\(4\delta\).
\end{proof}

\begin{corollary}[Projected \textsc{TreatmentNet} consistency]
\label{cor:treatmentnet_projected_consistency_plugin}
For \(\delta\in(0,1/4)\), define
\[
\rho_{\varphi,\nu_\varphi}(\delta)
:=
4\kappa_{\varphi,2,\nu_\varphi}
+6B_\theta^2\kappa_{\varphi,1,\nu_\varphi}
+12B_\theta^2\Delta_{\varphi,1,\nu_\varphi}(\delta)
+4\Delta_{\varphi,2,\nu_\varphi}^{\mathrm{RF},U}(\delta).
\]

Define the optimization-aware projected rate
\[
\tilde\rho_{\varphi,\nu_\varphi}(\delta)
:=
\rho_{\varphi,\nu_\varphi}(\delta)
+6B_\theta^2\xi_{\varphi,1,\nu_\varphi}
+2\xi_{\varphi,2,\nu_\varphi}.
\]
Then, with probability at least \(1-4\delta\),
\[
\|T_\varphi\hat\varphi-r_0
\|_{L^2(\sP_{Q_\varphi})}^2
\le \tilde\rho_{\varphi,\nu_\varphi}(\delta).
\]
Let \(\delta_{\varphi,\nu_\varphi}\downarrow0\) satisfy
\(\log(1/\delta_{\varphi,\nu_\varphi})/
\min\{n_\varphi,m_\varphi\}\to0\), and abbreviate
\(\rho_{\varphi,\nu_\varphi}
:=\rho_{\varphi,\nu_\varphi}(\delta_{\varphi,\nu_\varphi})\) and
\(\tilde\rho_{\varphi,\nu_\varphi}
:=\tilde\rho_{\varphi,\nu_\varphi}(\delta_{\varphi,\nu_\varphi})\).
If
\[
\kappa_{\varphi,1,\nu_\varphi}\to0,\qquad
\kappa_{\varphi,2,\nu_\varphi}\to0,\qquad
\widehat{\mathfrak R}_{D_1^{(\varphi)}}
(\mathcal L_{\varphi,1,\nu_\varphi})\xrightarrow{p}0,\qquad
\widehat{\mathfrak R}^{(2),U}_{D_2^{(\varphi)}}
(\mathcal L_{\varphi,2,\nu_\varphi}^{\mathrm{RF}})
\xrightarrow{p}0,
\]
\[
\xi_{\varphi,1,\nu_\varphi}\xrightarrow{p}0,
\qquad
\xi_{\varphi,2,\nu_\varphi}\xrightarrow{p}0,
\]
as \(n_\varphi,m_\varphi\to\infty\), then
\(\tilde\rho_{\varphi,\nu_\varphi}\xrightarrow{p}0\) and
\[
\|T_\varphi\hat\varphi-r_0
\|_{L^2(\sP_{Q_\varphi})}^2
=O_p(\tilde\rho_{\varphi,\nu_\varphi}),
\qquad
\|T_\varphi\hat\varphi-r_0
\|_{L^2(\sP_{Q_\varphi})}^2
\xrightarrow{p}0.
\]

If \(T_\varphi\varphi_0^{\mathrm{ATE}}=r_0\), then
\(\|\hat\varphi-\varphi_0^{\mathrm{ATE}}\|_{T_\varphi}^2\to0\) in probability.
\end{corollary}

\begin{remark}[Projected rather than strong bridge recovery]
The theorem controls the projected treatment-bridge error
\(T_\varphi\hat\varphi-r_0\). It does not imply
\[
\|\hat\varphi-\varphi_0^{\mathrm{ATE}}\|_{L^2(\sP_{A,X,Z})}\to0.
\]
Strong bridge recovery requires an additional inverse-stability condition.
The next subsection uses only the projected error to establish dose-response
consistency.
\end{remark}

\subsection{Dose-response consistency of  {\textsc{TreatmentNet}}}
\label{app:treatment_dose_response_rademacher}

We now convert the projected treatment-bridge result into consistency of the
ATE dose-response curve. This step does not require strong convergence of
\(\hat\varphi\) in \(L^2(\sP_{A,X,Z})\). Instead, it uses the outcome bridge as the
dual object and controls the population dose-response error through the
projected treatment-bridge residual.

Let \(Y\in\mathbb R^{d_Y}\). The treatment-bridge representation of the
ATE dose-response curve is
\[
f_{\mathrm{ATE}}(a)
=
\E[Y\varphi_0^{\mathrm{ATE}}(a,X,Z)\mid A=a]
\in\mathbb R^{d_Y},
\]
where
\(
T_\varphi\varphi_0^{\mathrm{ATE}}=r_0.
\)
For any scalar candidate treatment bridge \(u\), define
\[
\bar f_u^{(\varphi)}(a)
:=
\E[Yu(a,X,Z)\mid A=a]\in\mathbb R^{d_Y}.
\]
The learned-bridge population curve is obtained by setting \(u=\hat\varphi\).

\begin{assumption}[Dual outcome bridge and boundedness]
\label{ass:treatment_dr_dual_outcome_bridge}

Assume that:
\begin{itemize}
    \item There exists an
    \(\mathbb R^{d_Y}\)-valued outcome bridge \(h_0\) satisfying
    \[
    \E[Y-h_0(a,X,W)\mid A=a,X,Z]=0_{\mathbb R^{d_Y}}.
    \]

    \item For \(\sP_A\)-almost every \(a\), set
    \[
    c_h(a)
    :=
    \left(
    \E[\|h_0(a,X,W)\|_2^2\mid A=a]
    \right)^{1/2},
    \qquad
    C_h
    :=
    \operatorname*{ess\,sup}_{a\sim\sP_A}c_h(a),
    \]
    and assume \(C_h<\infty\).
\end{itemize}
\end{assumption}

\begin{lemma}[Projected treatment residual controls the population curve]
\label{lem:treatment_projected_to_dose_response}

For any square-integrable scalar candidate treatment bridge \(u(a,X,Z)\)
for which \(\bar f_u^{(\varphi)}\) exists, and for
\(\sP_A\)-almost every \(a\),
\[
\left\|
\bar f_u^{(\varphi)}(a)-f_{\mathrm{ATE}}(a)
\right\|_2
\le
c_h(a)
\left\|
(T_\varphi u)(a,X,W)-r_0(a,X,W)
\right\|_{L^2(\sP_{X,W\mid A=a})},
\]
Consequently,
\[
\|\bar f_u^{(\varphi)}-f_{\mathrm{ATE}}
\|_{L^2(\sP_A;\mathbb R^{d_Y})}^2
\le
C_h^2
\|T_\varphi u-r_0\|_{L^2(\sP_{Q_\varphi})}^2.
\]
\end{lemma}

\begin{proof}

For any \(u\),
\[
\bar f_u^{(\varphi)}(a)-f_{\mathrm{ATE}}(a)
=
\E[
Y\{u(a,X,Z)-\varphi_0^{\mathrm{ATE}}(a,X,Z)\}
\mid A=a].
\]

Since \(u(a,X,Z)-\varphi_0^{\mathrm{ATE}}(a,X,Z)\) is scalar and
measurable with respect to \((X,Z)\), the outcome-bridge equation gives
\[
\E[
\{Y-h_0(a,X,W)\}
\{u(a,X,Z)-\varphi_0^{\mathrm{ATE}}(a,X,Z)\}
\mid A=a]
=0_{\mathbb R^{d_Y}}.
\]
Therefore,
\[
\bar f_u^{(\varphi)}(a)-f_{\mathrm{ATE}}(a)
=
\E[
h_0(a,X,W)
\{u(a,X,Z)-\varphi_0^{\mathrm{ATE}}(a,X,Z)\}
\mid A=a].
\]
Conditioning on \((X,W)\),
\[
\bar f_u^{(\varphi)}(a)-f_{\mathrm{ATE}}(a)
=
\E\!\left[
h_0(a,X,W)
\E[
u(a,X,Z)-\varphi_0^{\mathrm{ATE}}(a,X,Z)
\mid A=a,X,W]
\mid A=a
\right].
\]

Since \(T_\varphi\varphi_0^{\mathrm{ATE}}=r_0\), the inner conditional
expectation equals
\[
(T_\varphi u)(a,X,W)-r_0(a,X,W).
\]

Cauchy--Schwarz gives
\[
\left\|
\bar f_u^{(\varphi)}(a)-f_{\mathrm{ATE}}(a)
\right\|_2^2
\le
c_h(a)^2
\left\|
(T_\varphi u)(a,X,W)-r_0(a,X,W)
\right\|_{L^2(\sP_{X,W\mid A=a})}^2.
\]
Integrating over \(a\sim \sP_A\) and using
\[
c_h(a)^2\le C_h^2
\]
almost surely proves the result.
\end{proof}

Conditional on the trained bridge \(\hat\varphi\), let
\(
D_3^{(\varphi)}
:=
\{( y_i, a_i, x_i, z_i)\}_{i \in \gI_3^{(\varphi)}}
\)
be an independent i.i.d. sample from \(\sP_{Y,A,X,Z}\) with $|\gI_3^{(\varphi)}|$. Set
\[
\Gamma_\varphi:=Y\hat\varphi(A,X,Z)\in\mathbb R^{d_Y},
\qquad
\mathcal F_{t_\varphi}^{(\varphi)}
\subset
\{f:\mathcal A\to\mathbb R^{d_Y}\}.
\]
Define
\begin{align*}
\widehat{\gR}_{\varphi,3}(f)
&:=
\frac1{t_\varphi}\sum_{i\in \gI_3^{(\varphi)}}
\left\|
{y_i^{(\varphi)}}
-f( a_i)
\right\|_2^2,\\
\hat f_{\mathrm{ATE}}^{(\varphi)}
&\in\mathcal F_{t_\varphi}^{(\varphi)}
{\text{ with }
\widehat{\gR}_{\varphi,3}(\hat f_{\mathrm{ATE}}^{(\varphi)})
\le
\inf_{f\in\mathcal F_{t_\varphi}^{(\varphi)}}
\widehat{\gR}_{\varphi,3}(f)
+\xi_{\varphi,3,t_\varphi},
\quad \xi_{\varphi,3,t_\varphi}\ge0,}\\
\kappa_{\varphi,3,t_\varphi}
&:=
\inf_{f\in\mathcal F_{t_\varphi}^{(\varphi)}}
\|f-\bar f_{\hat\varphi}^{(\varphi)}
\|_{L^2(\sP_A;\mathbb R^{d_Y})}^2.
\end{align*}
{Here
\(y_i^{(\varphi)}:=y_i\hat\varphi(a_i,x_i,z_i)\), as in the algorithmic
description, and for the neural class
\(\hat f_{\mathrm{ATE}}^{(\varphi)}(a)
=k^{\varphi}(a;\hat\theta_3^{\varphi})\).}
{The fitted function is assumed measurable and satisfies the preceding
approximate empirical minimization condition.} For a continuous treatment,
this is the third-stage response regression. For a finite treatment space,
the saturated choice of \(\mathcal F_{t_\varphi}^{(\varphi)}\) yields the
within-treatment sample averages on nonempty treatment cells, so no response
network is required.

Conditional on \(\hat\varphi\), define the third-stage loss class

\[
\mathcal L_{\varphi,3,t_\varphi}(\hat\varphi)
:=
\left\{
(y,a,x,z)\mapsto
\left\|y\hat\varphi(a,x,z)-f(a)\right\|_2^2
:
f\in\mathcal F_{t_\varphi}^{(\varphi)}
\right\}.
\]
We assume this conditional loss class is measurable and separable.

\begin{assumption}[Third-stage boundedness]
\label{ass:treatment_third_stage_boundedness}
Conditional on the trained treatment bridge \(\hat\varphi\), assume that:
\begin{itemize}
    \item The third-stage loss class is uniformly bounded:

    \[
    0\le \ell(y,a,x,z)\le M_{\varphi,3}
    \quad
    \text{for every }
    \ell\in\mathcal L_{\varphi,3,t_\varphi}(\hat\varphi)
    \text{ and }(y,a,x,z),
    \]
    for a common deterministic constant \(M_{\varphi,3}<\infty\), uniform
    in \(t_\varphi\).

    \item The third-stage sample \(D_3^{(\varphi)}\) is independent of the samples used to train \(\hat\varphi\).
\end{itemize}
\end{assumption}

Define the third-stage empirical process error
\[
\Delta_{\varphi,3,t_\varphi}(\delta)
:=
2\widehat{\mathfrak R}_{D_3^{(\varphi)}}(
\mathcal L_{\varphi,3,t_\varphi}(\hat\varphi))
+
3M_{\varphi,3}
\sqrt{\frac{\log(2/\delta)}{2t_\varphi}}.
\]

\begin{lemma}[Third-stage Rademacher bound]
\label{lem:treatment_third_stage_rademacher}

Under Assumption~\ref{ass:treatment_third_stage_boundedness}, conditional on
the trained treatment bridge \(\hat\varphi\), with probability at least
\(1-2\delta\),
\[
\|\hat f_{\mathrm{ATE}}^{(\varphi)}
-
\bar f_{\hat\varphi}^{(\varphi)}
\|_{L^2(\sP_A;\mathbb R^{d_Y})}^2
\le
\kappa_{\varphi,3,t_\varphi}
+
2\Delta_{\varphi,3,t_\varphi}(\delta)
{{}+\xi_{\varphi,3,t_\varphi}}.
\]

\end{lemma}

\begin{proof}
Conditional on \(\hat\varphi\), this is ordinary least-squares regression with response

\[
\Gamma_\varphi
:=
Y\hat\varphi(A,X,Z)
\in\mathbb R^{d_Y}
\]

and covariate \(A\). Since
\[
\bar f_{\hat\varphi}^{(\varphi)}(A)
=
\mathbb E[\Gamma_\varphi\mid A,\hat\varphi],
\]
the least-squares projection identity gives

\[
\mathbb E[
\|\Gamma_\varphi-f(A)\|_2^2
\mid \hat\varphi]
-
\mathbb E[
\|\Gamma_\varphi-\bar f_{\hat\varphi}^{(\varphi)}(A)\|_2^2
\mid \hat\varphi]
=
\|f-\bar f_{\hat\varphi}^{(\varphi)}
\|_{L^2(\sP_A;\mathbb R^{d_Y})}^2.
\]

By Proposition~\ref{prop:uniformRademacherDeviation}, each deviation
direction is bounded by
\(\Delta_{\varphi,3,t_\varphi}(\delta)\) with conditional failure
probability at most \(\delta\). On their intersection, {approximate}
empirical optimality
of \(\hat f_{\mathrm{ATE}}^{(\varphi)}\) and the projection identity give
the result{, the optimization gap entering additively once}.

\end{proof}

\begin{theorem}[{\textsc{TreatmentNet}} dose-response consistency]
\label{thm:treatment_dose_response_rademacher}

Suppose Assumptions~\ref{ass:treatment_dr_dual_outcome_bridge} and
\ref{ass:treatment_third_stage_boundedness} hold, and suppose
\(\hat\varphi\) is square-integrable. For every
\(\delta\in(0,1/2)\), conditional on the trained bridge, with probability at
least \(1-2\delta\),
\begin{align*}
\|\hat f_{\mathrm{ATE}}^{(\varphi)}-f_{\mathrm{ATE}}
\|_{L^2(\sP_A;\mathbb R^{d_Y})}^2
\le\;&
2\kappa_{\varphi,3,t_\varphi}
+
4\Delta_{\varphi,3,t_\varphi}(\delta)
{{}+2\xi_{\varphi,3,t_\varphi}}
+
2C_h^2
\|T_\varphi\hat\varphi-r_0\|_{L^2(\sP_{Q_\varphi})}^2.
\end{align*}
{The multiplier \(2\) on \(\xi_{\varphi,3,t_\varphi}\) is the same
\(\|u+v\|^2\le2\|u\|^2+2\|v\|^2\) factor that doubles \(\kappa_{\varphi,3,t_\varphi}\) and \(2\Delta_{\varphi,3,t_\varphi}(\delta)\).} Since \(T_\varphi\varphi_0^{\mathrm{ATE}}=r_0\), the projected error is
equivalently
\[
\|T_\varphi\hat\varphi-r_0\|_{L^2(\sP_{Q_\varphi})}^2
=
\|\hat\varphi-\varphi_0^{\mathrm{ATE}}\|_{T_\varphi}^2.
\]
Thus, the preceding bound requires only projected-norm consistency of the
treatment bridge.
If, in addition, the assumptions of
Theorem~\ref{thm:treatment_projected_plugin_rademacher} hold, then for
\(\delta\in(0,1/6)\), intersecting this event with
that theorem's event gives joint probability at least \(1-6\delta\), and
\begin{align*}
&\|\hat f_{\mathrm{ATE}}^{(\varphi)}-f_{\mathrm{ATE}}
\|_{L^2(\sP_A;\mathbb R^{d_Y})}^2\\
&\quad\le
2\kappa_{\varphi,3,t_\varphi}
+4\Delta_{\varphi,3,t_\varphi}(\delta)
+8C_h^2 {\kappa_{\varphi,2,\nu_\varphi}}\\
&\qquad
+12C_h^2B_\theta^2\kappa_{\varphi,1,\nu_\varphi}
+24C_h^2B_\theta^2\Delta_{\varphi,1,\nu_\varphi}(\delta)
+8C_h^2\Delta_{\varphi,2,\nu_\varphi}^{\mathrm{RF},U}(\delta)\\
&\qquad
{{}+2\xi_{\varphi,3,t_\varphi}
+12C_h^2B_\theta^2\xi_{\varphi,1,\nu_\varphi}
+4C_h^2\xi_{\varphi,2,\nu_\varphi}}.
\end{align*}
This follows by substituting the projected-norm bound in
Theorem~\ref{thm:treatment_projected_plugin_rademacher} of
Appendix~\ref{app:treatment_bridge_consistency_rademacher}. 

\end{theorem}

\begin{proof}
Decompose

\[
\hat f_{\mathrm{ATE}}^{(\varphi)}-f_{\mathrm{ATE}}
=
(\hat f_{\mathrm{ATE}}^{(\varphi)}-\bar f_{\hat\varphi}^{(\varphi)})
+
(\bar f_{\hat\varphi}^{(\varphi)}-f_{\mathrm{ATE}}).
\]

Using \(\|u+v\|_2^2\le2\|u\|_2^2+2\|v\|_2^2\),
\[
\|\hat f_{\mathrm{ATE}}^{(\varphi)}-f_{\mathrm{ATE}}
\|_{L^2(\sP_A;\mathbb R^{d_Y})}^2
\le
2\|\hat f_{\mathrm{ATE}}^{(\varphi)}
-\bar f_{\hat\varphi}^{(\varphi)}
\|_{L^2(\sP_A;\mathbb R^{d_Y})}^2
+
2\|\bar f_{\hat\varphi}^{(\varphi)}-f_{\mathrm{ATE}}
\|_{L^2(\sP_A;\mathbb R^{d_Y})}^2.
\]

The first term is controlled by
Lemma~\ref{lem:treatment_third_stage_rademacher}.

The second is controlled by
Lemma~\ref{lem:treatment_projected_to_dose_response}. Substituting both bounds gives the result.
\end{proof}

\begin{corollary}[ {\textsc{TreatmentNet}} dose-response consistency]
\label{cor:treatmentnet_ate_consistency}

Suppose the assumptions of
Theorems~\ref{thm:treatment_projected_plugin_rademacher}
and~\ref{thm:treatment_dose_response_rademacher} hold,
\(n_\varphi,m_\varphi,t_\varphi\to\infty\), and
\[
\kappa_{\varphi,1,\nu_\varphi}\to0,
\qquad
 {\kappa_{\varphi,2,\nu_\varphi}}\to0,
\qquad
\kappa_{\varphi,3,t_\varphi}\xrightarrow{p}0.
\]
Assume also that
\[
\widehat{\mathfrak R}_{D_1^{(\varphi)}}
(\mathcal L_{\varphi,1,\nu_\varphi})
\xrightarrow{p}0,
\qquad
\widehat{\mathfrak R}^{(2),U}_{D_2^{(\varphi)}}
(\mathcal L_{\varphi,2,\nu_\varphi}^{\mathrm{RF}})
\xrightarrow{p}0,
\]
and
\[
\widehat{\mathfrak R}_{D_3^{(\varphi)}}
(\mathcal L_{\varphi,3,t_\varphi}(\hat\varphi))
\xrightarrow{p}0.
\]

Assume in addition that the optimization gaps vanish,
\[
\xi_{\varphi,1,\nu_\varphi}\xrightarrow{p}0,
\qquad
\xi_{\varphi,2,\nu_\varphi}\xrightarrow{p}0,
\qquad
\xi_{\varphi,3,t_\varphi}\xrightarrow{p}0.
\]

Then
\[
\|\hat f_{\mathrm{ATE}}^{(\varphi)}-f_{\mathrm{ATE}}
\|_{L^2(\sP_A;\mathbb R^{d_Y})}^2
\xrightarrow{p}0.
\]

\end{corollary}

\subsection{Dose-response consistency of doubly robust method}

\label{app:dr_ate_rademacher}

We analyze the {\textsc{DR-PCL-Net}} estimator for
\(Y\in\mathbb R^{d_Y}\). Condition throughout on the joint training of \(\hat h\) and \(\hat\varphi\); the two fitted bridges need
not be independent of one another. The final-stage fold, including any data-dependent preprocessing, is independent of these bridge function training procedures. The result
uses only projected bridge residuals and does not require strong \(L^2\)-convergence of either bridge.

\paragraph{Population plug-in DR functional.}

Define
\[
\mu_{\hat h}(a)
:=
\mathbb E[\hat h(a,X,W)], \qquad
K_{\hat h,\hat\varphi}(a)
:=
\mathbb E[
\hat\varphi(a,X,Z)\{Y-\hat h(a,X,W)\}
\mid A=a].
\]
The population plug-in {\textsc{DR-PCL-Net}} functional is
\[
\bar f_{\mathrm{ATE}}^{\mathrm{DR}}(a)
:=
\mu_{\hat h}(a)+K_{\hat h,\hat\varphi}(a).
\]

\begin{lemma}[ATE doubly robust identity]
\label{lem:ate_dr_identity_projected}
{For \(\sP_A\)-almost every \(a\),}
\[
f_{\mathrm{ATE}}(a)
-
\bar f_{\mathrm{ATE}}^{ {\mathrm{DR}}}(a)
=
\mathbb E[
\{ {\varphi_0^{\mathrm{ATE}}}(a,X,Z)-\hat\varphi(a,X,Z)\}
\{h_0(a,X,W)-\hat h(a,X,W)\}
\mid A=a].
\]
\end{lemma}

\begin{proof}
By the outcome bridge equation,
\[
\mathbb E[Y-h_0(a,X,W)\mid A=a,X,Z]=0.
\]
Since \(\hat\varphi(a,X,Z)\) is a function of \((X,Z)\),
\[
\mathbb E[
\hat\varphi(a,X,Z)\{Y-h_0(a,X,W)\}
\mid A=a]=0.
\]
Therefore,
\[
\bar f_{\mathrm{ATE}}^{ {\mathrm{DR}}}(a)
=
\mathbb E[\hat h(a,X,W)]
+
\mathbb E[
\hat\varphi(a,X,Z)\{h_0(a,X,W)-\hat h(a,X,W)\}
\mid A=a].
\]
By the outcome bridge representation,
\[
f_{\mathrm{ATE}}(a)=\mathbb E[h_0(a,X,W)].
\]
Hence
\begin{align*}
\bar f_{\mathrm{ATE}}^{ {\mathrm{DR}}}(a)-f_{\mathrm{ATE}}(a)
&=
\mathbb E[\hat h(a,X,W)-h_0(a,X,W)]
\\
&\quad+
\mathbb E[
\hat\varphi(a,X,Z)\{h_0(a,X,W)-\hat h(a,X,W)\}
\mid A=a].
\end{align*}
The treatment bridge reweighting identity gives, for any square-integrable
\(q(a,X,W)\),
\[
\mathbb E[ {\varphi_0^{\mathrm{ATE}}}(a,X,Z)q(a,X,W)\mid A=a]
=
\mathbb E[q(a,X,W)].
\]
Applying this identity with
\[
q(a,X,W)=\hat h(a,X,W)-h_0(a,X,W)
\]
yields
\[
\mathbb E[\hat h(a,X,W)-h_0(a,X,W)]
=
\mathbb E[
 {\varphi_0^{\mathrm{ATE}}}(a,X,Z)\{\hat h(a,X,W)-h_0(a,X,W)\}
\mid A=a].
\]
Substitution gives
\[
\bar f_{\mathrm{ATE}}^{ {\mathrm{DR}}}(a)-f_{\mathrm{ATE}}(a)
=
\mathbb E[
\{\hat\varphi(a,X,Z)- {\varphi_0^{\mathrm{ATE}}}(a,X,Z)\}
\{h_0(a,X,W)-\hat h(a,X,W)\}
\mid A=a].
\]
Multiplying both sides by \(-1\) proves the claim.
\end{proof}

\paragraph{Projected residual bounds for the population DR remainder.}
Define the conditional bridge-error constants
\[
B_{e,\varphi}^2
:=
\operatorname*{ess\,sup}_{a\sim \sP_A}
\mathbb E[
\{\hat\varphi(a,X,Z)- {\varphi_0^{\mathrm{ATE}}}(a,X,Z)\}^2
\mid A=a],
\]
and

\[
B_{e,h}^2
:=
\operatorname*{ess\,sup}_{a\sim \sP_A}
\mathbb E[
\|h_0(a,X,W)-\hat h(a,X,W)\|_2^2
\mid A=a].
\]
Conditional on bridge training, assume \(B_{e,\varphi},B_{e,h}<\infty\).
The consistency statement below strengthens this to
\(B_{e,\varphi},B_{e,h}=O_p(1)\); this is boundedness, not strong
bridge convergence.

Define the projected residuals

\[
\mathcal R_h^{\mathrm{proj}}
:=
\|T_h\hat h-m_0\|_{L^2(\sP_{Q_h};\mathbb R^{d_Y})}^2,
\]
and
\[
\mathcal R_\varphi^{\mathrm{proj}}
:=
\|T_\varphi\hat\varphi-r_0\|_{L^2(\sP_{Q_\varphi})}^2.
\]

\begin{lemma}[Population DR remainder controlled by projected residuals]
\label{lem:dr_projected_population_bound}

Assume \(B_{e,\varphi}<\infty\) and \(B_{e,h}<\infty\). Then
\[
\|\bar f_{\mathrm{ATE}}^{\mathrm{DR}}-f_{\mathrm{ATE}}
\|_{L^2(\sP_A;\mathbb R^{d_Y})}^2
\le
B_{e,\varphi}^2
\mathcal R_h^{\mathrm{proj}},
\]
and
\[
\|\bar f_{\mathrm{ATE}}^{\mathrm{DR}}-f_{\mathrm{ATE}}
\|_{L^2(\sP_A;\mathbb R^{d_Y})}^2
\le
B_{e,h}^2
\mathcal R_\varphi^{\mathrm{proj}}.
\]
Consequently,
\[
\|\bar f_{\mathrm{ATE}}^{\mathrm{DR}}-f_{\mathrm{ATE}}
\|_{L^2(\sP_A;\mathbb R^{d_Y})}^2
\le
\min
\left\{
B_{e,\varphi}^2\mathcal R_h^{\mathrm{proj}},
\;
B_{e,h}^2\mathcal R_\varphi^{\mathrm{proj}}
\right\}.
\]
\end{lemma}

\begin{proof}
By Lemma~\ref{lem:ate_dr_identity_projected},
\[
f_{\mathrm{ATE}}(a)-\bar f_{\mathrm{ATE}}^{ {\mathrm{DR}}}(a)
=
\mathbb E[
\{ {\varphi_0^{\mathrm{ATE}}}(a,X,Z)-\hat\varphi(a,X,Z)\}
\{h_0(a,X,W)-\hat h(a,X,W)\}
\mid A=a].
\]

First condition on \((X,Z)\). Since
\[
\mathbb E[h_0(a,X,W)-\hat h(a,X,W)\mid A=a,X,Z]
=
m_0(a,X,Z)-T_h\hat h(a,X,Z),
\]
we obtain
\[
f_{\mathrm{ATE}}(a)-\bar f_{\mathrm{ATE}}^{ {\mathrm{DR}}}(a)
=
\mathbb E[
\{ {\varphi_0^{\mathrm{ATE}}}(a,X,Z)-\hat\varphi(a,X,Z)\}
\{m_0(a,X,Z)-T_h\hat h(a,X,Z)\}
\mid A=a].
\]
By Cauchy--Schwarz,

\begin{align*}
\|f_{\mathrm{ATE}}(a)-\bar f_{\mathrm{ATE}}^{\mathrm{DR}}(a)\|_2^2
&\le
\mathbb E[
\{ {\varphi_0^{\mathrm{ATE}}}(a,X,Z)-\hat\varphi(a,X,Z)\}^2
\mid A=a]
\\
&\times
\mathbb E[
\|m_0(a,X,Z)-T_h\hat h(a,X,Z)\|_2^2
\mid A=a].
\end{align*}

Using the definition of \(B_{e,\varphi}\) and integrating over
\({a\sim\sP_A}\) gives

\[
\|\bar f_{\mathrm{ATE}}^{\mathrm{DR}}-f_{\mathrm{ATE}}
\|_{L^2(\sP_A;\mathbb R^{d_Y})}^2
\le
B_{e,\varphi}^2
\mathcal R_h^{\mathrm{proj}}.
\]

Secondly, condition on \((X,W)\). Since
\[
\mathbb E[ {\varphi_0^{\mathrm{ATE}}}(a,X,Z)-\hat\varphi(a,X,Z)\mid A=a,X,W]
=
r_0(a,X,W)-T_\varphi\hat\varphi(a,X,W),
\]
we similarly obtain
\[
f_{\mathrm{ATE}}(a)-\bar f_{\mathrm{ATE}}^{ {\mathrm{DR}}}(a)
=
\mathbb E[
\{h_0(a,X,W)-\hat h(a,X,W)\}
\{r_0(a,X,W)-T_\varphi\hat\varphi(a,X,W)\}
\mid A=a].
\]
Cauchy--Schwarz and integration over \(a\sim \sP_A\) give

\[
\|\bar f_{\mathrm{ATE}}^{\mathrm{DR}}-f_{\mathrm{ATE}}
\|_{L^2(\sP_A;\mathbb R^{d_Y})}^2
\le
B_{e,h}^2
\mathcal R_\varphi^{\mathrm{proj}}.
\]

{Taking the smaller of the two valid bounds completes the
proof.}
\end{proof}

\paragraph{Final-stage residual regression.}

Define the residual pseudo-outcome
\[
\Gamma_{\mathrm{DR}}
:=
\hat\varphi(A,X,Z)\{Y-\hat h(A,X,W)\}
\in\mathbb R^{d_Y}.
\]
Then
\[
K_{\hat h,\hat\varphi}(a)
=
\mathbb E[\Gamma_{\mathrm{DR}}\mid A=a].
\]
Let
\(
\mathcal D_{\mathrm{DR}}
=
\{(y_i,a_i,x_i,z_i,w_i)\}_{i \in \gI_3^{(\mathrm{DR})}}
\)
be the independent final-stage sample, with $| \gI_3^{(\mathrm{DR})}| =  {t_{\mathrm{DR}}}$. {Set
\(y_i^{(\mathrm{DR})}:=\hat\varphi(a_i,x_i,z_i)
\{y_i-\hat h(a_i,x_i,w_i)\}\), matching the main-text pseudo-outcome
notation.} Let \(\nu_{\mathrm{DR}}\) index a
class
\[
\mathcal K_{\mathrm{DR},\nu_{\mathrm{DR}}}
\subset\{k:\mathcal A\to\mathbb R^{d_Y}\}.
\]

Writing
\(
\widehat{\gR}_{\mathrm{DR},3}(k)
:=
t_{\mathrm{DR}}^{-1}
\sum_{i \in \gI_3^{(\mathrm{DR})}}
\| y_i^{(\mathrm{DR})}-k(a_i)\|_2^2
\)
for the final-stage empirical risk, let
\(\hat K_{\mathrm{DR}}\in\mathcal K_{\mathrm{DR},\nu_{\mathrm{DR}}}\) be a measurable
approximate empirical minimizer,
\[
\widehat{\gR}_{\mathrm{DR},3}(\hat K_{\mathrm{DR}})
\le
\inf_{k\in\mathcal K_{\mathrm{DR},\nu_{\mathrm{DR}}}}
\widehat{\gR}_{\mathrm{DR},3}(k)
+\xi_{\mathrm{DR},3,\nu_{\mathrm{DR}}},
\qquad
\xi_{\mathrm{DR},3,\nu_{\mathrm{DR}}}\ge0.
\]

 {For the neural parameterization in the algorithmic section,
\(\hat K_{\mathrm{DR}}(a)=
k^{\mathrm{DR}}(a;\hat\theta_3^{\mathrm{DR}})\).}
Define the approximation error
\[
\kappa_{\mathrm{DR},\nu_{\mathrm{DR}}}
:=
\inf_{k\in\mathcal K_{\mathrm{DR},\nu_{\mathrm{DR}}}}
\|k-K_{\hat h,\hat\varphi}\|_{L^2(\sP_A;\mathbb R^{d_Y})}^2.
\]
Define the residual-regression loss class
\[
\mathcal L_{\mathrm{DR},\nu_{\mathrm{DR}}}
:=
\left\{
(\gamma,a)\mapsto
\|\gamma-k(a)\|_2^2
:
k\in\mathcal K_{\mathrm{DR},\nu_{\mathrm{DR}}}
\right\}.
\]

\begin{assumption}[DR final-stage boundedness]
\label{ass:dr_final_stage_boundedness}
Conditional on the trained bridge estimators \((\hat h,\hat\varphi)\), assume:
\begin{itemize}
    \item The residual-regression loss class has a
    common deterministic envelope \(M_{\mathrm{DR}}<\infty\):
    \[
    0\leq \ell\leq M_{\mathrm{DR}}
    \quad\text{for every }
    \ell\in\mathcal L_{\mathrm{DR},\nu_{\mathrm{DR}}}.
    \]

    \item The final-stage sample
    \(\mathcal D_{\mathrm{DR}}\) is independent of all data used to train
    \(\hat h\) and \(\hat\varphi\), including fitted preprocessing
    transformations.
\end{itemize}
\end{assumption}

Define
\[
\Delta_{\mathrm{DR},\nu_{\mathrm{DR}}}(\delta)
:=
2\widehat{\mathfrak R}_{\mathcal D_{\mathrm{DR}}}(\mathcal L_{\mathrm{DR},\nu_{\mathrm{DR}}})
+
3M_{\mathrm{DR}}\sqrt{\frac{\log(2/\delta)}{2 {t_{\mathrm{DR}}}}}.
\]

\begin{lemma}[Final-stage residual-regression bound]
\label{lem:dr_final_stage_rademacher}

Under Assumption~\ref{ass:dr_final_stage_boundedness}, conditional on the
trained bridge estimators, with probability at least \(1-2\delta\),
\[
\|\hat K_{\mathrm{DR}}-K_{\hat h,\hat\varphi}
\|_{L^2(\sP_A;\mathbb R^{d_Y})}^2
\le
\kappa_{\mathrm{DR},\nu_{\mathrm{DR}}}
+
2\Delta_{\mathrm{DR},\nu_{\mathrm{DR}}}(\delta)
{{}+\xi_{\mathrm{DR},3,\nu_{\mathrm{DR}}}}.
\]
\end{lemma}

\begin{proof}
Conditional on \((\hat h,\hat\varphi)\), this is ordinary least-squares regression with response
\(\Gamma_{\mathrm{DR}}\) and covariate \(A\). Since

\[
K_{\hat h,\hat\varphi}(A)
=
\mathbb E[\Gamma_{\mathrm{DR}}\mid A,\hat h,\hat\varphi],
\]
the least-squares projection identity gives
\[
\mathbb E[
\|\Gamma_{\mathrm{DR}}-k(A)\|_2^2
\mid \hat h,\hat\varphi]
-
\mathbb E[
\|\Gamma_{\mathrm{DR}}-K_{\hat h,\hat\varphi}(A)\|_2^2
\mid \hat h,\hat\varphi]
=
\|k-K_{\hat h,\hat\varphi}\|_{L^2(\sP_A;\mathbb R^{d_Y})}^2.
\]

Apply both directions of
Proposition~\ref{prop:uniformRademacherDeviation} and use {approximate}
empirical
optimality of \(\hat K_{\mathrm{DR}}\){, which adds
\(\xi_{\mathrm{DR},3,\nu_{\mathrm{DR}}}\) once}. Their intersection has
conditional probability at least \(1-2\delta\), which proves the result.
\end{proof}

On the same final-stage sample, set
\[
\hat\mu_h(a):=\frac1{ {t_{\mathrm{DR}}}}
\sum_{i=1}^{ {t_{\mathrm{DR}}}}\hat h(a,x_i,w_i),
\qquad
\mathcal E_{\mu,h}
:=
\|\hat\mu_h-\mu_{\hat h}\|_{L^2(\sP_A;\mathbb R^{d_Y})}^2.
\]
Theorem~\ref{thm:outcome_empirical_averaging_error} gives
\(\mathcal E_{\mu,h}=O_p( {t_{\mathrm{DR}}^{-1}})\). The two final-stage
components may use the same fold; their mutual independence is unnecessary.

The {\textsc{DR-PCL-Net}} estimator is
\[
\hat f_{\mathrm{ATE}}^{\mathrm{DR}}(a)
:=
\hat\mu_h(a)+\hat K_{\mathrm{DR}}(a).
\]

\begin{theorem}[{ {\textsc{DR-PCL-Net}} ATE error bound}]
\label{thm:dr_ate_rademacher}

Suppose valid \(h_0\) and \( {\varphi_0^{\mathrm{ATE}}}\) exist, the conditional
bridge-error bounds above are finite, and
Assumption~\ref{ass:dr_final_stage_boundedness} holds. Conditional on bridge
training, with probability at least \(1-2\delta\),
\begin{align*}
\|\hat f_{\mathrm{ATE}}^{\mathrm{DR}}-f_{\mathrm{ATE}}
\|_{L^2(\sP_A;\mathbb R^{d_Y})}^2
\le{}&
4\mathcal E_{\mu,h}
+
4\kappa_{\mathrm{DR},\nu_{\mathrm{DR}}}
+
8\Delta_{\mathrm{DR},\nu_{\mathrm{DR}}}(\delta)
{{}+4\xi_{\mathrm{DR},3,\nu_{\mathrm{DR}}}}\\
&+
2
\min
\left\{
B_{e,\varphi}^2\mathcal R_h^{\mathrm{proj}},
\;
B_{e,h}^2\mathcal R_\varphi^{\mathrm{proj}}
\right\}.
\end{align*}
{The multiplier \(4\) is the one already applied to the residual-regression
error: the decomposition uses \(\|u+v\|^2\le2\|u\|^2+2\|v\|^2\) twice. The projected
bridge errors \(\mathcal R_h^{\mathrm{proj}},\mathcal R_\varphi^{\mathrm{proj}}\) are those of
Theorems~\ref{thm:outcome_projected_rademacher} and
\ref{thm:treatment_projected_plugin_rademacher}, and therefore already carry their own
optimization gaps; the exact doubly robust product-remainder identity is unchanged.}
\end{theorem}

\begin{proof}
Decompose
\[
\hat f_{\mathrm{ATE}}^{ {\mathrm{DR}}}-f_{\mathrm{ATE}}
=
(\hat f_{\mathrm{ATE}}^{ {\mathrm{DR}}}
-\bar f_{\mathrm{ATE}}^{ {\mathrm{DR}}})
+
(\bar f_{\mathrm{ATE}}^{ {\mathrm{DR}}}-f_{\mathrm{ATE}}).
\]

Using \(\|u+v\|_2^2\le2\|u\|_2^2+2\|v\|_2^2\),
\[
\|\hat f_{\mathrm{ATE}}^{\mathrm{DR}}-f_{\mathrm{ATE}}
\|_{L^2(\sP_A;\mathbb R^{d_Y})}^2
\le
2\|\hat f_{\mathrm{ATE}}^{\mathrm{DR}}-\bar f_{\mathrm{ATE}}^{\mathrm{DR}}
\|_{L^2(\sP_A;\mathbb R^{d_Y})}^2
+
2\|\bar f_{\mathrm{ATE}}^{\mathrm{DR}}-f_{\mathrm{ATE}}
\|_{L^2(\sP_A;\mathbb R^{d_Y})}^2.
\]

Moreover,

\[
\hat f_{\mathrm{ATE}}^{\mathrm{DR}}-\bar f_{\mathrm{ATE}}^{\mathrm{DR}}
=
(\hat\mu_h-\mu_{\hat h})
+
(\hat K_{\mathrm{DR}}-K_{\hat h,\hat\varphi}).
\]
Therefore,
\[
\|\hat f_{\mathrm{ATE}}^{\mathrm{DR}}-\bar f_{\mathrm{ATE}}^{\mathrm{DR}}
\|_{L^2(\sP_A;\mathbb R^{d_Y})}^2
\le
2\mathcal E_{\mu,h}
+
2\|\hat K_{\mathrm{DR}}-K_{\hat h,\hat\varphi}
\|_{L^2(\sP_A;\mathbb R^{d_Y})}^2.
\]

The residual-regression term is bounded by Lemma~\ref{lem:dr_final_stage_rademacher}. The population DR remainder is bounded by Lemma~\ref{lem:dr_projected_population_bound}. Combining these bounds gives the theorem.
\end{proof}

\begin{assumption}[Bounded complementary bridge errors]
\label{assum:bounded_complementary_bridge_errors}
Along the asymptotic sequence, the action-uniform conditional bridge-error
envelopes defined above satisfy
\[
    B_{e,\varphi}=O_p(1),
    \qquad
    B_{e,h}=O_p(1).
\]
This assumption requires stochastic boundedness of the bridge error on
the route complementary to the vanishing projected error; it does not
require strong consistency of either fitted bridge.
\end{assumption}

\begin{corollary}[Projected doubly robust consistency]
\label{cor:dr_ate_projected_consistency}

Let \( {\delta_t}\downarrow0\) satisfy
\(\log(1/ {\delta_t})/ {t_{\mathrm{DR}}}\to0\), and define
\[
 {\rho_{\mathrm{DR},3,\nu_{\mathrm{DR}}}}
:=
\kappa_{\mathrm{DR},\nu_{\mathrm{DR}}}
+2\Delta_{\mathrm{DR},\nu_{\mathrm{DR}}}( {\delta_t}).
\]

Define the optimization-aware final-stage rate
\(
\tilde\rho_{\mathrm{DR},3,\nu_{\mathrm{DR}}}
:=
\rho_{\mathrm{DR},3,\nu_{\mathrm{DR}}}
+\xi_{\mathrm{DR},3,\nu_{\mathrm{DR}}}.
\)

{Suppose
Assumption~\ref{assum:bounded_complementary_bridge_errors}
holds,} \( {t_{\mathrm{DR}}}\to\infty\),
\[
\kappa_{\mathrm{DR},\nu_{\mathrm{DR}}}=o_p(1),
\qquad
\widehat{\mathfrak R}_{\mathcal D_{\mathrm{DR}}}
(\mathcal L_{\mathrm{DR},\nu_{\mathrm{DR}}})=o_p(1),
\qquad
\xi_{\mathrm{DR},3,\nu_{\mathrm{DR}}}=o_p(1).
\]
Suppose further that the rate conclusions
\[
\mathcal R_h^{\mathrm{proj}}=O_p(\tilde\rho_{h,\nu_h}),
\qquad
\mathcal R_\varphi^{\mathrm{proj}}
=O_p(\tilde\rho_{\varphi,\nu_\varphi})
\]
of Corollaries~\ref{cor:outcomenet_projected_consistency} and \ref{cor:treatmentnet_projected_consistency_plugin} hold, and that Theorem~\ref{thm:outcome_empirical_averaging_error} applies with \(t_h= {t_{\mathrm{DR}}}\).
Then \( \tilde\rho_{\mathrm{DR},3,\nu_{\mathrm{DR}}}=o_p(1)\), and these results give
\begin{align}
\|\hat f_{\mathrm{ATE}}^{\mathrm{DR}}-f_{\mathrm{ATE}}
\|_{L^2(\sP_A;\mathbb R^{d_Y})}^2
=O_p\!\left(
{t_{\mathrm{DR}}^{-1}}
+\tilde\rho_{\mathrm{DR},3,\nu_{\mathrm{DR}}}
+\min\{\tilde\rho_{h,\nu_h},\tilde\rho_{\varphi,\nu_\varphi}\}
\right).
\label{eq:DRATEConsistencyRate}
\end{align}
Hence {\textnormal{\textsc{DR-PCL-Net}}} is consistent if the final-stage errors and the
final-stage optimization gap vanish and either
\(\tilde\rho_{h,\nu_h}\to0\) or
\(\tilde\rho_{\varphi,\nu_\varphi}\to0\), that is, if one route's approximation,
empirical-process and optimization terms all vanish. There is no additive
{density-ratio estimation error}.

\end{corollary}

\begin{remark}[Projected double robustness]
The theorem gives double robustness in projected-residual form. The estimator is consistent if either the outcome-side projected residual
\[
T_h\hat h-m_0
\]
vanishes, or the treatment-side projected residual
\[
T_\varphi\hat\varphi- r_0
\]
vanishes. No strong \(L^2\)-convergence of either bridge is required. A second-order product rate in strong bridge errors would require additional strong-norm or conditional-profile assumptions and is not used here.

The bridge on the opposite side enters only through the
\(O_p(1)\) conditional error bound. For a fixed number
\(K_{\mathrm{cf}}\) of cross-fitting
folds, the fold-averaged estimator inherits the result because
\[
\left\|K_{\mathrm{cf}}^{-1}\sum_{k=1}^{K_{\mathrm{cf}}}
(\hat f_k-f_{\mathrm{ATE}})\right\|_
{L^2(\sP_A;\mathbb R^{d_Y})}^2
\leq
K_{\mathrm{cf}}^{-1}\sum_{k=1}^{K_{\mathrm{cf}}}
\|\hat f_k-f_{\mathrm{ATE}}\|_{L^2(\sP_A;\mathbb R^{d_Y})}^2.
\]

\end{remark}

\subsection{Self-normalized density-ratio-free heterogeneous dose-response}
\label{app:self_normalized_cate_consistency}

Write \(X=(S,V)\) and
\(\hat\varphi^{\mathrm{ATE}}(a,v,s,z):=
\hat\varphi(a,(s,v),z)\), where \(\hat\varphi\) is the estimator from Appendix~\ref{app:treatment_bridge_consistency_rademacher}. Condition throughout on the samples used to fit \(\hat h\) and \(\hat\varphi^{\mathrm{ATE}}\). Let

\begin{align*}
r_A&:=r_0:=r_{0,\mathrm{ATE}},
&r_C&:= {r_{0,\mathrm{CATE}}},\\
e_A&:=T_\varphi\hat\varphi^{\mathrm{ATE}}-r_A,
&\mathcal E_{\varphi,A}&:=
\|e_A\|_{L^2(\sP_{A,V,S,W})}^2.
\end{align*}

Also, define
\begin{align*}
D_{\hat\varphi}(a,v)
&:=\E[\hat\varphi^{\mathrm{ATE}}(a,v,S,Z)\mid A=a,V=v],\\
\bar e_A(a,v)
&:=\E[e_A(a,v,S,W)\mid A=a,V=v],\\
\widetilde\varphi(a,v,s,z)
&:=\frac{\hat\varphi^{\mathrm{ATE}}(a,v,s,z)}
{D_{\hat\varphi}(a,v)}.
\end{align*}

\begin{lemma}[ATE-to-CATE projected-error transfer]
\label{lem:ate_to_cate_projected_transfer}
Suppose, on the evaluation region,
\[
D_{\hat\varphi}(a,v)\ge d_C>0,
\qquad
|r_C(a,v,s,w)|\le B_C<\infty.
\]
Then
\begin{equation}
\mathcal E_{\varphi,C}
:=
\|T_\varphi^{\mathrm{CATE}}\widetilde\varphi-r_C
\|_{L^2(\sP_{A,V,S,W})}^2
\le
C_C\mathcal E_{\varphi,A},
\qquad
C_C:=\frac{2(1+B_C^2)}{d_C^2}.
\label{eq:ate_to_cate_projected_transfer}
\end{equation}
\end{lemma}

\begin{proof}
By total law of expectation and
Proposition~\ref{prop:AlternativeCATEidentification}, we have
\(D_{\hat\varphi}=d_0+\bar e_A\) and \(r_C=r_A/d_0\). As  result
\[
T_\varphi^{\mathrm{CATE}}\widetilde\varphi-r_C
=\frac{e_A-r_C\bar e_A}{D_{\hat\varphi}}.
\]
Then, the result follows from \((x+y)^2\le2x^2+2y^2\) and
\(\|\bar e_A\|_{L^2(\sP_{A,V})}\le
\|e_A\|_{L^2(\sP_{A,V,S,W})}\), by Jensen Inequality.
\end{proof}
\paragraph{Joint third-stage regression.}
For \(Y\in\mathbb R^{d_Y}\), define
\begin{align*}
\Gamma_N&:=Y\hat\varphi^{\mathrm{ATE}}(A,V,S,Z),\\
\Gamma_K&:=\hat\varphi^{\mathrm{ATE}}(A,V,S,Z)
\{Y-\hat h(A,V,S,W)\},\\
\Gamma_D&:=\hat\varphi^{\mathrm{ATE}}(A,V,S,Z),
\end{align*}
and let \(N:=N_{\hat\varphi}:=\E[\Gamma_N\mid A,V]\) and
\(K:=K_{\hat h,\hat\varphi}:=\E[\Gamma_K\mid A,V]\); also note that
\(\E[\Gamma_D\mid A,V]=D_{\hat\varphi}\). For \(G\in\{N,K\}\), let
\[
\mathcal D_{3,G}:=\{(a_i,v_i,\gamma_{G,i},\gamma_{D,i})\}_{i=1}^{t_G}
\]
be conditionally i.i.d. and independent of the bridge-fitting samples. Let
\(\mathcal F_{G,t_G}\) be a shared-network of \((d_Y+1)\)-output class contained
in
\[
\{(g,d):\mathcal A\times\mathcal V
\longrightarrow\mathbb R^{d_Y}\times\mathbb R\}
\]
with target \(F_G^0:=(G,D_{\hat\varphi})\), and define
\[
\kappa_{G,t_G}
:=
\inf_{F\in\mathcal F_{G,t_G}}
\|F-F_G^0\|_{L^2(\sP_{A,V};\mathbb R^{d_Y+1})}^2.
\]
Let \(\mathcal L_{G,t_G}\) be the induced real-valued joint squared-loss
class,
\[
\mathcal L_{G,t_G}
:=
\left\{
(a,v,\gamma_G,\gamma_D)\mapsto
\|\gamma_G-g(a,v)\|_2^2+\{\gamma_D-d(a,v)\}^2
:(g,d)\in\mathcal F_{G,t_G}
\right\}.
\]
Assume this class is bounded by \(M_G<\infty\), and set
\[
\Delta_{G,t_G}(\delta)
:=
2\widehat{\mathfrak R}_{\mathcal D_{3,G}}(\mathcal L_{G,t_G})
+3M_G\sqrt{\frac{\log(2/\delta)}{2t_G}},
\qquad
 {\rho_{G,3,t_G}(\delta)}:=\kappa_{G,t_G}+2\Delta_{G,t_G}(\delta).
\]

Let \(\hat F_G=(\hat G,\hat D_G)\in\mathcal F_{G,t_G}\) be a measurable approximate
empirical minimizer of the joint empirical risk \(\widehat{\gR}_{G,3}\),
\[
\widehat{\gR}_{G,3}(\hat F_G)
\le
\inf_{F\in\mathcal F_{G,t_G}}\widehat{\gR}_{G,3}(F)
+\xi_{G,3,t_G},
\qquad \xi_{G,3,t_G}\ge0,
\qquad G\in\{N,K\},
\]
and define the optimization-aware third-stage rate
\(
\tilde\rho_{G,3,t_G}(\delta):=\rho_{G,3,t_G}(\delta)+\xi_{G,3,t_G}.
\)
Because the numerator and the normalizer are fitted by one shared network, there is a single
gap per regression and no separate denominator gap. With probability at least \(1-2\delta\),
Proposition~\ref{prop:uniformRademacherDeviation} gives

\begin{equation}
\|\hat G-G\|_{L^2(\sP_{A,V};\mathbb R^{d_Y})}^2
+\|\hat D_G-D_{\hat\varphi}\|_{L^2(\sP_{A,V})}^2
\le  \tilde\rho_{G,3,t_G}(\delta).
\label{eq:cate_joint_third_stage_bound}
\end{equation}

Here \((\hat G,\hat D_G)=(\hat N^{(\varphi)},\hat D^{(\varphi)})\) for {\textsc{TreatmentNet}} and \((\hat K^{\mathrm{DR}},\hat D^{\mathrm{DR}})\) for the doubly robust estimator. Thus, joint learning changes only the third-stage approximation and Rademacher terms: the shared class must approximate both targets simultaneously.

Choose \(0<\epsilon_D<d_C\) and set
\(\hat D_{G,\epsilon}:=\max\{\hat D_G,\epsilon_D\}\). If
\(\|G(a,v)\|_2\le B_G\), then Equation~\ref{eq:cate_joint_third_stage_bound} yields
\begin{equation}
\left\|
\frac{\hat G}{\hat D_{G,\epsilon}}-\frac{G}{D_{\hat\varphi}}
\right\|_{L^2(\sP_{A,V};\mathbb R^{d_Y})}^2
\le C_G {\tilde\rho_{G,3,t_G}(\delta)},
\qquad
C_G:=\frac{2}{\epsilon_D^2}
\left(1+\frac{B_G^2}{d_C^2}\right).
\label{eq:cate_joint_ratio_bound}
\end{equation}
This follows from the \(1\)-Lipschitz property of clipping and a ratio decomposition.

\begin{theorem}[Self-normalized CATE consistency]
\label{thm:self_normalized_cate_consistency}
Suppose the conditions above hold and that valid bridges
\(h_0\) and \(\varphi_0^{\mathrm{CATE}}\) exist. Let
\begin{align*}
(T_h^{\mathrm{CATE}}h)(a,v,s,z)
&:=\E[h(a,v,s,W)\mid A=a,V=v,S=s,Z=z],\\
m_0^{\mathrm{CATE}}(a,v,s,z)
&:=\E[Y\mid A=a,V=v,S=s,Z=z].
\end{align*}
Define
\begin{align*}
M_{\hat h}(a,v)
&:=\E[\hat h(a,v,S,W)\mid V=v],\\
\mathcal E_{h,C}
&:=\|T_h^{\mathrm{CATE}}\hat h-m_0^{\mathrm{CATE}}
\|_{L^2(\sP_{A,V,S,Z};\mathbb R^{d_Y})}^2,\\
\mathcal E_M
&:=\|\hat f_{\mathrm{CATE}}^{(h)}
-M_{\hat h}
\|_{L^2(\sP_{A,V};\mathbb R^{d_Y})}^2.
\end{align*}
Assume the relevant conditional second-moment bounds are finite, and write
them as
\begin{align*}
C_{h,C}^2
&:=\operatorname*{ess\,sup}_{a,v}
\E[\|h_0(a,v,S,W)\|_2^2\mid A=a,V=v],\\
B_{\varphi,C}^2
&:=\operatorname*{ess\,sup}_{a,v}
\E[(\widetilde\varphi(a,v,S,Z)
-\varphi_0^{\mathrm{CATE}}(a,v,S,Z))^2\mid A=a,V=v],\\
B_{h,C}^2
&:=\operatorname*{ess\,sup}_{a,v}
\E[\|\hat h(a,v,S,W)-h_0(a,v,S,W)\|_2^2\mid A=a,V=v].
\end{align*}
For every \(\delta\in(0,1/4)\), each bound below holds conditionally with
probability at least \(1-2\delta\); if both third-stage regressions are
fitted, the two bounds hold simultaneously with probability at least
\(1-4\delta\). The {\textnormal{\textsc{TreatmentNet}}} estimator in
Equation~\ref{eq:CATESelfNormalizedTreatmentEstimator} satisfies
\begin{equation}
\|\hat f_{\mathrm{CATE}}^{(\varphi)}-f_{\mathrm{CATE}}
\|_{L^2(\sP_{A,V};\mathbb R^{d_Y})}^2
\le
2C_N {\tilde\rho_{N,3,t_N}(\delta)}
+2C_{h,C}^2C_C\mathcal E_{\varphi,A}.
\label{eq:self_normalized_treatment_cate_bound}
\end{equation}
Moreover, the estimator in
Equation~\ref{eq:CATESelfNormalizedDREstimator} satisfies
\begin{align}
\|\hat f_{\mathrm{CATE}}^{\mathrm{DR}}-f_{\mathrm{CATE}}
\|_{L^2(\sP_{A,V};\mathbb R^{d_Y})}^2
\le{}&
3\mathcal E_M+3C_K {\tilde\rho_{K,3,t_K}(\delta)}\nonumber\\
&+3\min\left\{
B_{\varphi,C}^2\mathcal E_{h,C},
B_{h,C}^2C_C\mathcal E_{\varphi,A}
\right\}.
\label{eq:self_normalized_dr_cate_bound}
\end{align}
\end{theorem}

\begin{proof}
The outcome bridge and a change of measure give
\[
\frac{N_{\hat\varphi}}{D_{\hat\varphi}}-f_{\mathrm{CATE}}
=
\E\!\left[
h_0(a,v,S,W)
\{T_\varphi^{\mathrm{CATE}}\widetilde\varphi-r_C\}
\,\middle|\,A=a,V=v
\right].
\]
Thus, conditional Cauchy--Schwarz gives
\(\|N_{\hat\varphi}/D_{\hat\varphi}-f_{\mathrm{CATE}}\|_2^2
\le C_{h,C}^2\mathcal E_{\varphi,C}\). Combine this with
Equations~\ref{eq:ate_to_cate_projected_transfer} and
\ref{eq:cate_joint_ratio_bound} to obtain
Equation~\ref{eq:self_normalized_treatment_cate_bound}. For the DR method,
the exact population remainder is
\[
\E[(\widetilde\varphi-\varphi_0^{\mathrm{CATE}})
\{h_0-\hat h\}\mid A,V].
\]
Conditioning successively on \((A,V,S,Z)\) and \((A,V,S,W)\), as in
Lemma~\ref{lem:dr_projected_population_bound}, bounds its squared
\(L^2(\sP_{A,V})\)-norm by
\[
\min\{B_{\varphi,C}^2\mathcal E_{h,C},
B_{h,C}^2\mathcal E_{\varphi,C}\}.
\]
Equation~\ref{eq:self_normalized_dr_cate_bound} now follows from
Equation~\ref{eq:cate_joint_ratio_bound} and
\(\|u+v+w\|_2^2\le3(\|u\|_2^2+\|v\|_2^2+\|w\|_2^2)\).
\end{proof}

Corollaries~\ref{cor:outcomenet_projected_consistency}
and~\ref{cor:treatmentnet_projected_consistency_plugin} give
\[
\mathcal E_{h,C}=O_p(\tilde\rho_{h,\nu_h}),
\qquad
\mathcal E_{\varphi,A}=O_p(\tilde\rho_{\varphi,\nu_\varphi}),
\]
so the CATE bounds inherit the bridge-stage optimization gaps through the ATE treatment
bridge, which the CATE construction reuses; no new bridge-training gap is introduced here.
Choose \(\delta_G\downarrow0\) such that \(\log(1/\delta_G)/t_G\to0\). If \(\kappa_{G,t_G}=o_p(1)\), \(\widehat{\mathfrak R}_{\mathcal D_{3,G}}(\mathcal L_{G,t_G})=o_p(1)\), \(M_G=O_p(1)\), and \(\xi_{G,3,t_G}=o_p(1)\), then \( \tilde\rho_{G,3,t_G}:=\tilde\rho_{G,3,t_G}(\delta_G)=o_p(1)\). Provided all remaining constants in the bounds are \(O_p(1)\), {\textsc{TreatmentNet}} is consistent when \( \tilde\rho_{N,3,t_N}=o_p(1)\) and \(\tilde\rho_{\varphi,\nu_\varphi}\to0\). The self-normalized DR estimator is consistent when \(\mathcal E_M, \tilde\rho_{K,3,t_K}=o_p(1)\) and either \(\tilde\rho_{h,\nu_h}\to0\) or \(\tilde\rho_{\varphi,\nu_\varphi}\to0\). Thus, joint learning adds one joint third-stage regression error and its optimization gap; it neither alters projected double robustness nor requires strong bridge convergence.

\subsection{{Treatment-bridge consistency for conditional dose-response}}
\label{app:att_treatment_bridge_consistency}

We analyze the complete ordered-pair objective in
Equation~\ref{eq:ATTEmpiricalRatioFreeRisk}. Retain the first-stage objects
\(
\Phi_{\nu_\varphi},\mathcal M_{\nu_\varphi},
\mu_\phi^{(\varphi)},\hat\mu_\phi^{(\varphi)},
\kappa_{\varphi,1,\nu_\varphi}
\), and
\(\Delta_{\varphi,1,\nu_\varphi}\)
from Appendix~\ref{app:treatment_bridge_consistency_rademacher}. The reference
level \(A'\) does not enter this conditional mean embedding.

Let
\[
\Theta_{\nu_\varphi}^{\mathrm{ATT}}
\subset
\{\theta:\mathcal A\times\mathcal A\times\mathcal X
\to\mathbb R^{d_{\nu_\varphi}}\},
\]
and define
\begin{align*}
\varphi_{\theta,\phi}^{\mathrm{ATT}}(a,a',x,z)
&:=\langle\theta(a,a',x),\phi(z)\rangle,\\
g_{\theta,\phi}^{\mathrm{ATT}}(a,a',x,w)
&:=\langle\theta(a,a',x),
\mu_\phi^{(\varphi)}(a,x,w)\rangle
=(T_{\mathrm{ATT}}\varphi_{\theta,\phi}^{\mathrm{ATT}})(a,a',x,w),\\
\hat g_{\theta,\phi}^{\mathrm{ATT}}(a,a',x,w)
&:=\langle\theta(a,a',x),
\hat\mu_\phi^{(\varphi)}(a,x,w)\rangle.
\end{align*}
For this operator, write
\[
\|u\|_{T_{\mathrm{ATT}}}
:=\|T_{\mathrm{ATT}}u\|_{L^2(\sP_{\mathrm{ATT}})}.
\]
Assume the bound \(\|\theta(a,a',x)\|_2\le B_\theta\) holds uniformly
over \(\Theta_{\nu_\varphi}^{\mathrm{ATT}}\), in addition to the first-stage
bounds in Assumption~\ref{ass:treatment_boundedness_rademacher}. Then
\begin{equation}
\sup_{\theta,\phi}
\|\hat g_{\theta,\phi}^{\mathrm{ATT}}
-g_{\theta,\phi}^{\mathrm{ATT}}\|_{L^2(\sP_{\mathrm{ATT}})}^2
\le B_\theta^2\epsilon_{\varphi,1,\nu_\varphi},
\label{eq:att_plugin_projection_error}
\end{equation}
because the independent anchor \(A'\) integrates out.

For \(q=(a,x,w)\) and \(q'=(a',x',w')\), define the conditional plug-in
kernel
\[
\hat\ell_{\theta,\phi}^{\mathrm{ATT}}(q,q')
:=
\{\hat g_{\theta,\phi}^{\mathrm{ATT}}(a,a',x,w)\}^2
-2\hat g_{\theta,\phi}^{\mathrm{ATT}}(a,a',x',w').
\]
Notice that the conditional embedding in the transported term is evaluated at
\(\hat\mu_\phi^{(\varphi)}(a,x',w')\): \(a\) remains the source treatment,
while \((a',x',w')\) comes from the anchor unit. Define the deterministic
pairwise loss class
\begin{align*}
\mathcal L_{\varphi,2,\nu_\varphi}^{\mathrm{ATT},\mathrm{RF}}
:=\Bigl\{(q,q')\mapsto{}&
\langle\theta(a,a',x),\eta(a,x,w)\rangle^2\\
&-2\langle\theta(a,a',x'),\eta(a,x',w')\rangle:
\theta\in\Theta_{\nu_\varphi}^{\mathrm{ATT}},\quad
\eta\in\mathcal M_{\nu_\varphi}
\Bigr\}.
\end{align*}
Its envelope is
\[
M_{\varphi,2}^{\mathrm{ATT},\mathrm{RF}}
:=B_\theta^2B_\mu^2+2B_\theta B_\mu.
\]
On the independent second-stage sample
\(D_2^{(\varphi)}=\{\tilde q_{\varphi,i}\}_{i \in \gI_2^{(\varphi)}}\), set
\begin{align*}
\Delta_{\varphi,2,\nu_\varphi}^{\mathrm{ATT},\mathrm{RF},U}(\delta)
:={}&2\widehat{\mathfrak R}^{(2),U}_{D_2^{(\varphi)}}
(\mathcal L_{\varphi,2,\nu_\varphi}^{\mathrm{ATT},\mathrm{RF}}) +3M_{\varphi,2}^{\mathrm{ATT},\mathrm{RF}}
\sqrt{\frac{8\log(2/\delta)}{m_\varphi}},
\end{align*}
and define
\[
\kappa_{\varphi,2,\nu_\varphi}^{\mathrm{ATT}}
:=
\inf_{\theta\in\Theta_{\nu_\varphi}^{\mathrm{ATT}},\,
\phi\in\Phi_{\nu_\varphi}}
\|g_{\theta,\phi}^{\mathrm{ATT}}-r_{0,\mathrm{ATT}}
\|_{L^2(\sP_{\mathrm{ATT}})}^2.
\]
{Write}
\[
{\widehat{\gR}_{\varphi,2}^{\mathrm{ATT},\mathrm{RF},U}(\theta,\phi)
:=}
\frac{1}{m_\varphi(m_\varphi-1)}
\sum_{i\ne j}
\hat\ell_{\theta,\phi}^{\mathrm{ATT}}
(\tilde q_{\varphi,i},\tilde q_{\varphi,j})
\]
{and let
\((\hat\theta^{\mathrm{ATT}},\hat\phi^{\mathrm{ATT}})
\in\Theta_{\nu_\varphi}^{\mathrm{ATT}}\times\Phi_{\nu_\varphi}\)
be a measurable approximate empirical minimizer, that is,
\(
\widehat{\gR}_{\varphi,2}^{\mathrm{ATT},\mathrm{RF},U}
(\hat\theta^{\mathrm{ATT}},\hat\phi^{\mathrm{ATT}})
\le
\inf_{\theta,\phi}
\widehat{\gR}_{\varphi,2}^{\mathrm{ATT},\mathrm{RF},U}(\theta,\phi)
+\xi_{\varphi,2,\nu_\varphi}^{\mathrm{ATT}}
\)
with \(\xi_{\varphi,2,\nu_\varphi}^{\mathrm{ATT}}\ge0\)}, and set
\(
\hat\varphi^{\mathrm{ATT}}
:=\varphi_{\hat\theta^{\mathrm{ATT}},
\hat\phi^{\mathrm{ATT}}}^{\mathrm{ATT}}.
\)

\begin{theorem}[Projected ATT bridge error]
\label{thm:att_projected_plugin_rademacher}
Suppose the pair-law setup of
Appendix~\ref{sec:NeuralMeanEmbedding_ConditionalDoseResponse} holds, including
\(\sQ_{\mathrm{ATT}}\ll\sP_{\mathrm{ATT}}\) and
\(r_{0,\mathrm{ATT}}\in L^2(\sP_{\mathrm{ATT}})\). Assume also that the
preceding classes are measurable and separable and that both stages use
{measurable approximate empirical minimizers with gaps
\(\xi_{\varphi,1,\nu_\varphi},\xi_{\varphi,2,\nu_\varphi}^{\mathrm{ATT}}\ge0\)}.
For every \(\delta\in(0,1/4)\), with probability at least \(1-4\delta\),
\begin{align*}
\mathcal R_{\varphi,\mathrm{ATT}}^{\mathrm{proj}}
&:=
\|T_{\mathrm{ATT}}\hat\varphi^{\mathrm{ATT}}
-r_{0,\mathrm{ATT}}\|_{L^2(\sP_{\mathrm{ATT}})}^2\\
&\le
4\kappa_{\varphi,2,\nu_\varphi}^{\mathrm{ATT}}
+6B_\theta^2\kappa_{\varphi,1,\nu_\varphi}
+12B_\theta^2\Delta_{\varphi,1,\nu_\varphi}(\delta)
+4\Delta_{\varphi,2,\nu_\varphi}^{\mathrm{ATT},\mathrm{RF},U}(\delta)\\
&\qquad
{{}+6B_\theta^2\xi_{\varphi,1,\nu_\varphi}
+2\xi_{\varphi,2,\nu_\varphi}^{\mathrm{ATT}}}.
\end{align*}
{The pairwise U-process deviation result itself is unchanged.}
If a valid bridge satisfies
\(T_{\mathrm{ATT}}\varphi_0^{\mathrm{ATT}}=r_{0,\mathrm{ATT}}\),
the left-hand side equals
\(\|\hat\varphi^{\mathrm{ATT}}-\varphi_0^{\mathrm{ATT}}
\|_{T_{\mathrm{ATT}}}^2\).
\end{theorem}

\begin{proof}
Conditional on the first-stage sample, the plug-in kernels form a subclass of
\(\mathcal L_{\varphi,2,\nu_\varphi}^{\mathrm{ATT},\mathrm{RF}}\).
Proposition~\ref{prop:uniformUProcessDeviation} therefore yields both
uniform deviation directions, each with probability at least \(1 - \delta\).
{Approximate} empirical optimality and Equation~\ref{eq:ATTRatioFreeExcessRisk} give
\[
\|\hat g_{\hat\theta^{\mathrm{ATT}},
\hat\phi^{\mathrm{ATT}}}^{\mathrm{ATT}}-r_{0,\mathrm{ATT}}
\|_{L^2(\sP_{\mathrm{ATT}})}^2
\le
2\kappa_{\varphi,2,\nu_\varphi}^{\mathrm{ATT}}
+2B_\theta^2\epsilon_{\varphi,1,\nu_\varphi}
+2\Delta_{\varphi,2,\nu_\varphi}^{\mathrm{ATT},\mathrm{RF},U}(\delta)
{{}+\xi_{\varphi,2,\nu_\varphi}^{\mathrm{ATT}}}.
\]
Since
\(T_{\mathrm{ATT}}\hat\varphi^{\mathrm{ATT}}
=g_{\hat\theta^{\mathrm{ATT}},\hat\phi^{\mathrm{ATT}}}^{\mathrm{ATT}}\),
a second use of Equation~\ref{eq:att_plugin_projection_error} yields
\begin{align*}
\mathcal R_{\varphi,\mathrm{ATT}}^{\mathrm{proj}}
\le{}&
4\kappa_{\varphi,2,\nu_\varphi}^{\mathrm{ATT}}
+6B_\theta^2\epsilon_{\varphi,1,\nu_\varphi}
+4\Delta_{\varphi,2,\nu_\varphi}^{\mathrm{ATT},\mathrm{RF},U}(\delta)
{{}+2\xi_{\varphi,2,\nu_\varphi}^{\mathrm{ATT}}}.
\end{align*}
Lemma~\ref{lem:treatment_stage1_rademacher} bounds
\(\epsilon_{\varphi,1,\nu_\varphi}\) by
\(\kappa_{\varphi,1,\nu_\varphi}
+2\Delta_{\varphi,1,\nu_\varphi}(\delta)
{{}+\xi_{\varphi,1,\nu_\varphi}}\){, which contributes
\(6B_\theta^2\xi_{\varphi,1,\nu_\varphi}\)}.
The two first-stage and two conditional U-process events have total failure
probability at most \(4\delta\).
\end{proof}

\begin{corollary}[Projected joint-anchor ATT bridge consistency]
\label{cor:att_projected_consistency}

Let $\gamma_{\varphi,1,\nu_\varphi}$ and $\gamma_{\varphi,2,\nu_\varphi}^{\mathrm{ATT},\mathrm{RF},U}$ be deterministic sequences such that

\begin{align*}
\widehat{\mathfrak R}_{D_1^{(\varphi)}}
(\mathcal L_{\varphi,1,\nu_\varphi})
&=O_p(\gamma_{\varphi,1,\nu_\varphi}),\quad \widehat{\mathfrak R}^{(2),U}_{D_2^{(\varphi)}}
(\mathcal L_{\varphi,2,\nu_\varphi}^{\mathrm{ATT},\mathrm{RF}})
&=O_p(\gamma_{\varphi,2,\nu_\varphi}^{\mathrm{ATT},\mathrm{RF},U}).
\end{align*}
Define
\begin{align*}
\rho_{\varphi,\mathrm{ATT},\nu_\varphi}
:={}&
\kappa_{\varphi,2,\nu_\varphi}^{\mathrm{ATT}}
+B_\theta^2\kappa_{\varphi,1,\nu_\varphi}
+B_\theta^2\gamma_{\varphi,1,\nu_\varphi}
+\gamma_{\varphi,2,\nu_\varphi}^{\mathrm{ATT},\mathrm{RF},U} +
\frac{B_\theta^2M_{\varphi,1}}{\sqrt{n_\varphi}}
+\frac{M_{\varphi,2}^{\mathrm{ATT},\mathrm{RF}}}{\sqrt{m_\varphi}}.
\end{align*}

Following the same convention of absorbing fixed bounded multipliers, define the
optimization-aware ATT bridge rate
\[
\tilde\rho_{\varphi,\mathrm{ATT},\nu_\varphi}
:=
\rho_{\varphi,\mathrm{ATT},\nu_\varphi}
+B_\theta^2\xi_{\varphi,1,\nu_\varphi}
+\xi_{\varphi,2,\nu_\varphi}^{\mathrm{ATT}}.
\]

Then
\[
\|T_{\mathrm{ATT}}\hat\varphi^{\mathrm{ATT}}
-r_{0,\mathrm{ATT}}\|_{L^2(\sP_{\mathrm{ATT}})}^2
=O_p({\tilde\rho_{\varphi,\mathrm{ATT},\nu_\varphi}}).
\]

Suppose $\gamma_{\varphi,1,\nu_\varphi} \to 0$, $\gamma_{\varphi,2,\nu_\varphi}^{\mathrm{ATT},\mathrm{RF},U} \to 0$, and $\kappa_{\varphi,1,\nu_\varphi} \to 0$, $\kappa_{\varphi,2,\nu_\varphi}^{\mathrm{ATT}} \to 0$ {and \(\xi_{\varphi,1,\nu_\varphi},\xi_{\varphi,2,\nu_\varphi}^{\mathrm{ATT}}\to0\)} in probability with \(n_\varphi,m_\varphi\to\infty\). Then,
\[
\|T_{\mathrm{ATT}}\hat\varphi^{\mathrm{ATT}}
-r_{0,\mathrm{ATT}}\|_{L^2(\sP_{\mathrm{ATT}})}^2 \to 0 \enspace \text{in probability.}
\]

If $T_{\mathrm{ATT}}\varphi_0^{\mathrm{ATT}} = r_{0, \mathrm{ATT}}$, then $\|\hat{\varphi}^{\mathrm{ATT}}
- \varphi_0^{\mathrm{ATT}}\|_{T_{\mathrm{ATT}}}^2 \to 0$ in probability.
\end{corollary}

The result controls the full joint-anchor projection under
\(\sP_{\mathrm{ATT}}\); it does not require strong recovery of the raw bridge.

\subsection{{Conditional dose-response consistency of \textsc{TreatmentNet}}}
\label{app:treatmentnet_att_consistency}

Let \(Y\in\mathbb R^{d_Y}\), and condition on the samples used to fit
\(\hat\varphi^{\mathrm{ATT}}\). For a scalar candidate bridge \(u\), define
the induced treatment-conditional counterfactual mean
\[
\bar f_u^{(\varphi,\mathrm{ATT})}(a,a')
:=
\E[Yu(a,a',X,Z)\mid A=a]
\in\mathbb R^{d_Y}.
\]

\begin{lemma}[Projected ATT residual controls the response surface]
\label{lem:att_projected_to_surface}
Suppose Assumption~\ref{ass:treatment_dr_dual_outcome_bridge} holds. Then, for
every candidate \(u\) for which the displayed products and conditional
expectations are well-defined and square-integrable,
\begin{align}
&\|\bar f_u^{(\varphi,\mathrm{ATT})}-f_{\mathrm{ATT}}
\|_{L^2(\sP_A\otimes\sP_A;\mathbb R^{d_Y})}^2 \le
C_h^2\|T_{\mathrm{ATT}}u-r_{0,\mathrm{ATT}}
\|_{L^2(\sP_{\mathrm{ATT}})}^2.
\label{eq:att_projected_to_surface}
\end{align}
\end{lemma}

\begin{proof}
The outcome-bridge equation and iterated expectation give
\begin{align*}
\bar f_u^{(\varphi,\mathrm{ATT})}(a,a')
&=\E[h_0(a,X,W)(T_{\mathrm{ATT}}u)(a,a',X,W)\mid A=a],\\
f_{\mathrm{ATT}}(a,a')
&=\E[h_0(a,X,W)r_{0,\mathrm{ATT}}(a,a',X,W)\mid A=a].
\end{align*}
The second identity is the change of measure from \(A=a\) to \(A=a'\).
Conditional Cauchy--Schwarz and integration over
\((a,a')\sim\sP_A\otimes\sP_A\) prove the result.
\end{proof}

Let \(D_{3,\mathrm{ATT}}^{(\varphi)}=\{o_i\}_{i=1}^{t_{\mathrm{ATT}}}\) be an i.i.d. sample independent of bridge training. For \(i\ne j\), define
\[
\gamma_{ij}^{\mathrm{ATT},\varphi}
:=
y_i\hat\varphi^{\mathrm{ATT}}(a_i,a_j,x_i,z_i)
\in\mathbb R^{d_Y}.
\]
For two independent random copies \(O_1\) and \(O_2\), define
\(\Gamma_{12}^{\mathrm{ATT},\varphi}
:=Y_1\hat\varphi^{\mathrm{ATT}}(A_1,A_2,X_1,Z_1)\). Its conditional mean is
\[
\E[\Gamma_{12}^{\mathrm{ATT},\varphi}\mid A_1=a,A_2=a']
=\bar f_{\hat\varphi^{\mathrm{ATT}}}^{(\varphi,\mathrm{ATT})}(a,a').
\]
Let
\(
\mathcal F_{\varphi,3,t_{\mathrm{ATT}}}^{\mathrm{ATT}}
\subset
\{f:\mathcal A\times\mathcal A\to\mathbb R^{d_Y}\},
\)
and define
\begin{align*}
\kappa_{\varphi,3,t_{\mathrm{ATT}}}^{\mathrm{ATT}}
:={}&
\inf_{f\in\mathcal F_{\varphi,3,t_{\mathrm{ATT}}}^{\mathrm{ATT}}}
\|f-\bar f_{\hat\varphi^{\mathrm{ATT}}}^{(\varphi,\mathrm{ATT})}
\|_{L^2(\sP_A\otimes\sP_A;\mathbb R^{d_Y})}^2,\\
\mathcal L_{\varphi,3,t_{\mathrm{ATT}}}^{\mathrm{ATT}}
(\hat\varphi^{\mathrm{ATT}})
:={}&
\left\{(o_i,o_j)\mapsto
\|y_i\hat\varphi^{\mathrm{ATT}}(a_i,a_j,x_i,z_i)
-f(a_i,a_j)\|_2^2:
f\in\mathcal F_{\varphi,3,t_{\mathrm{ATT}}}^{\mathrm{ATT}}
\right\}.
\end{align*}
Conditional on bridge training, assume this real-valued loss class is
measurable, separable, and bounded by a deterministic envelope
\(M_{\varphi,3}^{\mathrm{ATT}}<\infty\). Also, define
\begin{align*}
\Delta_{\varphi,3,t_{\mathrm{ATT}}}^{\mathrm{ATT},U}(\delta)
:={}&
2\widehat{\mathfrak R}^{(2),U}_{D_{3,\mathrm{ATT}}^{(\varphi)}}
\left(
\mathcal L_{\varphi,3,t_{\mathrm{ATT}}}^{\mathrm{ATT}}
(\hat\varphi^{\mathrm{ATT}})
\right) +3M_{\varphi,3}^{\mathrm{ATT}}
\sqrt{\frac{8\log(2/\delta)}{t_{\mathrm{ATT}}}}.
\end{align*}
{Let \(\hat f_{\mathrm{ATT}}^{(\varphi)}\) be a measurable approximate
minimizer of the complete ordered-pair loss
\[
\widehat{\gR}_{\varphi,3}^{\mathrm{ATT}}(f)
:=
\frac{1}{t_{\mathrm{ATT}}(t_{\mathrm{ATT}}-1)}
\sum_{i\ne j}
\|\gamma_{ij}^{\mathrm{ATT},\varphi}-f(a_i,a_j)\|_2^2,
\]
that is,
\(
\widehat{\gR}_{\varphi,3}^{\mathrm{ATT}}(\hat f_{\mathrm{ATT}}^{(\varphi)})
\le
\inf_{f}\widehat{\gR}_{\varphi,3}^{\mathrm{ATT}}(f)
+\xi_{\varphi,3,t_{\mathrm{ATT}}}^{\mathrm{ATT}}
\)
with \(\xi_{\varphi,3,t_{\mathrm{ATT}}}^{\mathrm{ATT}}\ge0\).}

\begin{lemma}[Pairwise ATT final-stage error]
\label{lem:att_third_stage_rademacher}
Conditional on bridge training, with probability at least \(1-2\delta\),
\[
\|\hat f_{\mathrm{ATT}}^{(\varphi)}
-\bar f_{\hat\varphi^{\mathrm{ATT}}}^{(\varphi,\mathrm{ATT})}
\|_{L^2(\sP_A\otimes\sP_A;\mathbb R^{d_Y})}^2
\le
\kappa_{\varphi,3,t_{\mathrm{ATT}}}^{\mathrm{ATT}}
+2\Delta_{\varphi,3,t_{\mathrm{ATT}}}^{\mathrm{ATT},U}(\delta)
{{}+\xi_{\varphi,3,t_{\mathrm{ATT}}}^{\mathrm{ATT}}}.
\]
\end{lemma}

\begin{proof}
Under the law of two independent observations,
\(\bar f_{\hat\varphi^{\mathrm{ATT}}}^{(\varphi,\mathrm{ATT})}(A_1,A_2)\)
is the conditional mean of \(\Gamma_{12}^{\mathrm{ATT},\varphi}\).
The vector-valued least-squares projection identity therefore holds, while
the induced squared loss remains scalar. {Apply both directions of
Proposition~\ref{prop:uniformUProcessDeviation} and approximate empirical optimality. The
latter contributes \(\xi_{\varphi,3,t_{\mathrm{ATT}}}^{\mathrm{ATT}}\) once to the
squared-error bound.}
\end{proof}

\begin{theorem}[\textsc{TreatmentNet} conditional dose-response error]
\label{thm:treatmentnet_att_consistency}
Under the preceding conditions, for \(\delta\in(0,1/6)\), conditional on
bridge training, with probability at least \(1-2\delta\),
\begin{align*}
&\|\hat f_{\mathrm{ATT}}^{(\varphi)}-f_{\mathrm{ATT}}
\|_{L^2(\sP_A\otimes\sP_A;\mathbb R^{d_Y})}^2\\
&\quad\le
2\kappa_{\varphi,3,t_{\mathrm{ATT}}}^{\mathrm{ATT}}
+4\Delta_{\varphi,3,t_{\mathrm{ATT}}}^{\mathrm{ATT},U}(\delta)
{{}+2\xi_{\varphi,3,t_{\mathrm{ATT}}}^{\mathrm{ATT}}}
+2C_h^2\mathcal R_{\varphi,\mathrm{ATT}}^{\mathrm{proj}}.
\end{align*}
{The multiplier \(2\) arises from the same inequality \(\|u+v\|^2\le2\|u\|^2+2\|v\|^2\)
step applies to \(\kappa_{\varphi,3,t_{\mathrm{ATT}}}^{\mathrm{ATT}}\), and
\(\mathcal R_{\varphi,\mathrm{ATT}}^{\mathrm{proj}}\) already carries the bridge-stage gaps of
Theorem~\ref{thm:att_projected_plugin_rademacher}.}
Intersecting this event with
Theorem~\ref{thm:att_projected_plugin_rademacher} gives an end-to-end event
of probability at least \(1-6\delta\).
\end{theorem}

\begin{proof}
Decompose the fitted surface around
\(\bar f_{\hat\varphi^{\mathrm{ATT}}}^{(\varphi,\mathrm{ATT})}\), apply
\(\|u+v\|_2^2\le2\|u\|_2^2+2\|v\|_2^2\), and use
Lemma~\ref{lem:att_third_stage_rademacher} and
Equation~\ref{eq:att_projected_to_surface}.
\end{proof}

Let \(\delta_t\downarrow0\) satisfy
\(\log(1/\delta_t)/t_{\mathrm{ATT}}\to0\), and define
\[
\rho_{\varphi,3,t_{\mathrm{ATT}}}^{\mathrm{ATT}}
:=
\kappa_{\varphi,3,t_{\mathrm{ATT}}}^{\mathrm{ATT}}
+2\Delta_{\varphi,3,t_{\mathrm{ATT}}}^{\mathrm{ATT},U}(\delta_t).
\]
If \(t_{\mathrm{ATT}}\to\infty\),
\(\kappa_{\varphi,3,t_{\mathrm{ATT}}}^{\mathrm{ATT}}=o_p(1)\), and
\[
\widehat{\mathfrak R}^{(2),U}_{D_{3,\mathrm{ATT}}^{(\varphi)}}
\left(
\mathcal L_{\varphi,3,t_{\mathrm{ATT}}}^{\mathrm{ATT}}
(\hat\varphi^{\mathrm{ATT}})
\right)=o_p(1),
\]
then \(\rho_{\varphi,3,t_{\mathrm{ATT}}}^{\mathrm{ATT}}=o_p(1)\).

Define the optimization-aware final-stage rate
\[
\tilde\rho_{\varphi,3,t_{\mathrm{ATT}}}^{\mathrm{ATT}}
:=
\rho_{\varphi,3,t_{\mathrm{ATT}}}^{\mathrm{ATT}}
+
\xi_{\varphi,3,t_{\mathrm{ATT}}}^{\mathrm{ATT}}.
\]
If
\[
\tilde\rho_{\varphi,3,t_{\mathrm{ATT}}}^{\mathrm{ATT}}
=o_p(1),
\qquad
\tilde\rho_{\varphi,\mathrm{ATT},\nu_\varphi}
=o_p(1),
\]
then Corollary~\ref{cor:att_projected_consistency} gives
\[
\left\|
\hat f_{\mathrm{ATT}}^{(\varphi)}
-f_{\mathrm{ATT}}
\right\|_{L^2(\sP_A\otimes\sP_A;\mathbb R^{d_Y})}^2
=
O_p\!\left(
\tilde\rho_{\varphi,3,t_{\mathrm{ATT}}}^{\mathrm{ATT}}
+
\tilde\rho_{\varphi,\mathrm{ATT},\nu_\varphi}
\right)=o_p(1).
\]
The bridge-stage gaps are already contained in
\(\tilde\rho_{\varphi,\mathrm{ATT},\nu_\varphi}\) and need not be required separately.

For a finite treatment space, let
\[
N_{a,a'}:=\sum_{i\ne j}\mathds 1\{a_i=a,a_j=a'\}.
\]
The saturated final-stage estimator is exactly
\[
\hat f_{\mathrm{ATT}}^{(\varphi)}(a,a')
=
\frac{1}{N_{a,a'}}
\sum_{i\ne j}\mathds 1\{a_i=a,a_j=a'\}
\gamma_{ij}^{\mathrm{ATT},\varphi},
\qquad N_{a,a'}>0.
\]

\subsection{Conditional dose-response consistency of the doubly robust method}
\label{app:dr_att_consistency}

Condition on the data used to fit the vector-valued outcome bridge
\(\hat h\) and the treatment bridge
\(\hat\varphi^{\mathrm{ATT}}\). The two bridge estimates need not be
independent of one another, but the final-stage sample below is independent
of both, including their fitted preprocessing transformations. Assume that
all displayed products and conditional expectations are well-defined and
integrable. For two
independent random copies \(O_i\) and \(O_j\), \(i\ne j\), define
\begin{equation}
\Gamma_{ij}^{\mathrm{ATT},\mathrm{DR}}
:=
\hat\varphi^{\mathrm{ATT}}(A_i,A_j,X_i,Z_i)
\{Y_i-\hat h(A_i,X_i,W_i)\}
+\hat h(A_i,X_j,W_j).
\label{eq:att_dr_pair_pseudooutcome}
\end{equation}
\begin{equation}
\bar f_{\hat h,\hat\varphi^{\mathrm{ATT}}}^{\mathrm{ATT},\mathrm{DR}}(a,a')
:=
\E[\Gamma_{12}^{\mathrm{ATT},\mathrm{DR}}
\mid A_1=a,A_2=a'].
\label{eq:att_dr_population_surface}
\end{equation}
Thus, one regression of the paired pseudo-outcome on \((A_i,A_j)\)
estimates the full conditional dose-response surface.

\begin{lemma}[Doubly robust remainder]
\label{lem:att_dr_identity_projected}
Suppose valid bridges \(h_0\) and \(\varphi_0^{\mathrm{ATT}}\) exist. Then,
for \(\sP_A\otimes\sP_A\)-almost every \((a,a')\),
\begin{align}
&f_{\mathrm{ATT}}(a,a')
-\bar f_{\hat h,\hat\varphi^{\mathrm{ATT}}}^{\mathrm{ATT},\mathrm{DR}}(a,a')
\nonumber\\
&\quad=
\E\!\left[
\{\varphi_0^{\mathrm{ATT}}(a,a',X,Z)
-\hat\varphi^{\mathrm{ATT}}(a,a',X,Z)\}
\{h_0(a,X,W)-\hat h(a,X,W)\}
\mid A=a
\right].
\label{eq:att_dr_exact_remainder}
\end{align}
Consequently, the population surface in
Equation~\ref{eq:att_dr_population_surface} equals
\(f_{\mathrm{ATT}}\) whenever either fitted bridge equals a valid bridge.
\end{lemma}

\begin{proof}
The outcome-bridge equation implies
\[
\E[\hat\varphi^{\mathrm{ATT}}(a,a',X,Z)
\{Y-h_0(a,X,W)\}\mid A=a]=0.
\]
Moreover, the joint-anchor treatment bridge transports every square-integrable
\(q(a,X,W)\) according to
\[
\E[\varphi_0^{\mathrm{ATT}}(a,a',X,Z)q(a,X,W)\mid A=a]
=\E[q(a,X,W)\mid A=a'].
\]
Apply the latter identity to
\(q=h_0-\hat h\), expand
Equations~\ref{eq:att_dr_pair_pseudooutcome}--\ref{eq:att_dr_population_surface},
and cancel terms.
\end{proof}

Retain \(B_{e,h}\) and \(\mathcal R_h^{\mathrm{proj}}\) from
Appendix~\ref{app:dr_ate_rademacher}, and \(\mathcal R_{\varphi,\mathrm{ATT}}^{\mathrm{proj}}\) from Theorem~\ref{thm:att_projected_plugin_rademacher}. Define
\[
B_{e,\varphi,\mathrm{ATT}}^2
:=
\operatorname*{ess\,sup}_{(a,a')\sim\sP_A\otimes\sP_A}
\E\!\left[
\{\hat\varphi^{\mathrm{ATT}}(a,a',X,Z)
-\varphi_0^{\mathrm{ATT}}(a,a',X,Z)\}^2
\mid A=a\right].
\]

\begin{lemma}[Projected joint-anchor doubly robust bound]
\label{lem:att_dr_projected_population_bound}
If \(B_{e,\varphi,\mathrm{ATT}},B_{e,h}<\infty\), then
\begin{align}
&\|\bar f_{\hat h,\hat\varphi^{\mathrm{ATT}}}^{\mathrm{ATT},\mathrm{DR}}
-f_{\mathrm{ATT}}\|_{L^2(\sP_A\otimes\sP_A;\mathbb R^{d_Y})}^2 \le
\min\left\{
B_{e,\varphi,\mathrm{ATT}}^2\mathcal R_h^{\mathrm{proj}},
\;B_{e,h}^2\mathcal R_{\varphi,\mathrm{ATT}}^{\mathrm{proj}}
\right\}.
\label{eq:att_dr_projected_remainder_bound}
\end{align}
\end{lemma}

\begin{proof}
Condition the remainder in Equation~\ref{eq:att_dr_exact_remainder} first on
\((X,Z)\). The conditional mean of \(h_0-\hat h\) is
\(m_0-T_h\hat h\). Conditional Cauchy--Schwarz, followed by integration over
\((a,a')\), gives the first bound. Alternatively, condition first on
\((X,W)\). The conditional mean of
\(\varphi_0^{\mathrm{ATT}}-\hat\varphi^{\mathrm{ATT}}\) is
\(r_{0,\mathrm{ATT}}-T_{\mathrm{ATT}}
\hat\varphi^{\mathrm{ATT}}\), which gives the second bound. Take their
minimum.
\end{proof}

Let
\(D_{3,\mathrm{ATT}}^{\mathrm{DR}}
=\{o_i\}_{i=1}^{t_{\mathrm{ATT},\mathrm{DR}}}\)
be the independent final-stage sample. Let
\[
\mathcal F_{\mathrm{DR},3,t_{\mathrm{ATT},\mathrm{DR}}}^{\mathrm{ATT}}
\subset
\{f:\mathcal A\times\mathcal A\to\mathbb R^{d_Y}\}.
\]
For \(o=(a,x,w,z,y)\) and \(o'=(a',x',w',z',y')\), set
\[
\gamma_{\mathrm{DR}}^{\mathrm{ATT}}(o,o')
:=
\hat\varphi^{\mathrm{ATT}}(a,a',x,z)
\{y-\hat h(a,x,w)\}
+\hat h(a,x',w').
\]
Define
\begin{align*}
\kappa_{\mathrm{DR},3,t_{\mathrm{ATT},\mathrm{DR}}}^{\mathrm{ATT}}
:={}&
\inf_{f\in
\mathcal F_{\mathrm{DR},3,t_{\mathrm{ATT},\mathrm{DR}}}^{\mathrm{ATT}}}
\|f-\bar f_{\hat h,\hat\varphi^{\mathrm{ATT}}}^{\mathrm{ATT},\mathrm{DR}}
\|_{L^2(\sP_A\otimes\sP_A;\mathbb R^{d_Y})}^2,\\
\mathcal L_{\mathrm{DR},3,t_{\mathrm{ATT},\mathrm{DR}}}^{\mathrm{ATT}}
(\hat h,\hat\varphi^{\mathrm{ATT}})
:={}&
\left\{(o_i,o_j)\mapsto
\|\gamma_{\mathrm{DR}}^{\mathrm{ATT}}(o_i,o_j)-f(a_i,a_j)\|_2^2:
f\in
\mathcal F_{\mathrm{DR},3,t_{\mathrm{ATT},\mathrm{DR}}}^{\mathrm{ATT}}
\right\}.
\end{align*}
Conditional on bridge training, assume this scalar loss class is measurable,
separable, and bounded by a deterministic envelope
\(M_{\mathrm{DR},3}^{\mathrm{ATT}}<\infty\). Define
\begin{align*}
\Delta_{\mathrm{DR},3,t_{\mathrm{ATT},\mathrm{DR}}}^{\mathrm{ATT},U}(\delta)
:={}&
2\widehat{\mathfrak R}^{(2),U}_{D_{3,\mathrm{ATT}}^{\mathrm{DR}}}
\left(
\mathcal L_{\mathrm{DR},3,t_{\mathrm{ATT},\mathrm{DR}}}^{\mathrm{ATT}}
(\hat h,\hat\varphi^{\mathrm{ATT}})
\right) + 3M_{\mathrm{DR},3}^{\mathrm{ATT}}
\sqrt{\frac{8\log(2/\delta)}{t_{\mathrm{ATT},\mathrm{DR}}}}.
\end{align*}
{Let \(\hat f_{\mathrm{ATT}}^{\mathrm{DR}}\) be a measurable approximate
minimizer of
\[
\widehat{\gR}_{\mathrm{DR},3}^{\mathrm{ATT}}(f)
:=
\frac{1}{t_{\mathrm{ATT},\mathrm{DR}}
(t_{\mathrm{ATT},\mathrm{DR}}-1)}
\sum_{i\ne j}
\|\gamma_{\mathrm{DR}}^{\mathrm{ATT}}(o_i,o_j)-f(a_i,a_j)\|_2^2
\]
over
\(\mathcal F_{\mathrm{DR},3,t_{\mathrm{ATT},\mathrm{DR}}}^{\mathrm{ATT}}\), with gap
\(\xi_{\mathrm{DR},3,t_{\mathrm{ATT},\mathrm{DR}}}^{\mathrm{ATT}}\ge0\).}

\begin{theorem}[Doubly robust conditional dose-response error]
\label{thm:dr_att_consistency}
Under the preceding conditions, for \(\delta\in(0,1/2)\), conditional on
bridge training, with probability at least \(1-2\delta\),
\begin{align*}
&\|\hat f_{\mathrm{ATT}}^{\mathrm{DR}}-f_{\mathrm{ATT}}
\|_{L^2(\sP_A\otimes\sP_A;\mathbb R^{d_Y})}^2\\
&\quad\le
2\kappa_{\mathrm{DR},3,t_{\mathrm{ATT},\mathrm{DR}}}^{\mathrm{ATT}}
+4\Delta_{\mathrm{DR},3,t_{\mathrm{ATT},\mathrm{DR}}}^{\mathrm{ATT},U}(\delta)
{{}+2\xi_{\mathrm{DR},3,t_{\mathrm{ATT},\mathrm{DR}}}^{\mathrm{ATT}}}\\
&\qquad
+2\min\left\{
B_{e,\varphi,\mathrm{ATT}}^2\mathcal R_h^{\mathrm{proj}},
\;B_{e,h}^2\mathcal R_{\varphi,\mathrm{ATT}}^{\mathrm{proj}}
\right\}.
\end{align*}
\end{theorem}

\begin{proof}
Under the law of two independent observations,
\(\bar f_{\hat h,\hat\varphi^{\mathrm{ATT}}}^{\mathrm{ATT},\mathrm{DR}}(A_1,A_2)\)
is the conditional mean of \(\Gamma_{12}^{\mathrm{ATT},\mathrm{DR}}\).
The vector-valued least-squares projection identity, both directions of Proposition~\ref{prop:uniformUProcessDeviation}, and {approximate} empirical optimality give
\[
\|\hat f_{\mathrm{ATT}}^{\mathrm{DR}}
-\bar f_{\hat h,\hat\varphi^{\mathrm{ATT}}}^{\mathrm{ATT},\mathrm{DR}}
\|_{L^2(\sP_A\otimes\sP_A;\mathbb R^{d_Y})}^2
\le
\kappa_{\mathrm{DR},3,t_{\mathrm{ATT},\mathrm{DR}}}^{\mathrm{ATT}}
+2\Delta_{\mathrm{DR},3,t_{\mathrm{ATT},\mathrm{DR}}}^{\mathrm{ATT},U}(\delta)
{{}+\xi_{\mathrm{DR},3,t_{\mathrm{ATT},\mathrm{DR}}}^{\mathrm{ATT}}}.
\]
Decompose around the population surface, apply
\(\|u+v\|_2^2\le2\|u\|_2^2+2\|v\|_2^2\), and use
Lemma~\ref{lem:att_dr_projected_population_bound}.
\end{proof}

Choose \(\delta_t\downarrow0\) such that
\(\log(1/\delta_t)/t_{\mathrm{ATT},\mathrm{DR}}\to0\), and define
\[
\rho_{\mathrm{DR},3,t_{\mathrm{ATT},\mathrm{DR}}}^{\mathrm{ATT}}
:=
\kappa_{\mathrm{DR},3,t_{\mathrm{ATT},\mathrm{DR}}}^{\mathrm{ATT}}
+2\Delta_{\mathrm{DR},3,t_{\mathrm{ATT},\mathrm{DR}}}^{\mathrm{ATT},U}
(\delta_t).
\]
If \(t_{\mathrm{ATT},\mathrm{DR}}\to\infty\),
\(M_{\mathrm{DR},3}^{\mathrm{ATT}}=O_p(1)\),
\(\kappa_{\mathrm{DR},3,t_{\mathrm{ATT},\mathrm{DR}}}^{\mathrm{ATT}}
=o_p(1)\), and
\[
\widehat{\mathfrak R}^{(2),U}_{D_{3,\mathrm{ATT}}^{\mathrm{DR}}}
\left(
\mathcal L_{\mathrm{DR},3,t_{\mathrm{ATT},\mathrm{DR}}}^{\mathrm{ATT}}
(\hat h,\hat\varphi^{\mathrm{ATT}})
\right)=o_p(1),
\]
then
\(\rho_{\mathrm{DR},3,t_{\mathrm{ATT},\mathrm{DR}}}^{\mathrm{ATT}}=o_p(1)\).

Define the optimization-aware final-stage rate
\[
\tilde\rho_{\mathrm{DR},3,t_{\mathrm{ATT},\mathrm{DR}}}^{\mathrm{ATT}}
:=
\rho_{\mathrm{DR},3,t_{\mathrm{ATT},\mathrm{DR}}}^{\mathrm{ATT}}
+
\xi_{\mathrm{DR},3,t_{\mathrm{ATT},\mathrm{DR}}}^{\mathrm{ATT}}.
\]

If, in addition,
\begin{align*}
B_{e,\varphi,\mathrm{ATT}},B_{e,h}&=O_p(1),\\
\mathcal R_h^{\mathrm{proj}}&=O_p({\tilde\rho_{h,\nu_h}}),&
\mathcal R_{\varphi,\mathrm{ATT}}^{\mathrm{proj}}
&=O_p({\tilde\rho_{\varphi,\mathrm{ATT},\nu_\varphi}}),
\end{align*}
then
\begin{align}
&\|\hat f_{\mathrm{ATT}}^{\mathrm{DR}}-f_{\mathrm{ATT}}
\|_{L^2(\sP_A\otimes\sP_A;\mathbb R^{d_Y})}^2 =
O_p\!\left(
{\tilde\rho_{\mathrm{DR},3,t_{\mathrm{ATT},\mathrm{DR}}}^{\mathrm{ATT}}}
+\min\{{\tilde\rho_{h,\nu_h}},
{\tilde\rho_{\varphi,\mathrm{ATT},\nu_\varphi}}\}
\right).
\label{eq:dr_att_consistency_rate}
\end{align}
Hence the estimator is consistent when {\(\tilde\rho_{\mathrm{DR},3,t_{\mathrm{ATT},\mathrm{DR}}}^{\mathrm{ATT}}=o_p(1)\)}
and
either projected bridge error vanishes. There is no additive
density-ratio estimation error.

\begin{remark}[Optimization accuracy]
\label{rem:optimization_accuracy}
Exact empirical minimization is the special case in which all \(\xi\)-terms equal zero, and
every bound above then reduces to its exact-minimizer form. Consistency requires only that
the optimization gaps vanish together with the approximation and empirical-process terms of
their own stage. If a stage has statistical scale \(\rho_{b,s}\), then
\(\xi_{b,s}=O_p(\rho_{b,s})\) preserves the overall order of that stage, while
\(\xi_{b,s}=o_p(\rho_{b,s})\) makes the optimization error asymptotically negligible; no
universal \(O_p(n^{-1})\) optimization requirement is imposed, since the familiar
\(n^{-1}\) criterion condition applies to parameter-distance targets under quadratic
curvature, whereas the quantities bounded here are already squared projected errors. These
results do not establish a rate for the specific alternating neural optimizer used in the
experiments. They make explicit what empirical suboptimality that optimizer would have to
achieve for the statistical conclusions to hold.
\end{remark}

\section{{Supplementary on numerical experiments}}
\label{sec:Appendix_NumericalExperimentsSupp}

\subsection{Cross-fitting}
\label{app:cross_fitting}
In our implementation, we utilize cross-fitting that randomly partitions the observation indices into three disjoint sets \(\mathcal I_1,\mathcal I_2,\mathcal I_3\). For rotation \(r\in\{1,2,3\}\), with subscripts interpreted modulo three, we use \(\mathcal I_r\) to learn the first-stage conditional-mean features, \(\mathcal I_{r+1}\) to estimate the second-stage outcome and treatment bridge heads, and \(\mathcal I_{r+2}\) to compute the corresponding final-stage causal estimate \(\hat f^{(r)}\). Thus each split serves once at each stage, while the bridge estimators used in a given final stage are trained without its observations. We report the cross-fitted estimate \(\hat f^{\mathrm{cf}}:=3^{-1}\sum_{r=1}^{3}\hat f^{(r)}\).

\subsection{{Vectorized computation of U-statistics with einsum}}

Our {\textsc{TreatmentNet}} algorithm requires computing U-statistics. In this section, we provide a vectorized computation of this loss component with einsum. To this end, write the compact feature maps in the following separable form

\begin{align*}
    \phi_{AXW, 1}^{(\varphi)} &= \phi_{A, 1}^{(\varphi)} \otimes \phi_{XW, 1}^{(\varphi)}\\
    \phi_{AX, 2}^{(\varphi)} &= \phi_{A, 2}^{(\varphi)} \otimes \phi_{X,2}^{(\varphi)}.
\end{align*}

Fix \(\check\mV_t^{(\varphi)}\), let
\(m:=|\gB_{2,t}^{(\varphi)}|\ge2\), and suppress \((t,\varphi)\) on
\(\vu_{ij}\), \(b_{\mathrm{prod}}\), \(b_U\), and \(\bar\vu\) below.

Recall we defined
\begin{align*}
    \vu_{ij, t}^{(\varphi)} = \phi_{A, 2}^{(\varphi)}(a_i) \otimes \phi_{X, 2}^{(\varphi)}(x_j) \otimes \left\{\check \mV_t^{(\varphi)}\right\}^\top  \left\{\phi_{A, 1}^{(\varphi)} (a_i) \otimes \phi_{XW, 1}^{(\varphi)} (x_j, w_j)\right\}.
\end{align*}

The U-statistic term in the ratio-free risk is

\[
\frac{1}{m(m-1)}\sum_{\substack{i, j \in \gB_{2, t}^{(\varphi)}\\ i \ne j}} \bm{\varphi}^\top \vu_{ij, t}^{(\varphi)}.
\]

Define the tensor-product features
\(
x_{ij}:=\phi_{A,2}^{(\varphi)}(a_i)\otimes
\phi_{X,2}^{(\varphi)}(x_j),
\qquad
y_{ij}:=\phi_{A,1}^{(\varphi)}(a_i)\otimes
\phi_{XW,1}^{(\varphi)}(x_j,w_j).
\)
The operator identity underlying the implementation is
\begin{align*}
\vu_{ij}
&=x_{ij}\otimes\{\check\mV_t^{(\varphi)}\}^{\top}y_{ij}\\
&=\left(I\otimes\{\check\mV_t^{(\varphi)}\}^{\top}\right)
(x_{ij}\otimes y_{ij}).
\end{align*}
Here \(I\) is the identity on the \(A_2\otimes X_2\) feature space.
A fixed commutation permutation groups the \(i\)-only and \(j\)-only
factors. Reshape \(\{\check\mV_t^{(\varphi)}\}^{\top}\) as
\(\mathsf V_t^{(\varphi)}[z,p,r]\), where \(z,p,r\) index the
\(Z_2,A_1,XW_1\) modes, respectively, and define
\begin{align*}
    M_A[p, c] &= \frac{1}{m} \sum_i \phi_{A, 1}^{(\varphi)}(a_i)[p] \phi_{A, 2}^{(\varphi)}(a_i)[c]\\
    M_{XW}[r, d] &= \frac{1}{m} \sum_j \phi_{XW, 1}^{(\varphi)}(x_j, w_j)[r]\phi_{X, 2}^{(\varphi)}(x_j)[d].
\end{align*}
Then
\begin{align}
    b_{\mathrm{prod}}[c,d,z]
    := \frac{1}{m^2}\sum_{i,j}\vu_{ij}[c,d,z]
    = \sum_{p,r}\mathsf V_t^{(\varphi)}[z,p,r]
    M_A[p,c]M_{XW}[r,d]
    \label{eq:VstatisticContraction}
\end{align}
Equation~\ref{eq:VstatisticContraction} is computed as
\begin{verbatim}
einsum('zpr,pc,rd->zcd', V, M_A, M_XW).
\end{verbatim}

Finally, let $\bar \vu = m^{-1} \sum_i \vu_{ii}$. Removing the diagonal turns the empirical-product V-statistic into the complete ordered U-statistic:
\begin{align}
    b_U \coloneqq \frac{1}{m (m - 1)} \sum_{i \neq j} \vu_{ij} = \frac{m b_{\mathrm{prod}} - \bar \vu}{m - 1}
    \label{eq:UstatisticContraction}
\end{align}

Hence, the linear term in Equation~\ref{eq:TreatmentBridgeRFSecondStageRisk} is \(-2\bm{\varphi}^\top b_U\). Equations~\ref{eq:VstatisticContraction} and~\ref{eq:UstatisticContraction} are exact and never construct an \(m\times m\) pair array. For fixed feature widths, this contraction has linear sample-axis memory and computation; feature-mode contractions and the bridge-head Gram matrix retain their feature-dimension costs.

\subsection{{Treatment bridge normalization}}

\begin{proposition}[Treatment-bridge normalization]
\label{prop:TreatmentBridgeNormalization}
Suppose that \(\varphi\) is integrable and satisfies
Eq.~\ref{eq:TreatmentBridgeEq_ATE}. Then
\begin{align}
    \E\{\varphi(A,X,Z)\mid A\}=1
    \qquad \sP_A\text{-almost surely}.
    \label{eq:TreatmentBridgeNormalizationEq}
\end{align}
\end{proposition}
\begin{proof}
For \(\sP_A\)-almost every \(a\), the law of total expectation and the bridge
equation give
\begin{align*}
\E\{\varphi(a,X,Z)\mid A=a\}
&=\int {r_{0,\mathrm{ATE}}}(a,x,w)\,
    d\sP_{X,W\mid A=a}(x,w)\\
&=\int
    \frac{d\sP_{X,W}}{d\sP_{X,W\mid A=a}}(x,w)\,
    d\sP_{X,W\mid A=a}(x,w)=1,
\end{align*}
where the second equality uses the same domination condition as the
treatment-bridge equation.
\end{proof}

Let \(T_\varphi\varphi(A,X,W):=
\E\{\varphi(A,X,Z)\mid A,X,W\}\). Whenever the right-hand side below is
finite, the same argument and Jensen's inequality yields
\begin{align}
\left\|\E\{\varphi(A,X,Z)\mid A\}-1\right\|_{L^2(\sP_A)}
\leq
\left\|T_\varphi\varphi- {r_{0,\mathrm{ATE}}}\right\|_{L^2(\sP_{A,X,W})}.
\label{eq:ProjectedBridgeImpliesNormalization}
\end{align}
Thus projected treatment-bridge consistency already implies asymptotic normalization. The constraint below is a finite-sample stabilization.

\paragraph{Finite moment restrictions.}
If \(\varphi\in L^2(\sP_{A,X,Z})\), then for every
\(g\in L^2(\sP_A)\),
\begin{align}
\E[(\varphi-1)g(A)]
=\left\langle \E(\varphi\mid A)-1,g\right\rangle_{L^2(\sP_A)}.
\label{eq:TreatmentNormalizationMomentIdentity}
\end{align}
Consequently, Eq.~\ref{eq:TreatmentBridgeNormalizationEq} is equivalent to the vanishing of these moments over any class whose linear span is dense in \(L^2(\sP_A)\). A finite collection only imposes the corresponding projected restrictions.

Choose \(g_1,\ldots,g_L\) and let
\(\gB_{2,t}^{(\varphi)}\) contain \(m\) observations. Write the direct
bridge feature and its linear head as
\begin{align*}
\vq_i^{(\varphi)}
&:=\phi_{AX,2}^{(\varphi)}(a_i,x_i)\otimes
   \phi_{Z,2}^{(\varphi)}(z_i)\in\R^D,
&
\varphi_{\bm\beta}(a_i,x_i,z_i)
&:=\bm\beta^\top\vq_i^{(\varphi)}.
\end{align*}
Define \(\mC_t\in\R^{L\times D}\) and \(\vd_t\in\R^L\) by
\begin{align}
(\mC_t)_{\ell,:}
&:=\frac1m\sum_{i\in\gB_{2,t}^{(\varphi)}}
g_\ell(a_i){\vq_i^{(\varphi)}}^\top,
&
(\vd_t)_\ell
&:=\frac1m\sum_{i\in\gB_{2,t}^{(\varphi)}}g_\ell(a_i).
\label{eq:TreatmentNormalizationConstraint}
\end{align}
The \(L\) empirical normalization moments are exactly \(\mC_t\bm\beta=\vd_t\). To obtain the constrained head update equation, set
\begin{align*}
\mH_t
&:=\mG_t^{(\varphi)}+\lambda_2^{(\varphi)}\mI_D,
&
\vr_t
&:=\vb_{U,t}^{(\varphi)}+\lambda_2^{(\varphi)}
   \hat{\bm{\varphi}}_t.
\end{align*}

For \(\lambda_2^{(\varphi)}>0\), \(\mH_t\) is positive definite. The constrained optimization problem for the normalized head is
\begin{align}
\min_{\bm\beta\in\R^D}
\left\{\frac12\bm\beta^\top\mH_t\bm\beta-\vr_t^\top\bm\beta\right\}
\quad\text{subject to}\quad
\mC_t\bm\beta=\vd_t.
\label{eq:TreatmentNormalizationQP}
\end{align}

Its KKT conditions are
\begin{align}
\begin{bmatrix}
\mH_t & \mC_t^\top\\
\mC_t & \bm 0
\end{bmatrix}
\begin{bmatrix}
\bm\beta^\star\\ \bm\eta^\star
\end{bmatrix}
=
\begin{bmatrix}
\vr_t\\ \vd_t
\end{bmatrix}.
\label{eq:TreatmentNormalizationKKT}
\end{align}

If \(\mC_t\) has full row rank, the block system is nonsingular. Let
\(\bm\beta_{\mathrm{unc}}:=\mH_t^{-1}\vr_t\). Its unique primal solution is
\begin{align}
\bm\beta^\star
=\bm\beta_{\mathrm{unc}}
+\mH_t^{-1}\mC_t^\top
\left(\mC_t\mH_t^{-1}\mC_t^\top\right)^{-1}
(\vd_t-\mC_t\bm\beta_{\mathrm{unc}}),
\label{eq:TreatmentNormalizationProjection}
\end{align}
and \(\hat{\bm{\varphi}}_{t+1}:=\bm\beta^\star\). Hence the KKT update is the minimum-\(\mH_t\)-norm correction of the unconstrained head that satisfies the selected moments. The implementation solves Eq.~\ref{eq:TreatmentNormalizationKKT} directly; it does not form either inverse in Eq.~\ref{eq:TreatmentNormalizationProjection}. In implementation, we evaluate the bridge features without dropout, orthonormalize the empirical treatment basis, and discard numerically redundant basis directions.

Table~\ref{tab:treatment_bridge_normalization_ablation} isolates the effect of
the KKT update on \textsc{TreatmentNet}. Here, None uses the unconstrained head update,
whereas KKT enforces the empirical normalization moments in
Eq.~\ref{eq:TreatmentNormalizationConstraint} during training. For ATT, we use
the corresponding tensor-product moments in \((a,a')\).

\begin{table}[H]
\centering
\caption{\textsc{TreatmentNet} normalization ablation. Entries are causal MSE (lower is better), reported as mean $\pm$ seed-level standard error over ten data-generating seeds after averaging three repetitions within each seed.}
\label{tab:treatment_bridge_normalization_ablation}
\scriptsize
\setlength{\tabcolsep}{3.2pt}
\renewcommand{\arraystretch}{0.94}
\begin{adjustbox}{max width=\textwidth}
\begin{tabular}{@{}llccccc@{}}
\toprule
Benchmark & Normalization
& $n=2\mathrm{k}$ & $n=5\mathrm{k}$ & $n=10\mathrm{k}$
& $n=15\mathrm{k}$ & $n=20\mathrm{k}$ \\
\midrule
\multirow{2}{*}{Demand}
& None & $367.42\pm55.39$ & $374.43\pm28.81$ & $371.06\pm37.03$ & $383.51\pm29.58$ & $350.96\pm31.13$ \\
& KKT  & $\mathbf{199.67\pm15.77}$ & $\mathbf{174.11\pm16.50}$ & $\mathbf{153.84\pm5.57}$ & $\mathbf{144.97\pm4.24}$ & $\mathbf{145.76\pm4.40}$ \\
\midrule
\multirow{2}{*}{Synthetic ATT ($10^3\!\times$)}
& None & $\mathbf{142.83\pm7.88}$ & $128.02\pm13.47$ & $142.15\pm8.00$ & $130.36\pm14.43$ & $131.01\pm8.86$ \\
& KKT  & $179.49\pm8.90$ & $\mathbf{103.68\pm9.55}$ & $\mathbf{95.08\pm6.96}$ & $\mathbf{71.87\pm8.37}$ & $\mathbf{72.49\pm7.61}$ \\
\bottomrule
\end{tabular}
\end{adjustbox}
\end{table}

KKT lowers the \textsc{TreatmentNet} error throughout the Demand sweep and for
$n\geq 5{,}000$ in the ATT sweep. At $n=2{,}000$ in the ATT design, the
unconstrained update performs better, consistent with viewing the restriction
as a finite-sample stabilization rather than a guarantee of uniform
improvement.

\subsection{{Benchmark descriptions and data-generating processes}}
\label{app:Supplement_NumericalExperiments}
{In this section, we give the variables, proxy partitions, structural equations,
and oracle targets used in the numerical experiments.}

\subsubsection{{Bridge-route preference benchmarks}}
\label{app:BridgeRoutePreference}

We use two complementary stress tests to compare the outcome- and treatment-bridge routes. In both designs, we consider binary treatment setting with \(A\in\{-1,+1\}\), and we set \(d_Y=3\). For a Rademacher variable \(R\), write \(B_\lambda(R):=R\xi_\lambda\), where \(\xi_\lambda \in\{-1,+1\}\) with \(\sP(\xi_\lambda=1)=(1+\lambda)/2\). Hence \(\E\{B_\lambda(R)\mid R\}=\lambda R\). 

\paragraph{Scenario A: treatment-route stress.}
Let \(U=(U_1,U_2)\), where \(U_1\) and \(U_2\) are independent Rademacher variables. Define
\[
\eta(u_1,u_2):=\kappa_1u_1+\kappa_2u_2+\kappa_{12}u_1u_2,
\qquad
\sP(A=a\mid U_1=u_1,U_2=u_2)=\frac{1+a\eta(u_1,u_2)}{2}.
\]

with the restriction \(|\kappa_1|+|\kappa_2|+|\kappa_{12}|<1\). We define the proxy variables and potential outcomes as
\begin{align*}
Z&=\bigl(B_{\lambda_{Z_1}}(U_1),B_{\lambda_{Z_2}}(U_2)\bigr),
&W&=\bigl(B_{\lambda_{W_1}}(U_1),B_{\lambda_{W_2}}(U_2)\bigr),\\
Y^{(a)}&=a\mathbf 1_3+(U_1,U_2,U_1U_2)^\top+\varepsilon,
&\varepsilon&\sim\mathcal N(0,0.25^2I_3).
\end{align*}

By Bayes' rule, $\sP(U_1=u_1,U_2=u_2 \mid A = a) = \frac{1 + a \eta(u_1,u_2)}{4}$. Consequently, we have $\E[Y \mid A = a] = a\mathbf 1_3 + a(
    \kappa_1, \enspace \kappa_2, \enspace \kappa_{12}
)^\top$ and \(f_{\mathrm{ATE}}(a)=a\mathbf 1_3\). Then, an outcome bridge is given by
\begin{equation}
h_0(a,w_1,w_2)
=a\mathbf 1_3+
\left(
\frac{w_1}{\lambda_{W_1}},
\frac{w_2}{\lambda_{W_2}},
\frac{w_1w_2}{\lambda_{W_1}\lambda_{W_2}}
\right)^\top.
\label{eq:BridgePreferenceScenarioAOutcomeBridge}
\end{equation}

Indeed, since $\E\left[\frac{W_1}{\lambda_{W_1}} \mid U_1,U_2\right] = U_1$, $\E\left[\frac{W_2}{\lambda_{W_2}} \mid U_1,U_2\right] = U_2$, and $\E\left[\frac{W_1W_2}{\lambda_{W_1}\lambda_{W_2}} \mid U_1,U_2 \right] = U_1U_2$, we have
\[
\E\left[h_0(a, W) \mid U_1,U_2\right] = a \mathbf 1_3 + (U_1,U_2,U_1U_2)^\top \coloneqq m_a(U_1,U_2).
\]
Therefore, by the law of total expectations, 
\begin{align*}
    \E[h_0(a, W) \mid A = a, Z] &= \E[\E[h_0(a, W) \mid A = a, Z, U_1,U_2] \mid A = a, Z]\\
    &=\E[\E[h_0(a, W) \mid U_1,U_2] \mid A = a, Z] \enspace \text{by Assumption \ref{assum:ProxyConditionalIndependenceAssumptions}}\\
    &= \E[m_a(U_1,U_2) \mid A = a, Z] = \E[\E[Y \mid A = a, Z, U_1,U_2] \mid A = a, Z]\\
    &= \E[Y \mid A = a, Z].
\end{align*}

To derive a treatment bridge, define the latent reweighting
\[
r_a(u_1,u_2)
:=
\frac{\sP(A=a)}
     {\sP(A=a\mid U_1=u_1,U_2=u_2)}
=
\frac{1}{1+a\eta(u_1,u_2)},
\qquad u_1,u_2\in\{-1,+1\},
\]

Since \((U_1,U_2)\) takes only four possible values, every function of
\((U_1,U_2)\) can be written uniquely as
\[
r_a(u_1,u_2)
=
d_0(a)+d_1(a)u_1+d_2(a)u_2+d_{12}(a)u_1u_2.
\]
The four coefficients are
\begin{align*}
d_0(a)
&:=
\frac14
\sum_{u_1,u_2\in\{-1,+1\}}
r_a(u_1,u_2), \quad d_1(a)
:=
\frac14
\sum_{u_1,u_2\in\{-1,+1\}}
u_1\,r_a(u_1,u_2),\\
d_2(a)
&:=
\frac14
\sum_{u_1,u_2\in\{-1,+1\}}
u_2\,r_a(u_1,u_2),\quad d_{12}(a)
:=
\frac14
\sum_{u_1,u_2\in\{-1,+1\}}
u_1u_2\,r_a(u_1,u_2).
\end{align*}

Then, a treatment bridge is given by
\begin{equation}
\varphi_0^{\mathrm{ATE}}(a,z_1,z_2)
=d_0(a)+\frac{d_1(a)}{\lambda_{Z_1}}z_1
+\frac{d_2(a)}{\lambda_{Z_2}}z_2
+\frac{d_{12}(a)}{\lambda_{Z_1}\lambda_{Z_2}}z_1z_2.
\label{eq:BridgePreferenceScenarioATreatmentBridge}
\end{equation}

Indeed, conditional independence of the two proxy-noise variables gives
\[
\E[Z_1\mid U_1,U_2]=\lambda_{Z_1}U_1,
\qquad
\E[Z_2\mid U_1,U_2]=\lambda_{Z_2}U_2,
\qquad
\E[Z_1Z_2\mid U_1,U_2]
=
\lambda_{Z_1}\lambda_{Z_2}U_1U_2.
\]
Therefore,
\begin{align*}
\E[
\varphi_0^{\mathrm{ATE}}(a,Z_1,Z_2)
\mid U_1,U_2]
&=
d_0(a)+d_1(a)U_1+d_2(a)U_2+d_{12}(a)U_1U_2 =
r_a(U_1,U_2).
\end{align*}

Consequently, by iterated expectation,
\begin{align*}
  \E[
\varphi_0^{\mathrm{ATE}}(a,Z_1,Z_2)
\mid A = a, W] &= \E[\E[\varphi_0^{\mathrm{ATE}}(a,Z_1,Z_2)
\mid A = a, W, U_1,U_2] \mid A = a, W] \\
&= \E[r_a(U_1,U_2) \mid A = a, W] \\
&= \frac{\sP(A = a) \sP(W)}{\sP(A = a, W)}.
\end{align*}
The second equality uses the conditional independence \(Z\perp(A,W)\mid(U_1,U_2)\) from Assumption~\ref{assum:ProxyConditionalIndependenceAssumptions}. The third equality uses Bayes' rule and \(A\perp W\mid(U_1,U_2)\) from Assumption~\ref{assum:ProxyConditionalIndependenceAssumptions}.

In our simulations, we use
\[
\kappa_1=\kappa_{12}=0.1,\quad
\kappa_2\in\{0,0.2,0.4,0.6\},\quad
(\lambda_{Z_1},\lambda_{Z_2},\lambda_{W_1},\lambda_{W_2})
=(0.8,0.6,0.8,0.8).
\]
In this data generating process,  the proxy channels and the exact outcome bridge are fixed, while increasing \(\kappa_2\) strengthens latent treatment selection through \(U_2\). In particular, the smallest latent-stratum treatment probability decreases from \(0.4\) at \(\kappa_2=0\) to \(0.1\) at \(\kappa_2=0.6\). Consequently, the latent reweighting function
\[
r_a(u_1,u_2)=\frac{1}{1+a\eta(u_1,u_2)}
\]
becomes more dispersed, which increases the magnitude and second moment of
the exact treatment bridge. Scenario A therefore stresses the weighting
route while leaving the exact outcome bridge unchanged. The conditional
operator also becomes moderately less well conditioned over the sweep, but
remains injective.

\paragraph{Scenario B: outcome-route stress.} Again let \(U_1\) and \(U_2\) be independent Rademacher variables, but now set
\[
\sP(A=a\mid U_1=u_1)=\frac{1+0.1au_1}{2}.
\]
The proxy channels retain the same form as in Scenario A, with
\[
(\lambda_{Z_1},\lambda_{W_1})=(0.9,0.9),
\qquad
\lambda_{Z_2}\lambda_{W_2}=0.36,
\]
and potential outcomes are defined by
\begin{align*}
Y^{(a)}
&=a\mathbf 1_3+(0.25U_1,U_2,0.8U_1U_2)^\top+\varepsilon,
\qquad \varepsilon \sim\mathcal N(0,0.25^2I_3).
\end{align*}
This construction also yields \(f_{\mathrm{ATE}}(a)=a\mathbf 1_3\). The exact bridges are given by
\begin{align}
h_0(a,w_1,w_2)
&=a\mathbf 1_3+
\left(
\frac{0.25w_1}{0.9},
\frac{w_2}{\lambda_{W_2}},
\frac{0.8w_1w_2}{0.9\lambda_{W_2}}
\right)^\top,
\label{eq:BridgePreferenceScenarioBOutcomeBridge}\\
\varphi_0^{\mathrm{ATE}}(a,z_1,z_2)
&=\frac{1-0.1az_1/0.9}{1-0.1^2}.
\label{eq:BridgePreferenceScenarioBTreatmentBridge}
\end{align}

The outcome-bridge derivation is the same as in Scenario A, up to
the scaling coefficients. Indeed,
\[
\E[W_1\mid U_1,U_2]=0.9U_1,
\qquad
\E[W_2\mid U_1,U_2]=\lambda_{W_2}U_2,
\qquad
\E[W_1W_2\mid U_1,U_2]
=
0.9\lambda_{W_2}U_1U_2.
\]

Consequently,
\[
\E[h_0(a,W_1,W_2)\mid U_1,U_2]
=
a\mathbf 1_3+(0.25U_1,U_2,0.8U_1U_2)^\top
=
\E[Y\mid A=a,U_1,U_2].
\]
Using \(W\perp(A,Z)\mid(U_1,U_2)\) and iterated expectation therefore
yields
\[
\E[h_0(a,W)\mid A=a,Z]
=
\E[Y\mid A=a,Z].
\]

For treatment bridge function, observe that 
\begin{align*}
    \E[\varphi_0^{\mathrm{ATE}}(a,Z_1,Z_2) \mid U_1,U_2]
    &= \frac{1-0.1a\E[Z_1\mid U_1,U_2]/0.9}{1-0.1^2} \\
    &= \frac{1 - 0.1 a U_1}{1 - 0.1^2}
    = \frac{\sP(A = a)}{\sP(A = a \mid U_1)}.
\end{align*}
Let $r_a(u_1)$ be $\sP(A = a) / \sP(A = a \mid U_1=u_1)$. Then, law of total expectations and Assumption~\ref{assum:ProxyConditionalIndependenceAssumptions} give
\begin{align*}
    \E[\varphi_0^{\mathrm{ATE}}(a, Z) \mid A = a, W] &= \E[\E[\varphi_0^{\mathrm{ATE}}(a, Z)\mid A = a, W, U_1,U_2] \mid A = a, W]\\
    &=\E[r_a(U_1) \mid A = a, W] = \frac{\sP(A = a)}{\sP(A = a \mid W)}.
\end{align*}

We sweep
\(
\lambda_{W_2}\in\{0.9,0.7,0.6,0.5,0.4\},
\enspace
\lambda_{Z_2}=0.36/\lambda_{W_2}.
\)
In this data generating process, treatment depends only on \(U_1\). Hence \(U_1\) is the latent confounder, whereas \(U_2\) is prognostic but does not affect treatment. The exact treatment bridge depends only on \(Z_1\) and is therefore independent of \(\lambda_{Z_2}\) and \(\lambda_{W_2}\). Moreover,
\[
\sP(Z_2=z_2,W_2=w_2)
=
\frac{1+\lambda_{Z_2}\lambda_{W_2}z_2w_2}{4}.
\]
Keeping \(\lambda_{Z_2}\lambda_{W_2}=0.36\) therefore fixes the joint law of \((A,Z,W)\), and hence fixes the entire treatment-bridge nuisance-learning problem and both conditional-expectation operators. Because \(\sP\{B_{\lambda_{W_2}}(U_2)\ne U_2\}=(1-\lambda_{W_2})/2\), decreasing \(\lambda_{W_2}\) makes \(W_2\) a noisier measurement of \(U_2\). The exact outcome bridge then requires coefficients proportional to \(1/\lambda_{W_2}\), increasing its variability and sensitivity to estimation error. Scenario~B therefore isolates deterioration of the outcome route while leaving the treatment bridge and its nuisance law fixed.

\subsubsection{Legalized Abortion and Crime}
\label{app:abortion_crime_dgp}

We use the application-motivated data of \citet{LegalizedAbortionData}, as
preprocessed and released by \citet{Mastouri2021ProximalCL}. The continuous
treatment \(A\) is the effective abortion rate, and the scalar outcome \(Y\) is
the murder rate. The treatment-inducing proxy \(Z\in\R\) is the generosity of
Aid to Families with Dependent Children. The outcome-inducing proxy \(W\in\R^3\)
contains beer consumption per capita, log prisoners per capita, and the
concealed-weapons-law indicator. The released benchmark also supplies the
intervention grid and oracle response used to compute causal MSE.

\subsubsection{Grade retention (Deaner IM and IR)}
\label{app:deaner_dgp}

We use the calibrated semi-synthetic ECLS-K grade-retention designs based on \citet{Fruehwirth_GradeRetentions} and \citet{deaner2023proxycontrolspaneldata}, as released by \citet{Mastouri2021ProximalCL}. In the IM and IR variants, \(Y\) is the age-11 mathematics or reading score, respectively. The scalar treatment \(A\) is a continuously dequantized retention measure. The proxy vectors contain earlier elementary-school and kindergarten cognitive measurements, with \(Z\in\R^4\) and \(W\in\R^3\) under the released partition. We retain the observational and intervention-grid treatment values in \((-0.1,2)\), and use the accompanying oracle responses to compute causal MSE.

\subsubsection{Demand}
\label{app:demand_dgp}

We use the nonlinear demand design of \citet{xu2021deep}. Draw latent demand
\(D\sim\operatorname{Unif}[0,10]\) and define
\[
\psi(d)=2\left\{\frac{(d-5)^4}{600}+\exp[-4(d-5)^2]+\frac d{10}-2\right\}.
\]
For mutually independent standard-normal errors, generate fuel costs
\((C_1,C_2)\), page views \(R\), ticket price \(P\), and sales \(Y\) as
\begin{align*}
C_1&=2\sin(2\pi D/10)+\varepsilon_1,
&C_2&=2\cos(2\pi D/10)+\varepsilon_2,\\
R&=7\psi(D)+45+\varepsilon_R,
&P&=35+(C_1+3)\psi(D)+C_2+\varepsilon_P,\\
Y&=P\min\!\left\{\exp[(R-P)/10],5\right\}-5\psi(D)+\varepsilon_Y.
\end{align*}
The treatment is \(A=P\), the treatment-inducing proxy is
\(Z=(C_1,C_2)\), and the outcome-inducing proxy is \(W=R\). The oracle
population dose-response is
\[
f_{\mathrm{ATE}}(a)
=\E\!\left[a\min\!\left\{\exp[(R-a)/10],5\right\}-5\psi(D)\right].
\]

\subsubsection{Image-valued treatment benchmark: dSprite}
\label{app:dsprite_benchmarks}

We construct a modified dSprite benchmark from \citep{xu2021deep} that uses images of size $64 \times 64$ described by the latent parameters: $\mathrm{shape}$, $\mathrm{scale}$, $\mathrm{rotation}$, $\mathrm{posX}$, and $\mathrm{posY}$ \citep{dsprites17}. We draw a confounder \(U\sim\operatorname{Unif}[0,1]\). The shape is fixed to the heart, while the scale, orientation, and vertical-position indices of the treatment image are drawn independently and uniformly from \(\{0,\ldots,5\}\), \(\{0,\ldots,39\}\), and \(\{0,\ldots,31\}\), respectively. Its horizontal-position index is
\[
\mathrm{posX}_A
=\operatorname{round}\!\left[31\,\operatorname{clip}
\{U+0.25\varepsilon_A,0,1\}\right],
\]
and the observed treatment is the resulting flattened \(64\times64\) image plus \(\eta_A\sim\mathcal N(0,0.1^2I_{4096})\). The treatment-inducing proxy is
\[
Z=(U+0.40\varepsilon_{Z,1},\ U+\varepsilon_{Z,2}).
\]
For the outcome-inducing proxy, we fix the scale, orientation, and horizontal-
position indices at \(3,0,16\), set
\[
\mathrm{posY}_W
=\operatorname{round}\!\left[31\,\operatorname{clip}
\{U+0.30\varepsilon_W,0,1\}\right],
\]
and let \(W\) be the corresponding flattened image plus
\(\eta_W\sim\mathcal N(0,0.1^2I_{4096})\). Here
\(\varepsilon_A,\varepsilon_{Z,1},\varepsilon_W\sim\mathcal N(0,1)\),
\(\varepsilon_{Z,2}\sim\operatorname{Unif}[-0.60,0.60]\), and all proxy,
rendering, and pixel noises are mutually independent.

With \(B_{ij}=|32-j|/32\), define
\[
f_{\mathrm{ATE}}(a)
=\frac{\{\operatorname{vec}(B)^\top a\}^2-3000}{500},
\qquad
q(u)=\frac{\exp\{3(u-1/2)\}}{(2/3)\sinh(3/2)}.
\]
We generate
\[
Y=q(U)f_{\mathrm{ATE}}(A)+\varepsilon_Y,
\qquad \varepsilon_Y\sim\mathcal N(0,0.5^2).
\]
The outcome noise is independent of all preceding variables. Since
\(\E\{q(U)\}=1\), the oracle population dose-response is exactly
\(f_{\mathrm{ATE}}(a)\).

\paragraph{Original benchmark of \citet{xu2021deep}.}
For comparison, \citet{xu2021deep} draw the latent rendering vector
\(U=(\mathrm{scale},\mathrm{rotation},\mathrm{posX},\mathrm{posY})\)
uniformly over the dSprite factors, with shape fixed to the heart. They set
\begin{align*}
A&=\operatorname{vec}\{\operatorname{Fig}(U)\}+\eta_A,
&Z&=(\mathrm{scale},\mathrm{rotation},\mathrm{posX}),\\
W&=\operatorname{vec}\{\operatorname{Fig}(0.8,0,0.5,\mathrm{posY})\}+\eta_W,
\end{align*}
where \(\eta_A,\eta_W\sim\mathcal N(0,0.1^2I_{4096})\). Thus, \(Z\)
directly copies three coordinates of \(U\), while \(W\) is a noisy image
sharing the same vertical position. With \(B_{ij}=|32-j|/32\), the structural
response is
\[
f_{\mathrm{ATE}}(a)
=\frac{\{\operatorname{vec}(B)^\top a\}^2-3000}{500},\qquad
Y=12(\mathrm{posY}-1/2)^2f_{\mathrm{ATE}}(A)+\varepsilon_Y,
\quad \varepsilon_Y\sim\mathcal N(0,0.5^2).
\]
Figure~\ref{fig:dsprite_v2_results} presents results for this original benchmark.

\begin{figure}[H]
\centering
\includegraphics[width=.90\textwidth]{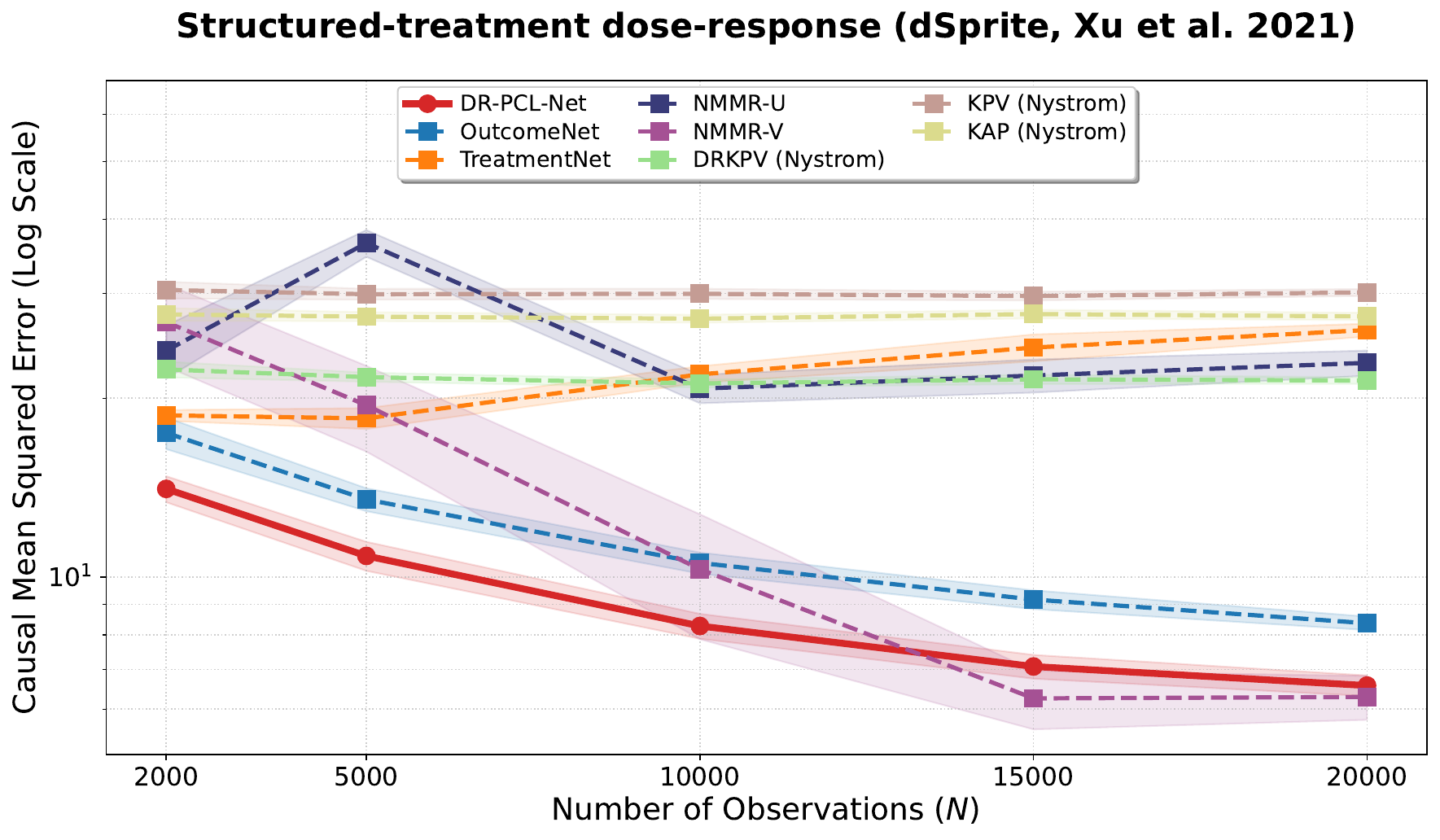}
\caption{Results on the original dSprite benchmark of \citet{xu2021deep}. We report causal MSE
over ten data-generating seeds after averaging three repetitions within each
seed; shaded regions are 95\% confidence intervals across the seed-level
means. The vertical axis is logarithmic and lower is better. }
\label{fig:dsprite_v2_results}
\end{figure}

\subsubsection{{Treatment-conditional dose-response benchmark}}
\label{app:att_dgp}
We use the data-generating process as in the synthetic low-dimensional dose-response benchmark from \citet{wu2024doubly}, but change the estimand to the conditional dose-response curve
\[
f_{\mathrm{ATT}}(a,a') := \mathbb{E}[Y^{(a)} \mid A = a']
\]
where \(a'\) indexes the observed-treatment reference population.

\textit{Observational data-generating process:} {The latent variables are generated as}
\[
U_1 \sim \mathrm{Unif}[-1,2], \quad U_2 := R - \mathbf{1}\{0 \le U_1 \le 1\},
\]
where \(R \sim \mathrm{Unif}[0,1].\) The observed proxy, treatment, and outcome variables are generated as
\[
W_1 = U_2 + \varepsilon_{W,1}, \qquad W_2 = U_1 + \varepsilon_{W,2},
\]
\[
Z_1 = U_2 + \varepsilon_{Z,1}, \qquad Z_2 = U_1 + \varepsilon_{Z,2},
\]
\[
A = U_1 + \varepsilon_A,
\]
\[
Y = 3\cos\!\bigl(2(0.3U_2 + 0.3U_1 + 0.2) + 1.5A\bigr) + \varepsilon_Y,
\]
where
\(
\varepsilon_{W,1}, \varepsilon_{Z,2} \sim \mathrm{Unif}[-1,1], \enspace
\varepsilon_{W,2}, \varepsilon_{Z,1}, \varepsilon_A, \varepsilon_Y \sim \mathcal{N}(0,1),
\)
and all exogenous variables are mutually independent. The treatment proxy is \(Z=(Z_1,Z_2)\), the outcome proxy is \(W=(W_1,W_2)\), and the latent confounding is driven by \((U_1,U_2)\).

\textit{Potential outcomes and the oracle ATT curve:}
Under the intervention \(A=a\), the potential outcome is
\[
Y^{(a)} = 3\cos\!\bigl(0.6U_2 + 0.6U_1 + 0.4 + 1.5a\bigr) + \varepsilon_Y.
\]
Since \(\mathbb{E}[\varepsilon_Y]=0\), the target ATT curve is
\[
f_{\mathrm{ATT}}(a,a') = \mathbb{E}\!\left[3\cos\!\bigl(0.6U_2 + 0.6U_1 + 0.4 + 1.5a\bigr)\,\middle|\, A=a' \right].
\]

We first integrate out \(U_2\) conditional on \(U_1=u\). Because \(U_2 = R - \mathbf{1}\{0 \le u \le 1\}\) with \(R \sim \mathrm{Unif}[0,1]\), we obtain
\[
m(a,u) := \mathbb{E}[Y^{(a)} \mid U_1=u]
= 3 \int_0^1 \cos\!\bigl(0.6r + c(u,a)\bigr)\,dr,
\]
where
\(
c(u,a) := 0.6u - 0.6\mathbf{1}\{0 \le u \le 1\} + 0.4 + 1.5a.
\)
Evaluating the integral yields
\[
m(a,u)
=
5\Bigl[\sin\!\bigl(c(u,a)+0.6\bigr)-\sin\!\bigl(c(u,a)\bigr)\Bigr].
\]
Next, since \(A = U_1 + \varepsilon_A\) with \(\varepsilon_A \sim \mathcal{N}(0,1)\) and \(U_1 \sim \mathrm{Unif}[-1,2]\), Bayes' rule gives the posterior density
\[
p(u \mid A=a')
=
\frac{\phi_{\mathrm N}(a'-u)\,\mathbf{1}\{-1 \le u \le 2\}}
{\int_{-1}^{2} \phi_{\mathrm N}(a'-t)\,dt}
=
\frac{\phi_{\mathrm N}(a'-u)\,\mathbf{1}\{-1 \le u \le 2\}}
{\Phi_{\mathrm N}(2-a')-\Phi_{\mathrm N}(-1-a')},
\]
where \(\phi_{\mathrm N}\) and \(\Phi_{\mathrm N}\) denote the standard normal density and distribution functions, respectively.

Therefore, the oracle conditional dose-response curve is
\[
f_{\mathrm{ATT}}(a,a')
=
\int_{-1}^{2} m(a,u)\,p(u \mid A=a')\,du,
\]
that is,
\[
f_{\mathrm{ATT}}(a,a')
=
\frac{5}{\Phi_{\mathrm N}(2-a')-\Phi_{\mathrm N}(-1-a')}
\int_{-1}^{2}
\Bigl[\sin\!\bigl(c(u,a)+0.6\bigr)-\sin\!\bigl(c(u,a)\bigr)\Bigr]\phi_{\mathrm N}(a'-u)\,du.
\]

This is the ground-truth ATT curve used in the experiments. In practice, we evaluate the one-dimensional integral above numerically on a grid of intervention values \(a\), for each chosen \(a'\).\looseness=-1

\subsubsection{Heterogeneous dose-response benchmark}
\label{app:cate_dgp}

We adapt the binary-treatment design of \citet{Abrevaya02102015} to the proxy
causal learning setting. Draw
\(\epsilon_0,\epsilon_1,\epsilon_2,\epsilon_3\stackrel{\mathrm{iid}}{\sim}
\operatorname{Unif}[-0.5,0.5]\), set \(V=\epsilon_0\), and define
\[
U_1=1+2V+\epsilon_1,\qquad
U_2=1+2V+\epsilon_2,\qquad
U_3=(V-1)^2+\epsilon_3.
\]
Treatment is assigned according to
\[
\sP(A=1\mid V,U_1,U_2,U_3)
=\operatorname{expit}\!\left\{\tfrac12(V+U_1+U_2+U_3)\right\}.
\]
With \(\nu\sim\mathcal N(0,0.25^2)\), the observed outcome is
\[
Y=A\{VU_1U_2U_3+\nu\}.
\]
The proxies are independent noisy measurements of the latent variables:
\begin{align*}
W_1&=U_1+\omega_1,& \omega_1&\sim\operatorname{Unif}[-b,b],
&Z_1&=U_1+\zeta_1,& \zeta_1&\sim\mathcal N(0,\sigma^2),\\
W_2&=U_2+\omega_2,& \omega_2&\sim\mathcal N(0,\sigma^2),
&Z_2&=U_2+\zeta_2,& \zeta_2&\sim\operatorname{Unif}[-b,b],\\
W_3&=U_3+\omega_3,& \omega_3&\sim\mathcal N(0,(\sigma/2)^2),
&Z_3&=U_3+\zeta_3,& \zeta_3&\sim\operatorname{Unif}[-b,b],
\end{align*}
where \(b=\sigma=0.1\), and all measurement noises are mutually independent
and independent of the preceding variables. The oracle response at \(A=1\) is
\[
f_{\mathrm{CATE}}(1,v)
=\E\{Y^{(1)}\mid V=v\}
=v(1+2v)^2(v-1)^2.
\]

\subsection{{Sensitivity to proxy informativeness}}
\label{appendix:proxy_informativeness}

We assess sensitivity to proxy informativeness in the Demand design. For \(P\in\{W,Z\}\), we perturb one proxy at a time according to
\[
\widetilde P_{ij}
=P_{ij}+s\widehat\sigma_{P_j}\varepsilon_{ij},
\qquad
\varepsilon_{ij}\sim\mathcal N(0,1),
\]
where \(\widehat\sigma_{P_j}\) is the empirical standard deviation of proxy coordinate \(j\). Thus, \(s\) measures the added noise relative to the scale of that coordinate. Within each seed, the underlying generated sample, sample splits, initialization, and training randomness are held fixed across \(s\); the \(s=0\) baseline is shared by the two perturbation arms. 

\begin{table}[H]
\centering
\caption{Sensitivity to proxy informativeness in the Demand design at
\(n=4{,}000\). Entries are causal MSE, reported as mean \(\pm\) seed-level standard error over ten data-generating seeds after averaging three
repetitions within each seed. The \(R^2\) rows measure the informativeness of
the perturbed proxy for \(U\) using quadratic regression on a separate
\(n=20{,}000\) diagnostic sample and therefore have no standard errors.
Boldface marks the smallest MSE at each noise level.}
\label{tab:demand_proxy_informativeness}
\scriptsize
\setlength{\tabcolsep}{3.2pt}
\renewcommand{\arraystretch}{0.94}
\begin{adjustbox}{max width=\textwidth}
\begin{tabular}{@{}lcccccc@{}}
\toprule
& \multicolumn{6}{c}{Noise scale \(s\)} \\
\cmidrule(lr){2-7}
Method & \(0\) & \(0.25\) & \(0.5\) & \(1\) & \(2\) & \(4\) \\
\midrule
\multicolumn{7}{l}{\textit{Noise added to the outcome proxy \(W\)}} \\
\textsc{OutcomeNet}
& \(27.5\pm2.5\) & \(27.4\pm2.4\) & \(30.9\pm2.5\)
& \(55.5\pm4.0\) & \(118.9\pm7.1\) & \(178.0\pm6.0\) \\
\textsc{TreatmentNet}
& \(163.4\pm11.1\) & \(175.0\pm8.6\) & \(186.9\pm10.5\)
& \(200.9\pm6.2\) & \(203.6\pm3.5\) & \(200.2\pm2.4\) \\
\textsc{DR-PCL-Net}
& \(\mathbf{22.3\pm2.2}\) & \(\mathbf{22.4\pm2.1}\)
& \(\mathbf{26.3\pm2.5}\) & \(\mathbf{50.3\pm3.7}\)
& \(\mathbf{113.9\pm6.3}\) & \(\mathbf{170.0\pm5.5}\) \\
\(R^2(U\mid\widetilde W)\)
& \(0.460\) & \(0.432\) & \(0.369\) & \(0.237\) & \(0.097\) & \(0.028\) \\
\midrule
\multicolumn{7}{l}{\textit{Noise added to the treatment proxy \(Z\)}} \\
\textsc{OutcomeNet}
& \(27.5\pm2.5\) & \(38.5\pm5.0\) & \(59.5\pm7.7\)
& \(101.7\pm7.8\) & \(157.0\pm7.7\) & \(196.3\pm7.9\) \\
\textsc{TreatmentNet}
& \(163.4\pm11.1\) & \(175.9\pm11.2\) & \(162.8\pm14.9\)
& \(169.9\pm14.7\) & \(183.3\pm17.8\) & \(194.1\pm18.4\) \\
\textsc{DR-PCL-Net}
& \(\mathbf{22.3\pm2.2}\) & \(\mathbf{31.4\pm4.5}\)
& \(\mathbf{50.2\pm7.6}\) & \(\mathbf{91.7\pm7.7}\)
& \(\mathbf{148.7\pm7.3}\) & \(\mathbf{190.5\pm7.8}\) \\
\(R^2(U\mid\widetilde Z)\)
& \(0.402\) & \(0.378\) & \(0.322\) & \(0.201\) & \(0.080\) & \(0.023\) \\
\bottomrule
\end{tabular}
\end{adjustbox}
\end{table}

The effect of corrupting one proxy is not confined to a single bridge route.
Indeed, for square-integrable \(u(a,x,w)\) and \(v(a,x,z)\),
\[
\langle T_hu,v\rangle_{L^2(\sP_{A,X,Z})}
=\E\{u(A,X,W)v(A,X,Z)\}
=\langle u,T_\varphi v\rangle_{L^2(\sP_{A,X,W})}.
\]
Thus, \(T_\varphi=T_h^*\); when compact, the two operators have the same
nonzero singular values. Independent measurement noise in either proxy
composes this adjoint pair with an \(L^2\)-contraction and can shrink the shared
informative spectrum. Hence, when one bridge problem becomes
less stable, deterioration of the other route is also expected. Their
finite-sample responses need not coincide because their targets, function
classes, and regularization differ. Consistent with this interpretation, causal MSE
generally rises as \(R^2(U\mid\widetilde P)\) falls. \textsc{DR-PCL-Net} has the smallest
MSE in every reported cell, although it also deteriorates under severe proxy
corruption.

\subsection{{Neural architectures}}
\label{appendix:NN_structures_for_experiments}

We report the component featurizers used in the six benchmarks below. The \textsc{OutcomeNet} tables use the compact feature notation of Appendix~\ref{sec:DFPCL_Xu_Review}. For \textsc{TreatmentNet}, the compact maps are implemented as tensor products of the listed component networks; for example, \(\phi_{AW,1}^{(\varphi)}=\phi_{A,1}^{(\varphi)}\otimes \phi_{W,1}^{(\varphi)}\). The doubly robust estimator uses the same two sets of bridge featurizers.  In the tables, ``hidden'' gives the successive hidden widths.  We write FC, LN, BN, SN, and DO for fully connected affine layers, layer normalization, batch normalization, spectral normalization, and dropout, respectively. 

\subsubsection{Image-valued treatment setting - dSprite}

The reported \textsc{OutcomeNet} uses a separate network for each image-valued input. Each image network has the form \(4096\!\to\!1024\!\to\!512\!\to\!128\!\to\!32\), with SN on every affine layer, ReLU activations, and BN after the 512-dimensional layer.  The two-dimensional treatment-proxy network uses ordinary affine layers and ReLU. The first-stage feature is \([\phi_{A,1}^{(h)}(A),1]\otimes[\phi_{Z,1}^{(h)}(Z),1]\), with the duplicated constant removed; it therefore has dimension \(33^2-1=1088\) before the model appends its intercept.  Each \textsc{TreatmentNet} component uses two hidden blocks of FC--LN--GELU--BN--DO\((0.05)\), followed by FC--LN--GELU.

\begin{table}[H]
\centering
\small
\caption{Component featurizers for the dSprite V3 benchmark.}
\label{tab:nn_dsprite_v3}
\begin{tabular*}{\textwidth}{@{\extracolsep{\fill}}llccc@{}}
\toprule
\textbf{Route} & \textbf{Feature map} & \textbf{Input (dimension)} & \textbf{Hidden} & \textbf{Output} \\
\midrule
\textsc{OutcomeNet} & \(\phi_{A,1}^{(h)}\) & \(A\) (4096) & 1024, 512, 128 & 32 \\
           & \(\phi_{Z,1}^{(h)}\) & \(Z\) (2)    & 128, 64        & 32 \\
           & \(\phi_{A,2}^{(h)}\) & \(A\) (4096) & 1024, 512, 128 & 32 \\
           & \(\phi_{W,2}^{(h)}\) & \(W\) (4096) & 1024, 512, 128 & 32 \\
\addlinespace
\textsc{TreatmentNet} & \(\phi_{A,1}^{(\varphi)}\) & \(A\) (4096) & 512, 256 & 16 \\
             & \(\phi_{W,1}^{(\varphi)}\) & \(W\) (4096) & 512, 256 & 8 \\
             & \(\phi_{Z,2}^{(\varphi)}\) & \(Z\) (2)    & 512, 256 & 8 \\
             & \(\phi_{A,2}^{(\varphi)}\) & \(A\) (4096) & 512, 256 & 8 \\
\bottomrule
\end{tabular*}
\end{table}

The \textsc{TreatmentNet} and doubly robust final-stage regressors take the flattened treatment image as input and use the architecture \(4096\to1024\to512\to1\), with GELU, DO\((0.05)\), and a linear output.

\subsubsection{Demand}

Each \textsc{OutcomeNet} featurizer has two FC--LN--GELU--DO\((0.01)\) hidden
blocks and an FC--LN--GELU output block.  Each \textsc{TreatmentNet} component has two
FC--LN--GELU--BN--DO\((0.01)\) hidden blocks and an FC--LN--GELU output
block.  The treatment-bridge first- and second-stage compact dimensions are
\(16\times8=128\) and \(8\times8=64\), respectively.

\begin{table}[H]
\centering
\small
\caption{Component featurizers for the Demand benchmark.}
\label{tab:nn_demand}
\begin{tabular*}{\textwidth}{@{\extracolsep{\fill}}llccc@{}}
\toprule
\textbf{Route} & \textbf{Feature map} & \textbf{Input (dimension)} & \textbf{Hidden} & \textbf{Output} \\
\midrule
\textsc{OutcomeNet} & \(\phi_{AZ,1}^{(h)}\) & \((A,Z)\) (3) & 32, 64 & 32 \\
           & \(\phi_{A,2}^{(h)}\)  & \(A\) (1)     & 32, 64 & 16 \\
           & \(\phi_{W,2}^{(h)}\)  & \(W\) (1)     & 32, 64 & 16 \\
\addlinespace
\textsc{TreatmentNet} & \(\phi_{A,1}^{(\varphi)}\) & \(A\) (1) & 32, 16 & 16 \\
             & \(\phi_{W,1}^{(\varphi)}\) & \(W\) (1) & 32, 16 & 8 \\
             & \(\phi_{Z,2}^{(\varphi)}\) & \(Z\) (2) & 32, 16 & 8 \\
             & \(\phi_{A,2}^{(\varphi)}\) & \(A\) (1) & 32, 16 & 8 \\
\bottomrule
\end{tabular*}
\end{table}

For the final stage, the fold-specific \textsc{OutcomeNet} treatment featurizer supplies
the 16-dimensional input.  \textsc{TreatmentNet} and the doubly robust correction use
\(16\to32\to64\to1\) networks with GELU, DO\((0.05)\), linear outputs,
and five-member averaging.

\subsubsection{Legalized Abortion and Crime}

The \textsc{OutcomeNet} featurizers apply two FC--LN--GELU--DO\((0.01)\) hidden blocks and an FC--GELU output block. The \textsc{TreatmentNet} block is the same FC--LN--GELU--BN--DO\((0.01)\) architecture used for Demand.

\begin{table}[H]
\centering
\small
\caption{Component featurizers for the Legalized Abortion and Crime benchmark.}
\label{tab:nn_abortion_crime}
\begin{tabular*}{\textwidth}{@{\extracolsep{\fill}}llccc@{}}
\toprule
\textbf{Route} & \textbf{Feature map} & \textbf{Input (dimension)} & \textbf{Hidden} & \textbf{Output} \\
\midrule
\textsc{OutcomeNet} & \(\phi_{AZ,1}^{(h)}\) & \((A,Z)\) (2) & 256, 128 & 64 \\
           & \(\phi_{A,2}^{(h)}\)  & \(A\) (1)     & 256, 128 & 8 \\
           & \(\phi_{W,2}^{(h)}\)  & \(W\) (3)     & 256, 128 & 16 \\
\addlinespace
\textsc{TreatmentNet} & \(\phi_{A,1}^{(\varphi)}\) & \(A\) (1) & 32, 16 & 16 \\
             & \(\phi_{W,1}^{(\varphi)}\) & \(W\) (3) & 32, 16 & 8 \\
             & \(\phi_{Z,2}^{(\varphi)}\) & \(Z\) (1) & 32, 16 & 8 \\
             & \(\phi_{A,2}^{(\varphi)}\) & \(A\) (1) & 32, 16 & 8 \\
\bottomrule
\end{tabular*}
\end{table}

The \textsc{TreatmentNet} and doubly robust final-stage regressors use \(1\to32\to64\to1\), GELU, DO\((0.01)\). 

\subsubsection{Deaner IM and IR}

The IM and IR designs use the same networks.  The \textsc{OutcomeNet} featurizers use two FC--LN--GELU--DO\((0.01)\) hidden blocks and an FC--GELU output block. The \textsc{TreatmentNet} components use two FC--LN--GELU--BN--DO\((0.01)\) hidden blocks and an FC--LN--GELU output block.  Their first- and second-stage compact dimensions are \(16\times8=128\) and \(8\times8=64\), respectively.

\begin{table}[H]
\centering
\small
\caption{Component featurizers for the Deaner IM and IR benchmarks.}
\label{tab:nn_deaner}
\begin{tabular*}{\textwidth}{@{\extracolsep{\fill}}llccc@{}}
\toprule
\textbf{Route} & \textbf{Feature map} & \textbf{Input (dimension)} & \textbf{Hidden} & \textbf{Output} \\
\midrule
\textsc{OutcomeNet} & \(\phi_{AZ,1}^{(h)}\) & \((A,Z)\) (5) & 256, 128 & 64 \\
           & \(\phi_{A,2}^{(h)}\)  & \(A\) (1)     & 256, 128 & 8 \\
           & \(\phi_{W,2}^{(h)}\)  & \(W\) (3)     & 256, 128 & 16 \\
\addlinespace
\textsc{TreatmentNet} & \(\phi_{A,1}^{(\varphi)}\) & \(A\) (1) & 32, 16 & 16 \\
             & \(\phi_{W,1}^{(\varphi)}\) & \(W\) (3) & 32, 16 & 8 \\
             & \(\phi_{Z,2}^{(\varphi)}\) & \(Z\) (4) & 32, 16 & 8 \\
             & \(\phi_{A,2}^{(\varphi)}\) & \(A\) (1) & 32, 16 & 8 \\
\bottomrule
\end{tabular*}
\end{table}

The \textsc{TreatmentNet} and doubly robust final stages first map the standardized scalar treatment to 128 frozen random Fourier features with unit bandwidth, followed by \(128\to32\to32\to1\) with GELU, DO\((0.05)\), and a linear output. 

\subsubsection{Heterogeneous dose-response}

The \textsc{OutcomeNet} components use two FC--LN--GELU--DO\((0.01)\) hidden blocks
and an FC--LN--GELU output block.  \textsc{TreatmentNet} uses the corresponding
FC--LN--GELU--BN--DO\((0.01)\) blocks.  Here \(V\) is the scalar effect
modifier and also enters the bridge as the observed covariate.  Consequently,
the treatment-bridge compact maps are
\(\phi_{A,1}^{(\varphi)}\otimes\phi_{V,1}^{(\varphi)}\otimes
\phi_{W,1}^{(\varphi)}\) and
\(\phi_{A,2}^{(\varphi)}\otimes\phi_{V,2}^{(\varphi)}\otimes
\phi_{Z,2}^{(\varphi)}\), both of dimension \(256\).

\begin{table}[H]
\centering
\small
\caption{Component featurizers for the synthetic CATE benchmark.}
\label{tab:nn_synthetic_cate}
\begin{tabular*}{\textwidth}{@{\extracolsep{\fill}}llccc@{}}
\toprule
\textbf{Route} & \textbf{Feature map} & \textbf{Input (dimension)} & \textbf{Hidden} & \textbf{Output} \\
\midrule
\textsc{OutcomeNet} & \(\phi_{AVZ,1}^{(h)}\) & \((A,V,Z)\) (5) & 64, 128 & 128 \\
           & \(\phi_{A,2}^{(h)}\)   & \(A\) (1)       & 64, 128 & 4 \\
           & \(\phi_{V,2}^{(h)}\)   & \(V\) (1)       & 64, 128 & 8 \\
           & \(\phi_{W,2}^{(h)}\)   & \(W\) (3)       & 64, 128 & 16 \\
\addlinespace
\textsc{TreatmentNet} & \(\phi_{A,1}^{(\varphi)}\) & \(A\) (1) & 32, 16 & 8 \\
             & \(\phi_{V,1}^{(\varphi)}\) & \(V\) (1) & 32, 16 & 4 \\
             & \(\phi_{W,1}^{(\varphi)}\) & \(W\) (3) & 32, 16 & 8 \\
             & \(\phi_{A,2}^{(\varphi)}\) & \(A\) (1) & 32, 16 & 8 \\
             & \(\phi_{V,2}^{(\varphi)}\) & \(V\) (1) & 32, 16 & 4 \\
             & \(\phi_{Z,2}^{(\varphi)}\) & \(Z\) (3) & 32, 16 & 8 \\
\bottomrule
\end{tabular*}
\end{table}

The \textsc{OutcomeNet} final stage regresses the 16-dimensional outcome-proxy feature on \(V\) with a \(1\to32\to64\to16\) network.  For \textsc{TreatmentNet} and the doubly robust estimator, a \(2\to32\to64\to2\) network takes \((A,V)\) and jointly estimates the numerator and denominator of the self-normalized response.  All three use GELU, DO\((0.01)\), and linear outputs.

\subsubsection{Treatment-conditional dose-response}

The \textsc{OutcomeNet} components use two FC--LN--GELU--DO\((0.05)\) hidden blocks and an FC--LN--GELU output block.  \textsc{TreatmentNet} again uses two FC--LN--GELU--BN--DO\((0.01)\) hidden blocks and an FC--LN--GELU output block.  Its bridge surface is represented by the separable feature \(\psi_A^{(\varphi)}(a)\otimes\psi_{A'}^{(\varphi)}(a')\).  Thus, the first-stage compact map has dimension \(16\times8=128\), and the second-stage map has dimension \(8\times8\times8=512\).

\begin{table}[H]
\centering
\small
\caption{Component featurizers for the synthetic ATT benchmark.}
\label{tab:nn_synthetic_att}
\begin{tabular*}{\textwidth}{@{\extracolsep{\fill}}llccc@{}}
\toprule
\textbf{Route} & \textbf{Feature map} & \textbf{Input (dimension)} & \textbf{Hidden} & \textbf{Output} \\
\midrule
\textsc{OutcomeNet} & \(\phi_{AZ,1}^{(h)}\) & \((A,Z)\) (3) & 128, 256 & 128 \\
           & \(\phi_{A,2}^{(h)}\)  & \(A\) (1)     & 128, 256 & 8 \\
           & \(\phi_{W,2}^{(h)}\)  & \(W\) (2)     & 128, 256 & 16 \\
\addlinespace
\textsc{TreatmentNet} & \(\phi_{A,1}^{(\varphi)}\) & \(A\) (1)  & 32, 16 & 16 \\
             & \(\phi_{W,1}^{(\varphi)}\) & \(W\) (2)  & 32, 16 & 8 \\
             & \(\psi_A^{(\varphi)}\)      & \(A\) (1)  & 32, 16 & 8 \\
             & \(\psi_{A'}^{(\varphi)}\)   & \(A'\) (1) & 32, 16 & 8 \\
             & \(\phi_{Z,2}^{(\varphi)}\) & \(Z\) (2)  & 32, 16 & 8 \\
\bottomrule
\end{tabular*}
\end{table}

The \textsc{OutcomeNet} final stage uses \(1\to32\to64\to16\) to regress the outcome-proxy feature on the anchor treatment.  \textsc{TreatmentNet} and the doubly robust estimator regress scalar pseudo-outcomes on \((A,A')\) with \(2\to32\to64\to1\).  These heads use GELU, DO\((0.01)\), and linear outputs.

\subsection{Compute resources.}
All experiments were run on a Linux x86\_64 compute node with Python 3.12.12. GPU experiments used single-node SLURM jobs with one task, eight CPU cores, 10GB system memory, and a 24-hour wall-time limit. Unless otherwise stated, each job used a single NVIDIA RTX A4500 GPU with driver version 580.95.05 and CUDA 13.0; the GPU provided approximately 20GB of memory.

\end{document}